\documentclass{article}

\newif\ifshownavigationpage
\newif\ifshowreminders
\newif\ifshownotationindex
\newif\ifshowtheoremlinks
\newif\ifshowtheoremtree
\newif\ifshowtheoremlist
\newif\ifshowequationlist

\newif\ifshowcomments
\newif\ifhighlight 
\newif\ifelaborate

\newif\ifshowaddressedcomments
\newif\ifshowrvin
\newif\ifshowrvout

\showcommentstrue

\usepackage{amsthm, amsmath, amssymb}
\usepackage{thmtools}
\usepackage{soul}
\usepackage{authblk}
\usepackage{framed}
\usepackage{mathrsfs}
\usepackage{hyperref}
\hypersetup{
    colorlinks,
    linkcolor={red!50!black},
    citecolor={blue!50!black},
    urlcolor={blue!80!black}
}
\usepackage{booktabs} 

\usepackage[dvipsnames]{xcolor}
\usepackage{graphicx, multicol}
\usepackage{marvosym, wasysym}
\usepackage{fancyhdr}
\usepackage{stackengine}
\usepackage{algorithm}
\usepackage{bm}
\usepackage[noend]{algpseudocode}
\usepackage{bbm}
\usepackage{verbatim}
\usepackage{mdframed}
\usepackage{needspace}
\makeatletter
\renewcommand{\ALG@beginalgorithmic}{\scriptsize}
\makeatother
\usepackage{xcolor}
\allowdisplaybreaks

\DeclareFontFamily{U}{mathx}{}
\DeclareFontShape{U}{mathx}{m}{n}{ <-> mathx10 }{}
\DeclareSymbolFont{mathx}{U}{mathx}{m}{n}
\DeclareFontSubstitution{U}{mathx}{m}{n}

\renewcommand{\dj}[1]{\bm{d}_{J_1}^{_{#1}}}
\newcommand{\dlp}[1]{\bm{d}_{L_p}^{_{#1}}}
\newcommand{\tofdd}{\stackrel{f.d.d.}{\to}}

\ifhighlight
\else
    
\fi

\newcommand{\rvoutopacity}{20}
\ifshowrvin
    
    \newcommand{\chrin}[1]{{\color{blue}{#1}}}
\else
    
    \newcommand{\chrin}[1]{#1}
\fi

\ifshowrvout
    \newcommand{\rvout}[1]{{\color{red!\rvoutopacity}{#1}} }
    \newcommand{\chrout}[1]{{\color{blue!\rvoutopacity}{#1}} }    
    \newcommand{\rvoutm}[1]{{\color{black!\rvoutopacity}{\ifmmode\text{\sout{\ensuremath{\displaystyle#1}}}\else\sout{#1}\fi}} } 
\else
    \newcommand{\rvout}[1]{}
    \newcommand{\chrout}[1]{}
\fi

\ifshowcomments
    \newcommand{\summ}[1]{{\color{blue}[summary: #1]} } 
    \newcommand{\chr}[1]{{\color{PineGreen}[CR: #1]} } 
    \newcommand{\xw}[1]{{\color{RoyalBlue}[XW: #1]} } 
    \ifshowaddressedcomments
        \newcommand{\chra}[1]{{\color{PineGreen}\sout{[CR: #1]}} } 
        \newcommand{\xwa}[1]{{\color{RoyalBlue}\sout{[XW: #1]}} } 
    \else
        \newcommand{\chra}[1]{} 
        \newcommand{\xwa}[1]{} 
    \fi
\else
    \newcommand{\summ}[1]{} 
    \newcommand{\chr}[1]{} 
    \newcommand{\chra}[1]{} 
    \newcommand{\xw}[1]{} 
    \newcommand{\xwa}[1]{} 
\fi

\usepackage[shortlabels]{enumitem}

\newlist{thmdependence}{itemize}{10}
\setlist[thmdependence]{nosep,label=-}

\newcommand{\thmtreenode}[5]{\item[#1] \linkdest{location, thm tree #3} {#2}~\ref{#3} \linktopf{#3} \thmsum{#4}{#5}}
\newcommand{\thmtreenodewopf}[5]{\item[#1] \linkdest{location, thm tree #3} {#2}~\ref{#3} \thmsum{#4}{#5}}
\newcommand{\thmtreeref}[2]{\item[\elsewhere] {{\hyperlink{location, thm tree #2}{\color{gray}#1}}}~\ref{#2}\thmsum{0.5}{}}

\ifshowtheoremlinks
    \newcommand{\linksinthm}[1]{\emph{\linkdest{location, #1}\linktopf{#1} \linktothmtree{location, thm tree #1} }}
    \newcommand{\linksinthmwopf}[1]{\emph{\linkdest{location, #1} \linktothmtree{location, thm tree #1} }}
    \newcommand{\linksinpf}[1]{\linkdest{location, proof of #1}\linktothm{#1} \linktothmtree{location, thm tree #1} }
\else
    \newcommand{\linksinthm}[1]{}
    \newcommand{\linksinthmwopf}[1]{}
    \newcommand{\linksinpf}[1]{}
\fi

\ifshownotationindex
    \newcommand{\notationdef}[2]{\linkdest{location, notation definition of #1}\hyperlink{location, notation index of #1}{#2}}
    \newcommand{\notationidx}[2]{\linkdest{location, notation index of #1}\hyperlink{location, notation definition of #1}{#2}}
\else
    \newcommand{\notationdef}[2]{#2}
\fi
    
\newcommand{\linktopf}[1]{\hyperlink{location, proof of #1}{\pflinksymbol}}
\newcommand{\linktothm}[1]{\hyperlink{location, #1}{\thmlinksymbol}}
\newcommand{\linktothmtree}[1]{\hyperlink{#1}{\thmtreelinksymbol}}
\newcommand{\thmlinksymbol}{{\tiny [Theorem]}}
\newcommand{\pflinksymbol}{{\tiny [Proof]}}
\newcommand{\thmtreelinksymbol}{{\tiny [ThmTree]}}

\newcommand{\complete}{{\color{black}\checkmark}}

\newcommand{\issue}{{\color{red}\checkmark}}
\newcommand{\elsewhere}{}

\newcommand{\thmsum}[2]{\quad{\color{gray}\begin{minipage}[t]{#1\linewidth}{#2}\vspace{0.5\baselineskip}\end{minipage}}}

\makeatletter
\newcommand{\linkdest}[1]{\Hy@raisedlink{\hypertarget{#1}{}}}
\makeatother

\newcommand{\elaborateopacity}{50}
\newcommand{\elaboratecolor}{RawSienna}
\ifelaborate
    \newcommand{\elaborate}[1]{{\color{\elaboratecolor!\elaborateopacity}{
    \begin{framed}
    \noindent {\footnotesize[Elaboration]}
    #1 
    \end{framed}
    }}\noindent}
\else
    \newcommand{\elaborate}[1]{}
\fi

\usepackage[margin=1.2in]{geometry}

\newtheorem{theorem}{Theorem}
\newtheorem{lemma}[theorem]{Lemma}
\newtheorem{corollary}[theorem]{Corollary}
\newtheorem{proposition}[theorem]{Proposition}

\newtheorem{definition}[theorem]{Definition}

\newtheorem{assumption}{Assumption}
\newtheorem{remark}{Remark}
\newtheorem{condition}{Condition}

\newtheorem*{theorem-nonumber}{Theorem}
\newtheorem*{condition-nonumber}{Condition}
\newtheorem*{proposition-nonumber}{Proposition}

\usepackage{mathtools}
\DeclarePairedDelimiter{\ceil}{\lceil}{\rceil}
\DeclarePairedDelimiter\floor{\lfloor}{\rfloor}

\newcommand{\D}{\mathbb D}

\newcommand{\cmt}[1]{#1} 
\renewcommand{\cmt}[1]{} 
\renewcommand{\P}{\mathbf{P}}
\newcommand{\Q}{\mathbf{Q}}
\newcommand{\E}{\mathbf{E}}
\newcommand{\RV}{\mathcal{RV}}
\newcommand{\R}{\mathbb{R}}
\newcommand{\Z}{\mathbb{Z}}
\renewcommand{\S}{\mathbb{S}}
\newcommand{\C}{\mathbb{C}}
\newcommand{\M}{\mathbb{M}}

\renewcommand{\complement}{c}

\usepackage[normalem]{ulem}

\usepackage{multirow}
\usepackage{longtable}

\usepackage{array}
\def\delequal{\mathrel{\ensurestackMath{\stackon[1pt]{=}{\scriptscriptstyle\Delta}}}}
\def\distequal{\mathrel{\ensurestackMath{\stackon[1pt]{=}{\scriptstyle d}}}}
\newcommand{\norm}[1]{\left\lVert#1\right\rVert}
\usepackage{mathtools}

\algrenewcommand\algorithmicrequire{\textbf{Require:}}
\algrenewcommand\algorithmicensure{\textbf{Postcondition:}}

\usepackage{subfigure}
\usepackage{float}
\usepackage{subfiles}
\title{Global Dynamics of Heavy-Tailed SGDs in Nonconvex Loss Landscape: Characterization and Control}

\DeclareMathAccent{\widecheck}{0}{mathx}{"71}

\author[1]{Xingyu Wang} 
\author[2]{Chang-Han Rhee}
\affil[1]{Quantitative Economics, University of Amsterdam\\
    Amsterdam, 1018 WB, NL}
\affil[2]{Industrial Engineering and Management Sciences, Northwestern University\\
    Evanston, IL, 60613, USA}
\begin{document}
\maketitle

\begin{abstract}
\noindent
Stochastic gradient descent (SGD) and its variants enable
modern artificial intelligence. 
However, theoretical understanding lags far behind their empirical success.
It is widely believed that SGD has a curious ability to avoid sharp local minima in the loss landscape, which are associated with poor generalization. To unravel this mystery and further enhance such capability of SGDs, it is imperative to go beyond the traditional local convergence analysis and obtain a comprehensive understanding of SGDs' global dynamics.
In this paper, we develop a set of technical machinery based on the recent large deviations and metastability analysis in \cite{wang2023large} and obtain sharp characterization of the global dynamics of heavy-tailed SGDs.
In particular, we reveal a fascinating phenomenon in deep learning: by injecting and then truncating heavy-tailed noises during the training phase, SGD can almost completely avoid sharp minima and achieve better generalization performance for the test data.
Simulation and deep learning experiments confirm our theoretical prediction that heavy-tailed SGD with gradient clipping finds local minima with a more flat geometry and achieves better generalization performance.
\end{abstract}

\counterwithin{equation}{section}
\counterwithin{lemma}{section}
\counterwithin{corollary}{section}
\counterwithin{theorem}{section}
\counterwithin{definition}{section}
\counterwithin{proposition}{section}
\counterwithin{figure}{section}
\counterwithin{table}{section}

\tableofcontents

\section{Introduction}

\xwa{
generalization gap; flat minima.
\cite{zhang2017understanding}: nearly perfect fit under random noise or labels.
\cite{doi:10.1073/pnas.1903070116}: challenge the classical bias-variance (underfitting vs overfitting).
\cite{hochreiter1997flat}: flat minima folklore goes back to this paper.
\cite{keskar2017on,jastrzębski2018factorsinfluencingminimasgd}: correlation between sharpness, generalization performance, and hyperparameters for training.
\cite{dinh2017sharp}: sharp minima can generalize.
\cite{NIPS2017_10ce03a1}: rigorously supported  the connection between sharpness and generation via the PAC-Bayesian framework (see also \cite{dziugaite2017computingnonvacuousgeneralizationbounds}).
Dynamical Stability perspective: \cite{NEURIPS2018_6651526b}.
Implicit regularization in sharpness through the study of approximating SDE: \cite{pmlr-v125-blanc20a,li2022happenssgdreacheszero}; see also \cite{NEURIPS2021_e6af401c} using coupling techniques.
Adaptive sharpness: proposed in \cite{pmlr-v139-kwon21b}, with \cite{pmlr-v202-andriushchenko23a} critique.
Empirical evidence: ?.
\cite{pmlr-v187-kaur23a}: challenge the connection between generalization (also in SAM) and largest eigenvalue of Hessian.
\cite{bahri-etal-2022-sharpness}: benefits of sharpness methods in language models.
\cite{Jiang*2020Fantastic}: large scale study.
\cite{chen2022vision}: benefits for visual transformers.
\cite{NEURIPS2021_bcb41ccd}: benefits for domain generalization tasks.
\cite{10.1007/978-3-031-40837-3_13}: benefits for graph neural networks.
}

Deep learning has seen unprecedented successes in a wide range of contexts, including image recognition, natural language processing, and game playing
\cite{lecun2015deep,NIPS2012_c399862d,NIPS2017_3f5ee243,silver2016mastering},
effectively laying the foundation for the modern machine learning and artificial intelligence revolution.
At the core of such sweeping empirical successes lies a central mystery: 
the ability of deep neural networks to generalize from the available training data to unseen test data.
In particular, modern deep learning tasks often employ heavily over-parameterized model architectures that are able to perfectly fit the training data or even random labels (see \cite{zhang2017understanding})
yet still generalize remarkably well during the test phase.
This observation challenges the classical bias-variance tradeoff (i.e., under-fitting vs.\ over-fitting) in the model capacity and generalization performance (see, e.g., \cite{doi:10.1073/pnas.1903070116})
and calls for new perspectives.

Regarding the generalization mystery in deep learning, a hypothesis that has become increasingly popular recently is that generalization is closely related to the sharpness of the loss landscape. 
More precisely, the training of the machine learning models is typically formulated as an optimization problem $\min_{\bm \theta}f(\bm \theta)$, where the training algorithm updates the model weights $\bm \theta$ in order to minimize the loss function $f(\cdot)$ induced by the training data and model architecture at hand. 
Such loss landscapes $f(\cdot)$ exhibit highly non-convex and sophisticated geometry with a plethora of local minima; see, e.g., \cite{li2018visualizing,NEURIPS2019_48042b1d}.
The flat-minima folklore dates back to \cite{hochreiter1997flat}, and carries a simple yet compelling intuition as argued in \cite{keskar2017on}:
models tend to generalize well at a local minimum $\bm \theta$ where the training loss landscape $f(\cdot)$ exhibits a flatter geometry, as such $\bm \theta$ ensures a consistent and robust model performance under the small perturbation of loss landscape when switching from the training to the test setting. 
Moreover, \cite{keskar2017on,jastrzębski2018factorsinfluencingminimasgd} observe that SGDs (i.e., with $\nabla f(\cdot)$ estimated over randomly chosen small batches of training data during each iteration)
yield solutions with flatter geometry and better generalization performance when compared to the deterministic gradient descent (GD) iterates (i.e., using the entire training set for each iteration).
Since then, the rigorous justification of the connection between sharpness and generalization has become an active field of research,
with existing work
built upon PAC-Bayes theory (see \cite{NIPS2017_10ce03a1,dziugaite2017computingnonvacuousgeneralizationbounds}),
taking the dynamical stability perspective (see \cite{NEURIPS2018_6651526b}),
or studying the implicit regularization of sharpness in SGDs (see \cite{pmlr-v125-blanc20a,li2022happenssgdreacheszero,NEURIPS2021_e6af401c}).
While these theoretical analyses are inevitably complicated by factors
such as
the wide range of candidates for sharpness metrics (leading eigenvalue of $\nabla^2 f(\bm \theta)$, trace of $\nabla^2 f(\bm \theta)$, expected sharpness \cite{pmlr-v97-zhu19e, NIPS2017_10ce03a1},
PAC-Bayes-based sharpness metrics \cite{NIPS2017_10ce03a1}, adaptive sharpness \cite{pmlr-v139-kwon21b}, etc.),
the lack of invariance property in many sharpness notions under equivalent reparameterization of model weights
(see \cite{dinh2017sharp}),
and the inherently data- and task-dependent nature of the problem (see \cite{pmlr-v202-andriushchenko23a}),
on the empirical front there is a growing body of evidence showing that SGD tends to find flatter minima and attain better test accuracy when compared to GD, and that seeking flatter minima leads to better generalization performance in practice across a wide range of contexts, including language and vision models, graph neural networks, and domain generalization tasks; see, e.g., \cite{Jiang*2020Fantastic,bahri-etal-2022-sharpness,chen2022vision,10.1007/978-3-031-40837-3_13,NEURIPS2021_bcb41ccd}.

Therefore, it is important to develop principled approaches for understanding and further enhancing SGD's ability to avoid sharp minima. 
In this paper, we focus on the characterization and control the global dynamics of SGD when exploring a multimodal loss landscape with several local minima.
In particular, we examine the global dynamics of SGD driven by truncated heavy-tailed noise:
given the step size (i.e., learning rate parameter) $\eta > 0$ and initial value $\bm x \in \R^d$,
we consider the recursion 
\begin{align}
     \bm X^{\eta|b}_0(\bm x) = \bm x;\quad 
     \bm X^{\eta|b}_{t+1}(\bm x) & = \bm X^{\eta|b}_{t}(\bm x) + \varphi_b\Big(-\eta \nabla f\big( \bm X^{\eta|b}_{t}(\bm x)\big) + \eta \bm\sigma\big( \bm X^{\eta|b}_{t}(\bm x)\big)\bm Z_{t+1}\Big)\quad \forall t = 0,1,2,\cdots,
     \label{def, intro, clipped SGD}
\end{align}
where the gradients $\nabla f(\cdot)$ are perturbed by noise terms $\bm Z_t$'s with power-law heavy-tailed distributions (formally captured by the notion of multivariate regular variation; see Section~\ref{subsec: global dynamics problem setting}),
the coefficient $\bm \sigma(\cdot):\R^{d} \to \R^{d \times d}$ captures the structure of noise at different states,
and the stochastic gradients are truncated by the gradient clipping operator with threshold $b > 0$, i.e., 
\begin{align}
    \notationdef{notation-truncation-operator-level-b}{\varphi_b}(\bm w) 
    \delequal{} 
    (b \wedge \norm{\bm w}) \cdot \frac{\bm w}{\norm{\bm w}},
    \ \ \ \forall \bm w \neq \bm 0;
    \qquad
    {\varphi_b}(\bm 0) \delequal \bm 0. \label{defTruncationClippingOperator, intro}
\end{align}
That is, $\varphi_b(\bm w)$ maintains the direction of the vector $\bm w$ but rescales it to ensure that the norm would not exceed the threshold $b$.
We show that under the presence of truncated heavy-tailed noise, SGD would \emph{almost always stay at the widest minima} over the loss landscape $f(\cdot)$.
Furthermore, this intriguing phenomenon inspires us to propose a new optimization algorithm for finding local minima and improving generalization performance in deep learning.
To be more precise, the main contributions of this paper can be summarized as follows.
\begin{itemize}
    \item 
        {\bf Theoretical Contributions: Characterization of Global Dynamics.}
        We establish a scaling limit of the (possibly truncated) heavy-tailed SGDs \eqref{def, intro, clipped SGD} over a multi-well potential at the process level.         
        The scaling limit is a Markov jump process whose state space consists of the local minima of the potential. 
        In particular, Theorem~\ref{corollary irreducible case} systematically characterize a curious phenomenon that the truncated heavy-tailed processes avoid narrow local minima altogether in the limit.
        As a direct application, we state an ergodic theorem (Corollary~\ref{corollary, elimination of sharp minima, metastability}), which shows that the fraction of time such processes spend in the narrow attraction field converges to zero as the step-size tends to zero.

    \item 
        {\bf Algorithmic Contributions: Control of SGDs using Truncated Heavy Tails.}
        Inspired by the sharp characterization of the global behavior of heavy-tailed SGDs,
        we propose a new training strategy for seeking flat minima in deep learning.
        Specifically, by injecting and then truncating heavy-tailed noise in SGD,
        this novel optimization algorithm consistently finds local minima with a flatter geometry and improved generalization performance when compared to vanilla SGD methods in deep learning experiments.
    \end{itemize}

\noindent
Below, we provide an overview of the paper and a comparison to related literature.

\subsection{Overview of the paper}

\begin{figure}[ht]
\vskip 0.2in

\begin{center}
\begin{tabular}{ c c}
\includegraphics[width=0.33\textwidth]{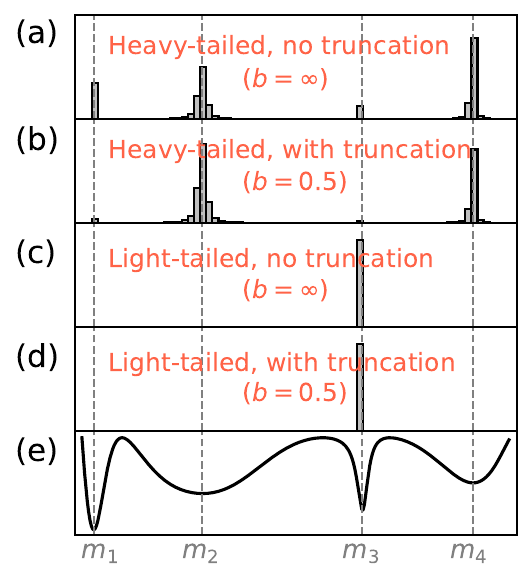}\qquad \qquad& \qquad\qquad \includegraphics[width=0.35\textwidth]{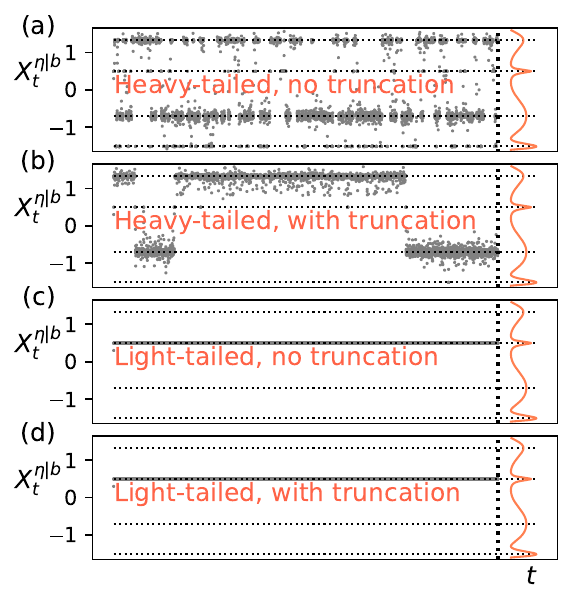}
\\
\end{tabular}
\caption{ 
\textbf{(Left)} Histograms of the locations visited by SGD when driven by different noise
and exploring the multimodal function $f$ plotted in part (e), 
with the dashed lines indicating the local minima of $f$. 
Under truncated heavy-tailed noise, SGD hardly ever visits the two narrower minima $m_1$ and $m_3$, and spends almost all its time around the wider minima $m_2$ and $m_4$.
\textbf{(Right)} Typical trajectories of different SGD methods when exploring $f$, with the dashed lines indicating the locations of the local minima.
Without truncation, heavy-tailed SGDs keep jumping between all local minima of $f$ (see part (a) of the figure). 
In contrast, when driven by truncated heavy-tailed noise, the global dynamics of SGD resemble those of a continuous-time Markov chain that only visits the wider minima of $f$ (see part (b) of the figure).
}
\label{fig histograms}
\end{center}
\vskip -0.2in
\end{figure}

\xwa{source of heavy tails: multiplicative noise (a new paper?) imbalanced dataset }

\xwa{
heavy tails in SGD (papers);
links to metastability analysis (papers);
disadvantage: not width;
disadvantage: lack of control (?);
cooling down L\'evy flights paper;
the RL + heavy tail paper (?)
}

We begin with a brief review of the related literature about heavy tails and gradient clipping,
the key ingredients in algorithm \eqref{def, intro, clipped SGD}.
Heavy tails formally capture the phenomenon where the probability of extreme outliers is relatively high,
which are not exceptions but rather a common feature in modern machine learning tasks.
They arise through multiple mechanisms, including the distribution of gradient noise in SGD \cite{csimcsekli2019heavy,simsekli2019tail,garg2021on},
the imbalance of the training datasets \cite{kunstner2024heavytailed,10.1145/3357713.3384290},
the stationary distribution of SGD under multiplicative noise \cite{gurbuzbalaban2020heavy,hodgkinson2020multiplicative},
and the implicit regularization of weight matrices in SGD \cite{mahoney2019traditional}.
As noted above, of particular interest and relevance to this work is the global dynamics of heavy-tailed SGDs over a multimodal function.
This is closely related to the field of \emph{metastability analysis}, which studies how a stochastic process stays in a semi-stable equilibrium state (in our context, around a local minimum) for a certain amount of time,
and then transitions between such states over longer time scales.

Metastability analyses trace back to the seminal works of Kramers and Eyring \cite{eyring1935chemical,kramers1940brownian,glasstone1941theory},
with the classical Freidlin–Wentzell theory \cite{freidlin1970onsmall,freidlin1984random}
establishing a systematic framework for metastability analysis under light-tailed dynamics.
While attempts have been made to interpret the global dynamics of SGD over non-convex loss landscapes using Freidlin-Wentzell theory (see \cite{azizian2024what,azizian2025the}),
the validity of this light-tailed approach in modern deep learning is challenged not only by the prevalence of heavy-tailed noise, but also by the unreasonably slow exploration predicted by the theory.
Indeed, Freidlin–Wentzell theory reveals that under light-tailed dynamics, the transition times between metastable sets grow exponentially with the noise scale (corresponding to $\eta$ in \eqref{def, intro, clipped SGD} in our setting).
That is, under the standard training paradigm with small step sizes, it would take an astronomically long time for SGD to escape from any local minimum, let alone explore the loss landscape (see Figure~\ref{fig histograms} (left, c \& d) and (right, c \& d) for numerical illustrations in the univariate setting).
This fails to align with SGD's ability to locate flatter minima within a reasonable time horizon \cite{keskar2017on}.
In contrast, \cite{imkeller2006first,pavlyukevich2008metastable}
reveal a fundamentally different metastable behavior under power-law heavy tails:
when driven by regularly varying Lévy processes,
the asymptotic limit of a univariate SDE (after appropriate scaling of time and noise magnitude)
is a continuous-time Markov chain visiting all local minima of the potential function.
In particular, the exit times from any local minimum now scale polynomially, with the prefactor depending on the width of the associated attraction field.
As highlighted in \cite{csimcsekli2019heavy,simsekli2019tail},
these metastability analyses imply that when driven by heavy-tailed noise,
SGD not only explores the landscape much more efficiently (i.e., transition times between local minima are polynomial in $\eta^{-1}$ rather than exponential),
but also tends to spend more time around the wider minima of the loss landscape,
thus providing new perspectives on the flat-minima folklore and the generalization mystery in SGD.
\cite{csimcsekli2019heavy,simsekli2019tail} also empirically verify the connections between the heavy-tailedness in stochastic gradients and the test accuracy in computer vision tasks.
It is worth noting that \cite{imkeller2009exponential} also investigates exit events driven by multiple big jumps, though these multiple big jumps are driven by dynamics exhibiting Weibull tails, a different type of heavy tail that decays faster than the power-law tails studied in this paper.

Despite the strong relevance of metastability theory in deep learning,
as well as its successful extension to multivariate and discrete-time (i.e., SGD-type) settings \cite{imkeller2010first,hogele2014exit,doi:10.1142/S0219493715500197,pavlyukevich2011first,debussche2013dynamics,NEURIPS2019_a97da629},
there have been relatively few attempts to translate such unique metastable behavior into algorithmic insights for seeking flat minima.
In reinforcement learning, \cite{JMLR:v25:21-1343} investigates exit times in policy gradient methods;
the results are in the same spirit as \cite{imkeller2006first,pavlyukevich2008metastable} and imply a preference for wider minima under heavy-tailed policy distributions.
For supervised learning tasks, \cite{gong2025adaptiveheavytailedstochasticgradient} proposes heuristics for injecting and iteratively modifying heavy-tailed noise in SGD,
though its application may suffer from a lack of theoretical justification and the nontrivial computational cost of estimating the trace of the Hessian of loss functions.
In the loosely related context of global optimization over non-convex functions, \cite{pavlyukevich2007levy} combines simulated annealing with heavy-tailed metastability theory to trap Lévy flights around the global minima.
It is worth noting that their algorithm hinges on the efficient exploration of the entire landscape by Lévy-driven SDEs,
which is due to the fast transitions under heavy-tailed dynamics among all local minima, regardless of whether they are wide or narrow 
(see Figure~\ref{fig histograms} (right, a)).
In summary, the potential of metastability-guided optimization toward flat minima remains largely unexplored, a gap this work addresses by characterizing the metastability of truncated heavy-tailed SGDs and their stronger preference for flat minima.

\xwa{clipping;
for convergence rates under heavy tails (papers);
lack of metastability analyses
}

Gradient clipping is a simple and effective technique that prevents excessively large gradients from causing model explosions or numerical instability during training.
First applied by \cite{pascanu2013difficulty} in the context of deep learning,
gradient clipping has since been employed as a default in various settings (e.g., \cite{Engstrom2020Implementation, merity2018regularizing, graves2013generating}). 
Gradient clipping also naturally lends itself to SGD under heavy-tailed noise,
as truncation techniques have long been recognized as effective tools for robust estimation in the presence of extreme variability
(see, e.g., \cite{a2b24e2d-61b3-33c3-8145-4a780deacc0d,6576820}).
Recent progress such as \cite{Zhang2020Why,NEURIPS2020_abd1c782,sun2024gradientnormalizationprovablybenefits,nguyen2023highprobabilityconvergenceclippedsgd,lee2025efficient} establishes faster or more stable convergence when heavy-tailed noise is clipped,
with some works extending beyond vanilla SGD to address adaptive first-order methods and decentralized settings.
Complementing the existing analyses focusing on convergence rates under clipped heavy tails, our results show that a proper clipping regime can also improve heavy-tailed SGD's ability to identify and stay around wide minima.

\xwa{
\begin{itemize}
    \item MC reduction result, explain, figure
    \item reveal much more refined structure in heavy-tailed metastability 
    \item principle of a single big jump, polynomial time scale
    \item width: jump number, exit time 
    \item widest 
    \item proof 
    \item corollary, special case
    \item algorithmic contribution 
    \item corollary, intuition (train long enough) 
    \item tail inflation strategy 
    \item effective in experiments 
\end{itemize}
}

On the theoretical front, the main contribution of this work is the metastability analyses for SGD iterates driven by truncated heavy-tailed noise over a multimodal function; see Figure~\ref{fig histograms} (left, e) for an illustration of a univariate example.
Under suitable conditions, Theorem~\ref{corollary irreducible case} establishes the following sample-path level convergence
\begin{align}
    \{ \bm X^{\eta|b}_{ \floor{ t / \lambda^*_b(\eta)  }  }(\bm x):\ t > 0  \} \Rightarrow \{ \bm Y^{*|b}_t:\ t > 0  \},\qquad\text{as }\eta \downarrow 0,
    \label{main result, intro}
\end{align}
where $\lambda^*_b(\eta)$ is a deterministic function representing the proper time scaling for observing the asymptotics \eqref{main result, intro},
and the limiting process $\bm Y^{*|b}_t$ is a continuous-time Markov chain whose generator only depends on the clipping threshold $b$ and the geometry of $f$.
In particular,  $\bm Y^{*|b}_t$ only visits the \emph{widest minima} over $f$, where the width of each minimum $m_i$ (and the associated attraction field) is captured by the notion of $\mathcal J_b(i)$ introduced in \eqref{def: J * b i,j, metastability}.

We present the rigorous definitions and statements in Section~\ref{subsec: main results},
and highlight here the main takeaway of Theorem~\ref{corollary irreducible case}:
under small $\eta$, the global dynamics of the truncated heavy-tailed SGD $\bm X^{\eta|b}_{t}(\bm x)$ closely resemble those of a Markov chain that only visits and make transitions between the widest region over $f$.
Figure~\ref{fig histograms} clearly illustrates these phenomena (see Section~\ref{sec:numerical-experiments} for details of the numerical experiments).
Under light-tailed gradient noise, SGD remains trapped in sharp minima, regardless of gradient clipping; see parts (c) and (d) of Figure~\ref{fig histograms} (left, right).
In contrast, when driven by heavy-tailed noise, SGD jumps between different local minima instead being trapped at one of them; see parts (a) and (b) of Figure~\ref{fig histograms} (left, right).
Furthermore, a clear distinction arises between clipped and unclipped cases:
without clipping, SGD constantly jumps around local minima ${m_1, m_2, m_3, m_4}$ and spends a significantly proportion of time at each of them (see part (a) of Figure~\ref{fig histograms} (left, right)),
whereas under clipping, heavy-tailed SGD resembles a Markov jump process that only visits the two wide minima ${m_1, m_3}$, and spends almost all time there (see part (b) of Figure~\ref{fig histograms} (left, right)).

Theorem~\ref{corollary irreducible case} extends far beyond the existing metastability analyses for heavy-tailed dynamics (e.g., \cite{imkeller2006first,pavlyukevich2008metastable,hogele2014exit,doi:10.1142/S0219493715500197}),
and reveals the existence of a much more refined mathematical structure when truncation is involved.
Prior works are manifestations of the \emph{principle of a single big jump}---a well-known phenomenon in extreme value theory---as the transitions between metastable sets are almost always caused by a single step with disproportionately large noise,
and
the transitions times are (roughly) of order $1/\eta^{\alpha}$ with $\alpha$ being the power-law tail index for the noise distribution.
See also Corollary~\ref{theorem: metastability, unclipped} where, essentially as a special case of Theorem~\ref{corollary irreducible case}, 
we send $b \to \infty$ in \eqref{main result, intro}
and recover the metastable behavior governed by the principle of a single big jump for heavy-tailed SGDs without truncation.
Nevertheless,
this intuition clearly fails under the gradient clipping mechanism, which confines the one-step movement of $\bm X^{\eta|b}_{t}(\bm x)$ within a bounded set of radius $b$ regardless of the original size of the heavy-tailed noise.
Instead, 
the number of steps required to escape from a local minimum $m_i$ now depends on the interplay between the clipping threshold $b$ and the geometry (in particular, width) of the local minimum.
This gives rise to the notion of width $\mathcal J_b(i)$ in \eqref{def: J * b i,j, metastability},
defined as the minimum number of jumps (with size bounded by $b$) required to escape from the attraction of $m_i$.

More precisely,  our proof of Theorem~\ref{corollary irreducible case} builds upon the first exit analyses for (truncated) heavy-tailed dynamics developed in \cite{wang2023large}.
Specializing the results to our setting, we obtain that, when initialized in an attraction  field $I_i$,
the time it takes $\bm X^{\eta|b}_t$ to escape from $I_i$ is (roughly) of order 
\begin{align}
    1/\eta^{ \mathcal J_b(i) \cdot (\alpha - 1) + 1  },
    \label{proof strategy, remark, exit time order}
\end{align}
i.e., it scales (roughly) polynomially with the exponent determined by the width metric $\mathcal J_b(i)$ in \eqref{def: J * b i,j, metastability},
and the exits are almost always driven by exactly $\mathcal J_b(i)$ big jumps (i.e., disproportionately large noise $\bm Z_t$'s).
This discrete hierarchy in exit times---depending on $\mathcal J_b(i)$---suggests that, compared to the time $\bm X^{\eta|b}_t$ spends at the widest minima (those with $\mathcal J_b(i) = \mathcal J^*_b$; see \eqref{def: J * b}), the time spent at narrower minima is almost negligible under small $\eta$
due to the smaller power-law rates in \eqref{proof strategy, remark, exit time order}.
To make this argument rigorous, we apply two technical tools.
First, the first exit analyses in  \cite{wang2023large} provide not only the scaling of the exit times but also the precise asymptotic prefactors as well as the asymptotic law of the exit locations,
thus revealing the transition probabilities between attraction fields. 
(See Section~\ref{sec: review, first exit analyses} in Appendix for a more detailed review.)
Moreover, in Section~\ref{subsec: proof, abstract framework, metastability} we develop a general framework for establishing sample-path level convergence in distributions to jump processes,
given the convergence of the jump times and locations (which is 3exactly the content of the first exit analyses).
Combining these tools, we provide in Section~\ref{subsec: proof, theorems, metastability} the proof of Theorem~\ref{corollary irreducible case},
with the proof of key propositions detailed in Section~\ref{subsec: proof, propositions, metastability}.


Furthermore, our metastability analysis translates to a novel algorithmic framework for finding wide minima in deep learning tasks.
As noted earlier, Theorem~\ref{corollary irreducible case} suggests that (a time scaled version of) $\bm X^{\eta|b}_{t}(\bm x)$ spends almost all time around the wide minima over $f$.
This is confirmed through a continuous mapping argument in Corollary~\ref{corollary, elimination of sharp minima, metastability}, which informally states that 
\begin{align}
  \frac{1}{T / \lambda^*_b(\eta) }\sum_{t=1}^{T/ \lambda^*_b(\eta)} \mathbf{I}\bigg\{
    \bm X^{\eta|b}_{t}(\bm x) \in  \bigcup_{ i:\ \bm m_i \in \text{widest minima}}B_\epsilon(\bm m_i)
    \bigg\}
     \stackrel{p}{\to}
    1,
    \qquad 
    \text{ as }\eta \downarrow 0
    \label{corollary, ergodicity, intro}
\end{align}
for any $T,\epsilon > 0$,
where $B_\epsilon(\bm y)$ denotes the $L_2$ open ball around $\bm y$ with radius $\epsilon > 0$.
In other words, provided that truncated heavy-tailed SGD has been running for long enough (by the criterion of the time scaling $\lambda^*_b(\eta)$ in \eqref{corollary, ergodicity, intro}),
it spends almost all time  around wide minima under small learning rate $\eta$.
We provide the rigorous statements in Section~\ref{subsec: implication, elimination of narrow minima}, and stress that \eqref{corollary, ergodicity, intro} suggests a highly effective method for finding wide minima in deep learning tasks using truncated heavy-tailed noise.
We flesh out this idea in Section~\ref{subsec: implication, elimination of narrow minima} by proposing a new training strategy that estimates the the gradient noise from data, inflates the tail distribution of the noise using heavy-tailed variables, and then truncates the heavy-tailed stochastic gradient by the gradient clipping operator.
Section~\ref{subsec: experiment 1, ablation study} conducts deep learning experiments and confirms that our truncated heavy-tailed optimizer finds solutions with flatter geometry and better generalization performance when compared to standard SGD.
Moreover, Section~\ref{subsec: experiment 2, adam wrn} shows that, even when incorporated with adaptive gradient methods, more complex model architecture, and training techniques to generalization performance of SGD,
our truncated heavy-tailed method still improves upon the fine-tuned baseline and finds flat solution with better generalization performance.

\subsection{Comparison to Related Works}

\xwa{
algorithms tailored for improving sharpness
}

\xwa{
SAM
\begin{itemize}
    \item OG SAM paper: \cite{foret2021sharpnessaware}
    \item extensions (modifying the perturbation directions, reducing the computational cost due to multiple gradient evaluations ): \cite{pmlr-v139-kwon21b,pmlr-v162-kim22f,zhuang2022surrogate,NEURIPS2024_5bf2b802}
    \item despite its popularity, global dynamics (selecting minima in different valleys) remains unsolved.
    \item implicit regularization of variance, magnitude of stochastic gradients, or value of smoothed functions:
        \cite{luo2025unveilingmsharpnessstructurestochastic,wen2022how,pmlr-v202-monzio-compagnoni23a}
    \item select wider minima, but lie within the same valley:
        \cite{zhou2025sharpnessaware,wen2022how}
    \item SDE work: \cite{pmlr-v202-monzio-compagnoni23a,luo2025unveilingmsharpnessstructurestochastic}
    \item exponential exit time (with different rates depending on noise structure)
    \item limitation of these analyses: \cite{NEURIPS2023_5305b789}
    \item flatness-aware SGLD and its invariant distribution: \cite{bruno2025flatnessawarestochasticgradientlangevin}
    
\end{itemize}
}

This paper focuses on the characterization and control of the \emph{global dynamics of SGD}
for attaining {strong preference to flat minima when exploring a multimodal loss landscape},
a crucial goal that remains unexplored in existing literature about optimization towards flat minima.
Specifically,
in light of the flat-minima folklore, several optimization algorithms have been proposed by incorporating explicit or implicit regularization on sharpness into stochastic first-order methods; e.g., \cite{chaudhari2017entropysgd,Zhou_Li_Feng_Huang_2025,NEURIPS2022_69b55345}.
Two of the most popular approaches, due to their effectiveness and scalability, are Sharpness-Aware Minimization (SAM) and Stochastic Weight Averaging (SWA). 
Originally proposed by \cite{foret2021sharpnessaware},
SAM 
interprets sharpness as $\max_{ \norm{\bm \delta} \leq \rho}f(\bm \theta + \bm \delta) - f(\bm \theta)$,
and 
aims to solve $\min_{\bm \theta}\max_{ \norm{\bm \delta} \leq \rho}f(\bm \theta + \bm \delta)$, which considers the loss under bounded perturbations to model weights.
Due to tractability and efficiency concerns regarding the min-max objective, SAM resorts to a first-order Taylor expansion
and updates the model weights by 
$\bm \theta \leftarrow \bm \theta - \eta \nabla f\big( \bm \theta + \rho\frac{ \nabla f(\bm \theta)  }{\norm{\nabla f(\bm \theta)} }  \big)$,
where $\eta$ denotes the step size (i.e., learning rate) parameter, and the $\nabla f(\cdot)$'s are estimated over small batches.
Since then, several extensions of SAM have been proposed (see, e.g., \cite{pmlr-v139-kwon21b,pmlr-v162-kim22f,zhuang2022surrogate,NEURIPS2024_5bf2b802}) to modify the perturbation directions in SAM or to reduce the computational cost of multiple gradient evaluations.
While it has been argued that SAM and its variants resemble SGDs with loss function regularized by its Hessian, by the magnitude of stochastic gradients, or under a smoothed version of the loss function (see  \cite{luo2025unveilingmsharpnessstructurestochastic,wen2022how,pmlr-v202-monzio-compagnoni23a}),
the question of interest here is whether (and to what extent) SAM is able to find and remain near minima with more stable geometry among the numerous minima in non-convex and multimodal loss landscapes.
Regarding this question, theoretical analyses (e.g., \cite{zhou2025sharpnessaware,wen2022how}) on SAM have so far provided affirmative answers only at a \emph{local} level:
locally within a certain attraction field (in particular, over a connected region attaining small values of the loss function $f$), SAM can avoid sharp regions (in terms of the trace of the Hessian) and move toward flatter areas.
However, this is not verified at a \emph{global} level, which would require SAM to efficiently traverse a multimodal landscape and identify flat minima from different attraction fields.\footnote{In the example of Figure~\ref{fig histograms} (Left, e), this refers to the efficient exploration of the disconnected minima $m_i$'s and identification of the ones with more stable geometry.
In fact, it is likely that SAM becomes rather inefficient when moving between different attraction fields,
as SAM resembles Brownian-motion-driven SDEs under small step size $\eta$ \cite{pmlr-v202-monzio-compagnoni23a,luo2025unveilingmsharpnessstructurestochastic},
which spends exponentially long time (in $\eta$) to escape any attraction field as characterized by the classical Freidlin-Wentzell theory \cite{freidlin1984random}.
}

\xwa{
the averaging methods (SWA).
\cite{izmailov2018averaging}: empirical observations (the first one?).
Recent empirical studies: \cite{10.1007/978-3-031-40837-3_13,NEURIPS2022_69b55345,pmlr-v238-nitanda24a}.
Theoretical justification: \cite{pmlr-v238-nitanda24a}, building upon the smoothed function perspective in \cite{pmlr-v80-kleinberg18a}.
Limitation: theoretical analysis is lacking; relies on one-point strong convexity.
}

Similar limitations arise in SWA, an approach that produces the final model weights by taking an average over the training trajectory.
As noted by \cite{izmailov2018averaging}, SWA finds wider solutions with improved generalization performance compared to SGD.
The benefits of SWA have been further confirmed across a wide range of tasks (see \cite{10.1007/978-3-031-40837-3_13,NEURIPS2022_69b55345,pmlr-v238-nitanda24a}).
Theoretical justifications are provided by drawing connections to convex optimization theory, where the Polyak–Ruppert type averaging scheme leads to optimal convergence rates in SGD (see, e.g., \cite{NIPS2011_40008b9a}).
In particular, \cite{pmlr-v238-nitanda24a} builds on an alternative perspective that treats the stochasticity in SGD as a smoothing of the objective function (see \cite{pmlr-v80-kleinberg18a}) and assumes that the smoothed function is nearly convex when viewed from a (likely flat and wide) minimum;
in this case, SWA resembles averaged SGD over convex functions, thereby enjoying its faster and more stable convergence.
Nevertheless, such one-point convexity assumptions are not guaranteed and are likely to fail for multimodal landscapes without aggressive smoothing (i.e., under reasonable step sizes and noise magnitudes),
rendering the analogy to averaged SGD over convex functions largely irrelevant when studying the global dynamics in the training of deep neural networks.
See also the analyses in \cite{NEURIPS2019_01d8bae2}, which confirm that, locally within an attraction field with asymmetric geometry, the averaging scheme can help bias the model weights toward the flatter side.

\xwa{
algorithmic stability: \cite{pmlr-v201-raj23a,pmlr-v202-raj23a,dang2025algorithmicstabilitystochasticgradient}.
Logic: Uniform-stability $\rightarrow$ generalization bounds.
Disadvantages:  Lack of insight about choice of $\alpha$ or comparison between light vs heavy; bounds rely on range of datasets, or are tailored for quadratic-like functions.
}

On a related note, recent works \cite{pmlr-v201-raj23a,pmlr-v202-raj23a,dang2025algorithmicstabilitystochasticgradient} study the generalization of heavy-tailed SGDs (and variants) through the lens of algorithmic stability.
Specifically, the notion of uniform stability is characterized by the change in the output of an algorithm when the training dataset differs by exactly one data point,
and verifying uniform stability immediately yields upper bounds on the generalization error of empirical risk minimization (see \cite{pmlr-v48-hardt16}).
We note that this line of research so far has focused on developing technical tools for establishing bounds on the uniform stability of heavy-tailed SGD (or the continuous-time SDE as its proxy), rather than providing algorithmic insights for improving generalization performance (e.g., quantitative comparison of the generalization error of light-tailed vs.\ heavy-tailed SGD, or suggestions on ideal heavy-tailedness or noise distributions in practice for minimizing generalization error).

Earlier versions of some of the results presented in this paper appeared in \cite{wang2022eliminating}. Specifically, Theorems 3 and 2 in \cite{wang2022eliminating} correspond to the one-dimensional (and constant diffusion coefficient) cases of Theorem~\ref{corollary irreducible case} and Corollary~\ref{corollary, elimination of sharp minima, metastability}, which establish the general multidimensional case with state-dependent diffusion coefficients.

The rest of this paper is organized as follows. 
Section~\ref{subsec: global dynamics problem setting} collects frequently used notations and definitions
and states the problem setting. 
Section~\ref{sec: metastability} presents the main results of this paper.
Specifically, Section~\ref{subsec: main results} studies the scaling limit of $\bm X^{\eta|b}_t(\bm x)$ and characterizes the global dynamics of heavy-tailed SGD under truncation.
Inspired by this result, 
Section~\ref{subsec: implication, elimination of narrow minima} proposes an algorithm that controls the training dynamics of SGD through tail inflation and truncation.
Section~\ref{sec: experiments combined} presents simulation and deep learning experiments.
The technical proofs are deferred to the Appendix.

\section{Notations and Problem Settings}
\label{subsec: global dynamics problem setting}

Let $\mathbb Z$ be the set of integers, $\mathbb{N} = \{1,2,\cdots\}$ be the set of positive integers, and $\mathbb Z_+ = \{0,1,2,\cdots\}$ be the set of non-negative integers.
Let $\notationdef{set-for-integers-below-n}{[n]} = \{1,2,\cdots,n\}$ for any positive integer $n$, with convention $[0] = \emptyset$.
For any $x \in \R$,
let
$\notationdef{floor-operator}{\floor{x}}\delequal \max\{n \in \mathbb{Z}:\ n \leq x\}$
and
$\notationdef{ceil-operator}{\ceil{x}} \delequal \min\{n \in \Z:\ n \geq x\}$
be the rounded-down and rounded-up operators, respectively .
Given $x,y \in \R$, let $x \wedge y \delequal \min\{x,y\}$ and $x \vee y \delequal \max\{x,y\}$.
Consider a metric space $(\mathbb{S},\bm{d})$ with
$\notationdef{notation-borel-sigma-algebra}{\mathscr{S}_\mathbb{S}}$
being its Borel $\sigma$-algebra.
For any $E\subseteq \mathbb{S}$,
let
$\notationdef{notation-interior-of-set-E}{E^\circ}$ and $\notationdef{notation-closure-of-set-E}{E^-}$ be the interior and closure of $E$, respectively.
For any $r > 0$, 
let
$\notationdef{notation-epsilon-enlargement-of-set-E}{E^r} \delequal 
\{ y \in \mathbb{S}:\ \bm{d}(E,y)\leq r\}$ be the $\epsilon$-enlargement of $E$.
Here, for any set $A \subseteq \mathbb{S}$ and any $x \in \mathbb{S}$,
we define $\bm{d}(A,x) \delequal \inf\{\bm{d}(y,x):\ y \in A\}$.
Let
$
\notationdef{notation-epsilon-shrinkage-of-set-E}{E_{r}} \delequal
((E^c)^r)^\complement
$
be the $r$-shrinkage
of $E$.
We say that set $A \subseteq \mathbb{S}$ is bounded away from $B \subseteq \mathbb{S}$ under $\bm d$
if $\inf_{ x\in A,y\in B }\bm{d}(x,y) > 0$.
Given two sequences of positive real numbers $(x_n)_{n \geq 1}$ and $(y_n)_{n \geq 1}$, 
we say that $x_n = \bm{O}(y_n)$ (as $n \to \infty$) if there exists some $C \in [0,\infty)$ such that $x_n \leq C y_n\ \forall n\geq 1$.
Besides, we say that $x_n = \bm{o}(y_n)$ if $\lim_{n \rightarrow \infty} x_n/y_n = 0$.

Throughout this paper, we
consider the $L_2$ norm 
$
\norm{(x_1,\cdots,x_k)} = \sqrt{\sum_{j = 1}^k x^2_k}
$
on Euclidean spaces.
Besides, we adopt the $L_2$ vector norm induced matrix norm
$
\norm{\textbf A} = \sup_{ \bm x \in \R^q:\ \norm{\bm x} = 1 }\norm{\textbf A\bm x}
$
for any $\textbf A \in \R^{p \times q}$.
For each $\bm x \in \R^d$ and $r > 0$,
we use $B_r(\bm x) \delequal \{ \bm y \in \R^d: \norm{\bm y - \bm x} < r  \}$
to denote the open ball centered at $\bm x$ with radius $r$,
and 
 $\bar B_r(\bm x) \delequal \{ \bm y \in \R^d: \norm{\bm y - \bm x} \leq  r  \}$
for the corresponding closed ball.

Throughout this paper, we fix some positive integer $d$ to denote the dimensionality of the problem at hand,
and use
$\mathbb{D}(I)$ to denote the space of all $\R^d$-valued càdlàg functions on the domain $I$,
where we only consider domains of the form $I = [0,T]$ or $I = [0,\infty)$.
In this paper, we characterize sample-path level convergence of $\R^d$-valued stochastic processes in terms of the following two modes.
First, 
we say that $\{S^\eta_t: t>0\}$ converges to $\{S^*_t: t > 0\}$ \emph{in finite-dimensional distributions} (f.d.d.) if we have
$
    \big(S^\eta_{t_1},\cdots,S^\eta_{t_k}\big) \Rightarrow \big(S^*_{t_1},\cdots, S^*_{t_k}\big)
$    
as $\eta \downarrow 0$ for any $k\geq 1$ and $0 < t_1 < t_2 < \cdots < t_k < \infty$.
We also denote this as $\{S^\eta_{t}: t>0\}\tofdd \{S^*_{t}: t>0\}$. 
Note that in this paper the convergence in f.d.d.\  is required only on $(0,\infty)$ and does not concern the law at $t=0$. 
Next, we recall the convergence 
w.r.t.\ the $L_p$ topology in $\D[0,\infty)$.
For any $p \in [1,\infty)$ and $T \in (0,\infty)$, let
\begin{align}
    \dlp{[0,T]}(x,y) \delequal \bigg( \int_0^T \norm{x_t - y_t}^p dt \bigg)^{1/p},
    \qquad \forall x,y \in \mathbb D[0,T]
    \label{def, Lp distance}
\end{align}
be the $L_p$ metric on $\mathbb D[0,T]$.
For any $T > 0$, define the projection $\pi_T: \D[0,\infty) \to \D[0,T]$ by
\begin{align}
    \pi_T(\xi)_t = \xi_t,\qquad\ \forall t \in [0,T].
    \label{def: projection mapping pi T}
\end{align}
Now, we define
\begin{align}
   \dlp{[0,\infty)}(x,y)
    \delequal 
    \sum_{k \geq 1}
    \frac{
        1 \wedge \dlp{[0,k]}\big(\pi_k(x),\pi_k(y)\big)
    }{
        2^k
    },
    \qquad
    \forall x,y \in \D[0,\infty)
     \label{def, Lp metric D 0 infty}
\end{align}
and note that $\dlp{[0,\infty)}$ is a metric on $\D[0,\infty)$.
We say that a sequence of càdlàg processes $\{S^\eta_t:\ t \geq 0\}$
converges in distribution to $\{S^*_t:\ t \geq 0\}$ \emph{w.r.t.\ the $L_p$ topology in $\D[0,\infty)$} as $\eta \downarrow 0$ if
$
\lim_{\eta \downarrow 0}\E g(S^\eta_{ \boldsymbol{\cdot} }) = \E g(S^*_{ \boldsymbol{\cdot} })
$
for all $g:\D[0,\infty) \to \R$ that is bounded and continuous (w.r.t.\ the topology induced by $\dlp{[0,\infty)}$).
We denote this by $S^\eta_{ \boldsymbol{\cdot} }\Rightarrow S^*_{ \boldsymbol{\cdot} }$ in $(\D[0,\infty),\dlp{[0,\infty)})$ 
or
$\{S^\eta_{t}:  t\geq 0\} \Rightarrow \{S^*_{t}: t \geq 0\}$ in $(\D[0,\infty),\dlp{[0,\infty)})$.

Next, we set up the problem by formally introducing truncated heavy-tailed SGDs and the assumptions on multimodal loss landscape.
Consider a multimodal potential function $f:\R^d \to \R$ with local minima $\{\bm m_1,\bm m_2, \ldots, \bm m_K\}$,
associated with attraction fields $\{I_1,I_2,\ldots,I_K\}$.
More precisely, let
\begin{align}
\bm y_0(\bm x) =\bm x,\qquad
   \frac{d\bm{y}_t(\bm x)}{dt} = -\nabla f\big(\bm{y}_t(\bm x)\big) \ \ \forall t \geq 0
   \label{def ODE path y t}
\end{align}
be the gradient flow path over $f$ under the initial value $\bm x$.
We make the following assumption throughout this section.
Recall that given a set $I$, we use $I^-$ to denote its closure.

\begin{assumption} 
\label{assumption: geometry, metastability}
    Let $f:\R^d \to \R$ be a function in $\mathcal{C}^1(\R^d)$,
    and let $K \geq 2$ be a positive integer. 
    There exist $(I_k)_{k \in [K]}$---a collection of non-empty open sets that are mutually disjoint---and $(\bm m_k)_{k \in [K]}$ with $\bm m_k \in I_k$ for each $k \in [K]$,
    such that $\bigcup_{k \in [K]}(I_k)^- = \R^d$, and the following claims hold.
    \begin{enumerate}[(i)]
        \item 
            \textbf{(Attraction fields of local minima)}
            For each $k \in [K]$, we have $\nabla f(\bm m_k) = \bm 0$, and the claim 
            \begin{align}
                \bm y_t(\bm x) \in I_k\ \forall t \geq 0;\qquad \lim_{t\to \infty}\bm y_t(\bm x) = \bm m_k
                \nonumber
            \end{align}
            holds for all $\bm x \in I_k$.

        \item 
            \textbf{(Contraction around local minima)}
            For each $k \in [K]$, it holds for all $\epsilon > 0$ small enough that
            $
            \nabla f(\bm x)^\top (\bm x - \bm m_k) > 0\ \forall\bm x \in \bar B_{\epsilon}(\bm m_k) \setminus \{\bm m_k\}.
            $

        \item 
            \textbf{(Dissipativity)} It holds for any $M$ large enough that $\inf_{ \norm{\bm x} \geq M }\nabla f(\bm x)^\top \bm x > 0$.

    \end{enumerate}
\end{assumption}

See Figure~\ref{fig typical transition graph} (Left) for an univariate example of such $f$ with $K= 3$,
where the local maxima $s_i$'s partition $\R$ into different regions $\notationdef{notation-attraction-field-I-i}{I_i} = (s_{i-1},s_i)$.
Such regions can be viewed as the \emph{attraction fields} 
of the local minima $m_i$'s.
That is, the ODE $\bm y_t(x)$ defined in \eqref{def ODE path y t} admits the limit $\bm y_t(x) \to m_i$ (as $t \to \infty$) for each $x \in I_i$.
We add two remarks regarding Assumption~\ref{assumption: geometry, metastability}.
First, 
we impose the condition $K \geq 2$ simply to avoid the trivial case where there exists only one attraction field (so there are no transitions between different attraction fields).
Besides, 
condition (ii) holds if $f$ is locally $\mathcal C^2$ and locally strongly convex around each $m_k$,
and condition (iii) is standard for ensuring that the gradient flows always return to a compact region of $\R^d$.

Next, we introduce SGDs driven by truncated heavy-tailed noise, the main object of study in this paper.
Specifically, let \notationdef{notation-Z-iid-noise-LDP}{$\bm Z_1,\bm Z_2,\ldots$} be iid copies of some random vector $\bm Z$ taking values in $\R^d$.
Given the initial value $\bm x \in \R^d$, step length $\eta > 0$, truncation threshold $b \in (0,\infty)$,
and the diffusion coefficient (i.e., noise magnitude matrix) $\notationdef{sigma}{\bm \sigma}:\mathbb{R}^d\to \mathbb{R}^{d\times d}$,
let the discrete-time process $\big\{ \notationdef{notation-X-eta-j-truncation-b-LDP}{\bm X^{\eta|b}_t(\bm x)}: t \in \mathbb N\big\}$ in $\R^d$ be defined by the recursion
\begin{align}
    \bm X^{\eta|b}_0(\bm x) = \bm x,
    \quad 
    {\bm X^{\eta|b}_t(\bm x)} = \bm X^{\eta|b}_{t - 1}(\bm x) +  \varphi_b\Big(-\eta \nabla f\big(\bm X^{\eta|b}_{t - 1}(\bm x)\big) + \eta \bm \sigma\big(\bm X^{\eta|b}_{t - 1}(\bm x)\big)\bm Z_t\Big)\ \ \forall t \geq 1,
    \label{def: X eta b j x, clipped SGD}
\end{align}
where
the gradient clipping operator $\varphi_\cdot(\cdot)$ is defined by
\begin{align}
    \notationdef{notation-truncation-operator-level-b}{\varphi_b}(\bm w) 
    \delequal{} 
    (b \wedge \norm{\bm w}) \cdot \frac{\bm w}{\norm{\bm w}},
    \ \ \ \forall \bm w \neq \bm 0;
    \qquad
    {\varphi_b}(\bm 0) \delequal \bm 0. \label{defTruncationClippingOperator}
\end{align}
In other words, the truncation operator $\varphi_b(\bm w)$ in \eqref{def: X eta b j x, clipped SGD} maintains the direction of the vector $\bm w$ but rescales it to ensure that the norm would not exceed the threshold $b$.
In particular, we are interested in the case where $\bm Z_i$'s are heavy-tailed, which is formally captured via the notion of multivariate regular variation.
We say that a measurable function $\phi:(0,\infty) \to (0,\infty)$ is regularly varying as $x \rightarrow\infty$ with index $\beta$ (denoted as $\phi(x) \in \RV_\beta(x)$ as $x \to \infty$) if $\lim_{x \rightarrow \infty}\phi(tx)/\phi(x) = t^\beta$ for each $t>0$,
and that 
$\phi(\eta)$
is regularly varying as $\eta \downarrow 0$ with index $\beta$ 
if $\lim_{\eta \downarrow 0} \phi(t\eta)/\phi(\eta) = t^\beta$ for each $t > 0$ (denoted by $\phi(\eta) \in \notationdef{notation-RV-LDP}{\RV_{\beta}}(\eta)$ as $\eta \downarrow 0$).
For a standard treatment to regularly varying functions, see, e.g.,
\cite{resnick2007heavy, foss2011introduction}. 
Let
\begin{align}
    \notationdef{notation-H}{H(x)} \delequal \P(\norm{\bm Z} > x). \label{def: H, law of Z_j}
\end{align}
For any $\alpha > 0$, let $\notationdef{notation-measure-nu-alpha}{\nu_\alpha}$ be the (Borel) measure on $(0,\infty)$ with
\begin{align}
     \nu_\alpha[x,\infty) = x^{-\alpha}. \label{def: measure nu alpha}
\end{align}
Let $\notationdef{notation-R-d-unit-sphere}{\mathfrak N_d} \delequal \{\bm x \in \R^d:\ \norm{\bm x} = 1\}$ be the unit sphere of $\R^d$.
Let $\Psi: \R^d \to [0,\infty) \times \mathfrak N_d$ be
\begin{align}
    \notationdef{notation-Phi-polar-transform}{\Psi(\bm x)} \delequal 
    \begin{cases}
         \Big(\norm{\bm x},\frac{\bm x}{\norm{\bm x}}\Big) &\text{ if }\bm x \neq 0
         \\
         \big( 0, (1,0,0,\cdots,0)\big) & \text{ otherwise}
    \end{cases},
    \label{def: Phi, polar transform in Rm}
\end{align}
where the origin is included in the domain of $\Psi$ as a convention and is of no consequence to the proofs. 
Thus, $\Psi$ can be interpreted as the polar transform with domain extended to $\bm 0$.
Throughout, we work with the following heavy-tailed assumption regarding the noise term $\bm Z$.
Note that in \eqref{claim, Rd heavy tailed assumption},
the vague convergence is equivalent to convergence in $\mathbb M\Big( 
            \big([0,\infty) \times \mathfrak N_d \big)
            \setminus
            \big( \{0\} \times \mathfrak N_d \big)
            \Big)$;
see Remark~2 in \cite{wang2023large} for details, and  \cite{lindskog2014regularly} for elaborations on the mode of $\mathbb M$-convergence for measures.

\begin{assumption}[Regularly Varying Noise]\label{assumption gradient noise heavy-tailed}
$\E \bm Z = \bm 0$. 
Besides, there exist some $\notationdef{alpha-noise-tail-index-LDP}{\alpha} > 1$ and 
a probability measure $\mathbf S(\cdot)$ on the unit sphere $\mathfrak N_d$ such that
\begin{itemize}
    \item 
        $H(x) \in \RV_{-\alpha}(x)$ as $x \to \infty$,
    \item 
        for the polar coordinates $(R,\bm \Theta) \delequal \Psi(\bm Z)$, we have (as $x \to \infty$)
        \begin{align}
            \frac{
                \P\Big( (x^{-1}R, \bm \Theta) \in\ \cdot\ \Big)
            }{
                H(x)
            }
            \xrightarrow{v}
            \nu_\alpha \times \mathbf S,
            \label{claim, Rd heavy tailed assumption}
        \end{align}
        where
        $\xrightarrow{v}$ denotes vague convergence,

    \item 
        the measure $\mathbf S(dx) = f_{\mathbf S}(x)dx$ admits a density over $\mathfrak N_d$, with $\inf_{ x \in \mathfrak N_d }f_{\mathbf S}(x) > 0$.
\end{itemize}
\end{assumption}

We also impose the following regularity conditions on $\nabla f(\cdot)$ and $\bm \sigma(\cdot)$.

\begin{assumption}[Lipschitz Continuity]
\label{assumption: lipschitz continuity of drift and diffusion coefficients}
There exists some $\notationdef{notation-Lipschitz-constant-L-LDP}{D} \in (0,\infty)$ such that
$$\norm{\bm \sigma(\bm x) - \bm \sigma(\bm y)} \vee \norm{\nabla f(\bm x)-\nabla f(\bm y)} \leq D\norm{\bm x - \bm y},
\qquad \forall \bm x,\ \bm y \in \mathbb{R}^d.$$
\end{assumption}

\begin{assumption}[Nondegeneracy]
\label{assumption: nondegeneracy of diffusion coefficients}
$\bm\sigma(\bm x)$ is not a singular matrix for any $\bm x\in \R^d$.
\end{assumption}

\section{Main Results}
\label{sec: metastability}

This section presents the main results of this paper.
Section~\ref{subsec: main results} shows that, after proper time-scaling, the sample paths of truncated heavy-tailed SGDs converge in distribution to those of a Markov jump process;
curiously, the state space of this limiting process consists of only the widest local minima of the loss landscape.
Inspired by such intriguing global dynamics in heavy-tailed SGDs,
Section~\ref{subsec: implication, elimination of narrow minima} proposes a novel algorithm for finding wide minima and improving the generalization performance in the training of deep leanring models.

\subsection{Characterization of Global Dynamics of Heavy-Tailed SGD}
\label{subsec: main results}

The goal of this paper is to rigorously show that the global dynamics of $\bm X^{\eta|b}_t(\bm x)$ (i.e., truncated heavy-tailed SGDs) closely resemble those of a Markov jump process that only visits the ``widest'' attraction fields over $f$.
To facilitate the presentation of the main results, we first introduce a few definitions.
For each $b > 0$ and $\bm x \in \R^d$, let $\mathcal{G}^{(0)|b}(\bm x) \delequal \{\bm x\}$,
and (for each $k \geq 1$)
\begin{align}
    {\mathcal G^{(k)|b}(\bm x)}
    \delequal
    \bigg\{
        \bm y_t(\bm z) + \varphi_b\Big( \bm \sigma\big( \bm y_t(\bm z)  \big)  \bm w \Big):\ 
        t > 0,\ \bm w \in \R^d,\ \bm z \in {\mathcal G^{(k-1)|b}(\bm x)}
    \bigg\},
    \label{property: hierarchy of set G k b}
\end{align}
where the gradient flow $\bm y_t(\cdot)$ is defined in \eqref{def ODE path y t}.
Intuitively speaking, 
${\mathcal G^{(k)|b}(\bm x)}$ is
the region accessible by the gradient flow path initialized at $\bm x$ and with $k$ perturbations , where the size of each perturbation is modulated by $\bm \sigma(\cdot)$ and truncated under $b$.
Note also that $\mathcal G^{(k)|b}(\bm x)$ is monotone in $k$ and $b$, in the sense that
$
\mathcal{G}^{(k)|b}(\bm x) \subseteq \mathcal{G}^{(k+1)|b}(\bm x), 
$
and
$
\mathcal{G}^{(k)|b}(\bm x) \subseteq \mathcal{G}^{(k)|b^\prime}(\bm x) 
$
for all $0 < b \leq b^\prime$.
Equipped with the definition of  $\mathcal G^{(k)|b}(\bm x)$, we are ready to introduce the notion of width for each attraction field that will be considered throughout this paper.
Recall that under Assumption~\ref{assumption: geometry, metastability}, there are $K$ distinct attraction fields over $f$, associated with the local minima $m_i$'s.
For each $i \in [K]$, let
\begin{align}
    \mathcal J_b(i) \delequal \min\big\{ k \geq 0:\ \mathcal G^{(k)|b}(\bm m_i) \cap (I_i)^\complement \neq \emptyset
    \big\}.
    \label{def: J * b i,j, metastability}
\end{align}
That is, we characterize the {width of $I_i$} by considering \emph{the minimum number of perturbations (with sizes truncated under $b$) required to escape the attraction of $m_i$}.

\begin{remark}[Connection to the Relative Width of $I_i$]
We add a few remarks regarding the connection between $\mathcal J_b(i)$ and the width of $I_i$.
Let $r(i) \delequal \inf\{ \norm{\bm m_i - \bm y}:\ \bm y \notin I_i  \}$
be the effective width of $I_i$ (starting from the local minimum $m_i$).
Note that 
(1) the term $\ceil{ r(i)/b  }$ is the width of $I_i$ relative to the truncation threshold $b$,
(2) the quantity $\mathcal J_b(i)$ in  \eqref{def: J * b i,j, metastability} is upper bounded by the relative width $\ceil{ r(i)/b  }$ due to the simple observation that $\mathcal G^{(k)|b}(\bm x) \supseteq \bar B_{kb}(\bm x)$,
and
(3) in the univariate setting,  the relative width $\ceil{ r(i)/b  }$ coincides with $\mathcal J_b(i)$; see, e.g., \cite{wang2022eliminating}.
\end{remark}

Theorem~\ref{corollary irreducible case} shows that under proper time-scaling, the sample path of $\bm X^{\eta|b}_t(\bm x)$ converges in distribution to a Markov jump process that only visits the local minima belonging to the \emph{widest attraction fields} of $f$.
Specifically, we use
\begin{align}
    \notationdef{notation-J-*-b-V}{\mathcal J^*_b} \triangleq \max_{i \in [K]}\mathcal{J}_b(i)
    \label{def: J * b}
\end{align}
to denote the largest width---characterized by $\mathcal J_b(i)$ in \eqref{def: J * b i,j, metastability}---of attraction fields over $f$.
As explained in Section~\ref{sec: review, first exit analyses} of the Appendix,
under Assumptions~\ref{assumption: geometry, metastability}--\ref{assumption: nondegeneracy of diffusion coefficients} we have that $\mathcal J_b(i) < \infty\ \forall i \in [K]$,
and hence $\mathcal J^*_b < \infty$.
Then, the set
\begin{align}
   \notationdef{notation-V-*-b}{V^*_b} \delequal \{m_i:\ \mathcal J_b(i) = \mathcal{J}^*_b\}
    \label{def: V * b, metastability}
\end{align}
is well-defined and 
contains all the local minima over $f$ that belongs to a \emph{widest} attraction field.

In order to formally present the law of the limiting process in Theorem~\ref{corollary irreducible case} (which only visits states in $V^*_b$),
we introduce a few more definitions.
Given $A \subseteq \R$, 
let
$
\notationdef{order-k-time-on-[0,t]}{A^{k \uparrow}} \delequal 
\{
(t_1,\cdots,t_k) \in A^k:\ t_1 < t_2 < \cdots < t_k
\}
$
be the set containing sequences of increasing real numbers on $A$ with length $k$.
For any $b$, $T \in (0,\infty)$ and $k \in \mathbb{N}$,
define the mapping 
$\notationdef{notation-h-k-t-bar-mapping-truncation-level-b-LDP}{h^{(k)|b}_{[0,T]}}: \mathbb{R}^d\times \mathbb{R}^{d \times k}  \times (0,{T}]^{k\uparrow} \to \mathbb{D}{[0,T]}$ as follows.
Given
$\bm x \in \mathbb{R}^d$,
$\textbf{W} = (\bm w_1,\cdots,\bm w_k) \in \mathbb{R}^{d \times k}$, 
and $\bm{t} = (t_1,\cdots,t_k)\in (0,T]^{k\uparrow}$, let $\xi = h^{(k)|b}_{[0,T]}(\bm x,\textbf{W},\bm{t})$ be the solution to
\begin{align}
    \xi_0 & = \bm x; \label{def: perturb ode mapping h k b, 1}
    \\
    \frac{d\xi_s}{d s} & = -\nabla f(\xi_s)\ \ \ \forall s \in [0,{T}],\ s \neq t_1,t_2,\cdots,t_k; \label{def: perturb ode mapping h k b, 2}
    \\
    \xi_s & = \xi_{s-} + \varphi_b\big( \bm \sigma(\xi_{s-})\bm w_j\big)\ \ \ \text{ if }s = t_j\text{ for some }j\in[k]. \label{def: perturb ode mapping h k b, 3}
\end{align}
That is,  $h^{(k)|b}_{[0,T]}(\bm x,\textbf{W},\bm{t})$ produces an ODE path perturbed by jumps $\bm w_1,\cdots,\bm w_k$ (with sizes modulated by $\bm \sigma(\cdot)$ and then truncated under threshold $b$) at times $t_1,\cdots,t_k$.
For $k = 0$,
we adopt the convention that $\xi = h^{(0)|b}_{[0,T]}(\bm x)$ is simply the gradient flow path
${d\xi_t}/{d t} = - \nabla f( \xi_t)$
under the initial condition $\xi_0 = \bm x$.
Next,
define 
$
\widecheck{g}^{(k)|b}: \R^d \times \R^{d\times k} \times (0,\infty)^{k\uparrow} \to \R^d
$
as the location of the gradient flow path with $k$ perturbation, right after the last perturbation; that is,
\begin{align}
    \notationdef{notation-check-g-k-b}{\widecheck{g}^{(k)|b}\Big(\bm x,\textbf W,(t_1,\ldots,t_k)\Big)}
    \delequal 
    h^{(k)|b}_{[0,t_k+1]}\Big(\bm x,\textbf W,(t_1,\ldots,t_k)\Big)(t_{k}).
    \label{def: mapping check g k b, endpoint of path after the last jump, first exit analysis}
\end{align}
Note that the definition remains the same if
we use mapping $h^{(k)|b}_{[0,T]}$
with any $T \in [t_k,\infty)$ instead of $h^{(k)|b}_{[0,t_k + 1]}$,
and we pick the $+1$ offset for simplicity.
Under $k = 0$,
we adopt the convention that $\widecheck{g}^{(0)|b}(\bm x) = \bm x$.
Note that an equivalent definition for $\mathcal G^{(k)|b}(\bm x)$ in \eqref{property: hierarchy of set G k b}
is that 
(for any $k \geq 1$, $b> 0$, and $\bm x \in \R^d$)
\begin{align}
    \notationdef{notation-set-G-k-b}{\mathcal G^{(k)|b}(\bm x)}
    & =
    \bigg\{
    \widecheck{g}^{(k - 1)|b}
    \Big( \varphi_b\big(\bm \sigma(\bm x)\bm w_1\big),
    (\bm w_2,\cdots, \bm w_k), \bm t 
    \Big):\ 
    \textbf W = (\bm w_1,\cdots,\bm w_k) \in \R^{d \times k},
    \bm t \in (0,\infty)^{k - 1 \uparrow}
    \bigg\}
     \label{def: set G k b}
\end{align}
Moreover, recall 
the measures $\nu_\alpha$ in \eqref{def: measure nu alpha} and $\mathbf S$ in Assumption~\ref{assumption gradient noise heavy-tailed},
and the polar transform $\Psi$ in \eqref{def: Phi, polar transform in Rm}.
Define Borel measures (for any $k \geq 1$, $\bm x \in \R^d$, and $b > 0$)
\begin{equation}\label{def: measure check C k b}
    \begin{aligned}
        \notationdef{notation-check-C-k-b}{\widecheck{ \mathbf C }^{(k)|b}(\ \cdot\ ; \bm x)}
    & \delequal 
    \int \mathbf{I}\bigg\{ \widecheck{g}^{(k-1)|b}\Big( \varphi_b\big(\bm\sigma(\bm x)\bm w_1\big),(\bm w_2,\cdots,\bm w_k),\bm t \Big) \in \ \cdot \  \bigg\}
     \big((\nu_\alpha \times \mathbf S)\circ \Psi\big)^k(d \textbf W) \times \mathcal{L}^{k-1\uparrow}_\infty(d\bm t),
    \end{aligned}
\end{equation}
where 
$\alpha > 1$ is the heavy-tail index in Assumption~\ref{assumption gradient noise heavy-tailed},
$\textbf W = (\bm w_1, \bm w_2, \cdots, \bm w_k) \in \R^{d \times k}$,
\notationdef{notation-measure-L-k-up-infty}{$\mathcal{L}^{k\uparrow}_\infty$}
is the Lebesgue measure restricted on $\{ (t_1,\cdots,t_k) \in (0,\infty)^k:\ 0 < t_1 < t_2 < \cdots < t_k \}$,
and
$\big((\nu_\alpha \times \mathbf S)\circ \Psi\big)^k$ is the $k$-fold of $(\nu_\alpha \times \mathbf S)\circ \Psi$,
which
is the composition of the product measure $\nu_\alpha \times \mathbf S$ with the polar transform $\Psi$:
\begin{align}
    \big((\nu_\alpha \times \mathbf S)\circ \Psi\big)(B) \delequal 
    (\nu_\alpha \times \mathbf S)\big( \Psi(B) \big),\qquad \forall \text{ Borel set }B \subseteq \R^d\setminus\{\bm 0\}.
    \label{def, nu alpha times S composition polar transform}
\end{align}
By the equivalence of \eqref{property: hierarchy of set G k b} and \eqref{def: set G k b},
one can that the measure ${\mathcal G^{(k)|b}(\bm x)}$ in \eqref{def: measure check C k b} is supported on the set $\mathcal G^{(k)|b}(\bm x)$.

We state a few regularity conditions for the technical analyses in Theorem~\ref{corollary irreducible case}.
First, Definition~\ref{def main paper typical transition graph} reveals the connectivity between different attraction fields over $f$.
In particular,
the intuition behind the condition $\mathcal G^{ ( \mathcal J_b(i) )|b  }(\bm m_i) \cap I_j \neq \emptyset$ below is that,
in terms of number of perturbations required in the gradient flow,
the ``hardness'' of going from local minimum $\bm m_i$ to a different attraction field $I_j$ is the same as that of simply escaping the current attraction field $I_i$ (see \eqref{def: J * b i,j, metastability}).

\begin{figure*}
\centering
\includegraphics[width=1\textwidth]{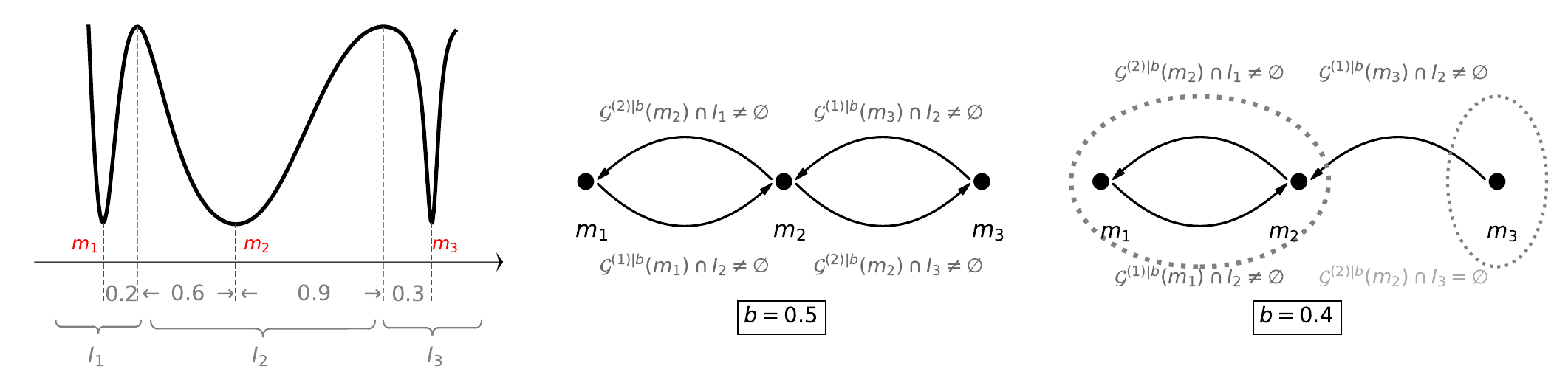}
\caption{\label{fig typical transition graph} 
Typical transition graphs under different choices of the truncation threshold $b$, illustrated with a univariate example.
\textbf{(Left)} A univariate function $f$ with three attraction fields, where the numbers indicate the distance between each local minimum $m_i$ and the neighboring attraction field to the left or right.
Note that in this univariate setting, we have $\mathcal J_b(i) = \ceil{ r(i) / b }$ where $r(i) = \inf\{ |m_i - y|:\ y \notin I_i  \}$,
and, for each $k \leq \mathcal J_b(i)$, we have $\mathcal G^{(k)|b}(m_i) = [ m_i - kb, m_i + kb ]$.
\textbf{(Middle)} The typical transition graph under $b = 0.5$. 
In particular, note that $\mathcal J^*_b(2) = \ceil{ 0.6/b } = 2$,
and the interval $\mathcal G^{( \mathcal J_b(2) )|b}(m_2) = [m_2 - 2 b, m_2 + 2b]$ intersects with both $I_1$ and $I_3$ (so the edges $m_2 \to m_1$ and $m_2 \to m_3$ are included in the typical transition graph).
The entire graph $\mathcal{G}_b$ is irreducible since all nodes communicate with each other. 
\textbf{(Right)} The typical transition graph under $b = 0.4$.
In this case, note that we still have $\mathcal J^*_b(2) = \ceil{ 0.6/b } = 2$,
but now $[m_2 - 2 b, m_2 + 2b]$ does not intersect with $I_3$.
As a result, 
the typical transition graph  does not contain the edge $m_2 \to m_3$, leading to two communication classes $G_1 = \{m_1,m_2\},\ G_2 = \{m_3\}$.
}
\end{figure*} 

\begin{definition}[Typical Transition Graph] \label{def main paper typical transition graph}
Given a function $f$ satisfying Assumption \ref{assumption: geometry, metastability}
and some $b > 0$, 
the \textit{typical transition graph} associated with threshold $b$
is a directed graph $\notationdef{notation-G-b-typical-transition-graph}(V,E_b)$ such that
\begin{itemize}
    \item $\notationdef{notation-V-of-G-b}{V} = \{\bm m_1,\cdots, \bm m_{ K }\}$;
    \item An edge $(\bm m_i\rightarrow \bm m_j)$ is in $\notationdef{notation-E-b-of-G-b}{E_b}$ iff 
    $\mathcal G^{ ( \mathcal J_b(i) )|b  }(\bm m_i) \cap I_j \neq \emptyset$.
\end{itemize}
\end{definition}
The typical transition graph $(V,E_b)$ can be decomposed into different communication classes that are mutually exclusive.
For $\bm m_i,\bm m_j$ with $i \neq j$, we say that  $\bm m_i$ and $\bm m_j$ communicate if and only if there exists a path $(\bm m_i \to \bm m_{k_1} \to \cdots \to \bm m_{k_n} \to \bm m_j)$ as well as a path $(\bm m_j \to \bm m_{k^\prime_1} \to \cdots \to \bm m_{k^\prime_{n^\prime}} \to \bm m_i)$ on the typical transition graph.
See Figure~\ref{fig typical transition graph} (Middle) and (Right) for the illustration of irreducible and reducible cases, respectively.
Specifically, we impose the following assumption and focus on the case where $\mathcal G_b$ is {irreducible},
i.e., all nodes communicate with each other in the graph $(V,E_b)$.
\begin{assumption}\label{assumption, value of b, irreducible, metastability}
The typical transition graph is irreducible.
\end{assumption}

We focus on  the irreducible case in the main paper for clarity of presentation,
and we note that in the reducible case, analogous results would hold  locally within each communication class of the typical transition graph:
when visiting a given communication class, the  truncated heavy-tailed SGDs
$\bm X^{\eta|b}_t(\bm x)$ closely resemble a Markov jump process that only visits the widest minima in that communication class:
see Section~\ref{sec: appendix, reducible case} of the Appendix for statements of analogous results in the reducible case;
see also
Theorem~H.2 and H.3 of \cite{wang2022eliminating} for results in a simplified univariate setting.

We also work with the following conditions on the choice of $b$. 
Similar regularity conditions are imposed in related works; see, e.g., \cite{doi:10.1142/S0219493715500197,wang2023large}.
Here, $\partial E = E^- \setminus E^\circ$ denotes the boundary set of $E$.

\begin{assumption} \label{assumption: regularity condition on b}
The following claims hold for each $i \in [K]$:
\begin{enumerate}[(i)]
    \item 
        $\widecheck{\mathbf C}^{ (\mathcal J_b(i))|b }\Big( \bigcup_{ j \in [K]} \partial I_j;\ \bm m_i  \Big) = 0$,
         and 
        $
        \widecheck{\mathbf C}^{ (\mathcal J_b(i))|b }\big( (I_i)^\complement;\ \bm m_i  \big) > 0;
        $

    \item 
        The set $(I_i)^\complement$ is bounded away from $\mathcal G^{ (\mathcal J_b(i) - 1)|b }(\bm m_i)$ (under the Euclidean norm).
       
\end{enumerate}
\end{assumption}

Recall 
the definition of largest width $\mathcal J^*_b$ in \eqref{def: J * b},
and
that $H(\cdot) = \P(\norm{\bm Z} > \cdot)$ and $\lambda(\eta) = \eta^{-1}H(\eta^{-1}) \in \RV_{\alpha - 1}(\eta)$.
Define the function
\begin{align}
    \notationdef{notation-scale-function-lambda-*-b}{\lambda^*_b(\eta)} \delequal \eta \cdot \big(\lambda(\eta)\big)^{ \mathcal J^*_b } \in \RV_{ \mathcal J^*_b\cdot (\alpha-1)  + 1 }(\eta)
    \qquad\text{as }\eta \downarrow 0,
    \label{def scale function lambda * b eta}
\end{align}
which will be used for the time scaling below.
We are now ready to state the main result.

\begin{theorem} \label{corollary irreducible case}
\linksinthm{corollary irreducible case}
Let Assumptions \ref{assumption: geometry, metastability}--\ref{assumption: regularity condition on b} hold.
Let $p \in [1,\infty)$, $i_0 \in [K]$, and $\bm x_0 \in I_{i_0}$.
As $\eta \downarrow 0$,
$$
\big\{\bm X^{\eta|b}_{\floor{ \boldsymbol{\cdot}/\lambda^*_b(\eta) }}(\bm x_0):\ t > 0\big\}\tofdd 
\{\bm Y^{*|b}_t:\ t > 0\}
\quad\text{and}\quad
\bm X^{\eta|b}_{\floor{ \boldsymbol{\cdot}/\lambda^*_b(\eta) }}(\bm x_0)
\Rightarrow
\bm Y^{*|b}_{\boldsymbol{\cdot}}
\text{ in $(\D[0,\infty),\dlp{[0,\infty)})$},
$$
where $\notationdef{notation-CTMC-Y-*-b}{\bm Y^{*|b}_t}$ is a continuous-time Markov chain with state space $V^*_b$ (see \eqref{def: V * b, metastability}).

\end{theorem}

We defer the detailed proof
to Section~\ref{subsec: proof, theorems, metastability} of the Appendix.
Here,
we discuss the the implication of Section~\ref{subsec: proof, theorems, metastability},
its connection  to existing works on metastability of heavy-tailed stochastic systems,
and state the law of the limiting process $\bm Y^{*|b}_t$.

Consider (untruncated) heavy-tailed SGDs defined by the recursion
$
\bm X^\eta_t(\bm x) = \bm X^\eta_{t-1}(\bm x) - \eta \nabla f\big(  \bm X^\eta_{t-1}(\bm x) \big) + 
\eta \bm \sigma\big(  \bm X^\eta_{t-1}(\bm x) \big)\bm Z_t,
$
given the initial value $\bm X^\eta_0(\bm x) = \bm x$ and step length $\eta > 0$.
Equivalently, $\bm X^{\eta}_t(\bm x)$ can be constructed by extending the definition of $\bm X^{\eta|b}_t(\bm x)$ in \eqref{def: X eta b j x, clipped SGD} and setting $b = \infty$ so the truncation operator $\varphi_\infty$ degenerates to the identity mapping.
The global dynamics of $\bm X^{\eta}_t(\bm x)$ can be revealed by sending the truncation threshold $b$ to $\infty$ in Theorem~\ref{theorem: metastability, unclipped}.
More specifically, let
\begin{align}
    \notationdef{notation-check-C}{\widecheck{\mathbf C}(\ \cdot\ ;\bm x)} \delequal \int \mathbf{I}\Big\{ \bm x + \bm \sigma(\bm x) \bm w \in\ \cdot\ \Big\}\nu_\alpha(d\bm w),
    \label{def: measure check C}
\end{align}
with $\nu_\alpha$ defined in \eqref{def: measure nu alpha}.
Also, we further impose a boundedness condition to facilitate the analyses in the untruncated case.
\begin{assumption}[Boundedness]  
\label{assumption: boundedness of drift and diffusion coefficients}
There exists some $\notationdef{notation-constant-C-boundedness-assumption}{C} \in (0,\infty)$ such that
\begin{align*}
    \norm{\nabla f(\bm x)} \vee \norm{\bm \sigma(\bm x)} \leq  C,\qquad \forall \bm x \in \mathbb{R}^d.
\end{align*}
\end{assumption}

Recall that $H(\cdot) = \P(\norm{\bm Z}> \cdot)$.
Corollary~\ref{theorem: metastability, unclipped} shows that, under the $1/H(\eta^{-1})$ time scaling, 
the sample path of $\bm X^\eta_t(\bm x)$ converges in distribution to that of a Markov jump process visiting all local minima over $f$.

\begin{corollary}\label{theorem: metastability, unclipped}
\linksinthm{theorem: metastability, unclipped}
Let Assumptions~\ref{assumption: geometry, metastability}--\ref{assumption: nondegeneracy of diffusion coefficients} and \ref{assumption: boundedness of drift and diffusion coefficients} hold.
Suppose that $\widecheck{\mathbf C}\big( \bigcup_{j \in [K]}\partial I_j; \bm m_i  \big) = 0$ holds for each $i \in [K]$.
Then, for each $p \in [1,\infty)$, $i_0 \in [K]$, and $\bm x_0 \in I_{i_0}$, we have
$$
\big\{\bm X^{\eta}_{\floor{t/H(\eta^{-1}) }}(\bm x_0):\ t > 0\big\}
\tofdd 
\{\bm Y^{*}_t:\ t > 0\}
\quad\text{and}\quad
\bm X^{\eta}_{\floor{ \boldsymbol{\cdot}/H(\eta^{-1}) }}(\bm x_0)
\Rightarrow
\bm Y^{*}_{\boldsymbol{\cdot}}
\text{ in $(\D[0,\infty),\dlp{[0,\infty)})$}
$$
as $\eta \downarrow 0$.
Here, $\notationdef{notation-CTMC-Y-*}{\bm Y^*_t}$ is a continuous-time Markov chain
with state space $V = \{\bm m_1,\ldots,\bm m_K\}$,
initial value $\bm Y^*_0 = \bm m_{i_0}$, and infinitesimal generator
\begin{align*}
    q(i,j) & =  \widecheck{\mathbf C}( I_j;\bm m_i)\qquad \forall \bm m_i,\ \bm m_j \in V\text{ with }\bm m_i \neq \bm m_j,
    \\ 
    q(i,i) & = - \sum_{ j \in [K]:\ j \neq i }q(i,j)
    = 
    -\widecheck{\mathbf C}\big( (I_i)^\complement; \bm m_i  \big)
    \qquad \forall \bm m_i \in V.
\end{align*}
    
\end{corollary}

We defer the detailed proof to Section~\ref{subsec: proof, theorems, metastability} of the Appendix,
and note that proof strategy is to send $b \to \infty$ in Theorem~\ref{theorem: metastability, unclipped} and carefully analyze the limits involved. 
In particular, under $b = \infty$, we have $\mathcal G^{(1)|\infty}(\bm x) = \R^d$ in \eqref{def: set G k b} and hence 
 $\mathcal J_\infty(i) = 1\ \forall i \in [K]$ in \eqref{def: J * b i,j, metastability}
as well as
$
\mathcal J^*_\infty = 1,
$
$
V^*_\infty = V
$
in \eqref{def: J * b} and \eqref{def: V * b, metastability}.
That is, without truncation, it is possible to reach any point in $\R^d$ with one jump when starting from a local minimum $\bm m_i$,
and each attraction field is considered equally wide---in terms of $\mathcal J_b(i)$ in \eqref{def: J * b i,j, metastability}---when compared to the infinite truncation threshold $b = \infty$.

Corollary~\ref{theorem: metastability, unclipped} is in the same spirit as prior work on metastability analyses under heavy-tailed noise. 
For instance, \cite{pavlyukevich2008metastable} studied univariate SDEs driven by regularly varying L\'evy processes,
and showed that transitions between different local minima are almost always caused by a single disproportionately large jump, while the rest of the dynamics follow a functional law of large numbers.
Moreover, the transition times scale polynomially in the noise magnitude, with the exponent determined solely by the power-law index of the (untruncated) L\'evy noise
(likewise, in Corollary~\ref{theorem: metastability, unclipped} the time scale revealing the global dynamics of $\bm X^\eta_t$ is dictated by $H(\eta^{-1}) = \P(\norm{\bm Z} > \eta^{-1}) \in \RV_{\alpha}(\eta)$, 
which only depends on the law of the heavy-tailed noise with $\alpha > 1$ being the corresponding heavy-tailed index in Assumption~\ref{assumption gradient noise heavy-tailed}).
See also \cite{doi:10.1142/S0219493715500197} for multivariate extensions to hyperbolic dynamical systems.
These results are manifestations of the \emph{principle of a single big jump}, a well-known phenomenon in extreme value theory that often governs rare events and metastable behaviors in heavy-tailed systems.

In contrast, this paper reveals a more refined mathematical structure in the global dynamics of heavy-tailed systems, where the governing factor is \emph{the number of jumps required to escape the attraction of a local minimum}.
Specializing to $\bm X^{\eta|b}_t$ where the stochastic dynamics are truncated above a fixed threshold $b$ by \eqref{defTruncationClippingOperator},
Theorem~\ref{corollary irreducible case} shows that the polynomial scaling of transition times now depends on both $\alpha$ (i.e., law of the noise) and ``width'' of the attraction fields (i.e., $\mathcal J_b(i)$ in \eqref{def: J * b i,j, metastability}).
The global dynamics of truncated heavy-tailed SGDs are in turn determined by the maximal width $\mathcal J^*_b$ in \eqref{def: J * b}.
In summary, our results provide a much more complete characterization of the metastability of heavy-tailed SGD: 
its global dynamics exhibit sophisticated phase transitions that depend in a discretized manner on the truncation threshold $b$ through key quantities $\mathcal J_b(i)$ and $\mathcal J^*_b$ that play the role of the width for the attraction fields.

To conclude Section~\ref{subsec: main results}, we specify the law of limiting process ${\bm Y^{*|b}_t}$ in Theorem~\ref{corollary irreducible case}.
Recall the measure $\widecheck{ \mathbf C}^{(k)|b}(\ \cdot\ ;\bm x)$ in \eqref{def: measure check C k b}.
Let
\begin{align}
    \notationdef{notation-q-b-i,j}{q_{b}(i,j)} \delequal \widecheck{\mathbf C}^{ ( \mathcal J_b(i) )|b }( I_j;\bm m_i),
    \qquad 
    \notationdef{notation-q-b-i}{q_b(i)} \delequal \widecheck{\mathbf C}^{ ( \mathcal J_b(i) )|b }\big( (I_i)^\complement; \bm m_i\big).
    \label{def: q b i, q b i j, generator for Y * b}
\end{align}
By condition (i) in Assumption~\ref{assumption: regularity condition on b}, we have
$
\sum_{j \in [K]:\ j \neq i}q_b(i,j) = q_b(i)
$ 
for each $i \in [K]$.
Furthermore, one can show that $q_b(i) \in (0,\infty)$ for each $i \in [K]$ (see the proof at the beginning of Section~\ref{subsec: proof, propositions, metastability} in Appendix).
This allows us to define a discrete-time Markov chain $(S_n)_{n \geq 0}$ over state space $V = \{\bm m_1, \bm m_2, \ldots, \bm m_K\}$, with any state $v \in V^*_b$ being an absorbing state, such that the one-step transition kernel is defined by
$\P(S_{n+1} = \bm m_j|S_n = \bm m_i) = q_b(i,j)/q_b(i)$
for any $\bm m_i \in V\setminus V^*_b$ and any $\bm m_j \in V$.
Next, define (for each $\bm m_i \in V$ and $\bm m_j \in V^*_b$)
\begin{align}
    \theta_b(\bm m_j|\bm m_i) \delequal \P(S_n = \bm m_j\text{ for some }n \geq 0\ |\ S_0 = \bm m_i)
    \label{def: absorption prob, theta b i j}
\end{align}
as the absorption probability at any $\bm m_j \in V^*_b$ when starting from $\bm m_i$.
By definition, for each $\bm m_i \in V^*_b$, we have
$
\theta_b(\bm m_i|\bm m_i) = 1.
$
Now, we are ready to define the initial distribution of
${\bm Y^{*|b}_t}$ by
\begin{align}
    \P(\bm Y^{*|b}_0 = \bm m_j) = \theta_b(\bm m_j|\bm m_{i_0}),\qquad \forall \bm m_j \in V^*_b,
    \label{def: Y * b, initial distribution}
\end{align}
where $\bm x_0$ is the initial value of SGD prescribed in Theorem~\ref{corollary irreducible case},
and
$i_0 \in [K]$ is the unique index with $\bm x_0 \in I_{i_0}$.
Next, the transition of this continuous-time Markov chain is governed by 
\begin{align}
    \P( \bm Y^{*|b}_{t+h} = \bm m_j\ |\ \bm Y^{*|b}_t = \bm m_i)
    =  h \cdot \sum_{j^\prime \in [K]:\ j^\prime \neq i} q_b(i,j^\prime)\theta_b(\bm m_j|\bm m_{j^\prime})
    + \bm{o}(h),\qquad\text{ as }h\downarrow 0
    \label{def, generator of Y * b}
\end{align}
for any $\bm m_i,\ \bm m_j \in V^*_b$ with $\bm m_i \neq \bm m_j$.
In other words, the infinitesimal generator of $\bm Y^{*|b}_t$ is 
\begin{align}
    Q^{*|b}(i,j) & = \sum_{j^\prime \in [K]:\ j^\prime \neq i} q_b(i,j^\prime)\theta_b(\bm m_j|\bm m_{j^\prime})
    \qquad 
    \forall 
    \bm m_i,\ \bm m_j \in V^*_b\text{ with }\bm m_i \neq \bm m_j,
    \label{def: generator of Y * b, 1}
    \\ 
    Q^{*|b}(i,i) & = - \sum_{\bm m_j \in V^*_b:\ j \neq i}Q^{*|b}(i,j)
    \qquad \forall \bm m_i \in V^*_b.
    \label{def: generator of Y * b, 2}
\end{align}


\subsection{Control of Global Dynamics of Heavy-Tailed SGD}
\label{subsec: implication, elimination of narrow minima}

In Section~\ref{subsec: implication, elimination of narrow minima}, we discuss the connection between Theorem~\ref{corollary irreducible case} and the control of training dynamics in deep learning.
Specifically, Theorem~\ref{corollary irreducible case} suggests that, under truncated heavy-tailed noise, SGD spends almost all time around the widest minima; this is rigorously characterized by Corollary~\ref{corollary, elimination of sharp minima, metastability} below.
Given its connection to the flat-minima folklore regarding the generalization in deep learning,
we then propose a novel algorithm that injects and then truncates heavy tails during the training of deep neural nets in order to find flat minima and improve generalization performance.

First of all, as the limiting process $\bm Y^{*|b}_t$ in Theorem \ref{corollary irreducible case} only visits the set $V^*_b$,
it is natural to expect that (under small step length $\eta$) the truncated heavy-tailed SGDs spend almost all time in the widest attraction fields of the loss landscape, and narrow minima are almost compeltely eliminated from their trajectories.
This conjecture can be easily made precise through a continuous mapping argument.
In particular, given any $\epsilon,\ T>0$, let
\begin{align*}
    g(\xi) = \frac{1}{T}\int_0^T\mathbf{I}\bigg\{ \xi_t \in \bigcup_{\bm m_j \in V^*_b }B_\epsilon(\bm m_j)  \bigg\}dt,
\end{align*}
and note that $g:\D[0,\infty)\to \R$
is continuous (w.r.t.\ $\dlp{[0,\infty)}$ in \eqref{def, Lp metric D 0 infty}) at any $\xi$ that
only takes values in $V^*_b$ and only makes finitely many jumps over $[0,T]$.
We then obtain Corollary~\ref{corollary, elimination of sharp minima, metastability} by combining the $L_p$ convergence stated in Theorem~\ref{corollary irreducible case} with the continuous mapping theorem.

\begin{corollary}
\label{corollary, elimination of sharp minima, metastability}
\linksinthm{corollary, elimination of sharp minima, metastability}
Let $\epsilon,\ T > 0$.
Under the conditions in Theorem~\ref{corollary irreducible case},
\begin{align*}
    \frac{1}{ \floor{T / \lambda^*_b(\eta)} }\sum_{t=1}^{ \floor{T / \lambda^*_b(\eta)} } \mathbf{I}\bigg\{
    \bm X^{\eta|b}_{t}(\bm x_0) \in  \bigcup_{\bm m_j \in V^*_b }B_{\epsilon}(\bm m_j)
    \bigg\}
     \stackrel{p}{\to}
    1,
    \qquad 
    \text{ as }\eta \downarrow 0,
\end{align*}
where $\stackrel{p}{\to}$ stands for convergence in probability.
\end{corollary}

Corollary~\ref{corollary, elimination of sharp minima, metastability} confirms that, as $\eta \downarrow 0$ and as long as we run truncated SGDs for long enough (i.e., the number of steps is comparable to the time scale $1/\lambda^*_b(\eta)$),
the proportion of time $\bm X^{\eta|b}_{t}(\bm x)$ spends around the widest minima (in terms of $\mathcal J_b(i)$) converges to 1. 
That is, truncated heavy-tailed noise can help SGD to almost always stay around the widest minima over the loss landscape; see, e.g., the numerical experiments in Figure~\ref{fig histograms} (left, b).

Such intriguing global dynamics are particularly relevant in deep learning.
Indeed, arriving at and staying around local minima with flatter geometry during  the training of deep neural networks often leads to better generalization performance in the test phase (see, e.g.,
\cite{Jiang*2020Fantastic,NEURIPS2022_69b55345,Li_2025_CVPR}). 
Corollary~\ref{corollary, elimination of sharp minima, metastability} then suggests that by running truncated heavy-tailed SGD for long enough (i.e., comparable to the time scale $1/\lambda^*_b(\eta)$) under a small step size $\eta$, we are almost certain to avoid the sharper, narrower local minima at the end of training. 

In order to translate our theoretical results into algorithmic insights,
we propose a novel training strategy that incorporates truncated heavy tails into the training of deep neural networks.  
While heavy-tailed noise has been empirically observed in deep learning, its presence and prevalence in specific tasks, as well as the validity of methods used to detect it, remain subtle topics of ongoing debate (see, e.g., \cite{panigrahi2019non,pmlr-v238-battash24a}).
Moreover, even when heavy-tailed noise is present, its exact degree of heavy-tailedness may not be ideal for efficient training.
For instance, the time scale at which the global dynamics described in Theorem~\ref{corollary irreducible case} and  Corollary~\ref{corollary, elimination of sharp minima, metastability} would manifest
is governed by the function $\lambda^*_b$ in \eqref{def scale function lambda * b eta} and depends on the heavy-tailed index $\alpha$.
As a result, under small step length $\eta$, the training time required to observe the preference toward the widest minima can become prohibitively long if $\alpha$ is too large (i.e., the tails in gradient noise are not sufficiently heavy). 
Therefore, it is also important to consider algorithmic framework that allows controlled injection of heavy-tailedness into the noise.

More precisely, given the current weights of a neural network $\theta$,
our approach is to update the model weights through the recursion of the form
\begin{align}
    \theta \leftarrow \theta - \varphi_b\big( \eta \cdot  g_{\text{heavy}}(\theta) \big),
    \label{def: truncated heavy tailed iteration, ablation study}
\end{align}
where $\varphi_b$ is the gradient clipping operator \eqref{defTruncationClippingOperator},
$\eta$ is the step length (i.e., learning rate),
and $g_{\text{heavy}}(\theta)$ is some stochastic gradient evaluated at $\theta$ perturbed by heavy-tailed dynamics.
Of course, the key step in implementing this training strategy is the construction of the heavy-tailed stochastic gradient $g_{\text{heavy}}(\theta)$ such that it is unbiased, i.e., 
$
\E g_{\text{heavy}}(\theta) = g_{\text{GD}}(\theta)
$
with $g_{\text{GD}}(\theta)$ being the (deterministic) true gradient evaluated using the entire training dataset,
and exhibits heavy-tailed laws. 
To this end, we estimate gradient noise via training data, and then conduct tail inflation for the noise term.
More precisely, let
\begin{align}
    g_{\text{heavy}}(\theta) \distequal g_{\text{SB}*}(\theta) + Z \big( g_{\text{SB}}(\theta) - g_{\text{GD}}(\theta) \big),
    \label{algorithm: def of heavy taild gradient}
\end{align}
where $g_{\text{SB}}(\theta)$ and $g_{\text{SB}*}(\theta)$ are the small-batch stochastic gradients,
and $Z \distequal cW$ with $W$ being a Pareto$(\alpha)$ random variable, and $c,\alpha$ being parameters of the algorithm (of course, a new independent copy of $Z$ will be drawn for each new gradient step). 
Here, note that the term $g_{\text{SB}}(\theta) - g_{\text{GD}}(\theta)$ represents gradient noise by definition,
and multiplying it with the heavy-tailed random variable $Z$ leads to inflation for the tail distribution of the noise.
We further note two details regarding the implementation of $g_{\text{heavy}}(\theta)$.
First,  due to the prohibitive cost of evaluating the true gradient $g_{\text{GD}}(\theta)$ in most tasks,
we instead use $g_{\text{LB}}(\theta)$, which is the stochastic gradient evaluated on a large batch of the training data (and is still unbiased for estimating $g_{\text{GD}}(\theta)$).
As a result, the heavy-tailed stochastic gradient is constructed by
\begin{align}
    g_{\text{heavy}}(\theta) \distequal g_{\text{SB}*}(\theta) + Z \big( g_{\text{SB}}(\theta) - g_{\text{LB}}(\theta) \big). 
    \label{update our heavy tailed method}
\end{align}
Second, depending on whether we use the same small batches for $g_{\text{SB}}(\theta)$ and $g_{\text{SB}*}(\theta)$,
we end up with two versions of the algorithm:
 in \emph{our method 1} (labeled as ``our 1'' in Table \ref{table ablation study}), we independently choose two small batches of the training data, 
 while in \emph{our method 2}  (labeled as ``our 2'' in Table \ref{table ablation study}), we use the same batch for $g_{\text{SB}}(\theta)$ and $g_{\text{SB}*}(\theta)$ in \eqref{update our heavy tailed method}.

In Section~\ref{sec: experiments combined}, we conduct simulation experiments and deep learning experiments to demonstrate the ability of our tail-inflation-and-truncation strategy \eqref{def: truncated heavy tailed iteration, ablation study}--\eqref{update our heavy tailed method} to find local minima with flat and wide geometry and improve the generalization performance of deep neural nets.
We conclude this section with a few remarks.
First, this tail-inflation-and-truncation strategy can be incorporated into first-order methods beyond vanilla SGD;
see Section~\ref{subsec: experiment 2, adam wrn} for its incorporation with the Adam optimizer  \cite{kingma2017adammethodstochasticoptimization}.
Second, 
several straightforward modifications can further reduce
the computational cost of this algorithm.
For example, 
when constructing $g_\text{heavy}$ in \eqref{algorithm: def of heavy taild gradient},
one can substitute $Z$ with $Z\mathbf{I}\{ Z > C \}$ for some prefixed threshold $C$ (i.e., we inject noise only if we know it is large),
and the updates \eqref{def: truncated heavy tailed iteration, ablation study} reduce to vanilla SGD steps when a small $Z$ is drawn.
See also Section~\ref{subsec: experiment 2, adam wrn} for demonstration of the effectiveness of the algorithm, even when the evaluation of $g_\text{LB}$---the arguably most costly step in the algorithm---is removed.

\section{Experiments}
\label{sec: experiments combined}

This section is devoted to numerical experiments.
Specifically, 
Section~\ref{sec:numerical-experiments} adopts the $\R^1$ simulation experiments in \cite{wang2022eliminating} to illustrate the global dynamics of (truncated) heavy-tailed SGDs established in Section~\ref{sec: metastability}.
Section~\ref{subsec: experiment 1, ablation study} follows the experimental design of the ablation study in \cite{wang2022eliminating} and verifies the effectiveness of the proposed tail-inflation-and-truncation strategy in improving the generalization performance of deep neural networks.
Then in Section~\ref{subsec: experiment 2, adam wrn}, we further show that our truncated heavy-tailed training strategy continues to perform well when combined with more recent network architectures, such as Wide Residual Networks \cite{zagoruyko2017wideresidualnetworks}, and popular optimization algorithms different from SGD, such as Adam \cite{kingma2017adammethodstochasticoptimization}.

\subsection{Simulation Experiments in $\R^1$}
\label{sec:numerical-experiments}

We 
adopt the design of simulation experiments in Section~3 of \cite{wang2022eliminating},
and consider a univariate function $f$ of the form
\begin{equation}\label{aeq:f}
\begin{aligned}
    & f(x)= (x+1.6)(x+1.3)^2(x-0.2)^2(x-0.7)^2(x-1.6)\big(0.05|1.65-x|\big)^{0.6} \\
    & \ \ \cdot \Big( 1 + \frac{1}{ 0.01 + 4(x-0.5)^2  } \Big)\Big( 1 + \frac{1}{0.1 + 4(x+1.5)^2} \Big)\Big( 1 - \frac{1}{4}\exp( -5(x + 0.8)^2  ) \Big).
\end{aligned}
\end{equation}
As illustrated in Figure~\ref{fig histograms} (left, e),
this function admits the local minima
 $m_1 = -1.51, m_2 = -0.66, m_3 = 0.49, m_4 = 1.32$,
 and attraction fields.
 $
 I_1 = (-\infty, -1.3),
 I_2 = (1,3, 0.2),
 I_3 = (0.2,0.7),
 I_4 = (0.7,+\infty).
 $
 Note that the attraction fields of the local minima $m_1$ and $m_3$ are narrower (in the sense that the distance between the local minimum and the region outside the attraction field is shorter),
 while the other two local minima $m_2$ and $m_4$ appear much wider in comparison.

 We compare the global dynamics of four different types of SGD algorithms (i.e., under the iteration \eqref{def: X eta b j x, clipped SGD}) when exploring the multimodal landscape of $f$.
 In the \textit{(a) heavy-tailed, no truncation} method,
 we set $b = \infty$,
 and let $Z_t$'s be iid copies of $Z = 0.1 U W$,
 where the $W$ is a Pareto Type II distribution (aka Lomax distribution) with tail index $\alpha = 1.2$,
 $\P(U = 1) = \P(U = -1) = 0.5$,
 and $U$ and $W$ are independent.
The same choice of heavy-tailed noise distribution is applied to the \textit{(b) heavy-tailed, with truncation} method,
but set the truncation threshold in \eqref{def: X eta b j x, clipped SGD} as $b = 0.5$.
Analogously in the \textit{(c) light-tailed, no truncation} and \textit{(d) light-tailed, with truncation}
methods,
we adopt the same choices of the truncation threshold from methods \textit{(a)} and \textit{(b)},
but set the noise distribution as $Z \sim \mathcal N(0,1)$.
In all methods tested,
we fix the step length as $\eta = 0.001$ and initial value as $x = 0.3$ (which belongs to the attraction field $I_3 = (0.2,0.7)$).
For each method,
we do 10 independent runs (i.e., generate 10 trajectories), each with 10, 000, 000 iterations.
Lastly, to prevent the cases of drifting to infinity due to extremely large noise,
each step the iterates are projected onto (i.e., confined with) the interval $[-1.6,1.6]$.

Figure~\ref{fig histograms} (left) present the histograms for the frequency of locations visited by SGDs,
using the 10 trajectories $\times$ 10, 000, 000 iterations in each of the four different methods.
Without truncation, we see from Figure~\ref{fig histograms} (left, a) that heavy-tailed SGD still frequently visit and spend some time around the narrower minima $m_1$ and $m_3$.
In comparison, Figure~\ref{fig histograms} (left, b) shows that the truncated heavy-tailed SGDs spend almost all time around the wider minima $m_2$ and $m_4$,
and the time spent around the narrower minima $m_1$ and $m_3$ is almost negligible in comparison.
This observation illustrate the claims in Corollary~\ref{corollary, elimination of sharp minima, metastability} that truncated heavy tails can guide SGDs to almost always stay around the wider region of the loss landscape. 
Note that this intriguing phenomenon is exclusive to the heavy-tailed setting:
as shown in Figure~\ref{fig histograms} (left, c\&d),
light-tailed SGD are easily trapped at sharp minima for extremely long time if initialized there, regardless of the truncation mechanism.
Furthermore, in Figure~\ref{fig histograms} (right)  we plot one sample path of SGD for each method tested.
Without truncation, heavy-tailed SGDs frequently visit and make transitions between all local minima (see Figure~\ref{fig histograms} (right, a)).
This is aligned with the global dynamics characterized in Corollary~\ref{theorem: metastability, unclipped} for (untruncated) heavy-tailed SGDs.
In contrast, 
Figure~\ref{fig histograms} (right, b) validates the global dynamics established in Theorem~\ref{corollary irreducible case} that, under small step length $\eta$, 
truncated heavy-tailed SGDs closely resemble a continuous-time Markov chain that only jumps between the widest minima of the loss landscape.

\subsection{Deep Learning Experiment 1: An Ablation Study}
\label{subsec: experiment 1, ablation study}

To demonstrate the effectiveness of the injection and truncation of heavy tails in the training of deep neural nets,
in this ablation study we adopt the experiment design in \cite{wang2022eliminating}, and benchmark our truncated heavy-tailed training strategy (using heavy-tailed gradients \eqref{update our heavy tailed method}) against the following algorithms:
\begin{itemize}
    \item 
    Large-batch SGD (LB): $\theta \leftarrow \theta - \eta \cdot g_{\text{LB}}(\theta)$,

    \item 
    Small-batch SGD (SB): $\theta \leftarrow \theta - \eta \cdot g_{\text{SB}}(\theta)$,

    \item 
    Small-batch SGD with clipping (SB + Clip): $\theta \leftarrow \theta - \varphi_b(\eta \cdot g_{\text{SB}}(\theta))$,

    \item 
    Small-batch SGD with heavy-tailed noise injection (SB + Noise):
    $
    \theta \leftarrow \theta - \eta \cdot g_\text{heavy}(\theta).
    $
\end{itemize}
Note that, unlike our truncated heavy-tailed training strategy, none of these algorithms incorporate both heavy-tailed noise injection and clipping.

Regarding the model architectures and deep learning tasks, we adopt the experiment setting and parameter choices in \cite{wang2022eliminating},
which also builds upon the design of experiments in \cite{pmlr-v97-zhu19e}:
\begin{enumerate}[(1)]
    \item 
    LeNet \cite{lecun1990handwritten} on corrupted FashionMNIST \cite{xiao2017/online},
    where
    we use a 1200-sample subset of the original FashionMNIST training set;
    the corruption is induced by picking
    200 samples from the training and randomly assigning a label (i.e., overwriting the correct labels);

    \item 
     VGG11 \cite{simonyan2014very} on SVHN \cite{netzer2011reading}, where we use a 25000-sample subset of the training dataset;

     \item 
     VGG11 on CIFAR10 \cite{krizhevsky2009learning},
     where we use the entire training set of CIFAR10.
\end{enumerate}
See Table~\ref{table appendix parameters ablation study} for the choice of parameters.
Here, we add a few comments on the design of experiments.
First, for each of the three tasks, 
the same choice of parameters in Table~\ref{table appendix parameters ablation study} is adopted across all optimization algorithms tested in the experiment;
the only exception is \textit{SB + Noise} due to its highly unstable behavior when driven by unclipped heavy-tailed dynamics, and we follow the suggested parameters in \cite{wang2022eliminating} to run extended training under fine-tuned step lengths for \textit{SB + Noise}.\footnote{
Specifically, we set  $\eta = 0.005$ in \textit{SB + Noise} across all tasks. 
For the corrupted FashionMNIST task, we train for 100,000 iterations and the heavy-tailed noise is removed for the final 50,000 iterations; for the other two tasks, we train for 150,000 iterations and heavy-tailed noise is removed for the last 70,000 iterations. 
Besides, when running \textit{SB + Noise} we always clip the model weights if its $L_\infty$ norm exceeds 1; 
otherwise, the models weights would quickly drift to infinity due to unclipped heavy-tailed noise.
}
Second, we stress again that $c$ and $\alpha$ is chosen for $Z \distequal c \cdot \text{Pareto}(\alpha)$ used for noise injection in the construction of $g_\text{heavy}$ in \eqref{update our heavy tailed method}.
Moreover, 
to ensure the convergence to local minima in \emph{our methods 1 and 2}, for last final 5,000 iterations we remove the injection of heavy-tailed noise  and run \emph{LB} instead.
Lastly, note that the choices of $b$ in Table~\ref{table appendix parameters ablation study} are different from the values reported in \cite{wang2022eliminating}, where the iterations $\theta \leftarrow \theta -\eta \cdot  \varphi_b\big(  g_{\text{heavy}}(\theta) \big)$ are considered when calculating the clipping threshold instead of using \eqref{def: truncated heavy tailed iteration, ablation study}; this corresponds to an enlargement of the values by the ratio $1/\eta$.

\begin{table}
  \caption{Parameters for the ablation study}
  \label{table appendix parameters ablation study}
  \centering
  \begin{tabular}{llll}
    \toprule
    Parameters    &  FashionMNIST, LeNet &  SVHN, VGG11  & CIFAR10, VGG11   \\
    \midrule
    step length $\eta$ & 0.05 & 0.05 & 0.05 \\
    batch size for $g_{\text{SB}}$ & 100 & 100 & 100 \\
    batch size for $g_{\text{LB}}$ & 1,200 & 1,000 & 1,000 \\
    training iterations & 10,000 & 30,000 & 30,000 \\
    clipping threshold $b$ & 0.25 & 1 & 1 \\
    $c$ & 0.5 & 0.5 & 0.5 \\
    $\alpha$ & 1.4 & 1.4 & 1.4 \\
    \bottomrule
  \end{tabular}
\end{table}

In this experiment, we are interested in not only the generalization performance of the obtained solution (i.e., the test accuracy of the trained model) but also its sharpness,
 measured by the \emph{expected sharpness} metric that has also been adopted in \cite{pmlr-v97-zhu19e, NIPS2017_10ce03a1}.
 Specifically, we report
 \begin{align}
     \E_{\nu \sim \mathcal{N}(\textbf{0},\delta^2 \textbf{I})}|f(\theta^* + \nu) - f(\theta^*)|,
     \label{def: expected sharpness metric}
 \end{align}
 where $f$ is the loss function induced by the entire training set (cross-entropy loss in this case),
 $\theta^*$ is the model weights obtained when training is done,
 and $\mathcal{N}(\textbf{0},\delta^2 \textbf{I})$ denotes the law of a random vectors with each coordinate being an iid copy of the univariate Gaussian $\mathcal N(0,\delta^2)$.
 A smaller value of \eqref{def: expected sharpness metric} indicates a more ``flat'' geometry locally around the solution obtained.
In Section~\ref{subsec: experiment 1, ablation study}, we set $\delta = 0.01$ and evaluate \eqref{def: expected sharpness metric} by averaging over 100 samples.
Also, to take into account the potential numerical instability in the estimation of \eqref{def: expected sharpness metric}, we set $f(\theta)$ to $5$ if the training loss  exceeds $5$ under the perturbation $v$.
We note that, in our experiment, this truncation mechanism on the training loss $f(\cdot)$ was in effect only for \textit{SB + Noise}.

\begin{table}
  \caption{Test accuracy and expected sharpness (mean $\pm$ range of 95\% CI, estimated over 5 runs; expected sharpness estimated under $\delta = 0.01$) of different methods across different tasks. }
  \label{table ablation study}
  \centering
    \scriptsize
  \begin{tabular}{lllllll}
    \hline
    Test accuracy (\%)    & LB     & SB   & SB + Clip & SB + Noise & Our 1 & Our 2 \\
    \hline
    
    FashionMNIST, LeNet & 68.77 {\tiny$\pm$ 0.97} & 68.91 {\tiny$\pm$ 0.59} & 68.38 {\tiny$\pm$ 1.38} & 52.52 {\tiny$\pm$ 29.7} & 69.60 {\tiny$\pm$ 0.76} & \textbf{70.03} {\tiny$\pm$ 0.55}  \\ 
    SVHN, VGG11         & 82.91 {\tiny$\pm$ 0.58} & 85.89 {\tiny$\pm$ 0.35} & 85.97 {\tiny$\pm$ 0.23} & 30.51 {\tiny$\pm$ 32.08} & \textbf{88.26} {\tiny$\pm$ 0.48} & 88.18 {\tiny$\pm$ 0.78}
 \\
    CIFAR10, VGG11      & 69.78 {\tiny$\pm$ 1.46} &  74.53 {\tiny$\pm$ 0.92} & 74.15 {\tiny$\pm$ 0.98} & 40.09 {\tiny$\pm$ 34.26} & \textbf{76.23} {\tiny$\pm$ 0.85} & 75.49 {\tiny$\pm$ 1.15}
\\

    \hline
    
    Expected Sharpness      & LB     & SB   & SB + Clip & SB + Noise & Our 1 & Our 2 \\
    \hline
    FashionMNIST, LeNet & 0.0280 {\tiny$\pm$ 0.0040} & 0.0082 {\tiny$\pm$ 0.0011} & 0.0090 {\tiny$\pm$ 0.0009} & 0.0842 {\tiny$\pm$ 0.1240} & 0.0028 {\tiny$\pm$ 0.0002} & \textbf{0.0016} {\tiny$\pm$ 0.0001} \\ 
    SVHN, VGG11         & 0.6140 {\tiny$\pm$ 0.1019} & 0.0412 {\tiny$\pm$ 0.0058} & 0.0372 {\tiny$\pm$ 0.0118} & 2.4508 {\tiny$\pm$ 2.9470} & \textbf{0.0023}{\tiny $\pm$ 0.0008} & 0.0030 {\tiny$\pm$ 0.0020} \\
    CIFAR10, VGG11      & 1.9476 {\tiny$\pm$ 0.1396} & 0.0388 {\tiny$\pm$ 0.0175} & 0.0548 {\tiny$\pm$ 0.0459} & 3.7084 {\tiny$\pm$ 5.0659} & \textbf{0.0231} {\tiny$\pm$ 0.0134} & 0.0602 {\tiny$\pm$ 0.0326}
 \\

    \hline
    
    
    
  \end{tabular}
\end{table}

The results are summarized in Table~\ref{table ablation study},
where we report the mean and a 95\% confidence interval (two-sided, under $t$-distribution) 
estimated by running 5 independent runs for each task.
Specifically, Table~\ref{table ablation study} shows that in all 3 tasks, \textit{our method 1} or \textit{our method 2} are consistently the best in terms of the test accuracy obtained or expected sharpness.
In comparison, when the heavy-tailed noise is removed,
the algorithm \textit{SB + Clip} yields worse test accuracies, and the performance is similar to that of \textit{SB}.
This is to be expected as the truncation is of little effect without observing large shift in one iteration.
On the other hand, when heavy-tailed noise is present but the gradient clipping mechanism is removed,
the performance of $\textit{SB + Noise}$ significantly deteriorates, despite the effort in fine-tuning this method as mentioned above (see also the details in \cite{wang2022eliminating}).
This observation also corroborates the existing empirical and theoretical analyses regarding the deterioration of convergence rates or even the arise of numerical instability due to the presence of heavy-tailed noise; see, e.g.,  \cite{Zhang2020Why,NEURIPS2020_abd1c782,DBLP:journals/corr/abs-2406-04443,lee2025efficient}.
In summary, the experiment results are well aligned with our theoretical analyses in Section~\ref{subsec: main results},
confirming that both heavy-tailed dynamics and the truncation mechanism (i.e., gradient clipping) are required for finding local minima with more flat geometry and better generalization performance.

\subsection{Deep Learning Experiment 2: Adam + Wide Residual Networks}
\label{subsec: experiment 2, adam wrn}

In Section~\ref{subsec: experiment 2, adam wrn}, we consider more sophisticated settings and demonstrate that
our truncated heavy-tailed training strategy remains effective and can still help improve the generalization performance.

Regarding the choice of optimizers, we incorporate truncated heavy tails into {Adam} \cite{kingma2017adammethodstochasticoptimization}, 
the popularity of which is related to its faster convergence rate when compared to SGD (see, e.g., \cite{pan2022toward,mazumder2024theoreticalempiricalstudyconvergence}).
In particular, {Adam} adaptively adjusts the learning rates based on moments estimation for the (small-batch) stochastic gradients,
resulting in smaller step lengths along coordinates with frequent large gradients. 
At the first glance, such an adaptive mechanism could play a role similar to gradient clipping when large gradients are presented.
Therefore,
the first goal of the experiments in Section~\ref{subsec: experiment 2, adam wrn} is to examine 
\emph{whether our truncated heavy-tailed training strategy can be efficiently combined with Adam and yield further improvements on the test performance}.
Specifically, 
we consider the following implementation (labeled as ``\textit{Adam + Truncated HT}'' in Table~\ref{table results, adam wrn}):
after each iteration of updating model weights $\theta$ using Adam with learning rate $\eta_{\text{Adam}}$,
we run another truncated heavy-tailed step of the form 
\begin{align}
    \theta \leftarrow \theta - \varphi_b\big( \eta_{\text{heavy}} \cdot  g_{\text{heavy}}(\theta) \big),
    \label{def: truncated heavy tailed iteration}
\end{align}
to further update $\theta$, 
where $g_{\text{heavy}}(\theta)$ is constructed by
\begin{align}
    g_{\text{heavy}}(\theta) \distequal g_{\text{SB}*}(\theta) + Z \cdot g_{\text{SB}}(\theta).
    \label{def: heavy tailed stochastic gradient, adam wrn experiment}
\end{align}
Here, 
$ g_{\text{SB}*}(\theta)$ and $g_{\text{SB}}(\theta)$ are estimated on two independently chosen small batches (which are also independent form the small batch used in the previous Adam step),
and
$Z \distequal c \cdot \text{Pareto}(\alpha)$.
Compared to \eqref{update our heavy tailed method}, note that in this experiment we remove the estimation of the true gradient on a large batch to further reduce the implementation cost of the truncated heavy-tailed updates. 
Also, note that another interpretation of the proposed optimization algorithm is that we alternative between the Adam step and the truncated heavy-tailed step \eqref{def: truncated heavy tailed iteration},
with the heavy-tailed stochastic gradient defined as in \eqref{def: heavy tailed stochastic gradient, adam wrn experiment}.
Regarding the choice of parameters, we adopt the default choice in PyTorch \cite{NEURIPS2019_bdbca288} for hyperparameters for moment estimation in Adam;
for the truncated heavy-tailed steps,
we set $c = 0.5$ and $\alpha = 1.4$ for $Z \distequal c \cdot \text{Pareto}(\alpha)$ when constructing the heavy-tailed stochastic gradients in \eqref{def: heavy tailed stochastic gradient, adam wrn experiment},
and set $b = 1$, $\eta_{\text{heavy}} = 0.1$ in \eqref{def: truncated heavy tailed iteration}.


\begin{table}[h]
  \caption{Parameters for the Adam + WRN experiment}
  \label{table appendix parameters, adam wrn}
  \small
  \centering
  \begin{tabular}{lllll}
    \toprule
    Dataset & Model   & Initial value of $\eta_{\text{Adam}}$ &  Number of epochs & Schedule for the decay of $\eta_{\text{Adam}}$ \\
    \midrule
    \multirow{3}{4em}{CIFAR10} & WRN16-8 & $2.5 \times 10^{-4}$ & 200 & $[60,120,160]$
    \\
                                &WRN28-10 &  $2.5 \times 10^{-4}$ & 300 &   $[90,180,240]$
    \\
                                &WRN40-4 & $2.5 \times 10^{-4}$ & 300 &   $[90,180,240]$
    \\
    \midrule
    \multirow{3}{4em}{CIFAR100}&WRN16-8 & $2.5 \times 10^{-4}$ & 200 & $[60,120,160]$
    \\
                                &WRN28-10 &  $2.5 \times 10^{-4}$ & 300 &   $[90,180,240]$
    \\
                                &WRN40-4 &  $2.5 \times 10^{-4}$ & 300 &   $[90,180,240]$
    \\
    \bottomrule
  \end{tabular}
\end{table}

\begin{table}[h]
  \caption{Test accuracy (\%) and expected sharpness in the Adam + WRN experiment: mean $\pm$ range of 95\% CI, estimated over 5 runs; expected sharpness estimated under $\delta = 2 \times 10^{-3}$.}
  \label{table results, adam wrn}
  \centering
  \begin{tabular}{lll}
    \toprule
    Test Accuracy  (\%)   &  Adam & Adam + Truncated HT  \\
    \midrule
    {CIFAR10}, WRN16-8 & 93.37 {\footnotesize $\pm$ 0.24} & \textbf{94.47} {\footnotesize$\pm$ 0.17} \\
    {CIFAR10}, WRN28-10 &  93.59 {\footnotesize$\pm$ 0.12} & \textbf{94.84} {\footnotesize$\pm$ 0.24} \\
    {CIFAR10}, WRN40-4 &  93.51 {\footnotesize$\pm$ 0.13} & \textbf{95.09} {\footnotesize$\pm$ 0.06} \\
    {CIFAR100}, WRN16-8 & 74.73 {\footnotesize$\pm$ 0.40} & \textbf{76.78} {\footnotesize$\pm$ 0.33} \\
    {CIFAR100}, WRN28-10 & 75.39 {\footnotesize$\pm$ 0.37} & \textbf{78.16} {\footnotesize$\pm$ 0.31} \\
    {CIFAR100}, WRN40-4 & 74.49 {\footnotesize$\pm$ 0.23} & \textbf{77.34} {\footnotesize$\pm$ 0.08} \\
    \midrule
    Expected Sharpness  &  Adam & Adam + Truncated HT  \\
    \midrule 
    {CIFAR10}, WRN16-8 & $5.7 \times 10^{-5}$ {\footnotesize$ \pm 2.6 \times 10^{-6} $} & ${\bf 1.1 \times 10^{-5}} ${\footnotesize$\pm 1.4 \times 10^{-6} $}
 \\
    {CIFAR10}, WRN28-10 &  $2.0 \times 10^{-5} ${\footnotesize$\pm 2.1 \times 10^{-6} $} & ${\bf 1.9\times 10^{-5}} ${\footnotesize$\pm 7.8 \times 10^{-6} $} \\
    {CIFAR10}, WRN40-4 & $1.2 \times 10^{-5} ${\footnotesize$\pm 2.1 \times 10^{-6} $} & ${\bf 3.7 \times 10^{-6}} ${\footnotesize$\pm 2.2 \times 10^{-6} $} \\
    {CIFAR100}, WRN16-8 & $9.8 \times 10^{-4} ${\footnotesize$\pm 1.0 \times 10^{-4} $} & ${\bf 2.3 \times 10^{-5}} ${\footnotesize$\pm 1.7 \times 10^{-6} $} \\
    {CIFAR100}, WRN28-10 & ${\bf 3.2 \times 10^{-4}} ${\footnotesize$\pm 6.0 \times 10^{-5} $} & ${3.6 \times 10^{-4}} ${\footnotesize$\pm 9.7 \times 10^{-5} $} \\
    {CIFAR100}, WRN40-4 & $6.8 \times 10^{-5} ${\footnotesize$\pm 5.8 \times 10^{-6} $} & ${\bf 1.6 \times 10^{-5}} ${\footnotesize$\pm 1.3 \times 10^{-6} $}  \\
    \bottomrule
  \end{tabular}
\end{table}

As for the model architectures, we consider Wide Residual Networks (WRNs)  \cite{zagoruyko2017wideresidualnetworks},
which could enjoy a faster training duration and improved generalization performance when compared to deeper models with narrower layers (see, e.g., \cite{NEURIPS2021_bef4d169}).
We test the models and algorithms on CIFAR10/100.
Specifically in this experiment, 
we adopt the choice of batch size $= 128$ (i.e., for evaluating the small-batch gradients in Adam and  the truncated heavy-tailed step \eqref{def: truncated heavy tailed iteration} during training) in \cite{zagoruyko2017wideresidualnetworks},
and consider the following three configurations of WRNs:
depth $= 16$, widening factor $ = 8$;
depth $= 28$, widening factor $ = 10$;
or
depth $= 40$, widening factor $ = 4$.
We also incorporate data augmentation (random crop of images padded by 4 pixels, and horizontal flips)
and learning rate scheduling (i.e., multiplying the learning rate by $0.2$ after certain amount of epochs).
These standard techniques were also applied for the training of WRNs in \cite{zagoruyko2017wideresidualnetworks},
and are known to further improve the generalization performance of the trained models.
Regarding the initial learning rate and the number of epochs during training (together with the scheduling for the decay of learning rates), we fine-tune over $\eta \in[10^{-3},\ 2.5 \times 10^{-4},\ 10^{-4}]$,
and \#Epoch $\in [200, 300]$;
in the case of \#Epoch$= 200$, we multiply the learning rate by 0.2 at the end of epochs $[60,120,160]$ (which is also the choice in  \cite{zagoruyko2017wideresidualnetworks}), and (similar to Section~\ref{subsec: experiment 1, ablation study}) remove the truncated heavy-tailed steps after the first 120 epochs to ensure the convergence of our \textit{Adam + Truncated HT} algorithm;
in the case of \#Epoch$= 300$, we scale the schedule proportionally to decay the learning rate at the end of epochs $[90, 180, 240]$ and remove the truncated heavy-tailed step after running 180 epochs.
In particular, the fine-tuning is done only for the {vanilla Adam} 
(with the best choice of parameters that attains the hightest test accuracy summarized in Table~\ref{table appendix parameters, adam wrn}),
while \textit{Adam + Truncated HT} simply adopts the same set of parameters.
In other words,
the second goal of this experiment is to examine
\emph{whether our truncated heavy-tailed training strategy remains effective when Adam has already been fine-tuned and several training techniques have already been implemented to improve the generalization performance}.
We note that the learning rate scheduling is carried out only for $\eta_{\text{Adam}}$, whereas the learning rate $\eta_{\text{heavy}}$ remains constant throughout each experiment for the truncated heavy-tailed steps \eqref{def: truncated heavy tailed iteration}.

The results are summarized in Table~\ref{table results, adam wrn},
where we set $\delta = 2 \times 10^{-3}$ in \eqref{def: expected sharpness metric} for the estimation of expected sharpness in WRNs.
We see that even though the vanilla Adam has been fine-tuned as described above for the training of WRNs,
our \textit{Adam + Truncated HT} algorithm consistently achieves better test performance.
Besides, in almost all cases we see that the \textit{Adam + Truncated HT} algorithm finds a solution with a smaller expected sharpness.
These experiments confirm the effectiveness of our theoretical analyses in Section~\ref{sec: metastability} beyond the settings studied in Section~\ref{subsec: experiment 1, ablation study},
and demonstrate that the proposed truncated heavy-tailed strategy finds solutions with flatter geometry and  improves the generalization performance of deep neural networks, even when combined with more popular and recent optimizers, modern architectures, and additional training techniques designed to enhance generalization.

\bibliographystyle{abbrv} 
\bibliography{bib_appendix} 

\newpage
\appendix

The appendices are structured as follows.
Section~\ref{sec: appendix, reducible case} states the results for metastability analyses in the reducible case.
Section~\ref{sec: review, first exit analyses} reviews the first exit analyses for heavy-tailed dynamical systems in \cite{wang2022eliminating} and adapts them to our setting.
Section~\ref{subsec: proof, abstract framework, metastability} develops a theoretical framework for establishing the sample path convergence of jump processes.
Applying this framework, in Sections~\ref{subsec: proof, theorems, metastability}--\ref{sec: appendix, CTMC Y * b} we provide the proof of Theorems~\ref{corollary irreducible case} and \ref{theorem: metastability, unclipped}.

\section{Metastability in Reducible Cases}
\label{sec: appendix, reducible case}

In this section, we present results analogous to Theorem~\ref{corollary irreducible case} for the case where Assumption~\ref{assumption, value of b, irreducible, metastability} (i.e., irreducibility of the typical transition graph; see \eqref{def main paper typical transition graph} for the definition) fails.
In such cases, the typical transition graph $(V,E_b)$ possesses multiple communication classes $G_1,\ldots,G_n$ (with $n \geq 2$),
and in the remainder of this section we focus on the metastability of truncated heavy-tailed SGD on one of these communication classes, denoted by $G$.

To formally present the results, we introduce a few definitions.
Analogous to \eqref{def: J * b}, let
\begin{align}
    \mathcal J^G_b \delequal \max_{  \bm m_i \in G }\mathcal J_b(i)
\end{align}
to denote the largest width of local minima in $G$.
Also, similar to \eqref{def scale function lambda * b eta}, we define
\begin{align}
    \lambda^G_b(\eta) \delequal \eta \cdot \big(\lambda(\eta)\big)^{ \mathcal J^G_b }.
\end{align}
Depending on the connectivity of $G$ with the other communication classes over the typical transition graph, $G$ could be either \emph{absorbing} or \emph{transient}.
We first consider the absorbing case.

\begin{theorem}[Metastability of $\bm X^{\eta|b}_t$: Absorbing Case]
Let Assumptions~\ref{assumption: geometry, metastability}--\ref{assumption: nondegeneracy of diffusion coefficients} and \ref{assumption: regularity condition on b} hold.
Let $G$ be an \emph{absorbing} communication class over the typical transition graph.
Given some $\bm m_{i_0} \in G$, let $\bm x_0 \in I_{i_0}$.
Let $p \in [1,\infty)$.
As $\eta \downarrow 0$,
$$
\big\{\bm X^{\eta|b}_{\floor{ \boldsymbol{\cdot}/\lambda^G_b(\eta) }}(\bm x_0):\ t > 0\big\}\tofdd 
\{\bm Y^{G|b}_t:\ t > 0\}
\quad\text{and}\quad
\bm X^{\eta|b}_{\floor{ \boldsymbol{\cdot}/\lambda^G_b(\eta) }}(\bm x_0)
\Rightarrow
\bm Y^{G|b}_{\boldsymbol{\cdot}}
\text{ in $(\D[0,\infty),\dlp{[0,\infty)})$},
$$
where $\bm Y^{G|b}_t$ is a recurrent continuous-time Markov chain with state space
$\{ \bm m_i \in G:\ \mathcal J_b(i) = \mathcal J^G_b  \}$.
\end{theorem}

In case that the communication class $G$ is transient, the process $\bm X^{\eta|b}_t$ will exit from $G$ (more precisely, all attraction fields with their local minima in $G$) at some point under the canonical time scale $1/\lambda^G_b(\eta)$,
and a few more definitions are needed.
First, let
\begin{align}
    \tau^{\dagger;\eta|b}_G(\bm x) \delequal \min\bigg\{
        t \geq 0:\ 
        \bm X^{\eta|b}_t(\bm x) \notin \bigcup_{ \bm m_i \in G }I_i
    \bigg\}
\end{align}
be the time $\bm X^{\eta|b}_t(\bm x)$ exits from the attraction fields over $G$.
By introducing a cemetery state $\dagger$,
we define a version of $\bm X^{\eta|b}_t(\bm x)$ killed at $ \tau^{\dagger;\eta|b}_G(\bm x)$:
\begin{align}
    \bm X^{\dagger;\eta|b}_t(\bm x)
    \delequal
    \begin{cases}
        \bm X^{\eta|b}_t(\bm x)\qquad& \text{ if } t < \tau^{\dagger;\eta|b}_G(\bm x)
        \\
        \dagger &\text{ otherwise}
    \end{cases}.
\end{align}

The next result reveals the metastable behavior of $\bm X^{\eta|b}_t(\bm x)$ before exiting $G$, where the scaling limit is a continuous-time Markov chain over $G$ that will be killed (denoted by entering an absorbing state $\dagger$) at a random time.

\begin{theorem}[Metastability of $\bm X^{\eta|b}_t$: Transient Case]
Let Assumptions~\ref{assumption: geometry, metastability}--\ref{assumption: nondegeneracy of diffusion coefficients} and \ref{assumption: regularity condition on b} hold.
Let $G$ be a \emph{transient} communication class over the typical transition graph.
Given some $\bm m_{i_0} \in G$, let $\bm x_0 \in I_{i_0}$.
Let $p \in [1,\infty)$.
As $\eta \downarrow 0$,
$$
\big\{\bm X^{\dagger;\eta|b}_{\floor{ \boldsymbol{\cdot}/\lambda^G_b(\eta) }}(\bm x_0):\ t > 0\big\}\tofdd 
\{\bm Y^{\dagger;G|b}_t:\ t > 0\}
\quad\text{and}\quad
\bm X^{\dagger;\eta|b}_{\floor{ \boldsymbol{\cdot}/\lambda^G_b(\eta) }}(\bm x_0)
\Rightarrow
\bm Y^{\dagger;G|b}_{\boldsymbol{\cdot}}
\text{ in $(\D[0,\infty),\dlp{[0,\infty)})$},
$$
where $\bm Y^{\dagger;G|b}_t$ is a continuous-time Markov chain with state space
$\{ \bm m_i \in G:\ \mathcal J_b(i) = \mathcal J^G_b  \} \cup \{\dagger\}$, with $\dagger$ being its only absorbing state and other states being transient.
\end{theorem}

The proofs of results in this section will be almost identical to that of Theorem~\ref{corollary irreducible case}, so omit the details to avoid repetition.

\section{First Exit Analyses and Related Lemmas}
\label{sec: review, first exit analyses}

This section reviews the first exit analyses for heavy-tailed dynamical systems in \cite{wang2022eliminating} and adapts them to the setting in Section~\ref{sec: metastability}.
These results lay the foundation for our subsequent proof of Theorem~\ref{corollary irreducible case}.

The first exit analyses in \cite{wang2022eliminating} are stated for a compact region within a certain attraction field of the multimodal potential.
Specifically, we w.l.o.g.\ assume in this section that one of the local minimum is located at the origin and work with the following assumption, where the gradient flow path $\bm y_t(\cdot)$ is defined in \eqref{def ODE path y t}.

\begin{assumption}
\label{assumption: shape of f, first exit analysis}
$\nabla f(\bm 0) = \bm 0$.
The open set $I \subset \R^d$ contains the origin and is bounded, i.e.,
$\sup_{\bm x \in I}\norm{\bm x} < \infty$
and
$\bm 0 \in I$.
Besides, for each $\bm x \in I \setminus \{\bm 0\}$,
\begin{align*}
    \bm y_t(\bm x) \in I\ \ \forall t \geq 0;\qquad\text{ and }
    \qquad
    \lim_{t \to \infty}\bm y_t(\bm x) = \bm 0.
\end{align*}
Moreover, there exists $\epsilon > 0$ such that
\begin{align}
    \nabla f(\bm x)^\top\bm x > 0,\qquad \forall \bm x \in \bar B_\epsilon(\bm 0) \setminus \{\bm 0\}.
    \label{property in assumption, contraction around minima, assumption: shape of f, first exit analysis}
\end{align}
\end{assumption}

Define the first exit time form $I$ by
\begin{align*}
    {\tau^{\eta|b}(\bm x)} \delequal \min\big\{j \geq 0:\ \bm X^{\eta|b}_j(\bm x) \notin I \big\}.
\end{align*}
Adapting Theorem 2.8 of \cite{wang2023large} to our setting, we obtain the following result.
We simplify the notations by writing $\widecheck{\mathbf C}^{(k)|b}(\ \cdot\ ) = \widecheck{\mathbf C}^{(k)|b}(\ \cdot\ ;\bm 0)$ for the measure $\widecheck{\mathbf C}^{(k)|b}$ defined in \eqref{def: measure check C k b}.

\begin{theorem}[Theorem 2.8 of \cite{wang2023large}]
\label{theorem: first exit analyses, adapted}
    Let Assumptions~\ref{assumption gradient noise heavy-tailed}, \ref{assumption: lipschitz continuity of drift and diffusion coefficients}, \ref{assumption: nondegeneracy of diffusion coefficients},
    and \ref{assumption: shape of f, first exit analysis} hold. 
    Let $\mathcal J^I_b \delequal
        \min\big\{ k \geq 1:\ \mathcal G^{(k)|b}(\bm 0) \cap I^\complement \neq \emptyset \big\}$.
    Suppose that $I^\complement$ is bounded away from $\mathcal G^{(\mathcal J^I_b - 1 )|b}(\bm 0)$ (see definitions in \eqref{property: hierarchy of set G k b}),
    and
        $
        \widecheck{\mathbf C}^{( \mathcal J^I_b )|b}(\partial I) = 0.
        $
        Then,
        for
        $\notationdef{notation-C-b-*}{C^I_b} \delequal  \widecheck{ \mathbf{C} }^{ (\mathcal{J}^I_b)|b }(I^\complement)$, we have $C^I_b < \infty$.
        Furthermore, if $C^I_b > 0$,
        then
        for each $\epsilon > 0$, $t \geq 0$, and measurable set $B \subseteq I^c$,
    \begin{align*}
    \limsup_{\eta\downarrow 0}\sup_{\bm x \in (I_\epsilon)^-}
    \P\bigg(
        C^I_b \eta \cdot \big(\lambda(\eta)\big)^{ \mathcal J^I_b } \cdot \tau^{\eta|b}(\bm x) > t;\ 
        \bm X^{\eta|b}_{ \tau^{\eta|b}(\bm x)}(\bm x) \in B
     \bigg)
     & \leq \frac{ \widecheck{\mathbf{C}}^{ (\mathcal J^I_b)|b }(B^-) }{ C^I_b }\cdot\exp(-t),
     \\
     \liminf_{\eta\downarrow 0}\inf_{\bm x \in (I_\epsilon)^-}
    \P\bigg(
        C^I_b \eta\cdot \big(\lambda(\eta)\big)^{ \mathcal J^I_b } \cdot\tau^{\eta|b}(\bm x) > t;\ 
        \bm X^{\eta|b}_{ \tau^{\eta|b}(\bm x)}(\bm x) \in B
     \bigg)
     & \geq \frac{ \widecheck{\mathbf{C}}^{ (\mathcal J^I_b)|b }(B^\circ) }{ C^I_b }\cdot\exp(-t).
    \end{align*}
    Here, $I_\epsilon = ((I^\complement)^\epsilon)^\complement$ is the $\epsilon$-shrinkage of the set $I$,
    and 
    $\lambda(\eta) = \eta^{-1}\P(\norm{\bm Z} > \eta^{-1})$.
\end{theorem}

It's worth noticing that the technical conditions in Theorem~\ref{theorem: first exit analyses, adapted} are less involved   compared to the original statements in  \cite{wang2022eliminating}.
This is thanks to the streamlined problem setup in Section~\ref{subsec: global dynamics problem setting}.
For completeness of the exposition, we highlight the differences below and formally explain how the results are adapted under the conditions in Section~\ref{subsec: global dynamics problem setting}.
We start by reviewing a few definitions in \cite{wang2022eliminating}.
First, analogous to the definitions in \eqref{def: perturb ode mapping h k b, 1}--\eqref{def: perturb ode mapping h k b, 3},
we 
define the mapping 
$\notationdef{notation-h-k-t-bar-mapping-truncation-level-b-LDP}{\bar h^{(k)|b}_{[0,T]}}: \mathbb{R}^d\times \mathbb{R}^{d \times k} \times \R^{d \times k} \times (0,{T}]^{k\uparrow} \to \mathbb{D}{[0,T]}$ as follows.
Given
$\bm x \in \mathbb{R}^d$,
$\textbf{W} = (\bm w_1,\cdots,\bm w_k) \in \mathbb{R}^{d \times k}$, 
$\textbf V = (\bm v_1,\cdots,\bm v_k) \in \R^{d \times k}$,
and $\bm{t} = (t_1,\cdots,t_k)\in (0,T]^{k\uparrow}$, let $\xi = \bar h^{(k)|b}_{[0,T]}(\bm x,\textbf{W},\textbf V,\bm{t})$ be the solution to
\begin{align*}
    \xi_0 & = \bm x; 
    \\
    \frac{d\xi_s}{d s} & = -\nabla f(\xi_s)\ \ \ \forall s \in [0,{T}],\ s \neq t_1,t_2,\cdots,t_k; 
    \\
    \xi_s & = \xi_{s-} + \bm v_j + \varphi_b\big( \bm \sigma(\xi_{s-} + \bm v_j)\bm w_j\big)\ \ \ \text{ if }s = t_j\text{ for some }j\in[k].
\end{align*}
In other words, we have
$
 h^{(k)|b}_{[0,T]}(\bm x, \bm W, \bm t)
    =
    \bar h^{(k)|b}_{[0,T]}\big(\bm x,\bm W, (\bm 0,\cdots,\bm 0), \bm t\big),
$
for the mapping $ h^{(k)|b}_{[0,T]}$ defined in \eqref{def: perturb ode mapping h k b, 1}--\eqref{def: perturb ode mapping h k b, 3},
and the difference in $\bar h^{(k)|b}_{[0,T]}$ is that we apply additional perturbations $\bm v_j$'s right before each jump.
Next, analogous to ${\widecheck{g}^{(k)|b}}$ defined in \eqref{def: mapping check g k b, endpoint of path after the last jump, first exit analysis}, let
\begin{align}
    \notationdef{notation-mapping-bar-g-k-b}{\bar g^{(k)|b}\big( \bm x, \textbf W, \textbf V, (t_1,\cdots,t_k)\big)}
    \delequal 
    \bar h^{(k)|b}_{ [0,t_k + 1] }
    \Big(
        \bm x,
        \textbf W,
        \textbf V,
        (t_1,\cdots,t_k)
    \Big)(t_k),
    \nonumber
\end{align}
and note that ${\widecheck{g}^{(k)|b}(\bm x,\textbf W,\bm t)}
    =
    \bar g^{(k)|b}\big(\bm x, \textbf W, (\bm 0,\cdots,\bm 0), \bm t\big)$.
This allows us to define (for each $k \geq 1$, $b,\epsilon > 0$, and $\bm x \in \R^d$)
\begin{align}
    {\mathcal G^{(k)|b}(\bm x;\epsilon)} 
    & \delequal 
    \bigg\{
    \bar g^{(k - 1)|b}
    \Big( \bm x + \bm v_1 + \varphi_b\big(\bm \sigma(\bm x + \bm v_1)\bm w_1\big),
    (\bm w_2,\cdots, \bm w_k), (\bm v_2,\cdots,\bm v_k), \bm t 
    \Big):
    \nonumber
    \\ 
    & \qquad\qquad
    \textbf W = (\bm w_1,\cdots, \bm w_k) \in \R^{d\times k},
    \textbf V = (\bm v_1,\cdots, \bm v_k) \in \Big(\bar B_\epsilon(\bm 0)\Big)^k,
    \bm t \in (0,\infty)^{k - 1 \uparrow}
    \bigg\},
    \label{def: set G k b epsilon}
\end{align}
where
$\bm y_t(\cdot)$ is the gradient flow defined in \eqref{def ODE path y t}, and
we use $\bar B_\epsilon(\bm x)$ to denote the \emph{closed} ball with radius $\epsilon$ centered at $\bm x \in \R^d$.
We also adopt the convention that ${\mathcal G^{(0)|b}(\bm x;\epsilon)} \delequal \bar B_\epsilon(\bm x)$.
Similar to \eqref{property: hierarchy of set G k b},
note that (for each $k \geq 1$)
\begin{align}
    {\mathcal G^{(k)|b}(\bm x;\epsilon)} 
    & = 
    \bigg\{
        \bm y_t(\bm z) + \bm v + \varphi_b\Big( \bm \sigma\big( \bm y_t(\bm z) + \bm v \big)  \bm w \Big):\ 
        t > 0,\ \bm w \in \R^d,\ \bm v \in \bar B_\epsilon(\bm 0) ,\  \bm z \in {\mathcal G^{(k-1)|b}(\bm x;\epsilon)}
    \bigg\}.
    \label{def: set G k b x epsilon, recursion}
\end{align}

We make a few important observations regarding the set ${\mathcal G^{(k)|b}(\bm x;\epsilon)}$.
First, by comparing \eqref{def: set G k b epsilon} to \eqref{def: set G k b},
note that 
\begin{align}
    \mathcal G^{(k)|b}(\bm x) = \mathcal G^{(k)|b}(\bm x;0)
\end{align}
In other words, the main difference in the construction of the set $\mathcal G^{(k)|b}(\bm x;\epsilon)$ in \eqref{def: set G k b epsilon} is that we apply $\epsilon$-bounded perturbations right before adding any jump onto the gradient flow paths.
Next, 
we stress that, given the problem setup in Section~\ref{subsec: global dynamics problem setting}, the set $ {\mathcal G^{(k)|b}(\bm x;\epsilon)} $ is bounded.
Indeed, we fix some $b > 0$ and $\bm x \in \R^d$, and note that for $k = 0$, we have
$\sup\{ \norm{\bm z}:\ \bm z \in \mathcal G^{(0)|b}(\bm x;\epsilon)  \}
=
\sup\{ \norm{\bm z}:\ \bm z \in \bar B_\epsilon(\bm x)  \}
\leq \norm{\bm x} + \epsilon$.
Next, 
by condition (iii) in Assumption~\ref{assumption: geometry, metastability},
we can fix $M$ large enough such that $\norm{\bm x } + \epsilon < M$ and $\inf_{\norm{\bm z} \geq M}\nabla f(\bm z)^\top \bm z  > 0$.
This implies $\norm{\bm y_t(\bm z)} \leq \norm{\bm z} \vee M$ for each $\bm z \in \R^d$ and $t \geq 0$.
Then, by definitions in \eqref{def: set G k b x epsilon, recursion},
it follows from an inductive argument that
\begin{align}
    \sup\Big\{ \norm{\bm z}:\ \bm z \in \mathcal G^{(k)|b}(\bm x;\epsilon)  \Big\}
    \leq M + k\cdot (b + \epsilon),
    \label{property, boundedness of set G k b x epsilon}
\end{align}
where $M$ is some constant that may vary with $\bm x$ and $\epsilon$ as noted above.
On the other hand, by definitions in \eqref{def: set G k b x epsilon, recursion}
and the non-degeneracy of $\bm \sigma(\cdot)$ (see Assumption~\ref{assumption: nondegeneracy of diffusion coefficients}), we have 
\begin{align}
    \mathcal G^{(k)|b}(\bm x;\epsilon) \supseteq B_{ kb + (k+1)\epsilon }(\bm x).
    \label{property, lower bound, size of set G k b epsilon}
\end{align}
Furthermore, by the Lipschitz continuity and non-degeneracy of $\bm \sigma(\cdot)$ (see Assumptions~\ref{assumption: lipschitz continuity of drift and diffusion coefficients} and \ref{assumption: nondegeneracy of diffusion coefficients})
as well as the boundedness of $\mathcal G^{(k)|b}(\bm x;\epsilon)$ (see \eqref{property, boundedness of set G k b x epsilon}),
the following can be established by Gronwall's inequality:
for any $\epsilon^\prime > 0$ and any $\bm x \in \R^d$, $b > 0$, and $k \in \mathbb Z_+$,
there exists $\epsilon > 0$ such that
\begin{align}
    \mathcal G^{(k)|b}(\bm x;\epsilon) \subseteq \big(\mathcal G^{(k)|b}(\bm x)\big)^{\epsilon^\prime},
    \label{property, inclusion of closure, G k b epsilon}
\end{align}
where we use $E^r$ to denote the $r$-enlargement of the set $E$. 

In the original statements of Theorem 2.8 in \cite{wang2023large}, 
it is required that $I^\complement$ is bounded away from $\mathcal G^{ (\mathcal J^I_b - 1)|b }(\bm 0;\epsilon)$ (for some $\epsilon > 0$ small enough) and that $\mathcal J^I_b < \infty$
where as in Theorem~\ref{theorem: first exit analyses, adapted} we only require 
$I^\complement$ to be bounded away from $\mathcal G^{ (\mathcal J^I_b - 1)|b }(\bm 0)$.
The reason is as follows:
\begin{itemize}
    \item 
    By \eqref{property, lower bound, size of set G k b epsilon} and the boundedness of $I$,
        we must have $\mathcal J^I_b =
        \min\big\{ k \geq 1:\ \mathcal G^{(k)|b}(\bm 0) \cap I^\complement \neq \emptyset \big\} < \infty$;

    \item 
        The condition that $I^\complement$ is bounded away from $\mathcal G^{(\mathcal J^I_b - 1 )|b}(\bm 0)$ 
        (i.e., $\inf\{ \norm{\bm x - \bm y}:\ \bm x \in I^\complement,\ \bm y \in \mathcal G^{(\mathcal J^I_b - 1 )|b}(\bm 0)  \} > 0$), 
        implies that $I^\complement$ is bounded away from $\big(\mathcal G^{(\mathcal J^I_b - 1 )|b}(\bm 0)\big)^{\epsilon^\prime}$ for some $\epsilon^\prime > 0$;
        By \eqref{property, inclusion of closure, G k b epsilon}, 
        the set $I^\complement$ must also be bounded away from $\mathcal G^{(\mathcal J^I_b - 1 )|b}(\bm 0;\epsilon)$ for some $\epsilon > 0$ small enough.
\end{itemize}
In short, the technical conditions in Theorem 2.8 of \cite{wang2023large} are automatically verified under the assumptions in Theorem~\ref{theorem: first exit analyses, adapted}, allowing us to adapt the first exit analyses and obtain the results in Theorem~\ref{theorem: first exit analyses, adapted}.

The remainder of Section~\ref{sec: review, first exit analyses} collects useful technical lemmas from \cite{wang2023large}.
First, Lemma~\ref{lemma: fixed cycle exit or return taking too long, first exit analysis} states that it is unlikely for $\bm X^{\eta|b}_t(\bm x)$ to take long excursion before exiting from $I_\epsilon$ or returning to a small enough neighborhood of the local minimum.

\begin{lemma}[Lemma 4.4 of \cite{wang2023large}]
\label{lemma: fixed cycle exit or return taking too long, first exit analysis}
\linksinthm{lemma: fixed cycle exit or return taking too long, first exit analysis}
Let Assumptions \ref{assumption gradient noise heavy-tailed}, 
\ref{assumption: lipschitz continuity of drift and diffusion coefficients}, 
and \ref{assumption: shape of f, first exit analysis} hold.
 Given any $k \geq 1$ and any $\epsilon > 0$ small enough,
 there exists $T = T(k,\epsilon) \in (0,\infty)$ such that for any $t \geq T$,
 \begin{align*}
      \lim_{\eta \downarrow 0}\ \sup_{ \bm x \in (I_\epsilon)^- }
     \ \, \frac1{\big(\lambda(\eta)\big)^{k-1}}\P\Big( \bm X^{\eta|b}_t(\bm x) \in I_\epsilon\setminus B_\epsilon(\bm 0)\ \  \forall t \leq T/\eta \Big)
      = 0,
 \end{align*}
 where $\lambda(\eta) = \eta^{-1}\P(\norm{\bm Z} > \eta^{-1})$.
\end{lemma}

Next, let
$
\notationdef{notation-R-eta-b-epsilon-return-time}{R^{\eta|b}_\epsilon(\bm x)} \delequal \min\big\{ t \geq 0:\ \bm X^{\eta|b}_t(\bm x) \in B_\epsilon(\bm 0) \big\}
$
be the first time $\bm X^{\eta|b}_t(x)$ returns to the $\epsilon$-neighborhood of the origin.
Lemma~\ref{lemma: cycle, efficient return} verifies that, when initialized within the attraction field $I$, the SGD iterates $\bm X^{\eta|b}_t(\bm x)$ would return to the local minimum efficiently with high probability.

\begin{lemma}[Lemma 4.5 of \cite{wang2023large}]
\label{lemma: cycle, efficient return}
\linksinthm{lemma: cycle, efficient return}
Let Assumptions \ref{assumption gradient noise heavy-tailed}, 
\ref{assumption: lipschitz continuity of drift and diffusion coefficients}, 
and \ref{assumption: shape of f, first exit analysis} hold.
For each $\epsilon > 0$ small enough, there exists a constant $T(\epsilon) \in (0,\infty)$ such that, for the event
    \begin{align*}
        E(\eta,\epsilon,\bm x)
        \delequal 
        \Big\{ R^{\eta|b}_\epsilon(\bm x) \leq \frac{T(\epsilon)}{\eta};\ \bm X^{\eta|b}_t(\bm x) \in I_{\epsilon/2}\ \forall t \leq R^{\eta|b}_\epsilon(\bm x) \Big\},
    \end{align*}
    we have
    $
     \lim_{\eta\downarrow 0}\sup_{\bm x \in (I_\epsilon)^-}
        \P\Big( \big(E(\eta,\epsilon,\bm x)\big)^\complement\Big) = 0.
    $
\end{lemma}

We also prepare two auxiliary technical lemmas that will be useful in the our subsequent proofs when applying Theorem~\ref{theorem: first exit analyses, adapted}.
First,
Lemma~\ref{lemma stuck at local minimum before large jump} shows that it is unlikely for $\bm X^{\eta|b}_t(\bm x)$ to deviate far from the local minimum without any ``large'' noise $\bm Z_t$.
Again, the proof makes heavy use of the results in \cite{wang2023large}.

\begin{lemma}
\label{lemma stuck at local minimum before large jump}
\linksinthm{lemma stuck at local minimum before large jump}
Let Assumptions \ref{assumption gradient noise heavy-tailed}, 
\ref{assumption: lipschitz continuity of drift and diffusion coefficients}, 
and \ref{assumption: shape of f, first exit analysis} hold.
Let $\tau^{>\delta}_1(\eta) \delequal \min\{ t \geq 1:\ \eta\norm{\bm Z_t} > \delta \}$.
Given any $\epsilon > 0$ small enough
and any positive integer $N$,
there exists $\bar \delta > 0$ such that
\begin{align*}
    \lim_{\eta \downarrow 0}
    \sup_{\bm x \in B_{\epsilon/2}(\bm 0)}
    \P\Big(
        \norm{ \bm X^{\eta|b}_t(\bm x) } \geq \epsilon\text{ for some }t < \tau^{>\delta}_1(\eta)
    \Big)\Big/\eta^N = 0,
    \qquad 
    \forall \delta \in (0,\bar\delta).
\end{align*}
\end{lemma}

\begin{proof}\linksinpf{lemma stuck at local minimum before large jump}
We start with a few observations.
First,  let $T^\eta_r(\bm x) \delequal \min\{t \geq 0:\ \norm{\bm X^{\eta|b}_t(\bm x) } \geq r \}$.
Due to the monotonicity in $\tau^{>\delta^\prime}_1(\eta) \leq \tau^{>\delta}_1(\eta)$
for any $0 < \delta^\prime < \delta$,
it suffices to show that for any positive integer $N$ and any small enough $\epsilon > 0$,
there is some $\delta = \delta(N,\epsilon) > 0$ such that 
\begin{align}
\limsup_{\eta \downarrow 0}
{ \sup_{\bm x \in B_{\epsilon}(\bm 0)}\P\big( T^\eta_{2\epsilon}(\bm x) < \tau^{>\delta}_1(\eta) \big)}\big/{ \eta^{N}} = 0,
\label{proof, goal 1, lemma stuck at local minimum before large jump}
\end{align}
where we also w.l.o.g.\ multiply $\epsilon$ by constant $2$ (compared to the original statements in Lemma~\ref{lemma stuck at local minimum before large jump}) to simplify notations in this proof.
Second, since the statements only concern the behavior of SGDs over a bounded region,
and
the values of $\nabla f(\cdot)$ and $\bm \sigma(\cdot)$ outside of $B_\epsilon(\bm 0)$ have not impact,
in light of Assumption~\ref{assumption: lipschitz continuity of drift and diffusion coefficients}
we can assume w.l.o.g.\ the existence of some $ C < \infty$ that
\begin{align}
    \sup_{\bm x \in \R^d}\norm{\bm\sigma(\bm x)} \vee \norm{\nabla f(\bm x)} \leq C.
    \label{proof, constant c and C, lemma stuck at local minimum before large jump}
\end{align}
Lastly, note that for any $\epsilon > 0$ small enough, we have: 
(i) 
 $B_{\epsilon}(\bm 0) \subseteq I$ (where $I$ is the open subset of the attraction field of $\bm 0$ stated in Assumption~\ref{assumption: shape of f, first exit analysis}),
 and
 (ii) the claim \eqref{property in assumption, contraction around minima, assumption: shape of f, first exit analysis} in Assumption~\ref{assumption: shape of f, first exit analysis} holds.
Henceforth in this proof, we only consider such $\epsilon$.



Recall that $\alpha > 1$ is the heavy-tailed index specified in Assumption~\ref{assumption gradient noise heavy-tailed}.
Also, fix some $\beta > \alpha$,
and observe that
\begin{align*}
     \P( \tau^{>\delta}_1(\eta) > 1/\eta^{\beta}) 
     = 
     \P\big( \text{Geom}\big( H(\delta/\eta) \big) > 1/\eta^\beta  \big),
\end{align*}
where $H(x) = \P(\norm{\bm Z_1} > x) \in \RV_{-\alpha}(x)$ as $x \to \infty$.
Combining our choice of $\beta > \alpha$ with standard bounds on the tail cdf of Geometric random variables (see, e.g., Lemma~D.1 of \cite{wang2023large}),
it hold for any $\theta \in (0,\beta - \alpha)$ that
$
 \P( \tau^{>\delta}_1(\eta) > 1/\eta^{\beta}) = \bm{o}\big(\exp(-1/\eta^\theta)\big)
$
(as $\eta \downarrow 0$).
Then, due to 
\begin{align*}
    \big\{T^\eta_{2\epsilon}(x) < \tau^{>\delta}_1(\eta)\big\}
& \subseteq
\big\{T^\eta_{2\epsilon}(x) < \tau^{>\delta}_1(\eta) \leq {1/\eta^\beta} \big\}
\cup  \big\{\tau^{>\delta}_1(\eta) > {1/\eta^\beta} \big\},
\end{align*}
it suffices to find some $\delta >0$ such that 
(here, note that by definitions, $ \tau^{>\delta}_1(\eta)$ and $T^\eta_{2\epsilon}(\bm x)$ only take integer values)
\begin{align}
    \sup_{ \bm x \in B_\epsilon(\bm 0) }\P\big(T^\eta_{2\epsilon}(\bm x) < \tau^{>\delta}_1(\eta) \wedge  \floor{1/\eta^\beta}\big)
    = \bm{o}(\eta^N),
    \ \ \ \text{ as }\eta\downarrow 0.
    \label{proof, goal 2, lemma stuck at local minimum before large jump}
\end{align}
Furthermore, 
let 
$$
K(\eta,t) \delequal 
\ceil{\frac{ \floor{1/\eta^\beta}  }{ \floor{t/\eta} } },
$$
and
suppose we can find $\delta,t,\tilde \epsilon > 0$ such that for all $\eta > 0$ small enough,
\begin{align}
    \sup_{\bm x \in B_\epsilon(\bm 0)} 
    \P\Big( T^\eta_{2\epsilon}(\bm x) < \tau^{>\delta}_1(\eta) \wedge \floor{1/\eta^\beta} \Big)
\leq 
\sup_{\bm x \in B_\epsilon(\bm 0)} 
\P\bigg( \bigcup_{k = 1}^{K(\eta,t)}\Big( {A}_k(\eta,t,\tilde\epsilon,\bm x) \Big)^\complement \bigg),
 \label{proof, goal 3, lemma stuck at local minimum before large jump}
\end{align}
where 
\begin{align*}
    {A}_k(\eta,t,\tilde\epsilon,\bm x) \delequal \Bigg\{ 
    \max_{(k-1)\floor{\frac{t}{\eta}} + 1 \leq j \leq k\floor{\frac{t}{\eta}}\wedge \big(\tau^{>\delta}_1(\eta) - 1\big) }
    \eta\norm{ 
        \sum_{i = (k-1)\floor{\frac{t}{\eta}} + 1  }^j \bm \sigma\big( \bm X^{\eta|b}_{i-1}(\bm x) \big)\bm Z_i 
    } 
    \leq \tilde{\epsilon}
    \Bigg\}.
\end{align*}
Then,
by part $(b)$ of Lemma~3.1 in \cite{wang2023large},
the claim
$\sup_{k \in [K(\eta,t)]}\sup_{\bm x \in B_\epsilon(\bm 0)}\P\Big( \big({A}_k(\eta,t,\tilde \epsilon,\bm x)\big)^\complement \Big) = \bm{o}(\eta^{N + \beta - 1})$
holds for all $\delta > 0$ small enough,
which leads to
\begin{align*}
    \sup_{\bm x \in B_\epsilon(\bm 0)} 
    \P\Big( T^\eta_{2\epsilon}(\bm x) < \tau^{>\delta}_1(\eta) \wedge \floor{1/\eta^\beta} \Big)
    \leq 
    K(\eta, t) \cdot \bm{o}(\eta^{N + \beta - 1})
    \leq \bm{O}(1/\eta^{\beta - 1})\cdot \bm{o}(\eta^{N + \beta - 1})
    = \bm{o}(\eta^{N}).
\end{align*}
This verifies claim \eqref{proof, goal 2, lemma stuck at local minimum before large jump}
and concludes the proof.
Now, it only remains to proof Claim~\eqref{proof, goal 3, lemma stuck at local minimum before large jump}.

\medskip
\noindent
\textbf{Proof of Claim~\eqref{proof, goal 3, lemma stuck at local minimum before large jump}}.
We consider some $t > 0$ large enough, whose value will be determined later.
Given such large $t$,
we pick some 
$\tilde{\epsilon} > 0$ small enough such that
$ 2\exp(tD)\tilde{\epsilon} < \epsilon/2$,
with $D < \infty$ being the Lipschitz constant in Assumption \ref{assumption: lipschitz continuity of drift and diffusion coefficients}.

For any $\bm x \in B_{\epsilon}(\bm 0)$, any $\delta \in (0,\frac{b}{2C})$ and any $\eta \in (0,\frac{\tilde{\epsilon}}{C} \wedge \frac{b\wedge 1}{2C} )$ (where $C$ is specified in \eqref{proof, constant c and C, lemma stuck at local minimum before large jump}),
on the event ${A}_1(\eta,t,\tilde\epsilon,\bm x)$,
we make a few observations.
First,
from part $(b)$ of Lemma~3.6 in \cite{wang2023large},
    \begin{align*}
       \sup_{s \leq {\frac{t}{\eta}}\wedge \big(\tau^{>\delta}_1(\eta) - 1\big) }
       \norm{ 
        \bm{y}_{\eta s}(\bm x) - \bm X^{\eta|b}_{\floor{s}}(\bm x)
        } < \exp(tD)\widetilde{\epsilon} + \exp(tD)\eta C
       < 2\exp(tD)\widetilde{\epsilon} < \epsilon/2,
    \end{align*}
where $\bm y_t(\bm x)$ is the gradient flow (ODE path) defined in \eqref{def ODE path y t},
and the last inequality follows from
our choice of $\tilde{\epsilon}$ and $\eta$ above.
Next, 
by the claim \eqref{property in assumption, contraction around minima, assumption: shape of f, first exit analysis} in Assumption \ref{assumption: shape of f, first exit analysis}, we have
$
\bm{y}_s(\bm x) \in B_{\epsilon}(\bm 0)\ \forall s \geq 0,\ \bm x \in B_\epsilon(\bm 0);
$
also, for any $t > 0$ large enough, we have 
$
\bm{y}_t(\bm x) \in B_{ \epsilon/2 }(\bm 0)\ \forall \bm x \in B_\epsilon(\bm 0).
$
We only consider such $t > 0$ in this proof.
Combining these facts, we see that on the event ${A}_1(\eta,t,\tilde\epsilon,\bm x)$:
\begin{itemize}
    \item 
    $\bm X^{\eta|b}_s(\bm x) \in B_{2\epsilon}(\bm 0)\ 
    \forall s \leq \floor{t/\eta} \wedge \big(\tau^{>\delta}_1(\eta)-1\big)$,
    so $T^\eta_{2\epsilon} \geq \tau^{>\delta}_1 \wedge \floor{t/\eta}$;

    \item 
    $\bm X^{\eta|b}_{\floor{{t}/{\eta}}}(\bm x) \in B_{\epsilon}(\bm 0)$ if $\tau^{>\delta}_1(\eta) \geq \floor{t/\eta}$.
\end{itemize}
In particular, the second bullet point allows us to 
repeat the arguments above inductively for $k = 2,3,\cdots,K(\eta,t)$,
and verify the following:
for any $\bm x \in B_{\epsilon}(\bm 0)$, any $\delta \in (0,\frac{b}{2C})$, and any $\eta \in (0,\frac{\tilde{\epsilon}}{C} \wedge \frac{b\wedge 1}{2C})$,
it holds on event $\bigcap_{k = 1}^{K(\eta,t)}{A}_k(\eta,t,\tilde \epsilon,\bm x)$
that
$$
\bm X^{\eta|b}_s(\bm x) \in B_{2\epsilon}(\bm 0),\qquad 
\forall s
\leq 
K(\eta,t)\cdot \floor{t/\eta} \wedge \big(\tau^{>\delta}_1(\eta) - 1\big).
$$
To conclude the proof for  Claim~\eqref{proof, goal 3, lemma stuck at local minimum before large jump},
simply note that
$
K(\eta,t ) \cdot \floor{t/\eta} = \ceil{\frac{ \floor{1/\eta^\beta}  }{ \floor{t/\eta} } } \cdot \floor{t/\eta} \geq \floor{1/\eta^\beta}.
$
\end{proof}

Lemma~\ref{lemma: exit rate strictly positive, first exit analysis} then states useful properties for the measure $\widecheck{\mathbf C}^{(k)|b}$ in \eqref{def: measure check C k b}.

\begin{lemma}
\label{lemma: exit rate strictly positive, first exit analysis}
\linksinthm{lemma: exit rate strictly positive, first exit analysis}
Let Assumptions~\ref{assumption: geometry, metastability},
\ref{assumption gradient noise heavy-tailed},
and
\ref{assumption: lipschitz continuity of drift and diffusion coefficients} hold.
For any $i,j \in [K]$ with $i \neq j$,
\begin{align}
    \widecheck{\mathbf C}^{ (\mathcal J_b(i))|b  }(I_j;\bm m_i) > 0
    \qquad\iff \qquad
    I_j \cap \mathcal G^{ (\mathcal J_b(i))|b }(\bm m_i) \neq \emptyset.
    \nonumber
\end{align}
\end{lemma}

\begin{proof}
\textbf{Proof of ``$\Rightarrow$''}.
By definitions in \eqref{def: set G k b} and \eqref{def: measure check C k b},
the measure $\widecheck{\mathbf C}^{(k)|b}(\ \cdot\ ;\bm x)$ is supported on $\mathcal G^{(k)|b}(\bm x)$.
Then $I_j \cap \mathcal G^{ (\mathcal J_b(i))|b }(\bm m_i) = \emptyset$ implies 
$
\widecheck{\mathbf C}^{ (\mathcal J_b(i))|b  }(I_j;\bm m_i) = 0.
$

\medskip
\noindent
\textbf{Proof of ``$\Leftarrow$''}.
Suppose that $\bm z \in I_j \cap \mathcal G^{ (\mathcal J_b(i))|b }(\bm m_i)$.
By definitions in \eqref{def: mapping check g k b, endpoint of path after the last jump, first exit analysis} and \eqref{def: set G k b},
there exist (with $k = \mathcal J_b(i)$ to lighten notations in this proof) some
$
\textbf{W} = (\bm w_1,\ldots,\bm w_{ k  }) \in \R^{d\times k},
$
and
$
\bm t = (t_1,\ldots,t_{k-1}) \in (0,\infty)^{(k-1)\uparrow}
$
(that is, $0 < t_1 < t_2 < \ldots < t_{k-1}$),
such that
\begin{align}
    \bm z = h^{(k-1)|b}_{ [0,1+t_{k-1}] }\Big(
        \varphi_b\big( \bm \sigma(\bm m_i)\bm w_1  \big), (\bm w_2,\ldots,\bm w_k), (t_1,\ldots,t_{k-1})
    \Big)(t_{k-1}).
    \nonumber
\end{align}
Then, by the continuity of the mapping $h^{(m)|b}_{[0,T]}$ (see Lemma~3.4 of \cite{wang2023large})
and the fact that $I_j$ is an open set,
there exist some $\epsilon \in (0,1)$ small enough such that the claim
$$
h^{(k-1)|b}_{ [0,1+t_{k-1}] }\Big(
        \varphi_b\big( \bm \sigma(\bm m_i)\tilde{\bm w}_1  \big), (\tilde{\bm w}_2,\ldots,\tilde{\bm w}_k), (\tilde t_1,\ldots,\tilde t_{k-1})
    \Big)(\tilde t_{k-1}) \in I_j
$$
holds whenever $\norm{\bm w_j - \tilde{\bm w}_j} < \epsilon\ \forall j \in [k]$,
and $|\tilde t_j - t_j| < \epsilon\ \forall j \in [k-1]$ (which also ensures $\tilde t_{k-1} < 1 + t_{k - 1}$ for the path evaluated at time $\tilde t_{k-1}$ to be well-defined in the display above).
Then, by \eqref{def: measure check C k b},
\begin{align*}
    \widecheck{\mathbf C}^{ (k)|b  }(I_j;\bm m_i)
    & \geq 
    \bigg( \prod_{j = 1}^{k} ((\nu_\alpha \times \mathbf S)\circ \Psi)\Big( B_\epsilon(\bm w_j) \Big)  \bigg)
    \cdot \epsilon^{k-1}.
\end{align*}
To conclude the proof, it suffices to note the following:
since the density of the spherical measure $\mathbf S$ is uniformly bounded from $0$ (see Assumption~\ref{assumption gradient noise heavy-tailed}),
the Lebesgue measure on $\R^d$ is absolutely continuous w.r.t.\ $(\nu_\alpha \times \mathbf S)\circ \Psi$,
thus implying
$
((\nu_\alpha \times \mathbf S)\circ \Psi)\big( B_\epsilon(\bm w_j) \big) > 0
$
for each $j$.
\end{proof}

\section{Sample Path Convergence to Jump Processes}
\label{subsec: proof, abstract framework, metastability}

This section develops a theoretical framework for establishing sample-path level convergence to jump processes in $\D[0,\infty)$,
which greatly facilitates our proof for Theorem~\ref{corollary irreducible case}.

Let $Y^\eta_{\boldsymbol{\cdot}}$ and $Y^*_{\boldsymbol{\cdot}}$ be random elements in $\D[0,\infty)$, i.e., $\R^d$-valued càdlàg processes.
We start by discussing a few properties of the weak convergence in $(\D[0,\infty),\dlp{[0,\infty)})$.
In particular, a similar mode of convergence in 
$(\D[0,T],\dlp{[0,T]})$ can be defined analogously for any $T \in (0,\infty)$.
Recall the projection mapping $\pi_T$ defined in \eqref{def: projection mapping pi T}.
We say that 
$Y^\eta_{ \boldsymbol{\cdot} }\Rightarrow Y^*_{ \boldsymbol{\cdot} }$ in $(\D[0,T],\dlp{[0,T]})$
if
\begin{align*}
    \lim_{\eta \downarrow 0}
    \E g\big(\pi_T(S^\eta_{ \boldsymbol{\cdot} })\big) 
    = 
    \E g\big(\pi_T(S^*_{ \boldsymbol{\cdot} })\big),
    \qquad 
    \forall g:\D[0,T]\to\R\text{ continuous and bounded};
\end{align*}
see \eqref{def, Lp distance} for the definition of $\dlp{[0,T]}$.
More precisely, the $L_p$ norm $\dlp{[0,T]}$ induces a metric over a quotient space $\D[0,T]/\mathcal N$,
where
we set $\mathcal N = \{ \xi \in \D[0,T]:\ \xi_t \equiv \bm 0\ \forall t \in [0,T) \}$,
which
is the set containing all paths in $\D[0,T]$ that stays at the origin except for the endpoint.
(That is, any two paths $x,\ y \in \mathbb D[0,T]$ are considered equivalent under $\dlp{[0,T]}$ if $x_t = y_t\ \forall t \in [0,T)$.)

First, Lemma~\ref{lemma: Lp convergence, from T to infty}
shows that the convergence in $(\D[0,\infty),\dlp{[0,\infty)})$ follows from the convergence in $(\D[0,T],\dlp{[0,T]})$.

\begin{lemma}
    \label{lemma: Lp convergence, from T to infty}
    \linksinthm{lemma: Lp convergence, from T to infty}
    Let $p \in [1,\infty)$.
    If $Y^\eta_{ \boldsymbol{\cdot} }\Rightarrow Y^*_{ \boldsymbol{\cdot} }$ in $(\D[0,T],\dlp{[0,T]})$ as $\eta \downarrow 0$
    for any positive integer $T$,
    then 
$Y^\eta_{ \boldsymbol{\cdot} }\Rightarrow Y^*_{ \boldsymbol{\cdot} }$ in $(\D[0,\infty),\dlp{[0,\infty)})$ as $\eta \downarrow 0$.
\end{lemma}

\begin{proof}
  \linksinpf{lemma: Lp convergence, from T to infty}  
  By Portmanteau Theorem, it suffices to show that
$
\lim_{\eta \downarrow 0}\E g(Y^\eta_{ \boldsymbol{\cdot} }) = \E g(Y^*_{ \boldsymbol{\cdot} })
$
holds
for any $g:\D[0,\infty) \to \R$ that is bounded and uniformly continuous (w.r.t.\ the topology induced by $\dlp{[0,\infty)}$).
To proceed, we arbitrarily pick one such $g$ and some $\epsilon > 0$.
By virtue of the uniform continuity of $g$, there exists some $\delta > 0$ such that
$
|g(x) - g(y)| < \epsilon
$
whenever $\dlp{[0,\infty)}(x,y) < \delta$.
By definition of $\dlp{[0,\infty)}$ in \eqref{def, Lp metric D 0 infty},
fo each $T > 0$,
we must have $\dlp{[0,\infty)}(x,y) < 1/2^{\floor{T}-1}$ if $x_t = y_t$ for all $t \in [0,T)$.
Now, we fix some positive integer $T$ large enough such that $1/2^{T-1} < \delta$.
Define $\widetilde \pi_T:\D[0,\infty)\to\D[0,\infty)$ by
\begin{align*}
    \widetilde\pi_T(\xi)_t \delequal 
    \begin{cases}
        \xi_t & \text{ if }t \in [0,T) \\ 
        0 & \text{ if }t \geq T
    \end{cases}
\end{align*}
and set $\widetilde g_T(\xi) \delequal g\big(\widetilde \pi_T(\xi)\big)$.
We now have 
$
\dlp{[0,\infty)}\big(\xi, \widetilde \pi_T(\xi)\big) < \delta
$
and hence $|g(\xi) - \widetilde g_T(\xi) | < \epsilon$
for any $\xi \in \D[0,\infty)$.
As a result,
\begin{align}
    \limsup_{\eta \downarrow 0}| \E g(Y^\eta_{\boldsymbol{\cdot}}) - \E \widetilde g_T(Y^\eta_{\boldsymbol{\cdot}}) |<\epsilon,
    \qquad 
    | \E g(Y^*_{\boldsymbol{\cdot}}) - \E \widetilde g_T(Y^*_{\boldsymbol{\cdot}}) | < \epsilon.
    \label{proof, limit 1, lemma: Lp convergence, from T to infty}
\end{align}
Furthermore, let $\pi^\dagger_T:\D[0,T]\to\D[0,\infty)$ be defined as
\begin{align*}
    \pi^\dagger(\xi)_t \delequal 
    \begin{cases}
        \xi_t & \text{ if }t \in [0,T)
        \\
        0 & \text{ if }t \geq T
    \end{cases},
\end{align*}
which can be interpreted as a ``pseudo inverse'' of the projection mapping $\pi_T$ defined in \eqref{def: projection mapping pi T}.
Also, let $g_T:\D[0,T]\to\R$ by $g_T(\cdot) \delequal g\big(\pi^\dagger_T(\cdot)\big)$.
It is easy to see that
$(i)$ $g_T$ is continuous due to the continuity of $g$ and $\pi^\dagger_T$,
and
$(ii)$ for any $\xi \in \D[0,\infty)$,
we have $\widetilde g_T(\xi) = g_T\big(\pi_T(\xi)\big)$.
Due to $Y^\eta_{ \boldsymbol{\cdot} }\Rightarrow Y^*_{ \boldsymbol{\cdot} }$ in $(\D[0,T],\dlp{[0,T]})$,
we now yield
\begin{align}
    \lim_{\eta \downarrow 0}| \E \widetilde g_T(Y^\eta_{\boldsymbol{\cdot}}) - \E \widetilde g_T(Y^*_{\boldsymbol{\cdot}}) | = 0.
    \label{proof, limit 2, lemma: Lp convergence, from T to infty}
\end{align}
Combining \eqref{proof, limit 1, lemma: Lp convergence, from T to infty} and \eqref{proof, limit 2, lemma: Lp convergence, from T to infty},
we get 
$
\limsup_{\eta \downarrow 0}|\E g(Y^\eta_{ \boldsymbol{\cdot} })  - \E g(Y^*_{ \boldsymbol{\cdot} })| < 2\epsilon.
$
Driving $\epsilon \to 0$, we conclude the proof.
\end{proof}

Lemma~\ref{proposition, fdd + tightness = Lp convergence} then provides a Prohorov-type argument where 
 weak convergence in $(\D[0,T],\dlp{[0,T]})$ can be established 
using the convergence in f.d.d.\ and a tightness condition.
The proof is a straightforward adaptation of its
$J_1$ counterparts.
For the sake of clarity, the next proof will, w.l.o.g., focus on the case where $T = 1$, but it's clear that the arguments can be easily extended to $\mathbb D[0,T]$ with arbitrary $T \in (0,\infty)$. 

\begin{lemma}
\label{proposition, fdd + tightness = Lp convergence}
\linksinthm{proposition, fdd + tightness = Lp convergence}
    Let $T \in (0,\infty)$, $p \in [1,\infty)$, and $\mathcal T$ be a dense subset of $(0,T)$.
    Suppose that the laws of $Y^{\eta_n}_{\boldsymbol{\cdot}}$ are tight in  $(\D[0,T],\dlp{[0,T]})$ for any sequence $\eta_n > 0$ with $\lim_n \eta_n = 0$,
    and 
    \begin{align}
        (Y^{\eta}_{t_1},\cdots,Y^\eta_{t_k}) \Rightarrow (Y^*_{t_1},\cdots,Y^*_{t_k})
    \text{ as }\eta \downarrow 0
    \qquad 
    \forall k = 1,2,\cdots,\ 
    \forall (t_1,\cdots,t_k) \in \mathcal T^{k\uparrow}.
    \label{condition, fdd convergence, proposition, fdd + tightness = Lp convergence}
    \end{align}
    Then $Y^\eta_{ \boldsymbol{\cdot} }\Rightarrow Y^*_{ \boldsymbol{\cdot} }$ in $(\D[0,T],\dlp{[0,T]})$ as $\eta \downarrow 0$.
\end{lemma}

\begin{proof}
\linksinpf{proposition, fdd + tightness = Lp convergence}
As mentioned above, the arguments are similar to those of the standard proofs in \cite{billingsley2013convergence} for $J_1$ topology,
and we provide the detailed proof for the sake of completeness.
Also, w.l.o.g.\ we focus on the case where $T = 1$ and write $\D = \D[0,1]$.

For any $0 \leq t_1 < t_2 < \cdots < t_k \leq 1$,
let $\pi_{(t_1,\cdots,t_k)}:\mathbb D\to \R^k$ be the projection mapping, i.e.,
$
\pi_{(t_1,\cdots,t_k)}(\xi) = (\xi_{t_1},\xi_{t_2},\cdots,\xi_{t_k}).
$
Let $\mathcal R^k$ be the Borel $\sigma$-algebra for $\R^{d \times k}$.
Let $p[\pi_{\bm t}: \bm t \in \mathcal T]$ be the collection of all sets of form
$\pi^{-1}_{(t_1,\cdots,t_k)}H$,
where $k \geq 1$, $H \in \mathcal R^k$, and $t_1 < \cdots < t_k$ with $t_i \in \mathcal T$ for each $i \in [k]$.
It suffices to show that (we write $\dlp{} = \dlp{[0,1]}$ and let $\mathcal D_p$ be the Borel $\sigma$-algebra of $(\D,\dlp{})$)
\begin{align}
    p[\pi_{\bm t}:\bm t \in \mathcal T]\text{ is a separating class for }(\D,\dlp{}).
    \label{proof, goal, proposition, fdd + tightness = Lp convergence}
\end{align}
In other words,
any two Borel probability measures $\mu$ and $\nu$ over $(\D,\dlp{})$
would coincide (i.e., $\mu(A)=\nu(A)\ \forall A \in \mathcal D_p$) if 
$\mu(A) = \nu(A)\ \forall A \in p[\pi_{\bm t}:\bm t \in \mathcal T]$.
To see why claim~\eqref{proof, goal, proposition, fdd + tightness = Lp convergence} is a sufficient condition,
note that the tightness condition implies that the sequence $Y^{\eta_n}_{\boldsymbol{\cdot}}$ has a converging sub-sequence, 
while the claim~\eqref{proof, goal, proposition, fdd + tightness = Lp convergence} and assumption \eqref{condition, fdd convergence, proposition, fdd + tightness = Lp convergence}
ensures that the limiting law must be that of $Y^{*}_{\boldsymbol{\cdot}}$.


The remainder of this proof is devoted to establishing claim \eqref{proof, goal, proposition, fdd + tightness = Lp convergence}.
First, we show that the projection mapping of form $\pi_{(t_1,\cdots,t_k)}:\D \to \R^{d \times k}$ is $\mathcal D_p/\mathcal R^k$ measurable when $0 \leq t_1 < \cdots < t_k < 1$,
which immediately confirms that $p[\pi_{\bm t}:\bm t \in \mathcal T]\subseteq\mathcal D_p$.
To do so, it suffices to prove that $\pi_{(t)}$ is measurable for any given $t \in [0,1)$.
Define $h_\epsilon(x): \D \to \R$ by $h_\epsilon(x) = \epsilon^{-1}\int_t^{t + \epsilon}x_s ds$.
W.l.o.g.\ we only consider $\epsilon$ small enough such that $t+\epsilon \leq 1$.
For any $x,\ y \in \D$ and $\Delta \in (0,1)$,
\begin{align*}
    \norm{h_\epsilon(x) - h_\epsilon(y)}
    &
    \leq
    \epsilon^{-1}\int_t^{t+\epsilon}\norm{x_s - y_s}\mathbf{I}\{ \norm{x_s-y_s} > \Delta \} ds
    +
    \epsilon^{-1}\int_t^{t+\epsilon}\norm{x_s - y_s}\mathbf{I}\{ \norm{x_s-y_s} \leq \Delta \} ds
    \\ 
    & \leq 
    \epsilon^{-1}\int_t^{t+\epsilon} \frac{\norm{x_s - y_s}^p}{ |\Delta|^p }ds
    +
    \Delta.
\end{align*}
Therefore, for any sequence $y^{(n)} \in \D$ such that $\dlp{}(y^{(n)},x) \to 0$,
we have $\limsup_{n\to\infty}\norm{h_\epsilon(x) - h_\epsilon(y^{(n)})} \leq \Delta$.
Driving $\Delta \downarrow 0$, we see that $h_\epsilon(\cdot)$ is a continuous mapping.
On the other hand, the right continuity of all paths in $\D$ implies that $h_{\epsilon}(x) \to \pi_{(t)}(x)$ as $\epsilon \to 0$ for all $x \in \D$.
As a result, the limiting mapping $\pi_{(t)}$ must be $\mathcal D_p/\mathcal R$ measurable.

Let $\sigma[\pi_{\bm t}:\bm t \in \mathcal T]$ be the $\sigma$-algebra generated by $p[\pi_{\bm t}:\bm t \in \mathcal T]$.
We have just verified $p[\pi_{\bm t}:\bm t \in \mathcal T]\subseteq\mathcal D_p$,
which implies $\sigma[\pi_{\bm t}:\bm t \in \mathcal T]\subseteq\mathcal D_p$ since $\mathcal D_p$ is also a $\sigma$-algebra.
Suppose we can show 
\begin{align}
    \sigma[\pi_{\bm t}:\bm t \in \mathcal T]\supseteq\mathcal D_p\qquad (\text{and hence }\sigma[\pi_{\bm t}:\bm t \in \mathcal T]=\mathcal D_p),
    \label{proof, goal 2, proposition, fdd + tightness = Lp convergence}
\end{align}
then we can confirm claim~\eqref{proof, goal, proposition, fdd + tightness = Lp convergence} using $\pi-\lambda$ Theorem.
Indeed, for any Borel probability measures $\mu$ and $\nu$ over $(\D,\dlp{})$, note that $\mathcal L \delequal \{A\in \mathcal D_p:\ \mu(A) = \nu(A)\}$ is a $\lambda$-system.
Whenever $p[\pi_{\bm t}:\bm t \in \mathcal T] \subseteq \mathcal L$, by applying $\pi-\lambda$ Theorem we then get
 $\sigma[\pi_{\bm t}:\bm t \in \mathcal T] = \mathcal D_p \subseteq \mathcal L$.
 This concludes that $p[\pi_{\bm t}:\bm t \in \mathcal T]$ is a separating class.

Now, it only remains to prove claim~\eqref{proof, goal 2, proposition, fdd + tightness = Lp convergence}.
Since $\mathcal T$ is a dense subset of $(0,T)$, for each $m \geq 1$
we can pick some positive integer $k$ and some $0 < s_1 < \cdots < s_k < 1$, with $s_i \in \mathcal T$, 
such that
$\max_{i \in [k+1]}| s_{i+1} - s_i | < m^{-1}$,
under the convention that $s_0 = 0$ and $s_{k+1} = 1$.
Now, construct a mapping $V_m:\R^{d\times k} \to \D$ as follows:
for each $\bm \alpha = (\alpha_1,\cdots,\alpha_k) \in \R^{d \times k}$,
define $\xi = V_m(\bm \alpha)$
by setting $\xi_t = \alpha_i$ if $t \in [s_i,s_{i+1})$ for each $i \in [k+1]$ (with the convention that $\alpha_0 = \bm 0$)
and $\xi_1 = \alpha_{k}$.
Note 3that $V_m$ is continuous, and hence $\mathcal R^k/\mathcal D_p$  measurable.
Besides, we have shown that $\pi_{(t_1,\cdots,t_k)}$ is $\sigma[\pi_{\bm t}:\bm t \in \mathcal T]/\mathcal R^k$ measurable.
As a result, the composition $V^*_m \delequal V_m \pi_{(s_1,\cdots,s_k)}: \D \to \D$ is $\sigma[\pi_{\bm t}:\bm t \in \mathcal T]/\mathcal D_p$ measurable.

To proceed, fix some $\epsilon > 0$.
For any $x \in \D$, define $x^\prime \in \D$ such that $x^\prime_t = x_t$ for all $t \in [\epsilon,1 -\epsilon)$ and $x^\prime_t = \bm 0$ otherwise.
The boundedness of any path in $\D$ implies the existence of some $M_x \in (0,\infty)$ such that $\sup_{t \in [0,1]}\norm{x_t} \leq M_x$.
Next, note that
\begin{align*}
    \dlp{}( V^*_m x, x)
    \leq \underbrace{\dlp{}( V^*_m x^\prime, x^\prime)}_{ (\text{I})}
    +
    \underbrace{\dlp{}( V^*_m x^\prime, V^*_m x)}_{ (\text{II})}
    +
    \underbrace{\dlp{}(x^\prime,x)}_{ (\text{III})}.
\end{align*}
First, it was shown in Theorem 12.5 of \cite{billingsley2013convergence} that
$
\lim_{m \to \infty} \bm d_{J_1}(V^*_m x^\prime,x^\prime) = 0.
$
This immediately implies that
$
\lim_{m\to \infty} \dlp{}(V^*_m x^\prime,x^\prime) = 0.
$
Next, by definition of $x^\prime$, we have $\limsup_{m \to \infty}\big[\text{(II)}\big]^p \leq (2M_x)^p\cdot 2 \epsilon$ and $\limsup_{m \to \infty}\big[\text{(III)}\big]^p \leq (2M_x)^p\cdot 2 \epsilon$.
Driving $\epsilon \downarrow 0$, 
we obtain that $\lim_{m \to \infty} \dlp{}(V^*_m x,x) = 0$ for all $x \in \D$.
This implies that the identity mapping $\textbf{I}(\xi) = \xi$ is also $\sigma[\pi_{\bm t}:\bm t \in \mathcal T]/\mathcal D_p$ measurable,
which leads to $\mathcal D_p \subseteq \sigma[\pi_{\bm t}:\bm t \in \mathcal T]$ and concludes the proof.
\end{proof}

Next, consider a family of $\R^d$-valued càdlàg processes
$\hat Y^{\eta,\epsilon}_t$, supported on the same underlying probability space with process $Y^\eta_t$,
that satisfies the following condition.

    

\begin{condition}\label{condition, metastability, fdd}
  For each $T \in (0,\infty)$ and $p \in [1,\infty)$,
the following claims hold for all $\epsilon > 0$ small enough:
\begin{enumerate}[(i)]
\item $\{\hat Y^{\eta,\epsilon}_t:\ t > 0\} \tofdd \{Y^*_t:\ t > 0\}$ and  $\hat Y^{\eta,\epsilon}_{ \boldsymbol{\cdot} }\Rightarrow Y^*_{ \boldsymbol{\cdot}}$ 
    in $(\D[0,T],\dlp{[0,T]})$
    as $\eta \downarrow 0$;

    \item $\lim_{\eta \to 0}\P\Big(\norm{\hat Y^{\eta,\epsilon}_{T} - Y^\eta_{T}} \geq \epsilon
    \Big) = 0$ and
    $\lim_{\eta \downarrow 0}\P\Big(
    \dlp{[0,T]}\big(\hat Y^{\eta,\epsilon}_{\boldsymbol{\cdot}}, Y^\eta_{\boldsymbol{\cdot}}) \geq 2\epsilon
    \Big) = 0$.
\end{enumerate}  
\end{condition}

Lemma \ref{lemma: metastability, abstract framework} shows that,
under Condition \ref{condition, metastability, fdd}, 
 both $Y^\eta_t$ and $\hat Y^{\eta,\epsilon}_t$ admit the same limit $Y^*_t$.

\begin{lemma}\label{lemma: metastability, abstract framework}
\linksinthm{lemma: metastability, abstract framework}
If Condition \ref{condition, metastability, fdd} holds, then
$
\{Y^\eta_t:\ t > 0\}\tofdd \{Y^*_t:\ t > 0\}
$
and, for any $T > 0$, 
$Y^\eta_{ \boldsymbol{\cdot} }\Rightarrow Y^*_{ \boldsymbol{\cdot} }$ in $(\D[0,T],\dlp{[0,T]})$ as $\eta \downarrow 0$.
\end{lemma}

\begin{proof}
    \linksinpf{lemma: metastability, abstract framework}
We start with the $L_p$ convergence.
By Portmanteau Theorem, it suffices to show that
$
\liminf_{\eta \downarrow 0}\P( Y^\eta_{\boldsymbol{\cdot}} \in G)
\geq 
\P( Y^*_{\boldsymbol{\cdot}} \in G)
$
for any open set $G$ in the $L_p$ topology of $\mathbb D[0,T]$.
Next, (recall that $G_\epsilon$ is the $\epsilon$-shrinkage of $G$,
and $G_\epsilon$ is also an open set)
\begin{align*}
    \P( Y^\eta_{\boldsymbol{\cdot}} \in G)
    & \geq 
    \P( Y^\eta_{\boldsymbol{\cdot}} \in G,\ \dlp{[0,T]}(\hat{Y}^{\eta,\epsilon}_{\boldsymbol{\cdot}},Y^\eta_{\boldsymbol{\cdot}})  < 2\epsilon)
    \geq 
    \P( \hat Y^{\eta,\epsilon}_{\boldsymbol{\cdot}} \in G_{2\epsilon},\ \dlp{[0,T]}(\hat{Y}^{\eta,\epsilon}_{\boldsymbol{\cdot}},Y^\eta_{\boldsymbol{\cdot}})  < 2\epsilon)
    \\ 
    & \geq  \P( \hat Y^{\eta,\epsilon}_{\boldsymbol{\cdot}} \in G_{2\epsilon}) - 
    \P(\dlp{[0,T]}(\hat{Y}^{\eta,\epsilon}_{\boldsymbol{\cdot}},Y^\eta_{\boldsymbol{\cdot}})  \geq 2\epsilon).
\end{align*}
For small enough $\epsilon> 0$, 
using  part $(i)$ of Condition \ref{condition, metastability, fdd} we get
$
\liminf_{\eta\downarrow 0}\P( \hat Y^{\eta,\epsilon}_{\boldsymbol{\cdot}} \in G_{2\epsilon})
\geq 
\P(Y^*_{\boldsymbol{\cdot}} \in G_{2\epsilon}),
$
and by  part $(ii)$ of Condition \ref{condition, metastability, fdd} we have
$\lim_{\eta \downarrow 0}\P(\dlp{[0,T]}(\hat{Y}^{\eta,\epsilon}_{\boldsymbol{\cdot}},Y^\eta_{\boldsymbol{\cdot}})  \geq 2\epsilon) = 0$.
Therefore,
$
\liminf_{\eta \downarrow 0}\P( Y^\eta_{\boldsymbol{\cdot}} \in G)
\geq 
\P(Y^*_{\boldsymbol{\cdot}} \in G_{ 2\epsilon }).
$
Driving $\epsilon\downarrow 0$, we conclude the proof for the $L_p$ convergence.
The proof for the f.d.d.\ convergence is almost identical and hence we omit the details.
\elaborate{Fix some $k \geq 1$ and $0 < t_1 < t_2 < \cdots < t_k$.
Equip $\R^k$ with the uniform metric 
$\bm d^{(k)}\big((x_1,\cdots,x_k),(y_1,\cdots,y_k)\big) = \sup_{i \in [k]}|x_i - y_i|$.
By Portmanteau theorem, it suffices to show that 
$
\liminf_{\eta \downarrow 0}\P\Big( (Y^\eta_{t_1},\cdots,Y^\eta_{t_k}) \in G\Big)
\geq 
\P\Big( (Y^*_{t_1},\cdots,Y^*_{t_k}) \in G\Big)
$
for any given open set $G$. 
Recall that $G_\epsilon$ is the $\epsilon$-shrinkage of $G$,
and $G_\epsilon$ is also an open set.
For 
$
\hat{\bm Y}^{\eta,\epsilon} = (\hat Y^{\eta,\epsilon}_{t_1},\cdots, \hat Y^{\eta,\epsilon}_{t_k})
$
and
$
\bm Y^\eta  = (Y^\eta_{t_1},\cdots,Y^\eta_{t_k})
$,
note that
\begin{align*}
    \P\Big( (Y^\eta_{t_1},\cdots,Y^\eta_{t_k}) \in G\Big)
    & \geq 
    \P\Big( (Y^\eta_{t_1},\cdots,Y^\eta_{t_k}) \in G,\ \bm d^{(k)}(\hat{\bm Y}^{\eta,\epsilon},\bm Y^\eta)  < \epsilon\Big)
    \\
    & \geq 
    \P\Big( (\hat Y^{\eta,\epsilon}_{t_1},\cdots, \hat Y^{\eta,\epsilon}_{t_k}) \in G_\epsilon,\ \bm d^{(k)}(\hat{\bm Y}^{\eta,\epsilon},\bm Y^\eta)  <\epsilon\Big)
    \\ 
    & \geq  \P\Big( (\hat Y^{\eta,\epsilon}_{t_1},\cdots, \hat Y^{\eta,\epsilon}_{t_k}) \in G_\epsilon\Big) - \P\Big(\bm d^{(k)}(\hat{\bm Y}^{\eta,\epsilon},\bm Y^\eta)  \geq \epsilon \Big).
\end{align*}
For small enough $\epsilon> 0$, 
using part $(i)$ of Condition \ref{condition, metastability, fdd} we get
$
\liminf_{\eta\downarrow 0}\P\Big( (\hat Y^{\eta,\epsilon}_{t_1},\cdots, \hat Y^{\eta,\epsilon}_{t_k}) \in G_\epsilon\Big)
\geq 
\P\Big((Y^*_{t_1},\cdots,Y^*_{t_k}) \in G_\epsilon\Big),
$
and
using part $(ii)$ of Condition \ref{condition, metastability, fdd},
we get
$\lim_{\eta \downarrow 0}\P\Big(\bm d^{(k)}(\hat{\bm Y}^{\eta,\epsilon},\bm Y^\eta)  \geq \epsilon \Big) = 0$.
\chr{Is this right?}\xw{Fixed. Now I'm using uniform distance}%
Therefore,
$
\liminf_{\eta \downarrow 0}\P\Big( (Y^\eta_{t_1},\cdots,Y^\eta_{t_k}) \in G\Big)
\geq 
\P\Big((Y^*_{t_1},\cdots,Y^*_{t_k}) \in G_{ \epsilon }\Big).
$
Taking $\epsilon\downarrow 0$, we conclude the for the convergence in f.d.d..}
\end{proof}

In light of Lemma \ref{lemma: metastability, abstract framework},
a natural approach to  Theorem \ref{corollary irreducible case}
is to identify some 
$\hat Y^{\eta,\epsilon}_t$
that converges to $Y^{*|b}_t$ while staying close enough to $X^{\eta|b}_{\floor{ t/\lambda^*_b(\eta) }}(x)$.
To this end,
we introduce the next key component of our framework,
i.e., a technical tool for establishing the weak convergence of jump processes.
Inspired by the approach in \cite{pavlyukevich2008metastable}, Lemma \ref{lemma weak convergence of jump process} shows that the convergence of jump processes can be established by verifying the convergence of the inter-arrival times and destinations of jumps.
Specifically, we introduce the following mapping $\Phi$ for constructing jump processes.

\begin{definition} \label{def jump process}
Let random elements $((U_j)_{j \geq 1},(V_j)_{j \geq 1})$ be such that $V_j \in \R^d\ \forall j \geq 1$
and
\begin{align}
    U_j\in[0,\infty)\quad \forall j \geq 1,
    \qquad 
    \lim_{i \to \infty}\P\bigg(\sum_{j=1}^i U_j > t\bigg) = 1\quad \forall t > 0.
    \label{def, mapping Phi for jump process, condition on interarrival times}
\end{align}
Let mapping $\Phi(\cdot)$ be defined as follows: the image $Y_{\boldsymbol{\cdot}} = \Phi\big((U_j)_{j \geq 1},(V_j)_{j \geq 1}\big)$
is a stochastic process taking values in $\mathbb R$ such that (under the convention $V_0 \equiv 0$)
\begin{align}
    Y_t = V_{\mathcal J(t)}\ \forall t \geq 0
    \qquad \text{where}\qquad
    \mathcal{J}(t) \delequal \max\{J \geq 0:\ \sum_{j = 1}^J U_j \leq t\}.
    \label{def, mapping Phi for jump process}
\end{align}
\end{definition}

\begin{remark}
    We add two remarks regarding Definition \ref{def jump process}.
First, $(U_j)_{j \geq 1}$ and $(V_j)_{j \geq 1}$ 
can be viewed as the inter-arrival times and destinations of jumps in $Y_t$, respectively. 
It is worth noticing that we allow for {instantaneous} jumps, i.e., $U_j = 0$. 
Nevertheless, the condition 
$
\lim_{i \to \infty}\P(\sum_{j =1}^i U_j > t) = 1\ \forall t > 0
$
prevents the concentration of infinitely many instantaneous jumps before any finite time $t \in (0,\infty)$,
thus ensuring that the process $Y_t =  V_{\mathcal J(t)}$ is almost surely well defined.
In case that $U_j > 0\ \forall j \geq 1$, the process $Y_t$ admits a more standard expression
and satisfies $Y_t = V_i$
for all
$
t \in [\sum_{j = 1}^{i}U_j,\sum_{j = 1}^{i+1}U_j).
$
Second, to account for the scenario where the process $Y_t$ stays constant after a (possibly random) timestamp $T$, one can introduce {dummy} jumps that keep landing at the same location.
For instance, suppose that after hitting the state $w \in \R^d$, the process $Y_t$ is absorbed at $w$,
then a representation compatible with Definition \ref{def jump process} is that, conditioning on $V_j = w$,
we set $U_k$ as iid Exp$(1)$ RVs and $V_k \equiv w$ for all $k \geq j+1$.
\end{remark}


As mentioned above, Lemma~\ref{lemma weak convergence of jump process} states that the convergence of jump processes in f.d.d.\ follows from the convergence in distributions of the inter-arrival times and destinations of jumps.

\begin{lemma} \label{lemma weak convergence of jump process}
\linksinthm{lemma weak convergence of jump process}
Let mapping $\Phi$ be specified as in Definition \ref{def jump process}.
Let $Y_{\boldsymbol{\cdot}} = \Phi\big( (U_j)_{j \geq 1},(V_j)_{j \geq 1} \big)$ and, for each $n \geq 1$,
$Y^n_{\boldsymbol{\cdot}} = \Phi\big( (U^n_j)_{j \geq 1},(V^n_j)_{j \geq 1} \big)$.
Suppose that
\begin{enumerate}[(i)]
    \item $(U^n_1,V^n_1,U^n_2,V^n_2,\cdots)$ converges in distribution to $(U_1,V_1,U_2,V_2,\cdots)$ as $n\rightarrow \infty$;
    
    \item For any $u > 0$ and any $j \geq 1$, $\P(U_1 + \cdots + U_j = u) = 0$;
    
    \item For any $u > 0$, $\lim_{j \rightarrow \infty}\P(U_1 + U_2 + \cdots U_j > u) = 1$.
\end{enumerate}
Then
$\{Y^{n}_t:\ t > 0\}\tofdd \{Y^*_t:\ t > 0\}$
as $n \to \infty$.
\end{lemma}

\begin{proof}
\linksinpf{lemma weak convergence of jump process}
Fix some $k \in \mathbb{N}$ and $0 < t_1 < t_2 < \cdots < t_k < \infty$. Set $t = t_k$.
Pick some $\epsilon > 0$.
By conditions $(i)$ and $(iii)$, one can find some $J(\epsilon) > 0$ such that $\P(\sum_{j = 1}^{J(\epsilon)} U_j \leq t) < \epsilon$,
and hence
$
\P(\sum_{j = 1}^{J(\epsilon)} U^n_j \leq t) < \epsilon
$
for all $n$ large enough.
Also, by condition $(ii)$, we can fix $\Delta(\epsilon) > 0$ such that
$
\P\Big(\sum_{i = 1}^j U_i \in \bigcup_{l \in [k]}[t_l - \Delta(\epsilon),t_l + \Delta(\epsilon)]
\text{ for some }j \leq J(\epsilon)\Big) < \epsilon.
$
Throughout the proof, we may abuse the notation slightly and write $J = J(\epsilon)$ and $\Delta = \Delta(\epsilon)$ when there is no ambiguity.

For any probability measure $\mu$, let $\mathscr L_\mu(X)$ be the law of the random element $X$ under $\mu$.
Applying Skorokhod's representation theorem, we can construct a probability space $(\widetilde \Omega,\mathcal{\widetilde F},\mathbf Q)$ that supports random elements $(\widetilde{U}^n_1,\widetilde{V}^n_1,\widetilde{U}^n_2,\widetilde{V}^n_2\cdots)_{n \geq 1}$ and $(\widetilde{U}_1,\widetilde{V}_1,\widetilde{U}_2,\widetilde{V}_2,\cdots)$ such that:
(1)  $\mathscr{L}_{\P}(U^n_1,V^n_1,U^n_2,V^n_2,\cdots) = \mathscr{L}_{\mathbf Q}(\widetilde{U}^n_1,\widetilde{V}^n_1,\widetilde{U}^n_2,\widetilde{V}^n_2\cdots)$ for all $n \geq 1$,
(2) $\mathscr{L}_{\P}(U_1,V_1,U_2,V_2,\cdots) = \mathscr{L}_{\mathbf Q}(\widetilde{U}_1,\widetilde{V}_1,\widetilde{U}_2,\widetilde{V}_2,\cdots)$,
and (3)
$\widetilde U^n_j \xrightarrow{\mathbf Q-a.s.} \widetilde U_j$ and $\widetilde V^n_j \xrightarrow{\mathbf Q-a.s.} \widetilde V_j$ as $n \rightarrow \infty$ for all $j \geq 1$.
This allows us to construct a coupling between processes $Y_t$ and $Y^n_t$ 
on $(\widetilde \Omega,\mathcal{\widetilde F},\mathbf Q)$
by setting $Y=\Phi\Big( (\widetilde U_j)_{j \geq 1},(\widetilde V_j)_{j \geq 1} \Big)$ 
and (for each $n \geq 1$)
$Y^n = \Phi\Big( (\widetilde U^n_j)_{j \geq 1},(\widetilde V^n_j)_{j \geq 1} \Big)$.
Next, 
for each $i \in [k]$, we define
\begin{align*}
   \mathcal{I}^{\leftarrow}_i(\Delta) = \max\{ j \geq 0:\ \widetilde U_1 + \cdots \widetilde U_j \leq t_i -\Delta  \},
   \qquad 
   \mathcal{I}^{\rightarrow}_i(\Delta) = \min\{ j \geq 0:\ \widetilde U_1 + \cdots \widetilde U_j \geq t_i + \Delta  \}.
\end{align*}
That is, $\mathcal{I}^{\leftarrow}_i(\Delta)$ is the index of the last jump in $Y_s$ before time $t_i - \Delta$,
and $\mathcal{I}^{\rightarrow}_i(\Delta)$ is the index of the first jump after time $t_i + \Delta$.
Recall that we have fixed $0< t_1 < \cdots < t_k = t < \infty$.
On the event 
$$A_n = \Big\{\sum_{i = 1}^j \widetilde U_i \notin \bigcup_{l \in [k]}[t_l - \Delta,t_l + \Delta]\ \forall j \leq J\Big\} \cap \Big\{ \sum_{j = 1}^J \widetilde U_j > t,\ \sum_{j = 1}^J \widetilde U^n_j > t \Big\},$$
we have $\mathcal{I}^{\rightarrow}_i(\Delta) = \mathcal{I}^{\leftarrow}_i(\Delta) + 1 \leq J$ for all $i \in [k]$.
Then, on $A_n$ it holds $\Q$-a.s.\ that (for all $i \in [k]$)
\begin{align*}
   & \lim_{n \to \infty}\sum_{j = 1}^{ \mathcal I^\leftarrow_i(\Delta) }\widetilde U^n_j = \sum_{j = 1}^{ \mathcal I^\leftarrow_i(\Delta) }\widetilde U_j < t_i - \Delta,
    \qquad
    \lim_{n \to \infty}\sum_{j = 1}^{ \mathcal I^\leftarrow_i(\Delta) + 1 }\widetilde U^n_j = \sum_{j = 1}^{ \mathcal I^\leftarrow_i(\Delta) + 1 }\widetilde U_j > t_i + \Delta,
\end{align*}
As a result, on $A_n$ it holds
for all $n$ large enough that
$
\sum_{j = 1}^{ \mathcal I^\leftarrow_i(\Delta)}\widetilde U^n_j < t_i
$
and
$
\sum_{j = 1}^{ \mathcal I^\leftarrow_i(\Delta) + 1}\widetilde U^n_j > t_i
$
for all $i \in [k]$,
implying that 
$
Y^n_{t_i} = \widetilde V^n_{ \mathcal I^\leftarrow_i(\Delta) } \ \forall i \in [k].
$
Furthermore,
due to $\widetilde V^n_j \to \widetilde V_j$ $\Q$-a.s. for all $j \leq J$,
it holds $\Q$-a.s. that 
$
\lim_{n \to \infty}\norm{\widetilde V^n_{ I^\leftarrow_i(\Delta) } - \widetilde V_{ I^\leftarrow_i(\Delta) }}
\leq 
\lim_{n \to \infty}\max_{j \leq J}\norm{\widetilde V^n_j - \widetilde V_j } = 0.
$
Therefore, on $A_n$ it holds $\Q$-a.s.\ that $\lim_{n\to\infty}Y^{n}_{t_i} = \lim_{n\to\infty} \widetilde V^n_{ I^\leftarrow_i(\Delta) } = \widetilde V_{ I^\leftarrow_i(\Delta) } = Y_{t_i}$ for all $i \in [k]$.
Then, 
for any $g:\R^{d \times k} \to \R$ that is bounded and continuous,
note that (let $\bm Y^n = (Y^n_{t_1},\cdots,Y^n_{t_k})$, $\bm Y = (Y_{t_1},\cdots,Y_{t_k})$,
and $\norm{g} = \sup_{ \bm y \in \R^{d\times k} }|g(\bm y)|$)
\begin{align*}
    \limsup_{n \to \infty}\Big|
    \E g(\bm Y^n) - \E g(\bm Y)
    \Big|
    &
    \leq
    \limsup_{n \to \infty}\E_{\Q}\Big| g(\bm Y^n) - g(\bm Y)  \Big|
    \\ 
    & =
    \limsup_{n\to\infty}\E_{\Q}\Big| g(\bm Y^n) - g(\bm Y)  \Big|\mathbf{I}_{ A_n } 
    +
     \limsup_{n\to\infty}\E_{\Q}\Big| g(\bm Y^n) - g(\bm Y)  \Big|\mathbf{I}_{ (A_n)^\complement } 
     \\ 
     & \leq 0 + 
     2\norm{g}\limsup_{n\to\infty}\Q\big((A_n)^\complement\big)
     \qquad 
     \text{due to }\bm Y^n\xrightarrow{\Q-a.s.}\bm Y\text{ on }A_n
     \\ 
     & \leq 
     2\norm{g}\cdot \bigg(
     \limsup_{n \to \infty}\Q(\sum_{i = 1}^J \widetilde U_j \leq t) + 
    \limsup_{n \to \infty}\Q(\sum_{i = 1}^J \widetilde U^n_j \leq t)
    \\ 
    &\qquad\qquad 
    +
    \limsup_{n\to\infty}
    \Q\Big(\sum_{i = 1}^j \widetilde U_i \in \bigcup_{l \in [k]}[t_l - \Delta,t_l + \Delta]
\text{ for some }j \leq J\Big)
     \bigg)
     \\ 
     & \leq 6\norm{g} \cdot \epsilon.
\end{align*}
The last inequality follows from our choice of $J = J(\epsilon)$ and $\Delta = \Delta(\epsilon)$ at the beginning.
Given the arbitrariness of the mapping $g$ and $\epsilon > 0$, we conclude the proof using Portmanteau theorem.
\end{proof}

\section{Proof of Theorem~\ref{corollary irreducible case} and Corollary~\ref{theorem: metastability, unclipped}}
\label{subsec: proof, theorems, metastability}

In this section, we apply the framework developed in Section~\ref{subsec: proof, abstract framework, metastability} to prove Theorem~\ref{corollary irreducible case} and Corollary~\ref{theorem: metastability, unclipped}.
In particular,
the verification of 
{part $(i)$ of Condition \ref{condition, metastability, fdd}}
hinges on the choice of the approximator $\hat Y^{\eta,\epsilon}_t$.
Here, we construct a process $\hat{\bm X}^{\eta,\epsilon|b}_t(\bm x)$ as follows.
Let $\hat \tau^{\eta,\epsilon|b}_0(\bm x) \triangleq 0$, 
\begin{align}
    \hat \tau^{\eta,\epsilon|b}_1(\bm x) & \delequal \min\Big\{ t \geq 0:\ \bm X^{\eta|b}_t(\bm x) \in \bigcup_{i \in [K]}B_\epsilon(\bm m_i) \Big\},
    \label{def: hat tau, 1, metastability}
\end{align}
    and (to lighten notations, we write $ \bm X^{\eta|b}_{\hat \tau_k}(\bm x) \delequal \bm X^{\eta|b}_{\hat \tau^{\eta,\epsilon|b}_k(\bm x)}(\bm x) $ )
\begin{align}
    \notationdef{notation-transition-marker-metastability}{\hat{\mathcal I}^{\eta,\epsilon|b}_1(x)} \triangleq i \iff \bm X^{\eta|b}_{\hat \tau_1}(\bm x) \in I_i.
    \label{def: hat I, metastability}
\end{align}
For $k \geq 2$, let
\begin{align}
    \notationdef{notation-transition-time-metastability}{\hat \tau^{\eta,\epsilon|b}_k(\bm x)} & \delequal \min\Big\{ t \geq \hat \tau^{\eta,\epsilon|b}_{k-1}(\bm x):\ 
    \bm X^{\eta|b}_t(\bm x) \in \bigcup_{i \neq \hat{\mathcal I}^{\eta,\epsilon|b}_{k-1}(\bm x)}
    B_\epsilon(\bm m_i)
    \Big\}\qquad \forall k \geq 2.
    \label{def: hat tau, 2, metastability}
\end{align}
and 
\begin{align}
    \notationdef{notation-transition-marker-metastability}{\hat{\mathcal I}^{\eta,\epsilon|b}_k(\bm x)} \triangleq i \iff \bm X^{\eta|b}_{\hat \tau_k}(\bm x) \in I_i.
    \label{def: hat I k, metastability}
\end{align}
Intuitively speaking,
$\hat \tau^{\eta,\epsilon|b}_k(\bm x)$ records the time $\bm X^{\eta|b}_t(\bm x)$ makes the $k$-th  transitions across the attraction fields over $f$ and visits (the $\epsilon$-neighborhood of) a local minimum,
and
$
\hat{\mathcal I}^{\eta,\epsilon|b}_k(\bm x)
$
denotes the index of the visited local minimum.
Let 
\begin{align}
    \notationdef{notation-marker-process-metastability-hat-X}{\hat{\bm X}^{\eta,\epsilon|b}_{ \boldsymbol{\cdot} }(\bm x)} \delequal 
\Phi\Big( 
    \Big( \big({\hat \tau^{\eta,\epsilon|b}_k(\bm x) - \hat \tau^{\eta,\epsilon|b}_{k-1}(\bm x)}\big) \cdot { \lambda^*_b(\eta) } \Big)_{k \geq 1},
    \big( \bm m_{ \hat{\mathcal I}^{\eta,\epsilon|b}_k(\bm x) } \big)_{k \geq 1} \Big).
    \label{def: marker process, hat X eta epsilon b}
\end{align}
By definition, $\hat{\bm X}^{\eta,\epsilon|b}_t(\bm x)$ keeps track of how $\bm X^{\eta|b}_t(\bm x)$ makes transitions between  the different local minima over $f$, under a time scaling $\lambda^*_b(\eta)$ in \eqref{def scale function lambda * b eta}.

Using Lemma \ref{lemma weak convergence of jump process},
the convergence of $\hat{\bm X}^{\eta,\epsilon|b}_{\boldsymbol{\cdot}}(\bm x)$
follows directly from the 
convergence of  ${\hat \tau^{\eta,\epsilon|b}_k(\bm x) - \hat \tau^{\eta,\epsilon|b}_{k-1}(\bm x)}$
and $\bm m_{ \hat{\mathcal I}^{\eta,\epsilon|b}_k(\bm x) }$,
i.e., the inter-arrival times and destinations
of the transitions in $\bm X^{\eta|b}_t(\bm x)$ between different attraction fields over the multimodal potential $f$.
This is exactly the content of the first exit time analysis.
In particular,
based on a straightforward adaptation of the first exit time analysis in \cite{wang2023large} (see Section~\ref{sec: review, first exit analyses} for details) to the current setting,
we obtain Proposition \ref{proposition: hat X converges to CTMC, metastability}.

\begin{proposition}
\label{proposition: hat X converges to CTMC, metastability}
\linksinthm{proposition: hat X converges to CTMC, metastability}
Under the conditions in Theorem~\ref{corollary irreducible case},
the following claims hold for any $\epsilon > 0$ small enough:
\begin{enumerate}[(i)]

    \item
$
\{\hat{\bm X}^{\eta,\epsilon|b}_t(\bm x_0):\ t > 0\}\tofdd \{\bm Y^{*|b}_t:\ t > 0\}
$
as $\eta \downarrow 0$;

\item Given any $T \in (0,\infty)$, $p \in [1,\infty)$, and any sequence of strictly positive reals $\eta_n$'s with $\lim_{n \to \infty}\eta_n = 0$,
the laws of $\hat{\bm X}^{\eta_n,\epsilon|b}_{\boldsymbol{\cdot}}(\bm x_0)$ are tight in $(\D[0,T],\dlp{[0,T]})$.
\end{enumerate}

\end{proposition}

Proposition \ref{proposition: hat X close to X, metastability} serves to verify part $(ii)$ of Condition \ref{condition, metastability, fdd}
in Lemma~\ref{lemma: metastability, abstract framework},
under the choice of $Y^\eta_t = \bm X^{\eta|b}_{ \floor{t/\lambda^*_b(\eta)} }(\bm x_0)$ and $\hat Y^{\eta,\epsilon}_t = \hat{\bm X}^{\eta,\epsilon|b}_t(\bm x_0)$.

\begin{proposition}
\label{proposition: hat X close to X, metastability}
\linksinthm{proposition: hat X close to X, metastability}
Let $T > 0$ and $p\in[1,\infty)$.
Under the conditions in Theorem~\ref{corollary irreducible case},
it holds for any $\epsilon > 0$ small enough that
$$
\lim_{\eta\downarrow 0}\P\bigg( \dlp{[0,T]}\Big(\bm X^{\eta|b}_{ \floor{ {\boldsymbol{\cdot}}/\lambda^*_b(\eta) } }(\bm x_0),\hat{\bm X}^{\eta,\epsilon|b}_{\boldsymbol{\cdot}}(\bm x_0)\Big) \geq 2\epsilon\bigg) = 0,
\qquad 
\lim_{\eta\downarrow 0}\P\Big(\norm{ \bm X^{\eta|b}_{  \floor{T/\lambda^*_b(\eta)} }(\bm x_0) - \hat{\bm X}^{\eta,\epsilon|b}_{T}(\bm x_0)} \geq \epsilon\Big) = 0.
$$
\end{proposition}

We defer the proofs of the two propositions to
Section~\ref{subsec: proof, propositions, metastability}.
Here, we apply these tools to establish Theorem~\ref{corollary irreducible case}.

\begin{proof}[Proof of Theorem~\ref{corollary irreducible case}]
\linksinpf{corollary irreducible case}
From Lemma \ref{proposition, fdd + tightness = Lp convergence} and Proposition \ref{proposition: hat X converges to CTMC, metastability},
we verify part $(i)$ of Condition \ref{condition, metastability, fdd}, i.e., given any $T>0$, the claim
\begin{align*}
\{\hat{\bm X}^{\eta,\epsilon|b}_t(\bm x_0):\ t > 0\} \tofdd \{\bm Y^{*|b}_t:\ t > 0\}\quad\text{and}\quad
   \hat{\bm X}^{\eta,\epsilon|b}_{\boldsymbol{\cdot} }(\bm x_0) \Rightarrow  \bm Y^{*|b}_{\boldsymbol{\cdot} }
   \text{ in $(\D[0,T],\dlp{[0,T]})$ as $\eta \downarrow 0$}
\end{align*}
holds for all $\epsilon > 0$ small enough.
Meanwhile, 
given any $T \in (0,\infty)$ and $p \in [1,\infty)$,
Proposition \ref{proposition: hat X close to X, metastability}
verifies part $(ii)$ of Condition \ref{condition, metastability, fdd}
under the choice of  $Y^\eta_t = \bm X^{\eta|b}_{ \floor{t/\lambda^*_b(\eta)} }(\bm x_0)$, $\hat Y^{\eta,\epsilon}_t = \hat{\bm X}^{\eta,\epsilon|b}_t(\bm x_0)$, and $Y^*_t = \bm Y^{*|b}_{t }$.
Applying Lemma~\ref{lemma: metastability, abstract framework}, we obtain that (for any $T \in (0,\infty)$ and $p \in [1,\infty)$)
\begin{align*}
\{\bm X^{\eta|b}_{ \floor{ t/\lambda^*_b(\eta)} }(\bm x_0):\ t > 0\}
\tofdd 
\{\bm Y^{*|b}_t:\ t > 0\}
\quad\text{and}\quad
   \bm X^{\eta|b}_{ \floor{ \boldsymbol{\cdot} /\lambda^*_b(\eta)} }(\bm x_0) \Rightarrow  \bm Y^{*|b}_{\boldsymbol{\cdot}}
   \text{ in $(\D[0,T],\dlp{[0,T]})$}
\end{align*}
as $\eta \downarrow 0$.
This allows us to conclude the proof using Lemma~\ref{lemma: Lp convergence, from T to infty}.
\end{proof}

Next, we show that Corollary~\ref{theorem: metastability, unclipped} follows directly from Theorem~\ref{corollary irreducible case}.

\begin{proof}[Proof of Corollary~\ref{theorem: metastability, unclipped}]
\linksinpf{theorem: metastability, unclipped}
Recall our convention of $\widecheck{g}^{(0)|b}(\bm x) = \bm x$ in \eqref{def: mapping check g k b, endpoint of path after the last jump, first exit analysis}.
By definitions in \eqref{def: set G k b},
we have 
$
\mathcal G^{(1)|b}(\bm m_i) = B_b(\bm m_i)
$
(i.e., the open ball centered at $\bm m_i$ with radius $b$).
Then according to \eqref{def: J * b i,j, metastability},
under all $b > 0$ large enough we would always have
$\mathcal J_b(i) = 1$ and $\mathcal G^{ ( \mathcal J_b(i) )|b  }(\bm m_i) \cap I_j \neq \emptyset$ for any $i,j\in [K]$ with $i \neq j$.
Therefore, under such large $b$,
 any edge $(\bm m_i \to \bm m_j)$ would always belong to $E_b$ of the typical transition graph (see Definition~\ref{def main paper typical transition graph}),
 and we have $\lambda^*_b(\eta) = \eta \cdot \lambda(\eta) = H(\eta^{-1})$ (see \eqref{def scale function lambda * b eta})
 and $\mathcal J^*_b = 1,\ V^*_b = \{\bm m_j: j \in [K]\}$ (see \eqref{def: J * b} and \eqref{def: V * b, metastability}).
As an immediate consequence, in \eqref{def: absorption prob, theta b i j} we have $\theta_b(\bm m_i|\bm m_i) = 1$ for any $i \in [K]$;
then in \eqref{def: generator of Y * b, 1}--\eqref{def: generator of Y * b, 2},
the infinitesimal generator of $\bm Y^{*|b}_t$ is now equal to 
\begin{align*}
    Q^{*|b}(i,j) & = q_b(i,j)\ \forall 
    \bm m_i,\ \bm m_j \in V\text{ with }\bm m_i \neq \bm m_j;
    \quad
    Q^{*|b}(i,i) = - \sum_{\bm m_j \in V:\ j \neq i}Q^{*|b}(i,j)\ \forall \bm m_i \in V.
\end{align*}
Henceforth in this proof, we only consider such large $b$.

We focus on the proof for the $L_p$ convergence on $\D[0,\infty)$, as the proof for convergence in f.d.d.\ is almost identical.
Furthermore, by Lemma~\ref{lemma: Lp convergence, from T to infty},
it suffices to prove the $L_p$ convergence on each $\D[0,T]$.
To proceed, we
pick some $T \in [0,\infty)$ and some closed set $A \subseteq \D[0,T]$ (w.r.t.\ $L_p$ topology).
Observe that
\begin{align}
    \P\Big( \bm X^{\eta}_{\floor{ {\boldsymbol{\cdot}}  /H(\eta^{-1}) }}(\bm x)  \in A  \Big)
    & =
    \P\Big(
    \bm X^{\eta}_{\floor{ {\boldsymbol{\cdot}}  /H(\eta^{-1}) }}(\bm x)  \in A;\ 
    \bm X^{\eta|b}_t(\bm x) = \bm X^\eta_t(\bm x)\ \forall t \leq \floor{ T/H(\eta^{-1}) } \Big)
    \label{proof, key ineq, theorem: metastability, unclipped}
    \\ 
    & + 
    \P\Big(
    \bm X^{\eta}_{\floor{ {\boldsymbol{\cdot}}  /H(\eta^{-1}) }}(\bm x)  \in A;\ 
    \bm X^{\eta|b}_t(\bm x) \neq \bm X^\eta_j(\bm x)\text{ for some } t \leq \floor{ T/H(\eta^{-1}) } \Big)
    \nonumber
    \\ 
    & \leq 
    \underbrace{\P\Big( \bm X^{\eta|b}_{\floor{ {\boldsymbol{\cdot}}  /H(\eta^{-1}) }}(\bm x)  \in A  \Big)}_{ \text{(I)} }
    +
    \underbrace{\P\Big(
    \bm X^{\eta|b}_t(\bm x) \neq \bm X^\eta_t(\bm x)\text{ for some } t \leq \floor{ T/H(\eta^{-1}) } \Big)}_{\text{(II)}}.
    \nonumber
\end{align}
For term (I), it follows from Theorem \ref{corollary irreducible case} that 
$
\limsup_{\eta \downarrow 0}\text{(I)} \leq \P\big( \bm Y^{*|b}_{ {\boldsymbol{\cdot}} } \in A \big).
$
For term (II), we make two observations.
First, 
recall that $C$ is the constant in Assumption \ref{assumption: boundedness of drift and diffusion coefficients}
such that $\sup_{x \in \R}\norm{\nabla f(\bm x)}\vee \norm{\bm \sigma(\bm x)} \leq C$.
Under any $\eta \in (0,\frac{b}{2C})$, on the event $\{\eta\norm{\bm Z_t} \leq \frac{b}{2C}\ \forall 
t\leq \floor{ T/H(\eta^{-1}) }\}$ the step-size (before truncation) $-\eta \nabla f\big(\bm X^{\eta|b}_{t - 1}(\bm x)\big) + \eta \bm \sigma\big(\bm X^{\eta|b}_{t - 1}(\bm x)\big)\bm Z_t$ of SGD is less than $b$ for each $t \leq \floor{ T/H(\eta^{-1}) }$. 
Therefore, $\bm X^{\eta|b}_t(\bm x)$ and $\bm X^\eta_t(\bm x)$ coincide for such $t$'s,
and for any $\eta \in (0,\frac{b}{2C})$, we have 
$
\{\eta\norm{\bm Z_t} \leq \frac{b}{2C}\ \forall t \leq \floor{ T/H(\eta^{-1}) }\}
\subseteq
\{\bm X^{\eta|b}_t(\bm x) = \bm X^{\eta}_t(\bm x)\ \forall t \leq \floor{ T/H(\eta^{-1}) }\}.
$
which leads to (recall that $H(\cdot) = \P(\norm{Z_1} > \cdot)$)
\begin{align*}
    \limsup_{\eta \downarrow 0}\text{(II)} 
    & \leq \limsup_{\eta \downarrow 0}\P\bigg(
    \exists t \leq \floor{ T/H(\eta^{-1}) }\ s.t.\ \eta\norm{\bm Z_t} > \frac{b}{2C}
    \bigg)
    \\
    & \leq 
    \limsup_{\eta \downarrow 0} \frac{T}{H(\eta^{-1})} \cdot H(\eta^{-1} \cdot \frac{b}{2C})
    = T \cdot \bigg( \frac{2C}{b}\bigg)^\alpha
    \qquad 
    \text{due to $H(x) \in \RV_{-\alpha}(x)$.}
\end{align*}
Now we have
$
\limsup_{\eta \downarrow 0}
 \P\big( \bm X^{\eta}_{\floor{ {\boldsymbol{\cdot}}  /H(\eta^{-1}) }}(\bm x)  \in A  \big)
\leq 
\P( \bm Y^{*|b}_{{\boldsymbol{\cdot}}} \in A )
+
T \cdot ( \frac{2C}{b})^\alpha.
$
Furthermore, 
for all $b$ large enough, due to $\mathcal J_b(i) = 1\ \forall i \in [K]$ (see the discussion at the beginning of the proof),
by definitions in \eqref{def: measure check C k b} and \eqref{def: measure check C} we have
\begin{align*}
    q(i,j) = \nu_\alpha\Big(\big\{ \bm w \in \R^d:\ \bm m_i +  \bm \sigma(\bm m_i) \bm w \in I_j \big\}\Big),
    \ \ 
    q_b(i,j) = \nu_\alpha\Big(\big\{ \bm w \in \R^d:\ \bm m_i + \varphi_b\big(\bm \sigma(\bm m_i) \bm w\big) \in I_j \big\}\Big),
\end{align*}
which implies $q_b(i,j) \to q(i,j)$ as $b \to \infty$.
To conclude, note that by the discussion at the beginning,
the infinitesimal generator (hence the law) of $\bm Y^{*|b}_t$ (the limiting Markov jump process in Theorem~\ref{corollary irreducible case})
converges to that of $\bm Y^*_t$ (the Markov jump process specified in Corollary~\ref{theorem: metastability, unclipped}).
Together with the fact that $\lim_{b \to \infty} \big( \frac{2C}{b}\big)^\alpha = 0$,
in \eqref{proof, key ineq, theorem: metastability, unclipped}
we obtain
$
\limsup_{\eta \downarrow 0}
 \P\big( \bm X^{\eta}_{\floor{ {\boldsymbol{\cdot}}  /H(\eta^{-1}) }}(\bm x)  \in A  \big)
\leq 
\P( \bm Y^{*}_{{\boldsymbol{\cdot}}}\in A ).
$
From the arbitrariness of the closed set $A$, we conclude the proof by Portmanteau theorem.
\end{proof}

\section{Proof of Propositions \ref{proposition: hat X converges to CTMC, metastability} and \ref{proposition: hat X close to X, metastability}}
\label{subsec: proof, propositions, metastability}

This section is devoted to proving Propositions \ref{proposition: hat X converges to CTMC, metastability} and \ref{proposition: hat X close to X, metastability}.
Henceforth in Section~\ref{subsec: proof, propositions, metastability}, we fix some  $b \in (0,\infty)$ be such that Assumption~\ref{assumption: regularity condition on b} holds. In particular,
\begin{align}
       \text{$(I_i)^\complement$ is bounded away from $\mathcal G^{ (\mathcal J_b(i) - 1)|b }(\bm m_i)$}
       \qquad\forall i \in [K].
    \label{condition on b, relative radius is not integer, metastability}
\end{align}
This allows us to fix some $\bar \epsilon \in (0,1\wedge b)$ small enough such that
\begin{align}
    (I_i)^\complement \cap \big(  \mathcal G^{ (\mathcal J_b(i) - 1)|b }(\bm m_i) \big)^{\bar\epsilon} = \emptyset,
    \text{ and }\bar B_{\bar\epsilon}(\bm m_i) \subseteq (I_i)_{\bar\epsilon}
    \qquad \forall i \in [K].
    \label{choice of bar epsilon, metastability}
\end{align}

To lighten notations in the subsequent analyses, we adopt the shorthands
\begin{align}
    \widecheck{\mathbf C}_k(\cdot) \delequal \widecheck{\mathbf C}^{ ( \mathcal J_{b}(k) )|b }(\ \cdot \ ;\bm m_k).
    \label{proof, shorthands, metastability}
\end{align}
We start by highlighting a few properties of the limiting Markov jump process $Y^{*|b}$ in Theorem~\ref{corollary irreducible case}.
Recall the definitions of $q_b(i)$ and $q_b(i,j)$ in \eqref{def: q b i, q b i j, generator for Y * b},
and
note that (for each $i \in [K]$)
\begin{align*}
    & q_b(i) \geq \sum_{j \in [K]:\ j \neq i}q_b(i,j),
    \qquad
    q_b(i) \leq \sum_{j \in [K]:\ j \neq i}q_b(i,j)  + 
    \widecheck{\mathbf C}_i\Big( \bigcup_{j \in [K]}\partial I_j  \Big),
    \\ 
    \Longrightarrow\ 
    & q_b(i) =  \sum_{j \in [K]:\ j \neq i}q_b(i,j) > 0
    \qquad
    \text{by condition (i) of Assumption~\ref{assumption: regularity condition on b}.}
\end{align*}
Moving on, we apply Theorem~\ref{theorem: first exit analyses, adapted} to show that $q_b(i) = \widecheck{\mathbf C}_i\big((I_i)^\complement\big) < \infty$.
First, 
by Assumption~\ref{assumption: regularity condition on b} (ii), 
the set $\mathcal G^{(\mathcal J_b(i) - 1)|b}(\bm m_i)$ is bounded away from $(I_i)^\complement$,
where $\mathcal J_b(i)$ is defined in \eqref{def: J * b i,j, metastability}.
Next, let $\tilde I_j = I_j \cap B_M(\bm 0)$, i.e., the restriction of $I_j$ on the open ball centered at the origin with radius $M$, for some $M$ large enough.
It is shown in \eqref{property, boundedness of set G k b x epsilon} that the set $\mathcal G^{(\mathcal J_b(i) - 1)|b}(\bm m_i)$ is bounded.
Then, for all $M$ large enough we know that 
$\mathcal G^{(\mathcal J_b(i) - 1)|b}(\bm m_i)$ is still bounded away from $(\tilde I_i)^\complement$.
Meanwhile, note that 
$
\partial \tilde I_i \subseteq \partial I_i \cup \partial B_M(\bm 0).
$
Again, by the boundedness property \eqref{property, boundedness of set G k b x epsilon},
as well as the fact that $\widecheck{\mathbf C}_i$ is supported on $\mathcal G^{ (\mathcal J_b(i))|b }(\bm m_i)$
(see definitions in \eqref{def: measure check C k b}),
we have 
$
\widecheck{\mathbf C}_i\big(\partial I_i \cup \partial B_M(\bm 0)\big) = 0
$
and hence
$
\widecheck{\mathbf C}_i(\partial \tilde I_i) = \widecheck{\mathbf C}_i(\partial I_i) = 0
$
for all $M$ large enough (see Assumption~\ref{assumption: regularity condition on b} (i)).
This allows us to apply the $C^I_b < \infty$ bound in Theorem~\ref{theorem: first exit analyses, adapted}
(by setting $I = I_i \cap B_M(\bm 0)$,
and get
\begin{align}
    \widecheck{\mathbf C}_i\Big( \big(I_i \cap B_M(\bm 0)\big)^\complement \Big) < \infty,
    \qquad\forall i \in [K]
    \label{property, finite upper bound for complement, I i restricted over M}
\end{align}
for any $M$ large enough.
Then, from the trivial bound
$
(I_i)^\complement \subseteq \big(I_i \cap B_M(\bm 0)\big)^\complement
$
as well as the bound $q_b(i) > 0$ noted above, we obtain (for each $i \in [K]$)
\begin{align}
     \sum_{j \in [K]:\ j \neq i}q_b(i,j) = q_b(i) = \widecheck{\mathbf C}_i\big((I_i)^\complement\big) \in (0,\infty).
     \label{property, sum of q b i j is q b i, metastability, proof}
\end{align}
Furthermore, Lemma~\ref{lemma: exit rate strictly positive, first exit analysis} verifies that
\begin{align}
    \widecheck{\mathbf C}^{ (\mathcal J_b(i))|b  }(I_j;\bm m_i) > 0
    \qquad\iff \qquad
    I_j \cap \mathcal G^{ (\mathcal J_b(i))|b }(\bm m_i) \neq \emptyset.
    \label{property, edge on graph iff q b i j strictly positive, metastability}
\end{align}
As a result, in Definition \ref{def main paper typical transition graph}
we know that the typical transition graph associated with threshold $b$ contains an edge $(\bm m_i \to \bm m_j)$ if and only if $q_b(i,j) > 0$.

Next, we stress that the law of the Markov jump process $\bm Y^{*|b}_t$ in Theorem~\ref{corollary irreducible case} can be expressed using the mapping $\Phi$ introduced in Definition~\ref{def jump process}.
Given any $\bm m_{i_0} \in \{\bm m_1,\bm m_2,\ldots,\bm m_{K}\}$,
we set $V_1 = \bm m_{i_0}$, $U_1 = 0$, and (for any $t > 0$, $l \geq 1$, and $i,j \in [K]$ with $i \neq j$)
\begin{equation} \label{def Y * b t, metastability}
    \begin{split}
         & \P\Big( U_{l+1} \leq t,\ V_{l+1} = \bm m_j\ \Big|\ V_l = \bm m_i, (V_j)_{j = 1}^{l-1},\ (U_j)_{j = 1}^l\Big)
    = \P\Big( U_{l+1} \leq t,\ V_{l+1} = \bm m_j\ \Big|\ V_l = \bm m_i \Big)
    \\ 
    & = 
    \begin{cases}
        \frac{q_b(i,j)}{q_b(i)} & \text{ if }\bm m_i \notin V^*_b,
        \\ 
        \frac{q_b(i,j)}{q_b(i)} \cdot \Big( 1 - \exp\big(-q_b(i)t\big)\Big) & \text{ if }\bm m_i \in V^*_b.
    \end{cases}
    \end{split}
\end{equation}
In other words, conditioning on $V_l = \bm m_i$,
we have $V_{l+1} = \bm m_j$ with probability $q_b(i,j)/q_b(i)$;
as for $U_{l+1}$, we set $U_{l+1} \equiv 0$ if $\bm m_i \notin V^*_b$ (i.e., the current value $\bm m_i$ is not a widest minimum), and set $U_{l+1}$ as an Exponential RV with rate $q_b(i)$ otherwise.
We claim that
\begin{align}
    \bm Y^{*|b}_{\boldsymbol{\cdot}}\distequal \Phi\Big((U_j)_{j \geq 1},(V_j)_{j \geq 1}\Big).
    \label{eq, jump process representation of Y * b}
\end{align}
In fact, under the conditions in Theorem~\ref{corollary irreducible case},
it is straightforward to show that 
\begin{enumerate}[$(i)$]

    \item For any $t > 0$, $\lim_{i \to \infty}\P(\sum_{j \leq i}U_j > t) = 1$;

    \item For any $u > 0$ and $i \geq 1$, $\P(U_1 + \cdots + U_i = u) = 0$;
    
    \item $\bm Y^{*|b}_{\boldsymbol{\cdot}} \distequal \Phi\Big((U_j)_{j \geq 1},(V_j)_{j \geq 1}\Big)$; that is, it
    is a continuous-time Markov chain with state space $V^*_b$,
    generator
\begin{align*}
    \P( \bm Y^{*|b}_{t+h} = m_j\ |\ \bm Y^{*|b}_t = m_i)
    =  h \cdot \sum_{j^\prime \in [K]:\ j^\prime \neq i} q_b(i,j^\prime)\theta_b(\bm m_j|\bm m_{j^\prime})
    + \bm{o}(h)\qquad\text{ as }h\downarrow 0,
\end{align*}
and initial distribution $\P(\bm Y^{*|b}_0 = \bm m_j) = \theta_b(\bm m_j|\bm m_{i_0})$;
see \eqref{def: q b i, q b i j, generator for Y * b} and \eqref{def: absorption prob, theta b i j} for the definitions of $q_b(i,j)$ and $\theta_b$, respectively.
\end{enumerate}
For the sake of completeness, we collect the proof in Section~\ref{sec: appendix, CTMC Y * b}.
The representation~\eqref{eq, jump process representation of Y * b} and the properties stated above will significantly streamline our proof in this section.

The proofs of Propositions \ref{proposition: hat X converges to CTMC, metastability} and \ref{proposition: hat X close to X, metastability}
hinge on the first exit analysis in Theorem~\ref{theorem: first exit analyses, adapted},
which is stated for a bounded region $I$.
To facilitate the application of Theorem~\ref{theorem: first exit analyses, adapted} onto the (perhaps unbounded) attraction fields over $f$,
we consider
\begin{align}
    S(\delta) & \delequal \bigcup_{j \in [K]}(\partial I_j)^\delta,
    \label{def: boundary set S delta}
    \\
    \notationdef{notation-attraction-field-I-i-truncated-delta-M}{I_{i;\delta,M}} 
    & \delequal 
    (I_i)_{\delta} \cap B_M(\bm 0),
    \label{def, attraction field I i truncated delta M}
\end{align}
for some $\delta,M > 0$.
Recall that we use $E^r$ to denote the $r$-enlargement of the set $E$ (with $E^r$ being closed),
and $E_r$ for the $r$-shrinkage of $E$ (with $E_r$ being open).
Meanwhile, define
\begin{align}
    \notationdef{notation-sigma-eta-b-i-epsilon}{\sigma^{\eta|b}_{i;\epsilon}(\bm x)} 
    &\delequal 
    \min\Big\{t \geq 0:\ \bm X^{\eta|b}_t(\bm x) \in \bigcup_{ j \neq i }B_\epsilon(\bm m_j)\Big\},
    \label{def stopping time sigma eta b i epsilon x, metastability}
    \\ 
    \notationdef{notation-tau-eta-b-i-delta-M-x}{\tau^{\eta|b}_{i;\delta,M}(\bm x)} 
    & \delequal 
    \min\Big\{ t \geq 0:\ \bm X^{\eta|b}_t(\bm x) \notin 
    I_{i;\delta,M}\Big\}.
    \label{def: tau eta b i delta M, metastability}
\end{align}
In other words, ${\tau^{\eta|b}_{i;\delta,M}(\bm x)}$ is the first exit time from $I_{i;\delta,M}$,
and ${\sigma^{\eta|b}_{i;\epsilon}(\bm x)}$ is the first hitting time to the $\epsilon$-neighborhood of a local minimum different that's not $\bm m_i$.

To prepare for the proof of Propositions \ref{proposition: hat X converges to CTMC, metastability} and \ref{proposition: hat X close to X, metastability},
we state a few properties of the measures $\widecheck{\mathbf C}_i$.
First, since $\widecheck{\mathbf C}_i$ is supported on $\mathcal G^{(\mathcal J_b(i))|b}(\bm m_i)$,
which is a bounded set (see \eqref{property, boundedness of set G k b x epsilon}),
for any $M$ large enough we have 
 \begin{align}
    \max_{i \in [K]}\widecheck{\mathbf C}_i\big( \{ \bm x \in \R^d:\ \norm{\bm x} \geq M  \} \big) = 0,
    \quad 
    M > \norm{\bm x_0},
    \quad\text{ and }
    \max_{i \in [K]}\norm{\bm m_i} < M.
    \label{goal 1, lemma: unlikely to exit M or visit s i, metastability} 
\end{align}
Here, $\bm x_0$ is the initial value prescribed in Theorem~\ref{corollary irreducible case} (and hence Propositions \ref{proposition: hat X converges to CTMC, metastability} and \ref{proposition: hat X close to X, metastability}).
Next, by Assumption~\ref{assumption: regularity condition on b} (i) and the continuity of measures,
\begin{align}
    \lim_{\delta \downarrow 0} {\widecheck{\mathbf C}_i\big( S(\delta)\big)}
=
\widecheck{\mathbf C}_i\bigg( \bigcup_{j \in [K]} \partial I_j\bigg) = 0,
\qquad\forall i \in [K].
\label{proof, property finitness of mass boundary set S delta}
\end{align}
As an immediate consequence, note that (recall that $E^-$ denotes the closure of the set $E$)
\begin{align}
    q_b(i,j) = \widecheck{\mathbf C}_i(I_j) = \widecheck{\mathbf C}_i(I_j^-),
    \qquad 
    \forall i,j \in [K]\text{ with }i \neq j.
    \label{proof, mass I j = I j closure, proposition: transition time, metastability}
\end{align}
On the other hand, by \eqref{property, finite upper bound for complement, I i restricted over M} and the continuity of measures,
for any $M$ large enough and $\delta > 0$ small enough, we have
$
 \widecheck{\mathbf C}_i\big( (I_{i;\delta,M})^\complement \big)  < \infty.
$
Together with \eqref{property, sum of q b i j is q b i, metastability, proof} and the trivial bound $
(I_i)^\complement \subseteq \big(I_i \cap B_M(\bm 0)\big)^\complement,
$
we see that for any $M$ large enough and $\delta > 0$ small enough,
\begin{align}
    \widecheck{\mathbf C}_i\big( (I_{i;\delta,M})^\complement \big) \in (0,\infty),
    \qquad\forall i \in [K].
    \label{proof, positive and finite bound for I i delta M}
\end{align}
Henceforth in this section, we only consider $M$ large enough such that the claims \eqref{goal 1, lemma: unlikely to exit M or visit s i, metastability}  and \eqref{proof, positive and finite bound for I i delta M} hold.
Then, given $\Delta > 0$, it holds for all $\delta > 0$ small enough that
\begin{align}
    \max_{i \in [K]}
    \frac{ \widecheck{\mathbf C}_i\big( (S(\delta))^-\big) }{
        \widecheck{\mathbf C}_i\big( (I_{i;\delta,M})^\complement \big)
    } < \Delta.
    \label{proof, range of epsilon, lemma: unlikely to exit M or visit s i, metastability}
\end{align}
Furthermore, observe that
$
\widecheck{\mathbf C}_i( \partial I_{i;\delta,M} )
\leq 
\widecheck{\mathbf C}_i( \partial I_{i} )
+
\widecheck{\mathbf C}_i\big( \partial S(\delta) \big)
+
\widecheck{\mathbf C}_i\big( \partial B_M(\bm 0)\big).
$
By \eqref{proof, property finitness of mass boundary set S delta}, 
for any $\delta_1$ small enough we have $\widecheck{\mathbf C}_i\big( (S(\delta_1))^- \big) < \infty$.
This further implies that
the claim $\widecheck{\mathbf C}\big( \partial S(\delta) \big) > 0$
could hold
for at most countably many $\delta \in (0,\delta_1]$,
due to the simple facts that the sets $\partial S(\delta)$ are mutually disjoint across different $\delta$'s,
and that $\partial S(\delta) \subseteq (S(\delta_1))^-$ when $\delta \in (0,\delta_1]$.
Then, together with \eqref{goal 1, lemma: unlikely to exit M or visit s i, metastability},
we know that \textbf{for all but countably many $\delta > 0$ small enough}, we have
\begin{align}
    \widecheck{\mathbf C}_i\big( \partial I_{i;\delta,M} \big) = 0,
    \qquad\forall i \in [K].
    \label{proof, property, zero mass on boundary of I i delta M}
\end{align}
Here, we say that a claim holds for {for all but countably many $\delta > 0$ small enough} if there exists some $\delta_1 > 0$ such that, over $\delta \in (0,\delta_1]$, the claim fails for at most countably $\delta$ (i.e., the claim holds for Lebesgue almost every $\delta \in (0,\delta_1]$).
Lastly, by \eqref{proof, property, zero mass on boundary of I i delta M},
we can pick a smaller $\bar\epsilon > 0$ if needed to ensure that \eqref{choice of bar epsilon, metastability} still holds, and 
\begin{align}
     \widecheck{\mathbf C}_i\big( \partial I_{i;\bar\epsilon,M} \big) = 0,
    \qquad\forall i \in [K].
     \label{proof, property, zero mass on boundary of I i delta M, delta = bar epsilon}
\end{align}

\subsection{Proof of Proposition \ref{proposition: hat X converges to CTMC, metastability}}

As a first application of Theorem~\ref{theorem: first exit analyses, adapted},
Lemma \ref{lemma: unlikely to exit M or visit s i, metastability} states that
it is unlikely for $\bm X^{\eta|b}_t(\bm x)$ to get close to any of the boundary set of attraction fields or exit a wide enough compact set
before visiting a different local minimum.

\begin{lemma}\label{lemma: unlikely to exit M or visit s i, metastability}
\linksinthm{lemma: unlikely to exit M or visit s i, metastability}
Given $\Delta > 0$ and $\epsilon \in (0,\bar\epsilon)$,
it holds for all but countably many $\delta > 0$ small enough that
\begin{align}
    \limsup_{\eta \downarrow 0}\max_{ i \in [K] }\sup_{ \bm x \in \bar B_\epsilon(\bm m_i) }
    \P\Big( \exists t < \sigma^{\eta|b}_{i;\epsilon}(\bm x)\ s.t.\ 
    \bm X^{\eta|b}_t(\bm x) \in S(\delta)\text{ or }\norm{\bm X^{\eta|b}_t(\bm x)} \geq M + 1
    \Big) & < \Delta.
\label{goal 2, lemma: unlikely to exit M or visit s i, metastability}
\end{align}
\end{lemma}

\begin{proof}
By \eqref{proof, range of epsilon, lemma: unlikely to exit M or visit s i, metastability} and \eqref{proof, property, zero mass on boundary of I i delta M},
it holds for all but countably many $\delta$ small enough that for each $k \in [K]$,
$
\widecheck{\mathbf C}_k\big( (S(2\delta))^-\big)\big/
\widecheck{\mathbf C}_k\big( (I_{k;2\delta,M})^\complement\big) < \Delta
$
and 
$
\widecheck{\mathbf C}_i\big( \partial I_{i;2\delta,M} \big) = 0.
$
Henceforth in this proof, we only consider such $\delta$.
Observe that (i) due to $I_{i;2\delta,M}\subseteq I_{i;\delta,M}$, we have
$
\tau^{\eta|b}_{i;2\delta,M}(\bm x) \leq \tau^{\eta|b}_{i;\delta,M}(\bm x) \leq \sigma^{\eta|b}_{i;\epsilon}(\bm x);
$
and
(ii)
by definitions,
it holds for all $t < \tau^{\eta|b}_{i;2\delta,M}(\bm x)$
that 
$
\bm X^{\eta|b}_t(\bm x) \notin S(2\delta),\ \norm{ \bm X^{\eta|b}_t(\bm x) } < M.
$
Then, by defining events 
\begin{align*}
A_0(\eta,\delta,\bm x) & \delequal 
\bigg\{
\bm X^{\eta|b}_{  \tau^{\eta|b}_{i;2\delta,M}(\bm x) }(\bm x) \in B_M(\bm 0);\ 
\bm X^{\eta|b}_{  \tau^{\eta|b}_{i;2\delta,M}(\bm x) }(\bm x) \notin S(2\delta)
\bigg\},
\\ 
A_1(\eta,\delta,\bm x) & \delequal
\bigg\{
    \bm X^{\eta|b}_t(\bm x) \notin S(\delta)\text{ and }\norm{ \bm X^{\eta|b}_t(\bm x) } < M +1\ \forall  t < \sigma^{\eta|b}_{i,\epsilon}(\bm x)
\bigg\},
\end{align*}
we have
\begin{align*}
    & \bigg\{
        \exists t < \sigma^{\eta|b}_{i;\epsilon}(\bm x)\ s.t.\ 
    \bm X^{\eta|b}_t(\bm x) \in S(\delta)\text{ or }\norm{ \bm X^{\eta|b}_t(\bm x) } \geq M + 1
    \bigg\}
    \\
    & \subseteq 
    \Big( A_0(\eta,\delta,\bm x) \Big)^\complement 
    \cup 
    \bigg(
        A_0(\eta,\delta,\bm x)  \cap \Big( A_1(\eta,\delta,\bm x) \Big)^\complement
    \bigg).
\end{align*}
Therefore, it suffices to prove (for all $\delta > 0$ small enough)
\begin{align}
    \limsup_{\eta \downarrow 0}
    \max_{ i \in [K] }\sup_{ \bm x \in \bar B_\epsilon(\bm m_i) }
    \P\bigg( \big( A_0(\eta,\delta, \bm x) \big)^\complement   \bigg) 
    & < \Delta, 
    \label{proof, overall goal 1, lemma: unlikely to exit M or visit s i, metastability}
    \\ 
     \lim_{\eta \downarrow 0}
    \max_{ i \in [K] }\sup_{ \bm x \in \bar B_\epsilon(\bm m_i) }
    \P\bigg(
        A_0(\eta,\delta,\bm x)  \cap \Big( A_1(\eta,\delta,\bm x) \Big)^\complement
    \bigg)
    & = 0.
    \label{proof, overall goal 2, lemma: unlikely to exit M or visit s i, metastability}
\end{align}

\medskip\noindent
\textbf{Proof of Claim \eqref{proof, overall goal 1, lemma: unlikely to exit M or visit s i, metastability}}.
It suffices to show that
\begin{align}
\limsup_{ \eta \downarrow 0 } 
 \max_{ i \in [K] }\sup_{ \bm x \in \bar B_\epsilon(\bm m_i) }
\P\bigg(
    \bm X^{\eta|b}_{  \tau^{\eta|b}_{i;2\delta,M}(\bm x) }(\bm x) \in S(2\delta)
    \bigg) & < \Delta,
     \label{proof, goal 2, lemma: unlikely to exit M or visit s i, metastability}
    \\
    \limsup_{\eta \downarrow 0} 
     \max_{ i \in [K] }\sup_{ \bm x \in \bar B_\epsilon(\bm m_i) }
    \P\bigg(
        \bm X^{\eta|b}_{  \tau^{\eta|b}_{i;2\delta,M}(\bm x) }(\bm x) \notin B_M(\bm 0)
    \bigg) & = 0.
    \label{proof, goal 1, lemma: unlikely to exit M or visit s i, metastability}
\end{align}
By Theorem~\ref{theorem: first exit analyses, adapted}
(in particular, note that the condition $\widecheck{\mathbf C}_i(\partial I) = 0$, under $I = I_{i;2\delta,M}$, is ensured by our choice of $\delta$ at the beginning of the proof),
we get (for each $i \in [K]$)
\begin{align*}
    \limsup_{ \eta \downarrow 0 } 
 \max_{ i \in [K] }\sup_{ \bm x \in \bar B_\epsilon(\bm m_i) }
\P\bigg(
    \bm X^{\eta|b}_{  \tau^{\eta|b}_{i;2\delta,M}(\bm x) }(\bm x) \in S(2\delta)
    \bigg)
    & \leq 
    \widecheck{\mathbf C}_i\Big( \big(S(2\delta)\big)^-\Big)\Big/
    \widecheck{\mathbf C}_i\Big( \big(I_{i;2\delta,M}\big)^\complement  \Big) < \Delta,
\end{align*}
where the last inequality also follows from our choice of $\delta$ at the beginning of the proof.
This 
 verifies Claim~\eqref{proof, goal 2, lemma: unlikely to exit M or visit s i, metastability}.
Likewise, Claim~\eqref{proof, goal 1, lemma: unlikely to exit M or visit s i, metastability} can be shown by combining Theorem~\ref{theorem: first exit analyses, adapted} with \eqref{goal 1, lemma: unlikely to exit M or visit s i, metastability}.
This concludes the proof of Claim \eqref{proof, overall goal 1, lemma: unlikely to exit M or visit s i, metastability}.

\medskip\noindent
\textbf{Proof of Claim \eqref{proof, overall goal 2, lemma: unlikely to exit M or visit s i, metastability}}.
Let
\begin{align}
    R^{\eta|b}_{j;\epsilon}(\bm x) \delequal \min\{ t \geq 0:\ \bm X^{\eta|b}_t(\bm x) \in B_\epsilon(\bm m_j) \}
    \label{def: return time R b j epsilon x}
\end{align}
be the first hitting time to the $\epsilon$-neighborhood of the local minimum $\bm m_j$.
By the strong Markov property at $\tau^{\eta}_{i;2\delta,M}(\bm x)$,
\begin{align*}
 & \max_{ i \in [K] }\sup_{ \bm x \in \bar B_\epsilon(\bm m_i) }
 \P\bigg(  A_0(\eta,\delta, \bm x)  \cap \Big( A_1(\eta,\delta, \bm x) \Big)^\complement
 \bigg)
 \\ 
    & \leq 
     \max_{ i \in [K] }\sup_{ \bm x \in \bar B_\epsilon(\bm m_i) }
     \P\bigg(   \Big( A_1(\eta,\delta,\bm x) \Big)^\complement\ \bigg|\ A_0(\eta,\delta,\bm x) 
    \bigg)
    \\ 
    &\leq \max_{ j \in [K] }\underbrace{\sup_{ \bm z \in (I_j)_{2\delta} \cap \bar B_M(\bm 0) }
    \P
    \bigg(
    \Big\{ \bm X^{\eta|b}_t(\bm z) \in (I_j)_{\delta} \cap B_{M+1}(\bm 0)\ \forall t < R^{\eta|b}_{j;\epsilon}(\bm z) \Big\}^\complement
    \bigg)}_{p_j(\eta)}.
\end{align*}
By Lemma~\ref{lemma: cycle, efficient return},
we get 
$
\lim_{\eta \downarrow 0}p_j(\eta) = 0
$
for each $j \in [K]$
and conclude the proof of Claim~\eqref{proof, overall goal 2, lemma: unlikely to exit M or visit s i, metastability}.
\end{proof}

Recall the time scaling $\lambda^*_b$ defined in \eqref{def scale function lambda * b eta},
the set $V^*_b$ defined in \eqref{def: V * b, metastability},
and the terms $q_b(i,j)$ and $q_b(i)$ defined in \eqref{def: q b i, q b i j, generator for Y * b}.
Lemma~\ref{proposition: transition time, metastability} applies  Theorem~\ref{theorem: first exit analyses, adapted}
to obtain first exit analyses for each attraction field over $f$.

\begin{lemma} \label{proposition: transition time, metastability}
\linksinthm{proposition: transition time, metastability} 
Let $\bar \epsilon > 0$ be specified as in \eqref{choice of bar epsilon, metastability}.
\begin{enumerate}[(i)]
    \item 
    Let $\notationdef{notation-return-time-to-I-i}{R_{i;\epsilon}^{\eta|b}(\bm x)}$ be defined as in the \eqref{def: return time R b j epsilon x}.
    For any $\epsilon \in (0,\bar\epsilon)$, $t > 0$ and $i \in [K]$, 
    \begin{align*}
        \liminf_{\eta\downarrow 0}
        \inf_{\bm x \in ((I_i)_\epsilon \cap B_M(\bm 0) )^-}
        \P\bigg(
        {R_{i;\epsilon}^{\eta|b}(\bm x)}\cdot {\lambda^*_b(\eta)} \leq t,\ 
        \bm X^{\eta|b}_t(\bm x) \in I_i\cap B_M(\bm 0)\ \forall t \leq {R_{i;\epsilon}^{\eta|b}(\bm x)}
        \bigg) = 1.
    \end{align*}

    \item Let $i,j \in [K]$ be such that $i \neq j$.
    Let ${\sigma^{\eta|b}_{i;\epsilon}(\bm x)}$ be defined as in \eqref{def stopping time sigma eta b i epsilon x, metastability}.
    If $\bm m_i \in V^*_b$, then for any $\epsilon \in (0,\bar\epsilon)$ and any $u > 0$,
    \begin{align*}
        \liminf_{\eta \downarrow 0}\inf_{\bm x \in \bar B_\epsilon(\bm m_i)}
        \P\bigg(
        {\sigma^{\eta|b}_{i;\epsilon}(\bm x)} \cdot {\lambda^*_b(\eta)} > u,\
        \bm X^{\eta|b}_{ \sigma^{\eta|b}_{i;\epsilon}(\bm x) }(\bm x) \in I_j
        \bigg)
        & \geq \exp\big( - q_b(i)\cdot u \big) \cdot \frac{q_b(i,j)}{q_b(i)},
        \\ 
        \limsup_{\eta \downarrow 0}\sup_{\bm x \in \bar B_\epsilon(\bm m_i)}
        \P\bigg(
        {\sigma^{\eta|b}_{i;\epsilon}(\bm x)} \cdot {\lambda^*_b(\eta)} > u,\
        \bm X^{\eta|b}_{ \sigma^{\eta|b}_{i;\epsilon}(\bm x) }(\bm x) \in I_j
        \bigg)
        & \leq \exp\big( - q_b(i) \cdot u \big)\cdot \frac{q_b(i,j)}{q_b(i)}.
    \end{align*}
    If $\bm  m_i \notin V^*_b$, then for any $\epsilon \in (0,\bar\epsilon)$ and any $u > 0$,
    \begin{align*}
    \frac{q_b(i,j)}{q_b(i)}
        & \leq \liminf_{\eta \downarrow 0}\inf_{\bm x \in \bar B_\epsilon(\bm m_i)}
        \P\bigg(
        {\sigma^{\eta|b}_{i;\epsilon}(\bm x)} \cdot {\lambda^*_b(\eta)} \leq u,\
        \bm X^{\eta|b}_{ \sigma^{\eta|b}_{i;\epsilon}(\bm x) }(\bm x) \in I_j
        \bigg)
        \\ 
        & \leq 
        \limsup_{\eta \downarrow 0}\sup_{\bm x \in \bar B_\epsilon(\bm m_i)}
        \P\bigg(
        {\sigma^{\eta|b}_{i;\epsilon}(\bm x)} \cdot {\lambda^*_b(\eta)} \leq u,\
        \bm X^{\eta|b}_{ \sigma^{\eta|b}_{i;\epsilon}(\bm x) }(\bm x) \in I_j
        \bigg)
        \leq 
        \frac{q_b(i,j)}{q_b(i)}.
    \end{align*}

\end{enumerate}
\end{lemma}

\begin{proof}
$(i)$
Recall the notations in \eqref{proof, shorthands, metastability},
and that
$\lambda^*_b(\eta) \in \RV_{ \mathcal J^*_b\cdot (\alpha - 1) + 1 }(\eta)$ as $\eta \downarrow 0$ (see \eqref{def scale function lambda * b eta}).
Due to $ \mathcal J^*_b\cdot (\alpha - 1) + 1 \geq \alpha > 1$,
given any $T > 0$
we have
$
\lim_{\eta \downarrow 0}\frac{T/\eta}{ t/\lambda^*_b(\eta) } = 0,
$
and hence
\begin{align*}
   & 
   \P\bigg(
        {R_{i;\epsilon}^{\eta|b}(\bm x)} \cdot {\lambda^*_b(\eta)} \leq t,\ 
        \bm X^{\eta|b}_u(\bm x) \in I_i \cap B_M(\bm 0)\ \forall u \leq R_{i;\epsilon}^{\eta|b}(\bm x)
        \bigg)
        \\ 
    &
    \geq 
    \P\bigg(
        {R_{i;\epsilon}^{\eta|b}(\bm x)}  \leq T/\eta,\ 
        \bm X^{\eta|b}_u(\bm x) \in I_i \cap B_M(\bm 0)\ \forall u \leq R_{i;\epsilon}^{\eta|b}(\bm x)
        \bigg)
\end{align*}
for all $\eta$ small enough.
Applying Lemma~\ref{lemma: cycle, efficient return} onto the bounded region $(I_i)_\epsilon \cap B_M(\bm 0)$ and sending $T \to \infty$,
we conclude the proof of part (i).

\medskip
$(ii)$
Let $\lambda^*_{b;i}(\eta) \delequal \eta \cdot \big(\lambda(\eta)\big)^{ \mathcal J_b(i) }$,
where $\mathcal J_{b}(i)$ is defined in \eqref{def: J * b i,j, metastability}.
To prove part (ii),
it suffices to establish the following upper and lower bounds: given $i,j \in [K]$ such that $i \neq j$, and $\epsilon \in (0,\bar\epsilon)$,  $t \geq 0$,
\begin{align}
  \liminf_{\eta \downarrow 0}\inf_{ \bm x \in \bar B_\epsilon(\bm m_i) }
        \P\bigg(
        {\sigma^{\eta|b}_{i;\epsilon}(\bm x)} \cdot {\lambda^*_{b;i}(\eta)} > t,\
        \bm X^{\eta|b}_{ \sigma^{\eta|b}_{i;\epsilon}(\bm x) }(\bm x) \in I_j
        \bigg)
& \geq \exp\big( - q_b(i)\cdot t \big) \cdot \frac{q_b(i,j)}{q_b(i)},
            \label{goal, LB, part ii, proposition: transition time, metastability}
        \\ 
        \limsup_{\eta \downarrow 0}\sup_{ \bm x \in \bar B_\epsilon(\bm m_i) }
       \P\bigg(
        {\sigma^{\eta|b}_{i;\epsilon}(\bm x)} \cdot {\lambda^*_{b;i}(\eta)} > t,\
        \bm X^{\eta|b}_{ \sigma^{\eta|b}_{i;\epsilon}(\bm x) }(\bm x) \in I_j
        \bigg)
& \leq \exp\big( - q_b(i)\cdot t \big) \cdot \frac{q_b(i,j)}{q_b(i)}.
             \label{goal, UB, part ii, proposition: transition time, metastability}
\end{align}
To see why, note that in case of $\bm m_i \in V^*_b$,
the claims in part (ii) are equivalent to \eqref{goal, LB, part ii, proposition: transition time, metastability} and \eqref{goal, UB, part ii, proposition: transition time, metastability}
due to $\mathcal J_b(i) = \mathcal J^*_b$
and hence $\lambda^*_{b;i}(\eta) = \lambda^*_b(\eta)$
(see \eqref{def: J * b i,j, metastability} and \eqref{def: J * b}).
In case that $\bm m_i \notin V^*_b$,
we have 
$
\lim_{\eta \downarrow 0}\frac{ t /\lambda^*_{b;i}(\eta) }{  u/\lambda^*_b(\eta) } = 0
$
for any $t,u\in (0,\infty)$.
We then recover the upper and lower bounds in part (ii) by sending $t \downarrow 0$ in \eqref{goal, LB, part ii, proposition: transition time, metastability} and \eqref{goal, UB, part ii, proposition: transition time, metastability}.

The rest of this proof is devoted to establishing \eqref{goal, LB, part ii, proposition: transition time, metastability} and \eqref{goal, UB, part ii, proposition: transition time, metastability}.
We begin by stating few useful facts about the measures $\widecheck{\mathbf C}_k$.
Combining \eqref{proof, mass I j = I j closure, proposition: transition time, metastability}
with the continuity of measures,
we get
$
\lim_{\delta \downarrow 0} \widecheck{\mathbf C}_i\big( ((I_i)_\delta)^\complement\big)
=
q_b(i) = \widecheck{\mathbf C}_i( I_i^\complement).
$
Given any $\Delta > 0$,
by \eqref{proof, range of epsilon, lemma: unlikely to exit M or visit s i, metastability} it  then holds all $\delta > 0$ small enough,
\begin{equation}\label{proof, bound for measure check C, proposition: transition time, metastability}
    \begin{split}
        \widecheck{\mathbf C}_i\big( (I_{i;\delta,M})^\complement\big)
    & \leq 
    \widecheck{\mathbf C}_i\big( (B_M(\bm 0))^c\big)
    +
    \widecheck{\mathbf C}_i\big(((I_i)_\delta)^\complement\big)
    < (1 + \Delta) \cdot q_b(i),\quad\forall i \in [K].
    \end{split}
\end{equation}
Besides,
due to $I_{i;\delta,M} \subseteq I_i$,
\begin{align}
    \widecheck{\mathbf C}_i\big( (I_{i;\delta,M})^\complement\big)
\geq 
\widecheck{\mathbf C}_i\big( (I_{i})^\complement\big) = q_b(i).
\label{proof, bound for measure check C, 2, proposition: transition time, metastability}
\end{align}
Lastly, recall that by \eqref{proof, property, zero mass on boundary of I i delta M},
the condition $\widecheck{\mathbf C}_i(\partial I_{i;\delta,M}) = 0\ \forall i \in [K]$ holds for all but countably many $\delta > 0$ small enough,
which supports the application of Theorem~\ref{theorem: first exit analyses, adapted} in the subsequent proof.

\medskip\noindent
\textbf{Proof of Lower Bound \eqref{goal, LB, part ii, proposition: transition time, metastability}}.
We fix some $i\neq j$ and $t > 0$ when proving \eqref{goal, LB, part ii, proposition: transition time, metastability}.
By the definition of $\tau^{\eta|b}_{ i;\delta,M}(\bm x)$
in 
\eqref{def: tau eta b i delta M, metastability},
\begin{align*}
    & 
    {\big\{
        {\sigma^{\eta|b}_{i;\epsilon}(\bm x)} \cdot {\lambda^*_{b;i}(\eta)} > t,\
        \bm X^{\eta|b}_{ \sigma^{\eta|b}_{i;\epsilon}(\bm x) }(\bm x) \in I_j
        \big\}
        }
        \\ 
        & 
        \supseteq
        \underbrace{
            \big\{ {\tau^{\eta|b}_{i;\delta,M}(\bm x)} \cdot {\lambda^*_{b;i}(\eta)} > t;\ 
         \bm X^{\eta|b}_{ \tau^{\eta|b}_{i;\delta,M}(\bm x) }(\bm x) \in I_{j;\delta,M + 1}
        \big\}
        }_{ \text{(I)} }
        \cap 
        \underbrace{\big\{
         \bm X^{\eta|b}_{ \sigma^{\eta|b}_{i;\epsilon}(\bm x) }(\bm x) \in I_j
        \big\}}_{ \text{(II)}  }.
\end{align*}
We first analyze $\P(\text{(II)}|\text{(I)})$.
By the strong Markov property at $\tau^{\eta|b}_{i;\delta,M}(\bm x)$,
we have 
$
  \inf_{\bm x \in\bar B_\epsilon(\bm m_i)}\P( \text{(II)}\ |\ \text{(I)} )
\geq   
\inf_{ \bm y \in I_{j;\delta, M + 1} }
\P\big( 
\bm X^{\eta|b}_{t}(\bm y) \in I_j\ \forall t \leq {R^{\eta|b}_{j;\epsilon}(\bm y)}\big),
$
where
$
{R_{j;\epsilon}^{\eta|b}(\bm y)}
$
is defined in \eqref{def: return time R b j epsilon x}
and the set $I_{j;\delta, N}$ is defined in \eqref{def, attraction field I i truncated delta M}.
Applying Lemma \ref{lemma: cycle, efficient return}, we yield 
\begin{equation}\label{proof, claim 1, LB, proposition: transition time, metastability}
    \begin{split}
        & \liminf_{\eta \downarrow 0} \inf_{\bm x \in \bar B_\epsilon(\bm m_i)}\P( \text{(II)}\ |\ \text{(I)} )
= 1.
    \end{split}
\end{equation}
Next,
due to $I_{j;\delta,M+1} \subseteq I_j$,
\begin{align*}
   \text{(I)} 
    & =
    \underbrace{
\big\{ {\tau^{\eta}_{i;\delta,M}(\bm x)} \cdot {\lambda^*_{b;i}(\eta)} > t;\ 
         \bm X^{\eta|b}_{ \tau^{\eta|b}_{i;\delta,M}(\bm x) }(\bm x) \in I_{j}
        \big\}
    }_{\text{(III)}}
    \cap 
    \underbrace{
\{
\bm X^{\eta|b}_{ \tau^{\eta|b}_{i;\delta,M}(\bm x) }(\bm x)  \in I_{j;\delta,M+1}
\}
    }_{ \text{(IV)} }.
\end{align*}
Given any $\Delta > 0$, observe that (for all but countably many $\delta > 0$ small enough)
\begin{align*}
    & \liminf_{\eta \downarrow 0}\inf_{ \bm x \in \bar B_\epsilon(\bm m_i) }
    \P(\text{(III)})
    \\ 
    & \geq 
    \exp\big( - \widecheck{\mathbf C}_i( (I_{i;\delta,M})^\complement)\cdot t \big)\cdot
    \frac{ 
     \widecheck{\mathbf C}_i( I_j )
    }{  \widecheck{\mathbf C}_i( (I_{i;\delta,M})^\complement)  }
    \quad 
    \text{ by Theorem~\ref{theorem: first exit analyses, adapted}}
    \\ 
    &
    >
    \frac{ \exp( - (1+\Delta) q_b(i)\cdot t)  }{1 + \Delta} \cdot \frac{q_b(i,j)}{q_b(i)}
    \qquad 
    \text{for any $\delta > 0$ small enough, due to 
\eqref{proof, bound for measure check C, proposition: transition time, metastability}.}
\end{align*}
Meanwhile, 
by Lemma \ref{lemma: unlikely to exit M or visit s i, metastability},
we have 
$
\limsup_{ \eta \downarrow 0 }\sup_{\bm x \in \bar B_\epsilon(\bm m_i)}\P(\text{(IV)}^\complement) < \Delta
$
for all but countably many $\delta > 0$ small enough.
In summary, given $\Delta > 0$, one can find $\delta > 0$ such that
\begin{align}
    \liminf_{\eta \downarrow 0}
    \inf_{ \bm x \in \bar B_\epsilon(\bm m_i) }
    \P(\text{(I)}) \geq 
    \frac{ \exp( - (1+\Delta) q_b(i)\cdot t)  }{1 + \Delta} \cdot \frac{q_b(i,j)}{q_b(i)} - \Delta.
    \label{proof, claim 2, LB, proposition: transition time, metastability}
\end{align}
Combining \eqref{proof, claim 1, LB, proposition: transition time, metastability} and \eqref{proof, claim 2, LB, proposition: transition time, metastability}
and sending $\Delta \downarrow 0$,
we establish the lower bound \eqref{goal, LB, part ii, proposition: transition time, metastability}.

\medskip\noindent
\textbf{Proof of Upper Bound \eqref{goal, UB, part ii, proposition: transition time, metastability}}.
Let $\text{(I)} = \{ {\sigma^{\eta|b}_{i;\epsilon}(\bm x)} \cdot {\lambda^*_{b;i}(\eta)} > t,\
        \bm X^{\eta|b}_{ \sigma^{\eta|b}_{i;\epsilon}(\bm x) }(\bm x) \in I_j\}$.
Given $\delta > 0$,
define the event $\text{(II)} = \{ \bm X^{\eta|b}_{  \tau^{\eta|b}_{i;\delta,M}(\bm x) }(\bm x) \in B_{M+1}(\bm 0) \setminus S(\delta)  \}$,
where $S(\delta)$ is defined in \eqref{def: boundary set S delta}.
Pick some $\Delta > 0$, and note the decomposition of events $\text{(I)} = \big( \text{(I)}  \setminus \text{(II)}  \big) \cup  \big( \text{(I)}  \cap \text{(II)}  \big)$.
Applying Lemma~\ref{lemma: unlikely to exit M or visit s i, metastability},
it holds for all but countably many $\delta > 0$ small enough that
\begin{align}
    \limsup_{ \eta \downarrow 0 }\sup_{ \bm x \in \bar B_\epsilon(\bm m_i) }\P(\text{(I)}  \setminus \text{(II)} ) 
\leq 
\limsup_{ \eta \downarrow 0 }\sup_{\bm x \in \bar B_\epsilon(\bm m_i)}\P(\text{(II)}^\complement)
< \Delta.
\label{proof, ineq 1 for UB, part ii, proposition: transition time, metastability}
\end{align}
Next, by the definition of ${\tau^{\eta}_{i;\delta,M}(\bm x)}$ in \eqref{def: tau eta b i delta M, metastability},
on the event $\text{(I)}\cap \text{(II)}$ there must be some $k \in [K]$ with $k \neq i$ such that 
$
\bm  X^{\eta|b}_{  \tau^{\eta|b}_{i;\delta,M}(\bm x) }(\bm x) \in I_{k;\delta,M+1}.
$
For each $k \in [K]$ with $k \neq i$, let 
$$
\text{($k$)} = \text{(I)}\cap \text{(II)} \cap \{
\bm  X^{\eta|b}_{  \tau^{\eta|b}_{i;\delta,M}(\bm x) }(\bm x) \in I_{k;\delta,M+1}
\}.
$$
Note that $\bigcup_{ k \in [K]:\ k \neq i}\text{($k$)} = \text{(I)}\cap \text{(II)}$.
To proceed, consider the following decomposition
\begin{align*}
    (k) & = 
    \underbrace{ \bigg((k) \cap \Big\{ \Big(\sigma^{\eta|b}_{i;\epsilon}(\bm x) - \tau^{\eta|b}_{i;\delta,M}(\bm x) \Big) \cdot \lambda^*_{b;i}(\eta) > \Delta\Big\}\bigg) }_{ (k,1) }
    \cup 
    \underbrace{\bigg((k) \cap \Big\{ \Big(\sigma^{\eta|b}_{i;\epsilon}(\bm x) - \tau^{\eta|b}_{i;\delta,M}(\bm x) \Big) \cdot \lambda^*_{b;i}(\eta) \leq \Delta\Big\}\bigg)}_{(k,2)}.
\end{align*}
To proceed, we fix some $k \in [K]$ with $k \neq i$.
First, 
due to
$\lim_{\eta \downarrow 0}\frac{T/\eta}{\Delta/\lambda^*_{b;i}(\eta)} = 0$
for any $T \in (0,\infty)$,
\begin{align}
    & \limsup_{\eta \downarrow 0}\sup_{\bm x \in \bar B_\epsilon(\bm m_i)}\P\big((k,1)\big) 
     \nonumber \\ \nonumber
    & \leq 
    \limsup_{\eta \downarrow 0}\sup_{\bm x \in \bar B_\epsilon(\bm m_i)}
    \P\Big(
        (k) \cap \{ \sigma^{\eta|b}_{i;\epsilon}(\bm x) - \tau^{\eta|b}_{i;\delta,M}(\bm x)   > T/\eta\}
    \Big)
    \\  \nonumber
    & \leq 
    \limsup_{\eta \downarrow 0}\sup_{ \bm y \in I_{k;\delta,M+1} }
    \P(  \sigma^{\eta|b}_{i;\epsilon}(\bm y) > T/\eta )
    \qquad 
    \text{by the definition of the event $(k)$ and strong Markov property}
    \\  \nonumber
    & \leq 
    \limsup_{\eta \downarrow 0}\sup_{ \bm y \in I_{k;\delta,M+1} }
    \P\Big( \bm X^{\eta|b}_t(\bm y) \notin B_\epsilon(\bm m_k)\ \forall t \leq T/\eta\Big)
    \qquad
    \text{due to }k \neq i
    \\
    & = 0
    \qquad 
    \text{for any $T>0$ large enough, due to Lemma \ref{lemma: cycle, efficient return}.}
    \label{proof, ineq 2 for UB, part ii, proposition: transition time, metastability}
\end{align}
Meanwhile,
\begin{align*}
    & \sup_{ \bm x \in \bar B_\epsilon(\bm m_i) }\P\big( (k,2) \big)
    \\
    & \leq 
    \sup_{ \bm x \in \bar B_\epsilon(\bm m_i) }\
    \P\Big( 
     {\tau^{\eta|b}_{i;\delta,M}(\bm x)} \cdot {\lambda^*_{b;i}(\eta)} > t - \Delta;\
      \bm X^{\eta|b}_{  \tau^{\eta|b}_{i;\delta,M}(\bm x) }(\bm x) \in I_{k;\delta,M + 1};\
        \bm X^{\eta|b}_{ \sigma^{\eta|b}_{i;\epsilon}(\bm x) }(\bm x) \in I_j
    \Big)
    \\ 
    & \leq 
    \sup_{\bm x \in \bar B_\epsilon(\bm m_i) }\
    \P\Big( 
     {\tau^{\eta|b}_{i;\delta,M}(\bm x)} \cdot {\lambda^*_{b;i}(\eta)} > t - \Delta;\
      \bm X^{\eta|b}_{  \tau^{\eta|b}_{i;\delta,M}(\bm x) }(\bm x) \in I_{k;\delta,M+1}
      \Big)
      \\ 
      &\qquad
      \cdot 
      \sup_{ \bm x \in \bar B_\epsilon(\bm m_i) }\
       \P\Big( 
        \bm X^{\eta|b}_{ \sigma^{\eta|b}_{i;\epsilon}(\bm x) }(\bm x) \in I_j\ \Big|\
     {\tau^{\eta|b}_{i;\delta,M}(\bm x)} \cdot {\lambda^*_{b;i}(\eta)} > t - \Delta;\
      \bm X^{\eta|b}_{  \tau^{\eta|b}_{i;\delta,M}(\bm x) }(\bm x) \in I_{k;\delta,M + 1}
    \Big)
    \\ 
    & \leq 
    \sup_{ \bm x \in \bar B_\epsilon(\bm m_i) }\
     \P\Big( 
     \underbrace{
     {\tau^{\eta|b}_{i;\delta,M}(\bm x)} \cdot {\lambda^*_{b;i}(\eta)} > t - \Delta;\
      \bm X^{\eta|b}_{  \tau^{\eta|b}_{i;\delta,M}(\bm x) }(\bm x) \in I_k
      }_{ \text{($k$,I)} }
      \Big)
      \cdot 
      \sup_{ \bm y \in I_{k;\delta, M + 1} }\P\Big(  
      \underbrace{ \bm X^{\eta|b}_{ \sigma^{\eta|b}_{i;\epsilon}(\bm y) }(\bm y) \in I_j }_{ \text{($k$,II)} } \Big).
\end{align*}
Applying Theorem~\ref{theorem: first exit analyses, adapted} onto $I_{i;\delta,M}$, we yield (for all but countably many $\delta > 0$ small enough) 
\begin{align}
    \limsup_{\eta \downarrow 0}\sup_{ \bm x \in \bar B_\epsilon(\bm m_i)}
    \P\big(\text{($k$,I)}\big)
    & \leq 
     \exp\Big( - \widecheck{\mathbf C}_i\big( (I_{i;\delta,M})^\complement\big)\cdot (t-\Delta) \Big)\cdot
    \frac{ 
    \widecheck{\mathbf C}_i\big( (I_k)^-\big)
    }{  \widecheck{\mathbf C}_i\big( (I_{i;\delta,M})^\complement\big) }
     \nonumber \\ 
    &
    \leq 
   \exp\big( - q_b(i)\cdot (t - \Delta)\big)\cdot \frac{q_b(i,k)}{q_b(i)}
   \quad\text{by \eqref{proof, bound for measure check C, 2, proposition: transition time, metastability} and \eqref{proof, mass I j = I j closure, proposition: transition time, metastability}.}
   \label{proof, ineq 3 for UB, part ii, proposition: transition time, metastability}
\end{align}
Next, we analyze the probability of event ($k$,II).
If $k = j$, we plug in the trivial upper bound $\P(\text{($k$,II)}) \leq 1$.
If $k \neq j$,
on the event ($k$,II), we have that $\big(\bm X^{\eta|b}_t(\bm y)\big)_{t \geq 0}$ visited $B_{\epsilon}(\bm m_j)$ before visiting the $\epsilon$-neighborhood of any other local minima, despite the fact that the initial value $\bm y$ belongs to $I_{k;\delta,M+1} \subset I_k$.
Then, by Lemma~\ref{lemma: cycle, efficient return},
for any $\delta > 0$ small enough (so that $I_{k;\delta,M} \neq \emptyset$) we have
$
\limsup_{\eta \downarrow 0}
\sup_{ \bm y \in I_{k;\delta, M+ 1} }
\P(\text{($k$,II)}) = 0\ \forall k \neq j.
$
Combining \eqref{proof, ineq 1 for UB, part ii, proposition: transition time, metastability}--\eqref{proof, ineq 3 for UB, part ii, proposition: transition time, metastability},
we get
\begin{align}
    & \limsup_{\eta \downarrow 0}\sup_{\bm x \in \bar B_\epsilon(\bm m_i)}
    \P\Big(
    {\sigma^{\eta|b}_{i;\epsilon}(\bm x)} \cdot {\lambda^*_{b;i}(\eta)} > t,\
        \bm X^{\eta|b}_{ \sigma^{\eta|b}_{i;\epsilon}(\bm x) }(\bm x) \in I_j
    \Big)
     \leq 
     \Delta + 
      \exp\big( - q_b(i)\cdot (t- \Delta) \big) \cdot \frac{q_b(i,j)}{q_b(i)}.
      \nonumber
\end{align}
Sending $\Delta \downarrow 0$, we conclude the proof of the upper bound \eqref{goal, UB, part ii, proposition: transition time, metastability}.
\end{proof}

Now, we are ready to prove Proposition \ref{proposition: hat X converges to CTMC, metastability}.

\begin{proof}[Proof of Proposition \ref{proposition: hat X converges to CTMC, metastability}]
\linksinpf{proposition: hat X converges to CTMC, metastability}
Recall the definitions in \eqref{def: hat tau, 1, metastability}--\eqref{def: hat I k, metastability}, and let
\begin{align*}
    U^{\eta,\epsilon}_k \delequal \big({\tau^{\eta,\epsilon|b}_k(\bm x_0) - \tau^{\eta,\epsilon|b}_{k-1}(\bm x_0)}\big)\cdot{\lambda^*_b(\eta)},
    \qquad 
    V^{\eta,\epsilon}_k \delequal \bm m_{ \hat{\mathcal I}^{\eta,\epsilon|b}_k(\bm x_0) };
\end{align*}
We first show that claims $(i)$ and $(ii)$ follow directly from the next claim:
for any $\epsilon > 0$ small enough,
\begin{align}
    (U^{\eta,\epsilon}_1,V^{\eta,\epsilon}_1,U^{\eta,\epsilon}_2,V^{\eta,\epsilon}_2,\cdots) \Rightarrow (U_1,V_2,U_2,V_2,\cdots)
    \qquad 
    \text{ as }\eta \downarrow 0,
    \label{proof, claim, proposition: hat X converges to CTMC, metastability}
\end{align}
where the law of $U_j$'s and $V_j$'s are defined in \eqref{def Y * b t, metastability}.
Specifically, we only consider $\epsilon > 0$ small enough such that claim \eqref{proof, claim, proposition: hat X converges to CTMC, metastability} holds.
In light of Lemma \ref{lemma weak convergence of jump process} and Proposition \ref{proposition: Y * b is a CTMC},
\eqref{proof, claim, proposition: hat X converges to CTMC, metastability} verifies part $(i)$ of Proposition~\ref{proposition: hat X converges to CTMC, metastability}.
Regarding part $(ii)$, note that $\hat{\bm X}^{\eta,\epsilon|b}_t$ in \eqref{def: marker process, hat X eta epsilon b} is a step function (i.e., piece-wise constant) that only takes values in $\mathcal M = \{\bm m_j:\ j = 1,2,\cdots,K\}$, which is a finite set.
Let
\begin{align*}
    A_N \delequal \{ \xi \in \D[0,T]:\ \xi\text{ is a step function with at most $N$ jumps and only takes values in }\mathcal M \}.
\end{align*}
First, the finite-dimensional nature of $A_N$ (i.e., at most $N$ jumps over $[0,T]$, only $K$ possible values) implies that $A_N$ is a compact set in $(\D[0,T],\dlp{[0,T]})$.
Besides,
\begin{align*}
\limsup_{n \to \infty}\P(\hat{\bm X}^{\eta_n,\epsilon|b}_{\boldsymbol{\cdot}}(\bm x_0) \notin A_N) & = \limsup_{n \to \infty}\P( \sum_{j = 1}^{N+1} U^{\eta_n,\epsilon}_j \leq T  )
\leq \P( \sum_{j = 1}^{N+1} U_j \leq T  ),
\end{align*}
where the last inequality follows from $(U^{\eta_n,\epsilon}_1,\cdots,U^{\eta_n,\epsilon}_N) \Rightarrow (U_1,\cdots,U_N)$.
By part $(i)$ of Proposition~\ref{proposition: Y * b is a CTMC},
we confirm that $\lim_{N\to\infty}\limsup_{n \to \infty}\P(\hat{\bm X}^{\eta_n,\epsilon|b}_{\boldsymbol{\cdot}}(\bm x_0) \notin A_N) = 0$,
which verifies the tightness of $\big(\hat{\bm X}^{\eta_n,\epsilon|b}_{\boldsymbol{\cdot}}(\bm x_0)\big)_{ n \geq 1 }$ with $\eta_n \downarrow 0$.

Now, it only remains to prove \eqref{proof, claim, proposition: hat X converges to CTMC, metastability}.
This is equivalent to proving that, for each $N \geq 1$, 
$(U^{\eta,\epsilon}_1,V^{\eta,\epsilon}_1,\cdots,U^{\eta,\epsilon}_N,V^{\eta,\epsilon}_N)$ converges in distribution to $(U_1,V_1,\cdots,U_N,V_N)$ as $\eta \downarrow 0$.
Fix some $N = 1,2,\cdots$.
First, by definitions we have $U_1 = 0$ and $V_1 = \bm m_{i_0}$.
From part $(i)$ of Lemma~\ref{proposition: transition time, metastability},
we get $(U^{\eta,\epsilon}_1,V^{\eta,\epsilon}_1) \Rightarrow (0,\bm m_i) = (U_1,V_1)$ as $\eta \downarrow 0$.
Next, for any $n \geq 1$, any $t_l \in (0,\infty)$, any sequence $v_l \in \{\bm m_i:\ i \in [K]\}$, and any $t > 0$, $i,j \in [K]$ with $i \neq j$,
it follows from part $(ii)$ of Lemma~\ref{proposition: transition time, metastability} that
\begin{align*}
    & \lim_{\eta \downarrow 0}\P\bigg( U^{\eta,\epsilon}_{n+1} \leq t,\ V^{\eta,\epsilon}_{n+1} = \bm m_j\ \bigg|\ 
    V^{\eta,\epsilon}_{n} = \bm m_i,\ V^{\eta,\epsilon}_l = v_l\ \forall l \in [n-1],\ U^{\eta,\epsilon}_l \leq t_l\ \forall l \in [n]
    \bigg)
    \\ 
    & = 
    \begin{cases}
        \frac{q_b(i,j)}{q_b(i)} & \text{ if }\bm m_i \notin V^*_b,
        \\ 
        \frac{q_b(i,j)}{q_b(i)} \cdot \Big( 1 - \exp\big(-q_b(i)t\big)\Big) & \text{ if }\bm m_i \in V^*_b.
    \end{cases}
\end{align*}
This coincides with the conditional law of $\P\Big( U_{n+1} \leq t,\ V_{n+1} = \bm m_j\ \Big|\ V_n = \bm m_i,\ (V_j)_{j = 1}^{n-1},\ (U_j)_{j = 1}^n\Big)$ specified in \eqref{def Y * b t, metastability}.
By arguing inductively, we conclude the proof.
\end{proof}

\subsection{Proof of Proposition~\ref{proposition: hat X close to X, metastability}}

Moving onto the proof of Proposition \ref{proposition: hat X close to X, metastability},
we first prepare a lemma that establishes the weak convergence from $\bm X^{\eta|b}_{ \floor{ {\boldsymbol{\cdot}}/\lambda^*_b(\eta) } }(\bm x)$ to $\hat{\bm X}^{\eta,\epsilon|b}_{\boldsymbol{\cdot}}(\bm x)$
in terms of finite dimensional distributions.

\begin{lemma}\label{lemma:  hat X close to X, fdd, metastability} 
\linksinthm{lemma:  hat X close to X, fdd, metastability} 
Given any $t > 0$ and $\bm x \in \bigcup_{i \in [K]}I_i$ with $\norm{\bm x} < M$,
\begin{enumerate}[(i)]
    \item 
    $
    \lim_{\eta \downarrow 0}\P\Big( \norm{\bm X^{\eta|b}_s(\bm x)} > M\text{ for some }s \leq t/\lambda^*_b(\eta)  \Big) = 0;
    $

    \item $\lim_{\eta \downarrow 0}
    \P\Big(
    \norm{ \bm X^{\eta|b}_{ \floor{ t/\lambda^*_b(\eta) } }(\bm x) - \hat{\bm X}^{\eta,\epsilon|b}_{t}(\bm x) } \geq \epsilon
    \Big) = 0
    $
    for all $\epsilon > 0$ small enough.
\end{enumerate}
\end{lemma}

\begin{proof}\linksinpf{lemma:  hat X close to X, fdd, metastability} 
Throughout this proof, let $\bar\epsilon$ be specified as in \eqref{choice of bar epsilon, metastability}.

\medskip
\noindent
$(i)$ We prove a stronger result.
Let $I^{(M,\delta)} = B_M(\bm 0)\setminus S(\delta)$,
where $S(\delta)$ is the $\delta$-enlargement of the boundary sets defined in \eqref{def: boundary set S delta}.
Recall the definition of $\hat \tau^{\eta,\epsilon|b}_k(\bm x)$ in \eqref{def: hat tau, 1, metastability} and \eqref{def: hat tau, 2, metastability}.
For each $N \in \Z_+$, on the event
\begin{align*}
& \Bigg(
\bigcap_{ k = 0 }^{N-1}
\underbrace{ \Big\{ 
\bm X^{\eta|b}_s(\bm x) \in I^{(M,\delta)}\ \forall s \in \big[ \hat \tau^{\eta,\epsilon|b}_k(\bm x), \hat \tau^{\eta,\epsilon|b}_{k+1}(\bm x) \big]
\Big\} }_{ A_k(\eta,\delta) }
\Bigg)
\cap
\underbrace{\Big\{ \hat \tau^{\eta,\epsilon|b}_1(\bm x) \leq t/\lambda^*_b(\eta)  \Big\}}_{ B_1(\eta) }
\cap 
\underbrace{
\Big\{ \hat \tau^{\eta,\epsilon|b}_N(\bm x) > t/\lambda^*_b(\eta) \Big\}
}_{B_2(\eta)},
\end{align*}
we have $\bm X^{\eta|b}_s(\bm x) \in I^{(M,\delta)}$ for all $s \in [ \hat \tau^{\eta,\epsilon|b}_1(\bm x), \hat \tau^{\eta,\epsilon|b}_{N}(\bm x)]$ and $\hat \tau^{\eta,\epsilon|b}_1(\bm x) \leq t/\lambda^*_b(\eta) < \hat \tau^{\eta,\epsilon|b}_N(\bm x)$.
Therefore, it suffices to show that given any $\Delta > 0$, there exist $N$ and $\delta > 0$ such that
\begin{align}
    \limsup_{\eta \downarrow 0}
    \P\Big(\big(B_1(\eta)\big)^\complement\Big) 
    + \P\Big( \big(B_2(\eta)\big)^\complement\Big) 
    + \sum_{k = 0}^{N-1}\P\Big(\big(A_k(\eta,\delta)\big)^\complement\Big)
    < \Delta.
\label{subgoal, goal 1, proposition: hat X close to X, metastability}
\end{align}

Let $i \in [K]$ be the unique index such that $\bm x \in I_i$
and let 
$R_{i;\epsilon}^{\eta|b}(\bm x)$ be the first hitting time to the $\epsilon$-neighborhood of $\bm m_i$ (see \eqref{def: return time R b j epsilon x}).
Since $\hat \tau^{\eta,\epsilon|b}_1(\bm x)$ is the first visit time to $\bigcup_{l \in [K]}B_\epsilon(\bm m_l)$ (see  \eqref{def: hat tau, 1, metastability}),
we have $\hat \tau^{\eta,\epsilon|b}_1(\bm x) \leq R_{i;\epsilon}^{\eta|b}(\bm x)$ and hence
\begin{equation} \label{bound 1, goal 1, proposition: hat X close to X, metastability}
    \begin{split}
    \limsup_{\eta \downarrow 0}
     \P\Big(\big(B_1(\eta)\big)^\complement\Big) 
    & =
    \limsup_{\eta \downarrow 0}\P\Big(
    \hat \tau^{\eta,\epsilon|b}_1(\bm x) > t/\lambda^*_b(\eta)
    \Big)
    \leq 
    \limsup_{\eta \downarrow 0}\P\Big(
     \lambda^*_b(\eta) \cdot R_{i;\epsilon}^{\eta|b}(\bm x) > t
    \Big)
    \\ 
    & =0
    \qquad 
    \text{using Lemma~\ref{proposition: transition time, metastability} $(i)$}.
    \end{split}
\end{equation}

Next, for the limiting process $\bm Y^{*|b}_t$ in Theorem~\ref{corollary irreducible case},
recall that we have collected a few important properties at the beginning of Section~\ref{subsec: proof, propositions, metastability} (with detailed proofs deferred to Section~\ref{sec: appendix, CTMC Y * b}).
In particular, for the $U_j$'s defined in \eqref{def Y * b t, metastability},
we can fix some $N$ large enough such that 
$
\P(U_1 + \cdots + U_N \leq t)< \Delta/2.
$
Then, by the proof of Proposition~\ref{proposition: hat X converges to CTMC, metastability} above,
\begin{equation}  \label{bound 2, goal 1, proposition: hat X close to X, metastability}
    \begin{split}
        \limsup_{\eta \downarrow 0}\P\Big( \big(B_2(\eta)\big)^\complement\Big) 
    & =
    \limsup_{\eta \downarrow 0}
    \P\Bigg(
    \sum_{k = 0}^{N-1}
    \big({\tau^{\eta,\epsilon|b}_{k+1}(\bm x) - \tau^{\eta,\epsilon|b}_{k}(\bm x)}\big)\cdot{\lambda^*_b(\eta)}
    \leq t
    \Bigg)
    \\ 
    & \leq \P(U_1 + \cdots + U_N \leq t)< \Delta/2.
    \end{split}
\end{equation}
Meanwhile, recall the definition of 
$\sigma^{\eta|b}_{i;\epsilon}(\bm x) = \min\{s \geq 0:\ \bm X^{\eta|b}_s(\bm x) \in \bigcup_{ l \neq i }B_\epsilon(\bm m_l)\}$ in \eqref{def stopping time sigma eta b i epsilon x, metastability}.
By the strong Markov property at $\hat \tau^{\eta,\epsilon|b}_k(\bm x)$,
\begin{align*}
    \sup_{k \geq 1}\P\Big(\big(A_k(\eta,\delta)\big)^\complement\Big)
    \leq 
    \max_{ l \in [K] }\sup_{ \bm y \in \bar B_\epsilon(\bm m_l) }
    \P\bigg( \exists t < \sigma^{\eta|b}_{l;\epsilon}(\bm y)\ s.t.\ 
        \bm X^{\eta|b}_t(\bm y) \in S(\delta)\text{ or }\norm{ \bm X^{\eta|b}_t(\bm y)} > M
    \bigg).
\end{align*}
By Lemma \ref{lemma: unlikely to exit M or visit s i, metastability},
for all but countably many $\delta > 0$ small enough we have
 $\limsup_{\eta \downarrow 0}\P\big(\big(A_k(\eta,\delta)\big)^\complement\big) \leq \frac{\Delta}{2N}\ \forall k \in [N-1]$.
Likewise, the case of $k  = 0$ can be bounded using part $(i)$ of Lemma~\ref{proposition: transition time, metastability}.
Combining this bound with \eqref{bound 1, goal 1, proposition: hat X close to X, metastability} and \eqref{bound 2, goal 1, proposition: hat X close to X, metastability},
we finish the proof of \eqref{subgoal, goal 1, proposition: hat X close to X, metastability}.

\medskip
\noindent
$(ii)$
If $\bm X^{\eta|b}_{ \floor{ t/\lambda^*_b(\eta) } }(\bm x) \in \bigcup_{ l \in [K] }B_\epsilon(\bm m_l)$,
then by the definition of $\hat{\bm X}^{\eta,\epsilon|b}_t(\bm x)$ as the marker of the {last visited local minimum} (see \eqref{def: hat tau, 1, metastability}--\eqref{def: marker process, hat X eta epsilon b}),
we must have $\norm{ \bm X^{\eta|b}_{ \floor{ t/\lambda^*_b(\eta) } }(\bm x) - \hat{\bm X}^{\eta,\epsilon|b}_t(\bm x) } < \epsilon$.
Therefore, it suffices to show that for any $\epsilon \in (0,\bar\epsilon)$
(where $\bar\epsilon$ is characterized in \eqref{choice of bar epsilon, metastability})
$$
\lim_{\eta\downarrow 0}\P\Big( \bm X^{\eta|b}_{ \floor{ t/\lambda^*_b(\eta) } }(\bm x) \in \bigcup_{ l \in [K] }B_\epsilon(\bm m_l)\Big) = 1.
$$
Pick some $\delta_t \in (0,\frac{t}{3})$, $\delta > 0$.
Recall that $H(\cdot) = \P(\norm{\bm Z_1}> \cdot)$, and define the event
$$
\text{(I)} = \Big\{ \bm X^{ \eta|b}_{  \floor{ t/\lambda^*_b(\eta) } - \floor{ {2\delta_t}/{ H(\eta^{-1}) } }  }(\bm x) \in I^{(M,\delta)} \Big\}.
$$
Let $t_1(\eta) = \floor{ t/\lambda^*_b(\eta) } - \floor{ {2\delta_t}/{ H(\eta^{-1}) } }$.
On the event (I), let
$$
R^{\eta} \delequal \min\bigg\{ s \geq t_1(\eta):\ \bm X^{\eta|b}_s(\bm x) \in \bigcup_{l \in [K]}B_{\epsilon/2}(\bm m_l)\bigg\},
$$
and set $\hat{\mathcal I}^\eta$ by the rule
$
\hat{\mathcal I}^\eta = j \iff \bm X^{\eta|b}_{R^{\eta}}(\bm x) \in I_j.
$
Then, we define the event
\begin{align*}
    \text{(II)}
    =
    \Big\{ R^\eta - t_1(\eta) \leq  {\delta_t}/{ H(\eta^{-1}) } \Big\}.
\end{align*}
On the event $\text{(I)}\cap \text{(II)}$, note that
$
\floor{ t/\lambda^*_b(\eta) } -  \floor{ {2\delta_t}/{ H(\eta^{-1}) } } \leq R^\eta \leq  \floor{ t/\lambda^*_b(\eta) }.
$
Furthermore,
let $\tau^\eta \delequal \min\{ s \geq R^\eta:\ \bm X^{\eta|b}_s(\bm x) \notin B_{\epsilon}(\bm m_{ \hat{\mathcal I}^\eta }) \}$,
and define event 
\begin{align*}
    \text{(III)} = \Big\{ \tau^\eta - R^\eta > 2\delta_t / H(\eta^{-1}) \Big\}.
\end{align*}
On the event $\text{(I)}\cap \text{(II)} \cap \text{(III)}$, we must have $\tau^\eta > \floor{ t/\lambda^*_b(\eta) } \geq R^\eta$,
and hence
$
\bm X^{\eta|b}_{ \floor{ t/\lambda^*_b(\eta) } }(\bm x) \in \bigcup_{ l \in [K] }B_{\epsilon}(\bm m_l).
$
Therefore, suppose that given each $\Delta > 0$ there exist $\delta_t \in(0,\frac{t}{3})$ and $\delta > 0$ such that 
\begin{align}
    \liminf_{\eta \downarrow 0}\P\big(\text{(I)}\big) & \geq 1 - \Delta,
    \label{goal, 1, proposition: hat X close to X, metastability}
    \\
     \liminf_{\eta \downarrow 0}\P\big(\text{(II)}\ \big|\ \text{(I)}\big) & \geq 1,
     \label{goal, 2, proposition: hat X close to X, metastability}
     \\ 
     \liminf_{\eta \downarrow 0}\P\big(\text{(III)}\ \big|\ \text{(I)}\cap \text{(II)}\big) & \geq 1 - \Delta.
\label{goal, 3, proposition: hat X close to X, metastability}     
\end{align}
Then, we immediately get $\liminf_{\eta \downarrow 0}\P(\text{(I)}\cap \text{(II)} \cap \text{(III)}) \geq (1 - \Delta)^2$.
Sending $\Delta \downarrow 0$, we conclude the proof.
The rest of this proof is devoted to establishing 
\eqref{goal, 1, proposition: hat X close to X, metastability}
\eqref{goal, 2, proposition: hat X close to X, metastability}
\eqref{goal, 3, proposition: hat X close to X, metastability},
where we fix some $\epsilon \in (0,\bar\epsilon)$ and $\Delta > 0$.

\medskip
\textbf{Proof of \eqref{goal, 1, proposition: hat X close to X, metastability}}.
This has been established in the proof for part $(i)$.

\medskip
\textbf{Proof of \eqref{goal, 2, proposition: hat X close to X, metastability}}.
This claim holds for any $\delta_t \in (0,t/3)$,
and can be obtained by combining Lemma~\ref{lemma: cycle, efficient return}
with the preliminary fact that, given each $T > 0$, the inequality  $T/\eta < \delta_t/H(\eta^{-1})$ holds for all $\eta$ small enough (due to $H(\eta^{-1}) \in \RV_{ \alpha }(\eta)$ as $\eta \downarrow 0$ with $\alpha > 1$).
\elaborate{We show that the claim holds for all $\delta_t \in (0,t/3)$.
Due to $H(x) \in \RV_{-\alpha}(x)$ and $\alpha > 1$,
given any $T > 0$ we have $T/\eta < \delta_t/H(\eta^{-1})$ eventually for all $\eta$ small enough.
Recall that $I_{j;\delta,M} = (s_{j - 1} + \delta, s_j - \delta) \cap (-M,M)$. 
By Markov property at $t_1(\eta)$, for any $T > 0$ it holds for all $\eta > 0$ small enough that
\begin{align*}
    \P\Big(\text{(II)}^c\ \Big|\ \text{(I)}\Big)
& \leq 
\max_{k \in [n_\text{min}]}\sup_{ y \in I_{k;\delta,M}  }
\P\bigg( X^{\eta|b}_j(y) \notin \bigcup_{ l \in [n_\text{min}] }(m_l - \frac{\epsilon}{2},m_l + \frac{\epsilon}{2})\ \forall j \leq \delta_t/H(\eta^{-1}) \bigg)
\\ 
& \leq 
\max_{k \in [n_\text{min}]}\sup_{ y \in I_{k;\delta,M}  }
\P\bigg(
R^{\eta|b}_{k;\epsilon/2}(y) > \delta_t/H(\eta^{-1})
\bigg)
\\ 
&
\leq 
\max_{k \in [n_\text{min}]}\sup_{ y \in I_{k;\delta,M}  }
\P\bigg(
R^{\eta|b}_{k;\epsilon/2}(y) > T/\eta
\bigg)
\end{align*}
where 
$
R^{\eta|b}_{k;\epsilon/2}(y) = \min\{j \geq 0:\ X^{\eta|b}_j(y) \in (m_k - \frac{\epsilon}{2},m_k + \frac{\epsilon}{2})\}.
$

Let 
$
\bm t_k(x,\epsilon) \delequal \inf\{ t \geq 0:\ \bm y_t(x) \in (m_k - \epsilon,m_k + \epsilon)   \}.
$
By Assumption \ref{assumption: geometry, metastability}, $\bm t_k(x,\frac{\epsilon}{4})< \infty$ for all $x \in [-M - 1,M + 1]\cap [s_{k-1} + \frac{\delta}{2},s_k - \frac{\delta}{2}]$,
with 
$\bm t_k(\ \cdot\ ,\frac{\epsilon}{4})$ being continuous over $[-M - 1,M+1]\cap [s_{k-1} + \frac{\delta}{2},s_k - \frac{\delta}{2}]$.
As a result, we can fix $T\in(0,\infty)$ large enough such that 
\begin{align*}
    T >
     \sup\Big\{ \bm t_k(x,\frac{\epsilon}{4}):\  x \in [-M - 1,M+1]\cap [s_{k-1} + \frac{\delta}{2},s_k - \frac{\delta}{2}]  \Big\}
     \qquad 
     \forall k \in [n_\text{min}].
\end{align*}
For each $k \in [n_\text{min}]$,
by applying Lemma \ref{lemma: cycle, efficient return} onto $(-M - 1,M+1) \cap (s_{k-1},s_k)$,
we are able to show that 
$
\limsup_{\eta \downarrow 0}
\sup_{ y \in I_{k;\delta,M}  }
\P\Big(
R^{\eta|b}_{k;\epsilon/2}(y) > T/\eta
\Big) = 0.
$
This concludes the proof of claim \eqref{goal, 2, proposition: hat X close to X, metastability}.
}

\medskip
\textbf{Proof of \eqref{goal, 3, proposition: hat X close to X, metastability}}.
By the strong Markov property at $R^\eta$,
\begin{align}
    \P\Big(\text{(III)}^\complement\ \Big|\ \text{(I)}\cap \text{(II)}\Big)
    \leq 
    \max_{k \in [K]}\sup_{ \bm y \in \bar B_{\epsilon/2}(\bm m_k) }
    \P\bigg(
    \exists s \leq \frac{2\delta_t}{H(\eta^{-1})}\ s.t.\ 
    \bm X^{\eta|b}_s(\bm y) \notin B_{\epsilon}(\bm m_k)
    \bigg).
    \label{intermediate bound, goal, 3, proposition: hat X close to X, metastability}
\end{align}
Also, note that $\epsilon < \bar\epsilon < b$; see \eqref{choice of bar epsilon, metastability}.
For each $k \in [K]$, by Theorem~\ref{theorem: first exit analyses, adapted} under the choice of $I = B_\epsilon(\bm m_k)$,
we obtain some $c_{k,\epsilon} \in (0,\infty)$ 
such that for any $u > 0$,
\begin{align}
    \limsup_{\eta\downarrow 0}
    \sup_{ \bm y \in \bar B_{\epsilon/2}(\bm m_k) }
    \P\bigg(
    \exists j \leq \frac{u}{H(\eta^{-1})}\ s.t.\ \bm X^{\eta|b}_j(\bm y) \notin B_\epsilon(\bm m_k)
    \bigg)
    \leq 1 - \exp( - c_{k,\epsilon} \cdot u).
    \label{final result, goal, 3, proposition: hat X close to X, metastability}
\end{align}
By picking $\delta_t$ small enough, we ensure that $\max_{k \in [K]}1 - \exp(-c_{k,\epsilon}\cdot 2\delta_t) < \Delta$, 
thus completing the proof of claim \eqref{goal, 3, proposition: hat X close to X, metastability}.
To conclude, we elaborate a bit more on the constant $c_{\epsilon}$ and the application of Theorem~\ref{theorem: first exit analyses, adapted} above.
The event on the RHS of \eqref{intermediate bound, goal, 3, proposition: hat X close to X, metastability}
is about the exit from an $\epsilon$-neighborhood of $\bm m_k$.
Due to $\epsilon < \bar\epsilon < b$, this can be achieved by one jump (in the sense of $\mathcal J^I_b$ in Theorem~\ref{theorem: first exit analyses, adapted}) if we start from $\bm m_j$.
Specifically, adopting the notations in Theorem~\ref{theorem: first exit analyses, adapted},
we let
\begin{align*}
    c_{k,\epsilon} \delequal \widecheck{\mathbf C}^{(1)|b}\big( (B_\epsilon(\bm m_k))^\complement  \big)
=
\int \mathbf{I}\Big\{ w \cdot \norm{\bm \sigma(\bm m_j)\bm \theta} > \epsilon \Big\} \nu_\alpha(d w)\mathbf S(d\theta),
\end{align*}
where the equality follows from \eqref{def: measure check C k b}.
On one hand, the non-degeneracy of $\bm \sigma(\cdot)$ (see Assumption~\ref{assumption: nondegeneracy of diffusion coefficients})
implies that $\inf_{ \norm{\bm \theta} = 1 }\norm{ \bm \sigma(\bm m_j)\bm \theta } > 0$,
and hence the existence of some $\underline{w}_{k,\epsilon} > 0$ such that 
$
c_\epsilon \geq \nu_\alpha[\underline{w}_{k,\epsilon},\infty) = (\underline{w}_{k,\epsilon})^{-\alpha} > 0;
$
see \eqref{def: measure nu alpha}.
On the other hand, 
under the choice of $I = B_\epsilon(\bm m_k)$,
we have 
$
\widecheck{\mathbf C}^{(1)|b}\big(\partial I) = 0
$
due to the absolute continuity of measures $\nu_\alpha$ and $\mathbf S$ (see Assumption~\ref{assumption gradient noise heavy-tailed}).
This verifies the conditions in Theorem~\ref{theorem: first exit analyses, adapted},
allowing us to conclude that $c_{k,\epsilon} <\infty$ and obtain \eqref{final result, goal, 3, proposition: hat X close to X, metastability}.
\end{proof}

The next result provides an upper bound over the proportion of time that $\bm X^{\eta|b}_t(\bm x)$ is not close enough to a local minimum.

\begin{lemma} \label{lemma: metastability, proportion of time not around local minima}
\linksinthm{lemma: metastability, proportion of time not around local minima}
    Given $\epsilon \in (0,\bar\epsilon)$,
    it holds for all $t \in (0,1)$ small enough that 
    \begin{align*}
        \limsup_{\eta \downarrow 0}
        \max_{ i:\ \bm m_i \in V^*_b }\sup_{\bm x \in B_{\epsilon/2}(\bm m_i)}
        \P\bigg(
            \int_0^t \mathbf{I}\Big\{
                \bm X^{\eta|b}_{ \floor{ s/\lambda^*_b(\eta) }  }(\bm x) \notin B_\epsilon(\bm m_i) 
            \Big\}ds > t^2
        \bigg) < q^* t,
    \end{align*}
where $q^* \in (0,\infty)$ is a constant that does not vary with $t$ or $\epsilon$.
\end{lemma}

\begin{proof} \linksinpf{lemma: metastability, proportion of time not around local minima}
There are only finitely many elements in $V^*_b$.
Therefore, it suffices to fix some $\bm m_i \in V^*_b$ (recall that $I_i$ is the attraction field associated with $\bm m_i$, and w.l.o.g.\ we assume $\bm m_i$ = 0 in this proof) and prove that
\begin{align}
    \limsup_{\eta \downarrow 0}
       \sup_{\bm x \in B_\epsilon(\bm 0)}
        \P\bigg(
            \int_0^t \mathbf{I}\Big\{
                \bm X^{\eta|b}_{ \floor{ s/\lambda^*_b(\eta) }  }(\bm x) \notin B_\epsilon(\bm 0)
            \Big\}ds > t^2
        \bigg) < q^* t
        \label{proof, main goal, lemma: metastability, proportion of time not around local minima}
\end{align}
holds for all $t > 0$ small enough, where $q^* \in (0,\infty)$ is a constant that does not vary with $\epsilon$ or $t$.

Let $T^\eta_0 = 0$, and (for all $k \geq 1$)
\begin{align*}
    S^\eta_k \delequal \min\{ u > T^\eta_{k-1}:\ \bm X^{\eta|b}_u(\bm x) \notin B_\epsilon(\bm 0) \},
    \qquad 
    T^\eta_k  \delequal \min\{ u > S^\eta_k:\ \bm X^{\eta|b}_u(\bm x) \in B_{\epsilon/2}(\bm 0) \}.
\end{align*}
Then, by defining $N^\eta \delequal \max\{ k \geq 0:\ S^\eta_k \leq t/\lambda^*_b(\eta)  \}$,
we have
\begin{align}
\#\Big\{ u \leq \floor{ t/\lambda^*_b(\eta) }:\ 
\bm X^{\eta|b}_u(\bm x) \notin B_\epsilon(\bm 0)
\Big\}
\leq \sum_{k = 1}^{N^\eta}T^\eta_k \wedge \floor{ t/\lambda^*_b(\eta) } - S^\eta_k.    
\label{proof, property on N eta, 1, lemma: metastability, proportion of time not around local minima}
\end{align}
Next,
recall that $\alpha > 1$ is the heavy-tailed index in Assumption \ref{assumption gradient noise heavy-tailed},
and the time scaling $\lambda^*_b(\eta)$ is defined in \eqref{def scale function lambda * b eta}
with $\lambda^*_b(\eta) \in \RV_{ \mathcal J^*_b\cdot (\alpha-1)  + 1 }(\eta)$.
Fix some $\beta \in (0,\alpha - 1)$, and let 
\begin{align}
    k(\eta) \delequal 1/\eta^{ (\mathcal J^*_b - 1)(\alpha - 1) + \beta }.
    \label{proof, def k eta, lemma: metastability, proportion of time not around local minima}
\end{align}
Given $\bm x \in B_{\epsilon/2}(\bm 0)$,
define events (with $I_{i;\delta,M}$ defined as in \eqref{def, attraction field I i truncated delta M})
\begin{align*}
    A^\eta_{t}(\bm x) & \delequal 
    \big\{
        \bm X^{\eta|b}_u(\bm x) \in I_{i;\epsilon,M}\text{ for all }u \leq \floor{t/\lambda^*_b(\eta)}
    \big\},
    \\
    B^\eta_{\delta}(\bm x) & \delequal 
    \big\{
        \text{for each $j \leq k(\eta)$}, \exists u \in [T^\eta_{j-1} + 1, S^\eta_j]\text{ s.t. }\eta\norm{\bm Z_u} >\delta
    \big\}.
\end{align*}
On  the event $B^\eta_{\delta}(\bm x)$,
note that 
$$N^\eta \wedge k(\eta) \leq 
W^\eta \delequal 
\#\{
    u \leq \floor{ t/\lambda^*_b(\eta) }:\ \eta \norm{\bm Z_u}> \delta
\}.$$
Next, define the event
\begin{align}
    F^\eta_t \delequal 
    \big\{
        k(\eta) > W^\eta
    \big\}.
    \nonumber
\end{align}
On $B^\eta_{\delta}(\bm x) \cap F^\eta_t$, we must have
\begin{align}
    N^\eta \leq W^\eta < k(\eta) = 1/\eta^{ (\mathcal J^*_b - 1)(\alpha - 1) + \beta }.
    \nonumber
\end{align}
Furthermore, given a constant $T \in (0,\infty)$,
let 
$$
E^\eta_{t,T}(\bm x) \delequal \{T^\eta_k \wedge \floor{ t/\lambda^*_b(\eta) } - S^\eta_k \leq T/\eta\ \forall k \geq 1\}.$$
On event $B^\eta_{\delta}(\bm x) \cap F^\eta_t \cap  E^\eta_{t,T}(\bm x)$,
observe that
\begin{align*}
   \#\{ u \leq \floor{ t/\lambda^*_b(\eta) }:\ 
\bm X^{\eta|b}_u(\bm x) \notin B_\epsilon(\bm 0)
\}
& \leq \sum_{j = 1}^{N^\eta}T^\eta_j \wedge \floor{ t/\lambda^*_b(\eta) } - S^\eta_j 
\qquad
\text{by \eqref{proof, property on N eta, 1, lemma: metastability, proportion of time not around local minima}}
\\ 
&
\leq 
k(\eta) \cdot T/\eta 
= T /\eta^{1 + \beta + (\mathcal J^*_b - 1)(\alpha - 1)},
\end{align*}
and hence
$$
\int_0^t \mathbf{I}\Big\{
                \bm X^{\eta|b}_{ \floor{ s/\lambda^*_b(\eta) }  }(\bm x) \notin B_\epsilon(\bm 0) 
            \Big\}ds
    \leq 
    \frac{T /\eta^{1 + \beta + (\mathcal J^*_b - 1)(\alpha - 1)} }{ 
        \floor{ 1/\lambda^*_b(\eta) }}.
$$
However, due to 
$
\lambda^*_b(\eta) \in \RV_{ \mathcal J^*_b\cdot (\alpha-1)  + 1 }(\eta)
$
and 
$\mathcal J^*_b\cdot (\alpha-1)  + 1
> 
(\mathcal J^*_b - 1)\cdot (\alpha-1)  + 1 + \beta
$
(recall that we've fixed some $\beta \in (0,\alpha - 1)$),
we have
\begin{align*}
    \lim_{\eta \downarrow 0}\frac{T /\eta^{1 + \beta + (\mathcal J^*_b - 1)(\alpha - 1)} }{ 
        \floor{ 1/\lambda^*_b(\eta) }
    } = 0.
\end{align*}
In summary,
to prove \eqref{proof, main goal, lemma: metastability, proportion of time not around local minima},
it suffices to show that given $t,\epsilon > 0$, there exist $\delta$ and $T$ such that
\begin{align}
    &\limsup_{\eta \downarrow 0}
        \sup_{\bm x \in B_{\epsilon/2}(\bm 0)}
        \P\Big( \big(A^\eta_{t}(\bm x)\big)^\complement\Big) < q^* t,
        \label{proof, goal 1, lemma: metastability, proportion of time not around local minima}
        \\ 
    &\lim_{\eta\downarrow 0}
        \sup_{\bm x \in B_{\epsilon/2}(\bm 0)}
        \P\Big(
        \big( B^\eta_{\delta}(\bm x) \big)^\complement
        \Big) = 0,
        \label{proof, goal 2, lemma: metastability, proportion of time not around local minima}
        \\ 
    &\lim_{\eta\downarrow 0}
        \P\big(
        \big( F^\eta_{t} \big)^\complement
        \big) = 0,
        \label{proof, goal 4, lemma: metastability, proportion of time not around local minima}
    \\
    &\lim_{\eta\downarrow 0}
        \sup_{\bm x \in B_{\epsilon/2}(\bm 0)}
        \P\Big(
         A^\eta_{t}(\bm x) \cap B^\eta_{\delta}(\bm x) \cap F^\eta_t  \cap \big( E^\eta_{t,T}(\bm x) \big)^\complement
        \Big) = 0,
        \label{proof, goal 3, lemma: metastability, proportion of time not around local minima}
\end{align}
where
$q^* \in (0,\infty)$ is a constant that does not vary with $\epsilon$ or $t$.

\medskip
\noindent
\textbf{Proof of Claim \eqref{proof, goal 1, lemma: metastability, proportion of time not around local minima}}.
This follows immediately from the first exit time analysis.
Specifically,
for each $\epsilon \in (0,\bar\epsilon)$, we have
$
\big(A^\eta_{t}(\bm x)\big)^\complement
\subseteq \{ \bm X^{\eta|b}_u(\bm x) \notin I_{i;\bar\epsilon,M}\text{ for some }u \leq \floor{t/\lambda^*_b(\eta)}  \}.
$
Then, 
the properties \eqref{choice of bar epsilon, metastability} and  \eqref{proof, property, zero mass on boundary of I i delta M, delta = bar epsilon} allow us to apply Theorem~\ref{theorem: first exit analyses, adapted} under the choice of 
$I = I_{i;\bar\epsilon,M}$ and get 
\elaborate{recall that we have assumed w.l.o.g.\ that the local minimum $m_i \in V^*_b$ at hand is located at the origin, i.e., $m_i = 0$.
This implies
$
\mathcal J^*_b(V) = \ceil{\min\{ |s_{i-1}|, s_i \}/b};
$
that is, starting from the local minimum, it requires at least $\mathcal J^*_b(V)$ jumps (each bounded by $b$) to escape from the attraction field $(s_{i-1},s_i)$.
Furthermore, by our choice of $\bar\epsilon$ in \eqref{choice of bar epsilon, metastability}
(which is essentially due to the assumption that $|s_j-m_i|/b \notin \Z$ for all $i\in[n_\text{min}]$ and $j \in [n_\text{min}-1]$),
it holds for all $\epsilon \in (0,\bar\epsilon)$ that
$
\mathcal J^*_b(V) = \ceil{\min\{ |s_{i-1} + \frac{\epsilon}{2}|, s_i - \frac{\epsilon}{2} \}/b}.
$
For any $M \in (0,\infty)$ large enough, we then have 
$
\mathcal J^*_b(V) = \ceil{\min\{ |s_{i-1} + \frac{\epsilon}{2}|, s_i - \frac{\epsilon}{2}, M \}/b},
$
thus implying that, starting from the origin, it also requires at least $\mathcal J^*_b(V)$ jumps to escape from $(s_{i-1}+ \frac{\epsilon}{2},s_i - \frac{\epsilon}{2})\cap(-M,M)$.
By applying part $(a)$ of Result~\ref{theorem: first exit time, unclipped} onto $(s_{i-1} + \frac{\epsilon}{2},s_i - \frac{\epsilon}{2})\cap (-M,M)$,
we can find $q \in (0,\infty)$ such that}
\begin{align*}
    \limsup_{\eta \downarrow 0}
        \sup_{\bm x \in B_{\epsilon/2}(\bm 0) }
        \P\big( (A^\eta_{t}(\bm x))^\complement\big) \leq 1 - \exp(-qt),\qquad \forall t > 0,
\end{align*}
where we set $q = \max_{j \in [K]}\widecheck{\mathbf C}_j\big( (I_{j;\bar\epsilon,M})^\complement \big)$ (note that it does not vary with $\epsilon$ or $t$).
Lastly, for any $t > 0$ small enough, we have $1 - \exp(-qt) \leq 2q t$.
We conclude the proof by picking $q^* = 2q$.

\medskip
\noindent
\textbf{Proof of Claim \eqref{proof, goal 2, lemma: metastability, proportion of time not around local minima}}.
By the strong Markov property at each $T^\eta_k$,
\begin{align*}
    \sup_{\bm x \in B_{\epsilon/2}(\bm 0)}
        \P\big(
        ( B^\eta_{\delta}(\bm x) )^\complement
        \big)
    \leq k(\eta) \cdot 
    \underbrace{ \sup_{\bm y \in B_{\epsilon/2}(\bm 0)}
    \P\big(
        \bm X^{\eta|b}_u(\bm y) \notin B_\epsilon(\bm 0)\text{ for some }u < \tau^{>\delta}_1(\eta)
    \big)
    }_{\delequal p_\delta(\eta)},
\end{align*}
where
$
\notationdef{notation-large-jump-time}{\tau^{>\delta}_1(\eta)} \delequal{} 
\min\{ j \geq 1:\ \eta \norm{\bm Z_j} > \delta  \}.
$
Applying Lemma~\ref{lemma stuck at local minimum before large jump},
it holds for any $\delta > 0$ small enough that 
$
p_\delta(\eta)
=\bm{o}\big(1/k(\eta)\big).
$
This concludes the proof of claim \eqref{proof, goal 2, lemma: metastability, proportion of time not around local minima}.

\medskip
\noindent
\textbf{Proof of Claim \eqref{proof, goal 4, lemma: metastability, proportion of time not around local minima}}.
Recall that $H(x) = \P(\norm{\bm Z_1} > x) \in \RV_{-\alpha}(x)$ as $x \to \infty$,
and that 
$
{\lambda^*_b(\eta)} \in \RV_{ \mathcal J^*_b\cdot (\alpha-1)  + 1 }(\eta)
$
as $\eta \downarrow 0$ (see \eqref{def scale function lambda * b eta}).
Observe that
\begin{align*}
    \P\big(
        \big( F^\eta_{t} \big)^\complement
    \big)
    & = 
    \P\Big(
        \#\big\{
            u \leq \floor{ t /\lambda^*_b(\eta) }:\ 
            \eta \norm{\bm Z_u} > \delta
        \big\} \geq k(\eta)
    \Big)
    \\ 
    & = 
    \P\Big(
        \text{Binomial}
        \big(
            \floor{t/\lambda^*_b(\eta)},\ H(\delta/\eta)
        \big) \geq k(\eta)
    \Big).
\end{align*}
For the expectation of the Binomial variable above, note that
$
\frac{ t }{\lambda^*_b(\eta)} \cdot H(\delta/\eta)
\in \RV_{ - (\mathcal J^*_b - 1)(\alpha - 1)  }(\eta)
$
as $\eta \downarrow 0$.
Then, Claim \eqref{proof, goal 4, lemma: metastability, proportion of time not around local minima} follows from Markov's inequality and the definition of $k(\eta)$ in \eqref{proof, def k eta, lemma: metastability, proportion of time not around local minima}.

\medskip
\noindent
\textbf{Proof of Claim \eqref{proof, goal 3, lemma: metastability, proportion of time not around local minima}}.
On $A^\eta_{t}(\bm x) \cap B^{\eta}_{\delta}(\bm x)$,
we have $T^\eta_k \wedge \floor{ t/\lambda^*_b(\eta) } = \tilde T^\eta_k \wedge \floor{ t/\lambda^*_b(\eta) }$ for each $k \geq 1$, where
$
\tilde T^\eta_k 
    \delequal
    \min\big\{ u > S^\eta_k:\ 
        \bm X^{\eta|b}_u( \bm x) \notin I_{i;\epsilon,M} \setminus B_{\epsilon/2}(\bm 0) \big\}.
$
Furthermore, it has been noted above that, on the event
$B^\eta_{\delta}(\bm x) \cap F^\eta_t$, we have $N^\eta \leq k(\eta)$. 
Therefore,
\begin{align*}
    & \sup_{ \bm x \in B_{\epsilon/2}(\bm 0 ) }
        \P\Big(
         A^\eta_{t}(\bm x) \cap B^\eta_{\delta}(\bm x) \cap F^\eta_t  \cap \big( E^\eta_{t,T}(\bm x) \big)^\complement
        \Big)
    \\
    & \leq 
    \sup_{ \bm x \in B_{\epsilon/2}(\bm 0 ) }
    \P\big(
        \tilde T^\eta_j - S^\eta_j > T/\eta\text{ for some }j \leq k(\eta)
    \big)
    \\ 
    & \leq 
    k(\eta) \cdot 
    \underbrace{\sup_{\bm y \in B_{\epsilon}(\bm 0) }
    \P\big(
        \bm X^{\eta|b}_u(\bm x) \in I_{i;\epsilon,M} \setminus B_{\epsilon/2}(\bm 0)\ \forall u \leq \floor{T/\eta}
    \big)}_{\delequal p^*_T(\eta)}.
\end{align*}
The last step follows from the strong Markov property at the $S^\eta_j$'s.
Applying Lemma~\ref{lemma: fixed cycle exit or return taking too long, first exit analysis},
we can find $T$ large enough such that $p^*_T(\eta) = \bm{o}\big(1/k(\eta)\big)$ as $\eta \downarrow 0$ and complete the proof.
\end{proof}

Now, we are ready to prove  Proposition \ref{proposition: hat X close to X, metastability}.

\begin{proof}[Proof of Proposition \ref{proposition: hat X close to X, metastability}]
\linksinpf{proposition: hat X close to X, metastability}
The claim $\lim_{\eta\downarrow 0}\P\Big(\norm{ \bm X^{\eta|b}_{  \floor{T/\lambda^*_b(\eta)} }(\bm x_0) - \hat{\bm X}^{\eta,\epsilon|b}_{T}(\bm x_0)} \geq \epsilon\Big) = 0$
has been verified by part $(ii)$ of Lemma~\ref{lemma:  hat X close to X, fdd, metastability}.
In the remainder of this proof, we focus on establishing the claim
$\lim_{\eta\downarrow 0}\P\bigg( \dlp{[0,T]}\Big(\bm X^{\eta|b}_{ \floor{ {\boldsymbol{\cdot}}/\lambda^*_b(\eta) } }(\bm x_0),\hat{\bm X}^{\eta,\epsilon|b}_{\boldsymbol{\cdot}}(\bm x_0)\Big) \geq 2\epsilon\bigg) = 0$.
W.l.o.g., in this proof we focus on the case where $T = 1$,
and write $\dlp{} = \dlp{[0,1]}$ to lighten notations.

We start with a few observations that allow us to bound
\begin{align}
    \bm \Delta (\eta) & \delequal 
    \bigg(\dlp{}\Big( \bm X^{\eta|b}_{ \floor{  { \boldsymbol{\cdot} }  /\lambda^*_b(\eta) } }(\bm x_0), 
    \hat{\bm X}^{\eta,\epsilon|b}_{ \boldsymbol{\cdot} } (\bm x_0)  \Big)\bigg)^p
    =
    \sum_{n = 0}^{N - 1}
    \underbrace{ \int_{n/N}^{(n+1)/N}
    \norm{ 
        \bm X^{\eta|b}_{ \floor{t/\lambda^*_b(\eta)} }(\bm x_0) - \hat{\bm X}^{\eta,\epsilon|b}_t(\bm x_0)
    }^pdt
    }_{ \delequal \bm d^{(\eta)}_p(n)  },
    \label{proof, def Delta eta and de eta p n, proposition: hat X close to X, metastability}
\end{align}
given a positive integer $N$.
First, for any $\eta > 0$, let
$
\mathcal I^{(\eta)}_N(n) \delequal \mathbf{I}\big\{
    \bm i^{(\eta)}_N(n) > 1/N^2
\big\},
$
where
\begin{align*}
    \bm i^{(\eta)}_N(n) \delequal \int_{n/N}^{(n+1)/N} \mathbf{I}\bigg\{
        \bm X^{\eta|b}_{ \floor{t/\lambda^*_b(\eta)} }(\bm x_0) \notin \bigcup_{j \in [K]}B_{\epsilon}(\bm m_j)
        \bigg\}dt,
    \qquad \forall n = 0,1,\cdots,N-1.
\end{align*}
That is, $\bm i^{(\eta)}_N(n)$ denotes the amount of time over $[\frac{n}{N},\frac{n+1}{N})$ that the SGD iterates (under a $\lambda^*_b(\eta)$ time scaling)  are not close enough to any local minima,
and $\mathcal I^{(\eta)}_N(n)$ is the indicator that $\bm i^{(\eta)}_N(n) > 1/N^2$.
Moreover,
let 
\begin{align}
    K^{(\eta)}_N \delequal \sum_{n = 1}^{N-1}\mathcal I^{(\eta)}_N(n).
    \nonumber
\end{align}
The proof hinges on the following claims:
there exist some $q^* \in (0,\infty)$ and a family of events $(A^\eta_N)_{ N \geq 1,\ \eta > 0}$ such that
\begin{enumerate}[$(i)$]
    \item on the event $A^\eta_N$, we have $\norm{ \bm X^{\eta|b}_t(\bm x_0) } \leq M$ for all $t \leq \floor{1/\lambda^*_b(\eta)}$;

    \item it holds for all $N$ large enough that $\lim_{\eta \downarrow 0}\P(A^\eta_N) = 1$;

    \item
        for all $N$ large enough, there exists $\bar \eta = \bar \eta(N) > 0$ such that under any $\eta \in (0,\bar\eta)$,
        \begin{align*}
            \P( K^{(\eta)}_N \geq j\ |\ A^\eta_N )
            \leq \P\bigg( \text{Binomial}(N,\frac{2q^*}{N}) \geq j\bigg), 
            \qquad 
            \forall j = 1,2,\cdots,N.
        \end{align*}
\end{enumerate}
To see why, 
note that 
by the definition in \eqref{def: hat tau, 1, metastability}--\eqref{def: marker process, hat X eta epsilon b},
the process 
$\hat{\bm X}^{\eta,\epsilon|b}_t(\bm x_0)$ only takes values in $\{\bm m_j:\ j \in [K]\}$,
so $\norm{\hat{\bm X}^{\eta,\epsilon|b}_t(\bm x_0)} < M\ \forall t > 0$ (see \eqref{goal 1, lemma: unlikely to exit M or visit s i, metastability}).
Together with the Claim $(i)$ above, we get $\norm{ \bm X^{\eta|b}_{ \floor{t/\lambda^*_b(\eta)} }(\bm x_0) - \hat{\bm X}^{\eta,\epsilon|b}_t(\bm x_0) } \leq 2M$ for all $t \in (0,1]$.
Moreover, note that
we must have $\norm{ \bm X^{\eta|b}_{ \floor{ t/\lambda^*_b(\eta) } }(\bm x_0) - \hat{\bm X}^{\eta,\epsilon|b}_t(\bm x_0) } < \epsilon$
whenever 
$
\bm X^{\eta|b}_{ \floor{t/\lambda^*_b(\eta)} }(\bm x_0)\in \bigcup_{j \in [K]}B_{\epsilon}(\bm m_j).
$
Then, the following holds on the event $A^\eta_N$ for the terms $\bm d^{(\eta)}_p(n)$ in \eqref{proof, def Delta eta and de eta p n, proposition: hat X close to X, metastability}:
if $\bm i^{(\eta)}_N(n) \leq 1/N^2$,
we have $\bm d^{(\eta)}_p(n) \leq \epsilon^p \cdot \frac{1}{N} + (2M)^p \cdot \frac{1}{N^2}$;
otherwise, we have the trivial bound $\bm d^{(\eta)}_p(n) \leq (2M)^p \cdot \frac{1}{N}$.
Therefore, on $A^\eta_N$,
\begin{align*}
    \bm \Delta (\eta) 
    &
    \leq 
    (2M)^p \cdot \frac{1}{N} + \sum_{n = 1}^{N-1}\bm d^{(\eta)}_p(n)
    \\ 
    & 
    \leq 
    (2M)^p \cdot \frac{1}{N} + K^{(\eta)}_N \cdot \frac{(2M)^p}{N} + (N - 1 -K^{(\eta)}_N) \cdot \Big(\frac{\epsilon^p}{N} + \frac{(2M)^p}{N^2}\Big)
    \leq 
    (2M)^p \cdot \frac{2 + K^{(\eta)}_N}{N} + \epsilon^p.
\end{align*}
Then, given any $N$ large enough, $\eta \in (0,\bar \eta(N))$
and any
$\beta \in (0,1)$,
\begin{align*}
    &
    \P\bigg(\bm \Delta (\eta) \geq \underbrace{\frac{2 + 2q^* + \sqrt{N^\beta } }{ N }}_{ \delequal \delta(N,\beta) }\cdot (2M)^p + \epsilon^p \bigg)
    \\
    & \leq 
    \P(K^{(\eta)}_N \geq 2q^* + \sqrt{N^\beta})
    =
     \P\big(  \{K^{(\eta)}_N \geq 2q^* + \sqrt{N^\beta} \} \cap A^\eta_N \big)
     +
      \P\big(  \{K^{(\eta)}_N \geq 2q^* + \sqrt{N^\beta} \} \setminus A^\eta_N \big)
     \\ 
     & \leq 
    \P\bigg( \text{Binomial}(N,\frac{2q^*}{N}) \geq 2q^* + \sqrt{N^\beta} \bigg) + \P\big( (A^\eta_N)^\complement\big)
    \qquad\text{by claim }(iii)
    \\ 
    & \leq 
    \frac{ \text{var}\Big[   \text{Binomial}(N,\frac{2q^*}{N})  \Big]  }{N^\beta} + \P\big( (A^\eta_N)^\complement\big)
\leq 
    \frac{2q^*}{N^\beta} + \P\big( (A^\eta_N)^\complement\big).
\end{align*}
Lastly, to conclude the proof with $\lim_{\eta \downarrow 0}\P(\bm \Delta(\eta) > 2^p\epsilon^p) = 0$,
note that 
\begin{itemize}
    \item by claim $(ii)$, $\lim_{\eta \downarrow 0}\P\big( (A^\eta_N)^\complement\big) = 0$;

    \item 
        due to $\beta \in (0,1)$ we have $\lim_{N \to \infty}\delta(N,\beta) = 0$,
and hence $\delta(N,\beta) \cdot (2M)^p + \epsilon^p < 2^p\epsilon^p$ eventually for all $N$ large enough.
\end{itemize}
Now, it only remains to verify claims $(i)$, $(ii)$, and $(iii)$.

\medskip
\noindent
\textbf{Proof of Claims $(i)$ and $(ii)$}.
We start by defining events $A^\eta_N$.
Let $t_N(n) = n/N$,
\begin{align*}
    & A^\eta_N(n)
    \\
    & \delequal 
    \underbrace{
        \bigg\{ \bm X^{\eta|b}_{ \floor{ t_N(j)/\lambda^*_b(\eta)  }  }(\bm x_0) \in \bigcup_{ \bm m_i \in V^*_b }B_{\epsilon/2}(\bm m_i)
    \ \forall j \in [n]
        \bigg\}
    }_{\delequal A^\eta_{N,1}(n)}
    \cap 
    \underbrace{ \Big\{
    \norm{ \bm X^{\eta|b}_{ \floor{ t/\lambda^*_b(\eta)  }  }(\bm x_0) } \leq M \ \forall t \leq t_N(n)
    \Big\} }_{ \delequal A^\eta_{N,2}(n) },
\end{align*}
and let $A^\eta_N = A^\eta_N(N)$.
Note that $A^\eta_N(1) \supseteq A^\eta_N(2) \supseteq \cdots \supseteq A^\eta_N(N) = A^\eta_N$.
Claim $(i)$ then holds by definition.
Furthermore,
by Lemma~\ref{lemma: metastability, abstract framework}
and that $\lim_{\eta\downarrow 0}\P(\norm{\bm X^{\eta|b}_{ T}(\bm x_0) - \hat{\bm X}^{\eta,\epsilon|b}_{T}(\bm x_0)} \geq \epsilon) = 0$ for any $T > 0$,
we have
$
\{\bm X^{\eta|b}_{ \floor{ t/\lambda^*_b(\eta) }  }(\bm x_0):\ t > 0\} \tofdd 
\{\bm Y^{*|b}_t:\ t > 0\};
$
then, since $\bm Y^{*|b}_t$ only visits states in $V^*_b$,
we get $\lim_{\eta \downarrow 0}\P\big(A^\eta_{N,1}\big) = 1$ for any $N \geq 1$.
On the other hand,  part $(i)$ of Lemma \ref{lemma:  hat X close to X, fdd, metastability} implies $\lim_{\eta \downarrow 0}\P\big(A^\eta_{N,2}\big) = 1\ \forall N \geq 1$ for any $M$ large enough.
This verifies Claim $(ii)$.


\medskip
\noindent
\textbf{Proof of Claim $(iii)$}.
Let $\big(\widetilde{\mathcal I}^\eta_N(n)\big)_{n \in [N-1]}$
be a random vector
with law
$
\mathcal L\Big( \big({\mathcal I}^\eta_N(n)\big)_{n \in [N-1]}\ \Big|\ A^\eta_N\Big).
$
It suffices to find some $q^* \in (0,\infty)$ such that,
for all $N$ large enough, there is $\bar \eta = \bar \eta(N) > 0$
for the following claim to hold:
Given any $n \in [N-1]$ and any sequence $i_j \in \{0,1\}\ \forall j \in [n-1]$,
\begin{align}
    \P\Big(  \widetilde{\mathcal I}^\eta_N(n) = 1\ \Big|\ \widetilde{\mathcal I}^\eta_N(j) = i_j\ \forall j \in [n-1] \Big) < 2q^*/N
    \qquad \forall \eta \in (0,\bar \eta).
    \label{goal, corollary, elimination of sharp minima, metastability}
\end{align}
To see why, under condition \eqref{goal, corollary, elimination of sharp minima, metastability} and for any $\eta \in (0,\bar \eta(N))$, 
there exists a coupling between 
iid Bernoulli random variables $(\mathcal Z_N(n))_{n \in [N-1]}$ with success rate $2q^*/N$
and 
$( \widetilde{\mathcal I}^\eta_N(n))_{n \in [N-1]}$
such that
$
\widetilde{\mathcal I}^\eta_N(n) \leq {\mathcal Z}_N(n)\ \forall n \in [N-1]
$
almost surely.
This stochastic comparison between $(\mathcal Z_N(n))_{n \in [N-1]}$ and $(\widetilde{\mathcal I}^\eta_N(n))_{n \in [N-1]}$
directly verifies Claim $(iii)$.

To prove condition \eqref{goal, corollary, elimination of sharp minima, metastability},
note that
given any $N$, any $n \in [N-1]$, and any sequence $i_j \in \{0,1\}\ \forall j \in [n-1]$,
\begin{align*}
    &  
    \P\Big(  \widetilde{\mathcal I}^\eta_N(n) = 1\ \Big|\ \widetilde{\mathcal I}^\eta_N(j) = i_j\ \forall j \in [n-1] \Big)
    \\ 
    & = 
    \frac{
         \P\Big(  \widetilde{\mathcal I}^\eta_N(n) = 1; \ \widetilde{\mathcal I}^\eta_N(j) = i_j\ \forall j \in [n-1] \Big)
    }{
         \P\Big(\widetilde{\mathcal I}^\eta_N(j) = i_j\ \forall j \in [n-1] \Big)
    }
    \\ 
    & = 
    \frac{
         \P\Big( \big\{  {\mathcal I}^\eta_N(n) = 1; \ {\mathcal I}^\eta_N(j) = i_j\ \forall j \in [n-1] \big\} \cap A^\eta_N \Big)
         \Big/\P(A^\eta_N)
    }{
         \P\Big(\big\{ {\mathcal I}^\eta_N(j) = i_j\ \forall j \in [n-1] \big\} \cap A^\eta_N \Big) \Big/\P(A^\eta_N)
    }
    \qquad 
    \text{by definition of }\big(\widetilde{\mathcal I}^\eta_N(n)\big)_{n \in [N-1]}
    \\ 
    & \leq 
    \frac{
         \P\Big( \big\{  {\mathcal I}^\eta_N(n) = 1; \ {\mathcal I}^\eta_N(j) = i_j\ \forall j \in [n-1] \big\} \cap A^\eta_N(n) \Big)
    }{
         \P\Big(\big\{ {\mathcal I}^\eta_N(j) = i_j\ \forall j \in [n-1] \big\} \cap A^\eta_N \Big)
    }
    \qquad\text{due to }A^\eta_N(n) \supseteq A^\eta_N
    \\ 
    & = 
     \frac{
         \P\Big( \big\{  {\mathcal I}^\eta_N(n) = 1; \ {\mathcal I}^\eta_N(j) = i_j\ \forall j \in [n-1] \big\} \cap A^\eta_N(n) \Big)
    }{
         \P\Big(\big\{ {\mathcal I}^\eta_N(j) = i_j\ \forall j \in [n-1] \big\} \cap A^\eta_N(n) \Big)
    }
    \cdot 
    \frac{
     \P\Big(\big\{ {\mathcal I}^\eta_N(j) = i_j\ \forall j \in [n-1] \big\} \cap A^\eta_N(n) \Big)
    }{
     \P\Big(\big\{ {\mathcal I}^\eta_N(j) = i_j\ \forall j \in [n-1] \big\} \cap A^\eta_N \Big)
    }
    \\ 
    & = 
     \underbrace{ 
        \P\Big(  {\mathcal I}^\eta_N(n) = 1\ \Big| \big\{ {\mathcal I}^\eta_N(j) = i_j\ \forall j \in [n-1]\big\} \cap  A^\eta_N(n) \Big)
    }_{ \delequal p^\eta_1(N) }
     \cdot 
     \underbrace{
        \frac{
        \P\Big(\big\{ {\mathcal I}^\eta_N(j) = i_j\ \forall j \in [n-1] \big\} \cap A^\eta_N(n) \Big)
        }{
         \P\Big(\big\{ {\mathcal I}^\eta_N(j) = i_j\ \forall j \in [n-1] \big\} \cap A^\eta_N \Big)
        }
    }_{ \delequal p^\eta_2(N) }.
\end{align*}
For the term $p^\eta_1(N)$, 
note that on  the event $A^\eta_N(n)$ we have $\bm X^{\eta|b}_t(\bm x_0) \in \bigcup_{\bm m_i \in V^*_b}B_{\epsilon/2}(\bm m_i)$
at
$t = \floor{ t_N(n)/\lambda^*_b(\eta) }$,
and hence (by Markov property)
\begin{align*}
    p^\eta_1(N)
    & \leq 
    \max_{\bm m_i \in V^*_b}
    \sup_{ \bm y \in B_{\epsilon/2}(\bm m_i) }
    \P\bigg(
        \int_0^{1/N} \mathbf{I}\Big\{
                \bm X^{\eta|b}_{ \floor{ s/\lambda^*_b(\eta) }  }(\bm y) \notin B_\epsilon(\bm m_i) 
            \Big\}ds > 1/N^2
    \bigg).
\end{align*}
Applying Lemma \ref{lemma: metastability, proportion of time not around local minima},
for all $N$ large enough there exist $\bar\eta = \bar\eta(N) > 0$,
such that 
$
p^\eta_1 \leq q^*/N
\ \forall \eta \in (0,\bar\eta)$,
where $q^*\in (0,\infty)$ is a constant that does not vary with $N$ or $\eta$.
As for the term $p^\eta_2(N)$,
note that for any event $B$ with $\P(B)>0$,
we have 
\begin{align}
    \frac{ \P(B \cap A^\eta_N(n)) }{ \P(B \cap A^\eta_N) }
    \leq \frac{ \P(B) }{ \P(B) - \P\big((A^\eta_N)^\complement\big) } \to 1,
    \qquad\text{  as $\eta \downarrow 1$, due to }\lim_{\eta\downarrow 0}\P(A^\eta_N) = 1.
    \label{proof, property 1, corollary, elimination of sharp minima, metastability}
\end{align}
Also, in the definition of $p^\eta_2(N)$ above, note that there are only finitely many choices of $n \in [N-1]$
and finitely many combinations for $i_j \in\{0,1\}\ \forall j \in [n-1]$.
By enumerating each of the finitely many choices for $B = \{ {\mathcal I}^\eta_N(j) = i_j\ \forall j \in [n-1]\}$
in \eqref{proof, property 1, corollary, elimination of sharp minima, metastability},
we can find some $\bar \eta = \bar \eta(N)$ such that
$
p^\eta_2(N) < 2\ \forall \eta \in (0,\bar \eta)
$
uniformly for all those choices.
Combining the bounds $p^\eta_1(N) < q^*/N$ and $p^\eta_2(N) < 2$,
we verify the condition \eqref{goal, corollary, elimination of sharp minima, metastability} and conclude the proof.
\end{proof}

\section{Properties of the Markov Jump Process $Y^{*|b}$}
\label{sec: appendix, CTMC Y * b}

\begin{proposition}\label{proposition: Y * b is a CTMC}
\linksinthm{proposition: Y * b is a CTMC}
Let Assumptions \ref{assumption: geometry, metastability} and \ref{assumption, value of b, irreducible, metastability} hold.
The following claims hold for $((U_j)_{j \geq 1},(V_j)_{j \geq 1})$ defined in \eqref{def Y * b t, metastability}:
\begin{enumerate}[(i)]

    \item For any $t > 0$, $\lim_{i \to \infty}\P(\sum_{j \leq i}U_j > t) = 1$;

    \item For any $u > 0$ and $i \geq 1$, $\P(U_1 + \cdots + U_i = u) = 0$;
    
    \item $\bm Y^{*|b}_{\boldsymbol{\cdot}} \distequal \Phi((U_j)_{j \geq 1},(V_j)_{j \geq 1})$ holds for the mapping $\Phi$ defined in \eqref{def jump process}; that is, it
    is a continuous-time Markov chain 
    with initial distribution \eqref{def: Y * b, initial distribution}
    and
    generator \eqref{def, generator of Y * b}.
\end{enumerate}
\end{proposition}

\begin{proof}\linksinpf{proposition: Y * b is a CTMC}
$(i)$
Recall the definitions of $q_b(i)$ and $q_b(i,j)$ in \eqref{def: q b i, q b i j, generator for Y * b}.
Also, recall the definition of the discrete-time Markov chain $(S_t)_{t \geq 0}$ at the end of Section~\ref{subsec: main results},
with state space $\{\bm m_1,\ldots, \bm m_K \}$
and one-step transition kernel
$
\P(S_{t+1} = \bm m_j|S_t = \bm m _i) = q_b(i,j)/q_b(i).
$
Note that the chain is well-defined due to \eqref{property, sum of q b i j is q b i, metastability, proof}.
We also introduce two notations.
First,  we use $S_t(\bm v)$ to denote the Markov chain under initial condition $S_0(\bm v) = \bm v$.
Second,  for each $t \geq 0$, set $I^S_t(\bm v) = i$ if and only if $S_n(\bm v) = \bm m_i$ (i.e., recording the indices rather than the exact values of the states visited).

Let $\bm x_0$ be the initial value prescribed in Theorem~\ref{corollary irreducible case},
and
$i_0 \in [K]$ be the unique index with $\bm x_0 \in I_{i_0}$.
Let $(E_i)_{i \geq 0}$ be a sequence of iid Exponential RVs with rate 1, which is independent of $(S_t(\bm m_{i_0}))_{t \geq 0}$.
By the law of $(U_l, V_l)_{l \geq 1}$ specified in \eqref{def Y * b t, metastability} (recall that $U_1 = 0$ and $V_1 = \bm m_{i_0}$), 
for each $i \geq 2$ we have
\begin{align}
    \sum_{j \in [i]}U_j & \distequal \sum_{j = 0,1,\cdots,i-2} 
    \frac{E_j}{ q_b\big( I^S_j(\bm m_{i_0}) \big)  } \cdot \mathbf{I}\big\{ S_j(\bm m_{i_0}) \in V^*_b  \big\}
    \label{proof, proposition: Y * b is a CTMC, part i, representation of U_j sum}
    \\
    & \geq \frac{1}{q^*} \cdot \sum_{j = 0,1,\cdots,i-2} E_j \cdot \mathbf{I}\big\{ S_j(\bm m_{i_0}) \in V^*_b \big\}
    \qquad
    \text{ where }q^* \delequal \max_{i:\ \bm m_i \in V^*_b }q_b(i) \in (0,\infty)
    \nonumber
    \\ 
    & \distequal 
    \sum_{ j = 0  }^{ N_{i-2}  } \frac{E_j}{q^*}
    \qquad 
    \text{ where }N_{i} \delequal \sum_{j = 0}^{i}\mathbf{I}\big\{ S_j(\bm m_{i_0}) \in V^*_b \big\}.
    \nonumber
\end{align}
Then, given $t > 0$ and positive integers $n,i$, we get
$
\P(\sum_{j \leq i} U_j > t)
\geq 
\P(\sum_{j = 0}^n E_j/q^* > t)\cdot \P(  N_{i-2} > n ).
$
To conclude the proof of part $(i)$,
it suffices to show that
for each $\epsilon > 0$,
there exists $n = n(\epsilon)$ such that
\begin{align}
    \P\bigg(\sum_{j = 0}^n E_j/q^* > t\bigg) > 1 - \epsilon,
    \qquad 
    \lim_{i \to \infty}\P( N_{i} > n ) = 1.
    \label{proof, goal 1, part i, proposition: Y * b is a CTMC}
\end{align}
The first claim holds for any $n$ large enough due to $q^* \in (0,\infty)$; see \eqref{property, sum of q b i j is q b i, metastability, proof}.
The second claim follows from the irreducibility of the Markov chain $S_t(\bm v)$; see Assumption~\ref{assumption, value of b, irreducible, metastability} and \eqref{property, edge on graph iff q b i j strictly positive, metastability}.

$(ii)$
In light of the representation \eqref{proof, proposition: Y * b is a CTMC, part i, representation of U_j sum},
this claim is an immediate consequence of the absolute continuity of exponential distributions.

$(iii)$ 
We start by considering an equivalent representation of the continuous-time Markov chain $\bm Y^{*|b}$ (recall the definitions in \eqref{def: Y * b, initial distribution}--\eqref{def: generator of Y * b, 2}),
based on the following straightforward observation:
the law of the process would remain the same if we allow the process to jump from any state $\bm m_i$ to itself at exponential rates (i.e., by including Markovian ``dummy'' jumps where the process does not move at all).
More precisely, 
using the mapping $\Phi$ in Definition~\ref{def jump process},
we have
$\bm Y^{*|b}_{\boldsymbol{\cdot}} \distequal \Phi((\tilde U_k)_{k \geq 1},(\tilde V_k)_{k \geq 1})$
with $\tilde U_k$'s and $\tilde V_k$'s defined as follows.
Let $\tilde V_1$ be sampled from the distribution $\theta_b(\cdot|\bm m_{i_0})$ defined in \eqref{def: absorption prob, theta b i j}
and let $\tilde U_1 \equiv 0$.
Next,
for any $t > 0$, $l \geq 1$, and $m_i,\ m_j \in V^*_b$ (with possibly $m_i = m_j$),
\begin{align}
    & \P( \tilde U_{l+1} < t,\ \tilde V_{l+1} = \bm m_j\ |\ \tilde V_l = \bm m_i, (\tilde V_j)_{j = 1}^{l-1},\ (\tilde U_j)_{j = 1}^l)
    = \P( \tilde U_{l+1} < t,\ \tilde V_{l+1} = \bm m_j\ |\ \tilde V_l = \bm m_i )
    \nonumber 
    \\ 
    & = 
    r^{*|b}(i,j) \cdot \big( 1 - \exp( - q_b(i)t)\big),
    \label{def Y * b t, alternative representation, metastability}
\end{align}
where
\begin{align}
    r^{*|b}(i,j) \delequal \sum_{j^\prime \in [K]:\ j^\prime \neq i}\frac{q_b(i,j^\prime)}{q_b(i)} \cdot \theta_b(\bm m_j|\bm m_{j^\prime})
    \label{proof, def of r * b}
\end{align}
with $q_b(i)$ and $q_b(i,j)$ defined in \eqref{def: q b i, q b i j, generator for Y * b}.
That is,
by introducing ``dummy'' jumps from $\bm m_i \in V^*_b$ to itself with exponential rate
$\sum_{j^\prime \neq i}q_b(i,j^\prime)\theta_b(\bm m_i|\bm m_{j^\prime})$,
we end up with the same process and obtain a reformulation $\bm Y^{*|b}_{\boldsymbol{\cdot}} \distequal  \Phi((\tilde U_k)_{k \geq 1},(\tilde V_k)_{k \geq 1})$.


Meanwhile, we state a useful property of the mapping $\Phi$.
Recall that $U_1 = 0$,
and set $\hat T_0 = 1$. For each $k \geq 1$, define (under the convention $U_0 = 0$)
\begin{align}
    \hat T_k \delequal \min\{ j > \hat T_{k -1}:\ U_j \neq 0 \},
    \qquad 
    \hat V_k \delequal V_{ -1 + \hat T_k},
    \qquad 
    \hat U_k \delequal \sum_{ j = \hat T_{k-1} }^{- 1 + \hat T_k}U_j = U_{ \hat T_{k-1} }.
    \label{proof, part iii, def hat U hat V transform, proposition: Y * b is a CTMC}
\end{align}
Note that we have $\hat U_1 = 0$ and $\hat T_1 \geq 2$, which implies $-1 + \hat T_1 \geq 1$. This confirms that $\hat V_1$ is well-defined.
Also, \eqref{def Y * b t, metastability} dictates that $\hat V_1$ admits the law of $\theta_b(\cdot|\bm m_{i_0})$
defined in \eqref{def: absorption prob, theta b i j}.
In simple terms, $((\hat U_k)_{k \geq 1},(\hat V_k)_{k \geq 1})$ can be interpreted as a transformation of $((U_j)_{j \geq 1},(V_j)_{j \geq 1})$ with consecutive instantaneous jumps grouped together.
As a result,
\begin{align}
    \Phi\Big((U_j)_{j \geq 1},(V_j)_{j \geq 1}\Big) = \Phi\Big((\hat U_k)_{k \geq 1},(\hat V_k)_{k \geq 1}\Big).
    \label{proof, part iii, goal 1, proposition: Y * b is a CTMC}
\end{align}
\elaborate{
To see why, let $Y = \Phi\Big((U_j)_{j \geq 1},(V_j)_{j \geq 1}\Big)$
 and $\hat Y = \Phi\Big((\hat U_k)_{k \geq 1},(\hat V_k)_{k \geq 1}\Big)$.
From Definition \ref{def jump process}, given any $t \geq 0$
there exists some non-negative integer $k = k(t)$ such that 
$\sum_{j \leq k}\hat U_j \leq t$,
$\sum_{j \leq k + 1}\hat U_j > t$,
and
$
\hat Y_t = \hat V_{k}=  V_{-1 + \hat T_k}.
$
Meanwhile, from \eqref{proof, part iii, def hat U hat V transform, proposition: Y * b is a CTMC},
we get
\begin{align}
\sum_{j \leq k}\hat U_j
= \sum_{j \leq k}\sum_{i = \hat T_{j-1}}^{-1 + \hat T_j}U_i
= \sum_{i \leq -1 + \hat T_k}U_i \leq t.
\label{proof, part iii, result 1 for goal 1, proposition: Y * b is a CTMC}
\end{align}
Similarly, we yield
$
\sum_{j \leq k + 1}\hat U_j = \sum_{i \leq -1 + \hat T_{k+1}}U_i > t.
$
Moreover, from \eqref{proof, part iii, def hat U hat V transform, proposition: Y * b is a CTMC} we know that $U_{\hat T_k} > 0$ and $U_{j} = 0$ for all $\hat T_k + 1 \leq j \leq -1 + \hat T_{k+1}$.
As a result,
\begin{align}
    \sum_{i \leq \hat T_k}U_i > t.
    \label{proof, part iii, result 2 for goal 1, proposition: Y * b is a CTMC}
\end{align}
Combining results \eqref{proof, part iii, result 1 for goal 1, proposition: Y * b is a CTMC} with \eqref{proof, part iii, result 2 for goal 1, proposition: Y * b is a CTMC},
we have $\mathcal J(t) = -1 + \hat T_k$ for the $\mathcal J(t)$ defined in \eqref{def, mapping Phi for jump process}.
This leads to $Y_t = V_{-1 + \hat T_k} = \hat Y_t$ and, due to the arbitrariness of $t \geq 0$, concludes the proof of claim \eqref{proof, part iii, goal 1, proposition: Y * b is a CTMC}.
}



In light of \eqref{proof, part iii, goal 1, proposition: Y * b is a CTMC} and the representation $\bm Y^{*|b}_{\boldsymbol{\cdot}} \distequal  \Phi((\tilde U_k)_{k \geq 1},(\tilde V_k)_{k \geq 1})$ established above,
to prove part $(iii)$ it suffices to show that
\begin{align}
    (\hat U_k,\ \hat V_k)_{k \geq 1} \distequal (\tilde U_k,\ \tilde V_k)_{k \geq 1}.
    \label{proof, part iii, goal 2, proposition: Y * b is a CTMC}
\end{align}
As noted above, we have $\hat U_1 = \tilde U_1 = 0$, and that both $\hat V_1$ and $\tilde V_1$ admit the law $\theta_b(\cdot|\bm m_{i_0})$. 
Next, fix some $k \geq 1$, $\bm m_i,\ \bm m_j \in V^*_b$ (possibly with $\bm m_i = \bm m_j$) , and some $t > 0$. Observe that
\begin{align*}
    & \P( \hat U_{k+1} < t,\ \hat V_{k+1} = \bm m_j,\ \hat V_k = \bm m_i)
    \\ 
    & = \sum_{ N \geq 1}\sum_{ n \geq 1 }
   \P(  \hat U_{k+1} < t,\ V_{ N + n } = \bm m_j,\ \hat T_{k+1} - 1 = N + n,\ V_{N} = \bm m_i,\ \hat T_k - 1 = N )
   \qquad\text{by \eqref{proof, part iii, def hat U hat V transform, proposition: Y * b is a CTMC}}
   \\ 
   & = \sum_{ N \geq 1}\sum_{ n \geq 1 }
   \P( U_{ N+1 } < t,\ V_p \notin V^*_b\ \forall N+1 \leq  p \leq N+n-1;\\
   &\qquad \qquad \qquad \qquad  V_{ N + n } = \bm m_j,\ \hat T_{k+1} - 1 = N + n,\ V_{N} =\bm m_i,\ \hat T_k - 1 = N)
   \qquad 
   \text{by \eqref{proof, part iii, def hat U hat V transform, proposition: Y * b is a CTMC} and \eqref{def Y * b t, metastability}}
   \\ 
   & = \sum_{ N \geq 1}\sum_{ n \geq 1 }\sum_{ (l_1,\cdots,l_{n-1}) \in \mathscr I(i,n-1)  }
    \P( U_{ N+1 } < t,\ V_{N+p} = \bm m_{l_p}\ \forall p \in [n-1];\\
   &\qquad \qquad \qquad \qquad \qquad \qquad   \qquad \qquad V_{ N + n } = \bm m_j,\ \hat T_{k+1} - 1 = N + n,\ V_{N} = \bm m_i,\ \hat T_k - 1 = N  )
   \\ 
   &
   \text{ where } \mathscr I(i,n-1) \delequal \big\{ (l_1,\ldots,l_{n-1}):\ l_p \neq l_{p-1}\text{ and }m_{l_p} \notin V^*_b\ \forall p \in [n-1] \big\}
   \text{ with convention }l_0 = i
   \\ 
   & = 
   \sum_{ N \geq 1}
   \P( V_N = \bm m_i,\ \hat T_k - 1 = N)
   \\ 
   &\qquad 
   \cdot 
   \sum_{ n \geq 1 }\sum_{ (l_1,\cdots,l_{n-1}) \in \mathscr I(i,n-1) }
   \frac{ q_b(i,l_1) }{ q_b(i) }\Big( 1 - \exp\big( - q_b(i)t\big)\Big)
   \frac{ q_b(l_1,l_2) }{q_b(l_1)} \cdots \frac{ q_b(l_{n-2},l_{n-1}) }{ q_b(l_{n-2}) }\frac{ q_b(l_{n-1},j) }{ q_b(l_{n-1}) }
   \\ 
   &\qquad \qquad \qquad \qquad \qquad \qquad \qquad \qquad \qquad \qquad \qquad \qquad \qquad \qquad \qquad \qquad 
   \text{using }\eqref{def Y * b t, metastability}
   \\ 
   & =
    \sum_{ N \geq 1}
   \P( V_N = \bm m_i,\ \hat T_k - 1 = N)
   \\ 
   & \qquad 
   \cdot 
   \sum_{ l_1 \neq i } \frac{ q_b(i,l_1) }{ q_b(i) }\Big( 1 - \exp\big( - q_b(i)t\big)\Big)
   \cdot 
   \sum_{n \geq 1}\P( \tau( \bm m_{l_1} ) = n-1,\ S_{\tau}(\bm m_{l_1}) = m_j).
\end{align*}
In the last line of the display above, we adopt the notations in part $(i)$  that $S_n(v)$ is a discrete-time Markov chain with initial value $S_0(v) = v$
and one-step transition kernel 
$
\P(S_{n+1} = \bm m_j|S_n = \bm m _i) = q_b(i,j)/q_b(i),
$
and define
 $\tau(v) = \min\{n \geq 0:\ S_n(v) \in V^*_b\}$ as the hitting time of the set $V^*_b$;
 for notational simplicity we write $S_\tau(v) = S_{ \tau(v) }(v)$.
Now, observe that
\begin{align*}
    & \P( \hat U_{k+1} < t,\ \hat V_{k+1} = \bm m_j,\ \hat V_k = \bm m_i)
    \\ 
    & = 
    \sum_{ N \geq 1}
   \P( V_N = \bm m_i,\ \hat T_k - 1 = N)
   \cdot 
   \sum_{  l_1 \neq i } \frac{ q_b(i,l_1) }{ q_b(i) }\Big( 1 - \exp\big( - q_b(i)t\big)\Big)
   \theta_b(\bm m_j|\bm m_{l_1})
   \qquad 
   \text{by \eqref{def: absorption prob, theta b i j}}
   \\ 
   & = 
   \sum_{ N \geq 1}
   \P( V_N = \bm m_i,\ \hat T_k - 1 = N) \cdot 
   r^{*|b}(i,j) \cdot \Big( 1 - \exp\big( - q_b(i)t \big)\Big)
   \qquad \text{ with $r^{*|b}(\cdot,\cdot)$ defined in \eqref{proof, def of r * b}}
   \\ 
   & = r^{*|b}(i,j) \cdot \Big( 1 - \exp\big( - q_b(i)t \big)\Big) \cdot \P( \hat V_k = \bm m_i).
\end{align*}
This verifies 
$\P( \hat U_{k+1} < t,\ \hat V_{k+1} = \bm m_j\ |\ \hat V_k = \bm m_i) = r^{*|b}(i,j) \cdot\big( 1 - \exp( - q_b(i)t )\big)$.
By \eqref{def Y * b t, alternative representation, metastability},
we conclude the proof of \eqref{proof, part iii, goal 2, proposition: Y * b is a CTMC}.
\end{proof}

\ifshowreminders
\newpage
\footnotesize
\newgeometry{left=1cm,right=1cm,top=0.5cm,bottom=1.5cm}

\section*{\linkdest{location of reminders}Some Reminders}
\begin{itemize}[leftmargin=*]
    \item Assumption \ref{assumption gradient noise heavy-tailed}
    \begin{itemize}
        \item $\E Z_j = 0;\ H(\cdot) = \P(|Z_1| > \cdot) \in \RV_{-\alpha}$
        \item
        $\lim_{x \rightarrow \infty}\frac{ H^{(+)}(x) }{H(x)} = p^{(+)},\ \lim_{x \rightarrow \infty}\frac{ H^{(-)}(x) }{H(x)} = p^{(-)} = 1 - p^{(+)}$
        with $p^{(+)},p^{(-)} \in [0,1]$ 
        and $p^{(+)} + p^{(-)} = 1$.
    \end{itemize}
    \item Assumption \ref{assumption: lipschitz continuity of drift and diffusion coefficients} (Lipschitz Continuity)
    \begin{enumerate}
        \item[] There exists some $D \in [1, \infty)$ such that $|\sigma(x) - \sigma(y)| \vee |a(x)-a(y)| \leq D|x-y|\ \ \ \forall x,y \in \mathbb{R}.$
    \end{enumerate}
    \item Assumption \ref{assumption: boundedness of drift and diffusion coefficients} (Boundedness)
    \begin{enumerate}
        \item[] There exist some $0<c \leq 1 \leq C < \infty$ such that $|a(x)| \leq C,\ c \leq \sigma(x) \leq C\ \ \forall x \in \mathbb{R}.$
    \end{enumerate}
    \item Assumption \ref{assumption: nondegeneracy of diffusion coefficients} (Nondegeneracy)
    \begin{enumerate}
        \item[] $\sigma(x) > 0$ $\forall x\in \R$.
    \end{enumerate}
    

    \item Assumption \ref{assumption: shape of f, first exit analysis}
    \begin{enumerate}
    \item[] 
        It holds for all $x \in (s_\text{left},s_\text{right})\setminus \{0\}$ that $a(x)x < 0$. Besides,
        \begin{itemize}
            \item $a(0)= 0$; $a(\cdot)$ is differentiable around $0$ with $a^\prime(0) \neq 0$;
            \item The following claims hold for $s \in \{s_\text{left},s_\text{right}\}$: 
            If $a(s) \neq 0$, then
            $a(\cdot)$ is differentiable around $s$ and $a^\prime(s) \neq 0$. 
        \end{itemize}
    \end{enumerate}
    
\end{itemize}
\fi

\ifshownotationindex
\newpage
\footnotesize
\newgeometry{left=1cm,right=1cm,top=0.5cm,bottom=1.5cm}

\section*{\linkdest{location of notation index}Notation Index}
\begin{itemize}[leftmargin=*]
\linkdest{location, notation index A}
\item 
    \chra{Asymptotic Equivalence}\xwa{definition added}
    \notationidx{asymptotic-equivalence}{Asymptotic Equivalence}:
    $X_n$ is asymptotically equivalent to $Y_n$ when bounded away from $\mathbb{C}\subseteq \mathbb{S}$
    w.r.t.\ $\epsilon_n$
    if for each $\Delta > 0$
    and each $B \in \mathscr{S}_\mathbb{S}$ that is bounded away from $\mathbb{C}$,
    \begin{align*}
    \lim_{n \rightarrow \infty}\frac{\P\Big( \bm{d}\big(X_n, Y_n\big)\mathbf{I}\big( X_n\in B\text{ or }Y_n \in B \big) > \Delta \Big)}{ \epsilon_n } = 0.
\end{align*}

\item 
    \notationidx{order-k-time-on-[0,t]}{$(0,t)^{k\uparrow}$}:
    $(0,t)^{k\uparrow}
    \delequal
    \big\{
    (t_1,\cdots,t_k) \in \mathbb{R}^k:\ 0 < t_1 < t_2 < \cdots < t_k < t
    \big\}.
    $
    

\item 
    \notationidx{set-for-integers-below-n}{[n]}: $[n] = \{1,2,\cdots,n\}$ for any $n \in \mathbb{Z}^+$.

\item 
    \notationidx{floor-operator}{$\floor{x}$}:
    $\floor{x}\delequal \max\{n \in \mathbb{Z}:\ n \leq x\}$.

\item 
    \notationidx{a}{$a$}: drift coefficient $a: \mathbb{R} \to \mathbb{R}$

\item 
    \chra{$\alpha$ is missing}
    \notationidx{alpha-noise-tail-index-LDP}{$\alpha$}:
    $\alpha > 1$; the heavy tail index for $(Z_j)_{j \geq 1}$ in Assumption \ref{assumption gradient noise heavy-tailed}
    
\item
    \notationidx{a-M}{$a_M$}: drift coefficient $a: \mathbb{R} \to \mathbb{R}$ truncated at level $\pm M$
    

\item
    \chra{$A_i(\eta,\epsilon,\delta,t,x)$ missing}\xwa{Fixed}
    \notationidx{notation-event-A-i-concentration-of-small-jumps}{$A_i(\eta,b,\epsilon,\delta,x)$}:
    $
    A_i(\eta,b,\epsilon,\delta,x)
    \delequal 
    \Big\{ \underset{j \in E_i(\eta,\delta) }{\max}\  \eta\Big|\sum_{n = \tau^{>\delta}_{i-1}(\eta) + 1}^j \sigma\big( X^{\eta|b}_{n-1}(x) \big)Z_n\Big| \leq \epsilon \Big\}.
    $
    
\item
    \notationidx{notation-event-A-Y-i-concentration-of-small-jumps}{$A^Y_i(\eta,b,\epsilon,\delta,x)$}:
    $
     A^Y_i(\eta,b,\epsilon,\delta,x)
    \delequal 
    \Big\{
    \sup_{t \in [0,1]:\  \tau^{>\delta}_{i-1;L}(\eta) < t < \tau^{>\delta}_{i;L}(\eta)}
    \Big|
    \int_{ s \in \big(\tau^{>\delta}_{i-1;L}(\eta), t\big]}\sigma\big(Y^{\eta|b}_{s-}(x)\big)d\bar{L}^\eta_s
    \Big| \leq \epsilon
    \Big\}.
    $


    

\linkdest{location, notation index B}

\item
    \notationidx{notation-set-B-0}{$B_0$}:
    $
    \{ \bm{X}^{\eta|b}(x)\in B\text{ or }\hat{\bm X}^{\eta|b;(k)}(x) \in B;\   
    \bm{d}_{J_1}\big(\bm{X}^{\eta|b}(x),\hat{\bm X}^{\eta|b;(k)}(x)\big) > \Delta
    \}
    $

\item 
    \notationidx{notation-B1}
    {$B_1$}:
    $ \{ \tau^{>\delta}_{k+1}(\eta) > \floor{1/\eta}\}$
    
\item
    \notationidx{notation-B2}
    {$B_2$}:
    $\{ \tau^{>\delta}_{k}(\eta) < \floor{1/\eta} \}$
    
\item
    \notationidx{notation-B3}
    {$B_3$}:
    $\big\{\eta |W^{>\delta}_{i}(\eta)| > \bar{\delta}\ \text{for all }i \in [k] \big\}$
    

\item 
    \notationidx{notation-atypical-1-clipped-SDE}{$B^{(j),Y }_{1;M\downarrow}(\eta,\delta,b,c,x)$}:
    $
     \{ \bm{Y}^{\eta|b}_{M\downarrow}(x)\in B\text{ or }\hat{\bm Y}^{\eta|b;(j)}_{M\downarrow}(x) \in B \}
    \cap 
    \{ \tau^L_{n,j + 1}(\delta) \leq 1 \}
    $

\item 
    \notationidx{notation-atypical-2-clipped-SDE}{$B^{(j),Y }_{2;M\downarrow}(\eta,\delta,b,c,x)$}:
    $
     \{ \bm{Y}^{\eta|b}_{M\downarrow}(x)\in B\text{ or }\hat{\bm Y}^{\eta|b;(j)}_{M\downarrow}(x) \in B \}
    \cap 
    \{ \tau^L_{n,j}(\delta) > 1 \}
    $

\item 
    \notationidx{notation-atypical-3-clipped-SDE}{$B^{(j),Y }_{3;M\downarrow}(\eta,\delta,b,c,x)$}:
    $
    \{ \bm{Y}^{\eta|b}_{M\downarrow}(x)\in B\text{ or }\hat{\bm Y}^{\eta|b;(j)}_{M\downarrow}(x) \in B \}
\cap \big\{\tau^L_{n,j}(\delta) \leq 1 < \tau^L_{n,j+1}(\delta) ;\ |W^{>\delta}_{i;L}(\eta)| \leq c\ \text{for some }i \in [j] \big\}
    $

\item 
    \notationidx{notation-atypical-4-clipped-SDE}{$B^{(j),Y }_{4;M\downarrow}(\eta,\delta,b,c,x)$}:
    $
    \{ \bm{Y}^{\eta|b}_{M\downarrow}(x)\in B\text{ or }\hat{\bm Y}^{\eta|b;(j)}_{M\downarrow}(x) \in B \}
    \cap \big\{\tau^L_{n,j}(\delta) \leq 1 < \tau^L_{n,j+1}(\delta);\ |W^{>\delta}_{i;L}(\eta)| > c\ \forall i \in [j] \big\}
    \cap 
    \big\{ |W^{>\delta}_{i;L}(\eta)| > (1/\epsilon)^{\frac{1}{2j}}\ \text{for some }i \in [j] \big\}
    $

\item 
    \notationidx{notation-atypical-5-clipped-SDE}{$B^{(j),Y }_{5;M\downarrow}(\eta,\delta,b,c,x)$}:
    $
    \{ \bm{Y}^{\eta|b}_{M\downarrow}(x)\in B\text{ or }\hat{\bm Y}^{\eta|b;(j)}_{M\downarrow}(x) \in B \}
    \cap \big\{\tau^L_{n,j}(\delta) \leq 1 < \tau^L_{n,j+1}(\delta);\ |W^{>\delta}_{i;L}(\eta)| \in \big(c, (1/\epsilon)^{\frac{1}{2j}} \big]\ \forall i \in [j] \big\}
\cap 
    \big\{ \bm{d}_{J_1}\big(\bm{Y}^{\eta|b}_{M\downarrow}(x),\hat{\bm Y}^{\eta|b;(j)}_{M\downarrow}(x)\big) > \rho^{(j)} \sqrt{\epsilon} \big\}
    $

\linkdest{location, notation index C}

\item 
    \chra{$C$ is missing}%
    \notationidx{notation-constant-C-boundedness-assumption}{$C$}:
    $C \in [1,\infty)$ is the constant in Assumption \ref{assumption: boundedness of drift and diffusion coefficients} with $|a(x)|\vee \sigma(x) \leq C\ \ \forall x \in \mathbb{R}$.

\item 
    \notationidx{notation-C-*-first-exit-time}{$C^*$}: 
    $
    C^* \delequal  \widecheck{ \mathbf{C} }\big( I^\complement\big)
    $

\item 
    \notationidx{notation-C-b-*}{$C_b^*$}: 
    $
    C_b^* \delequal \widecheck{ \mathbf{C} }^{ (\mathcal{J}^*_b)|b }(I^\complement)
    $

\item 
    \notationidx{notation-measure-C-k-t-mu-LDP}{$\mathbf{C}^{(k)}_{[0,T]}(\ \cdot\ ;x)$}:
    $
    \int \mathbf{I}\Big\{ h^{(k)}_{[0,T]}\big( x,(w_1,\cdots,w_k),(t_1,\cdots,t_k)   \big) \in\ \cdot\  \Big\} \nu^k_\alpha(d w_1,\cdots,dw_k) \times\mathcal{L}^{k\uparrow}_T(dt_1, dt_2,\cdots,dt_k)$
    \\
    $\mathbf{C}^{(k)}\delequal\mathbf{C}^{(k)}_{[0,1]}
    =
    \int \mathbf{I}\Big\{ h^{(k)}\big( x,(w_1,\cdots,w_k),(t_1,\cdots,t_k)   \big) \in\ \cdot\  \Big\} \nu^k_\alpha(d w_1,\cdots,dw_k) \times\mathcal{L}^{k\uparrow}_1(dt_1, dt_2,\cdots,dt_k)
    $

\item 
    \notationidx{notation-measure-C-k-t-truncation-b-LDP}{$ {\mathbf{C}}^{(k)|b}_{[0,T]}(\ \cdot\ ;x)$}:
    $
    \int \mathbf{I}\Big\{ h^{(k)|b}_{[0,T]}\big( x,(w_1,\cdots,w_k),(t_1,\cdots,t_k)   \big) \in \ \cdot\  \Big\}  \nu^k_\alpha(d w_1,\cdots,dw_k) \times\mathcal{L}^{k\uparrow}_T(dt_1, dt_2,\cdots,dt_k)$
    \chra{how about $\mathbf C^{(k)|b}_{[0,T]}$?}\xwa{Done}%
    \\
    $
    {\mathbf{C}}^{(k)|b} = {\mathbf{C}}^{(k)}_{b;[0,1]}
    =
    \int \mathbf{I}\Big\{ h^{(k)|b} \big( x,(w_1,\cdots,w_k),(t_1,\cdots,t_k)   \big) \in \ \cdot\  \Big\} \nu^k_\alpha(d w_1,\cdots,dw_k) \times\mathcal{L}^{k\uparrow}_1(dt_1, dt_2,\cdots,dt_k)
    $


\item  
    \notationidx{notation-check-C}{$\widecheck{\mathbf C}(\ \cdot\ ;x)$}:
    $
    {\widecheck{\mathbf C}(\ \cdot\ ;x)} \delequal \int \mathbf{I}\Big\{ x +\sigma(x) \cdot w \in\ \cdot\ \Big\}\nu_\alpha(dw)
    $
    \\ 
    $\widecheck{\mathbf C}(\cdot)\delequal \widecheck{\mathbf C}(\ \cdot\ ;0)$.

\item 
    \notationidx{notation-check-C-k-b}{$\widecheck{ \mathbf C }^{(k)|b}(\ \cdot\ ;x )$}:
    $
    {\widecheck{ \mathbf C }^{(k)|b}(\ \cdot\ ;x )}
    \delequal 
    \int \mathbf{I}\Big\{ g^{(k-1)}_b\big( x + \varphi_b\big(\sigma(x)\cdot w_k\big), w_1,\cdots,w_{k-1},\bm t \big) \in \ \cdot \  \Big\}
    \nu^k_\alpha(d w_1,\cdots,dw_k) \times \mathcal{L}^{k-1\uparrow}_\infty(d\bm t)
    $
    \\ 
    $
    \widecheck{ \mathbf C }^{(k)|b}(\cdot) \delequal
    \widecheck{ \mathbf C }^{(k)|b}(\ \cdot\ ;0 )
    $

\item 
    \notationidx{notation-mathcal-C-S-exclude-C}{$\mathcal{C}({ \mathbb{S}\setminus \mathbb{C} })$}:
    $
\mathcal{C}({ \mathbb{S}\setminus \mathbb{C} })
$
is the set of all real-valued, non-negative, bounded and continuous functions with support bounded away from $\mathbb{C}$

    \chra{$\mathcal C(\mathbb S\setminus \mathbb C)$ instead of $\mathcal{C}_{ \mathbb{S}\setminus \mathbb{C} }$}\xwa{Fixed}

\linkdest{location, notation index D}
\item
    \chra{$L$ is missing}%
    \notationidx{notation-Lipschitz-constant-L-LDP}{$D$}:
    The Lipschitz $D \in[1,\infty)$ in Assumption \ref{assumption: lipschitz continuity of drift and diffusion coefficients}:
    $|\sigma(x) - \sigma(y)| \vee |a(x)-a(y)| \leq D|x-y|\ \ \ \forall x,y \in \mathbb{R}$

\item 
    \notationidx{notation-D-A-k-t-LDP}{$\mathbb{D}^{(k)}_A[0,T]$}:
    $\mathbb{D}^{(k)}_A[0,T] \delequal h^{(k)}_{[0,T]}\big( A \times \mathbb{R}^{k} \times (0,T)^{k\uparrow} \big)$ with convention that $\mathbb{D}_A^{(-1)}[0,T] = \emptyset$
    \\
    $
    \mathbb{D}^{(k)}_A \delequal \mathbb{D}^{(k)}_A[0,1] = h^{(k)}\big( A \times \mathbb{R}^{k} \times (0,1)^{k\uparrow} \big)
    $

\item 
    \notationidx{notation-D-A-k-t-truncation-b-LDP}{$\mathbb{D}_{A}^{(k)|b} [0,T]$}:
    ${ \mathbb{D}}_{A|b}^{(k)}[0,T] \delequal h^{(k)|b} \big( A \times \mathbb{R}^k\times(0,T)^{k\uparrow} \big)$ with convention that $\mathbb{D}_{A|b}^{(-1)}  = \emptyset$
    \chr{How about $\D_A^{(k)|b}[0,T]$?}%
    \xw{Done}
    \\
    ${\mathbb{D}}_{A}^{(k)|b} = {\mathbb{D}}_{A}^{(k)|b}[0,1] \delequal h^{(k)|b} \big( A \times \mathbb{R}^k\times(0,1)^{k\uparrow} \big)$

\item 
    \notationidx{notation-D-A-k-t-truncation-b-M-LDP}{$\mathbb{D}_{A;M\downarrow}^{(k)|b}$}: 
    $\mathbb{D}_{A;M\downarrow}^{(k)|b} \delequal 
h^{(k)|b}_{M\downarrow}\big( \mathbb{R} \times \mathbb{R}^{k} \times (0,1]^{k\uparrow} \big).
$
\chr{How about $\mathbb{D}_{A;M\downarrow}^{(k)|b}$?}
\xw{Done}

    


\item 
    \notationidx{notation-D-J1}{$\dj{[0,T]}$}:
    Skorokhod $J_1$ metric on $\mathbb{D}[0,T]$.
    \\
    $\bm{d}_{J_1} = \dj{[0,1]}$ is the Skorodhod metric on $\mathbb{D} = \mathbb{D}[0,1]$.



\linkdest{location, notation index E}

\item 
    \notationidx{notation-closure-of-set-E}{$E^-$}: closure of set $E$ 

\item 
    \notationidx{notation-interior-of-set-E}{$E^\circ$}: interior of set $E$ 


\item 
    $\notationidx{notation-E-b-of-G-b}{E_b}$: $(m_i \rightarrow m_j) \in E_b \iff \mathcal J^*_b(i,j) = \mathcal J^*_b(i)$

\item 
    \notationidx{notation-set-E-delta-LDP}{$E^\delta_{c,k}(\eta)$}:
    $E^\delta_{c,k}(\eta) \delequal \Big\{ \tau^{>\delta}_{k}(\eta) < \floor{1/\eta} < \tau^{>\delta}_{k+1}(\eta);\ \eta|W^{>\delta}_j(\eta)| > c\ \ \forall j \in [k] \Big\}$
    \quad
    $(c > \delta)$
    \quad
    (event that there are exactly $k$ ``big'' jumps)

\item 
    \notationidx{notation-epsilon-enlargement-of-set-E}{$E^\epsilon$}: 
    $E^\epsilon \delequal 
\{ y \in \mathbb{S}:\ \bm{d}(E,y)\leq\epsilon \}$
    ($\epsilon$-fattening)
    \xwa{Now it is a closed set}%
    \chra{Not necessarily closed: e.g., $E = (0,1)$}%
    \xwa{Redefined}%

\item 
    \notationidx{notation-epsilon-shrinkage-of-set-E}{$E_{\epsilon}$}:
    $E_{\epsilon} \delequal
\{x \in E: \bm{d}(x,y) < \epsilon \Longrightarrow y \in E\}$
    ($\epsilon$-shaving)

 \item 
    \notationidx{notation-E-L-eta-k}{$E^\delta_{c,k;L}(\eta)$}:
    $
    E^\delta_{c,k;L}(\eta) \delequal\big\{ \tau^{> \delta}_{k;L}(\eta) \leq 1 <   \tau^{>\delta}_{k+1;L}(\eta)\big\}.
    $
    

\item 
    \notationidx{notation-eta}{$\eta$}: step length
    
\item
    $\notationidx{notation-sigma-algebra-F}{\mathcal{F}}$:
    the $\sigma-$algebra generated by iid copies $(Z_j)_{j \geq 1}$

\linkdest{location, notation index F}
\item 
    \notationidx{notation-F}{$\mathbb F$}: 
    $\mathbb{F} = (\mathcal{F}_j)_{j \geq 0}$ where  $\mathcal{F}_0 \delequal \{\Omega, \emptyset\}$
    and $\mathcal{F}_j$ is the $\sigma$-algebra generated by $Z_1,\cdots,Z_j$


\linkdest{location, notation index G}

\item
    \notationidx{notation-check-g-k-b}{$\widecheck{g}^{(k)|b}(x,w_1,\cdots,w_k,t_1,\cdots,t_{k})$}:
    $
    {\widecheck{g}^{(k)|b}(x,w_1,\cdots,w_k,t_1,\cdots,t_{k})}
    \delequal 
    h^{(k)|b}_{[0,t_k+1]}(x,w_1,\cdots,w_k,t_1,\cdots,t_{k})(t_{k})
    $


\item
    $\notationidx{notation-G-b-typical-transition-graph}{\mathcal{G}_b}$: Directed graph ${\mathcal{G}_b} = (V,E_b)$,
    where $V = \{m_1,\cdots,m_{ n_\text{min} }\}$,
    and
    an edge $(m_i\rightarrow m_j)$ is in $E_b$ iff $\mathcal J^*_b(i,j) = \mathcal J^*_b(i)$.

\item 
    \notationidx{notation-Gamma-M-adapted-process-bounded-by-M-LDP}{$\bm{\Gamma}_M$}:
    $\bm{\Gamma}_M \delequal \big\{ (W_j)_{j \geq 0}\text{ is adapted to }\mathbb{F}:\ |W_j| \leq M\ \forall j \geq 0\text{ almost surely} \big\};$
    see \eqref{def: Gamma M, set of bounded adapted process}
    
\item
    \notationidx{notation-Gamma-M-cont-adapted-process-bounded-by-M-LDP}{$\bm{\Gamma}_M^{\text{cont}}$}:
    $\bm{\Gamma}^{\text{cont}}_M \delequal \Big\{ V\text{ takes value in }\mathbb{D}\text{ and is adapted to }\mathbb{F}:\ \sup_{t \in [0,1]}|V(t)| \leq M\text{ almost surely} \Big\}.$

\linkdest{location, notation index H}
\item 
    \notationidx{notation-H-plus}{$H^{(+)}$}:
    $H^{(+)}(x)  \delequal \P(Z_1 > x) \in\RV_{-\alpha}$ 

\item 
    \notationidx{notation-H-minus}{$H^{(-)}$:} $H^{(-)}(x) \delequal \P(Z_1 < -x) \in \RV_{-\alpha}$

\item 
    \notationidx{notation-H}{$H$:} $H(x)  \delequal H^{(+)}(x) + H^{(-)}(x) = \P(|Z_1| > x) \in \RV_{-\alpha}$
    
\item
    \notationidx{notation-H-L-plus}{$H_L^{(+)}$}:
    $H^+_L(x)\delequal \nu(x,\infty) \in \RV_{-\alpha}(x)$.
    
\item 
    \notationidx{notation-H-L-minus}{$H_L^{(-)}$}:
    $H^-_L(x)\delequal \nu(-\infty,-x) \in \RV_{-\alpha}(x)$.
    
\item
    \notationidx{notation-H-L}{$H_L$}:
    $H^-_L(x)\delequal H^+_L(x) + H^-_L(x) = \nu\big( \R \setminus [-x,x] \big) \in \RV_{-\alpha}(x)$.


\item 
    \notationidx{notation-h-k-t-mapping-LDP}{$h^{(k)}_{[0,T]}$}: 
    an operator that maps a starting point ($x_0$), jump sizes ($\bm w$), and jump times ($\bm t$) to a piecewise gradient flow satisfying \eqref{def: perturb ode mapping h k, 1}-\eqref{def: perturb ode mapping h k, 3}
    \\
    $h^{(k)}\delequal h^{(k)}_{[0,1]}$.

\item 
    \notationidx{notation-h-k-t-mapping-truncation-level-b-LDP}{$h^{(k)|b}_{[0,T]}$}: 
    ODE mapping with $k$ perturbations (truncated at level $b>0$) on $[0,T]$.
    \chr{how about $h^{(k)|b}_{[0,T]}$?}\xw{Done}
    \\
    $h^{(k)|b}  = h^{(k)}_{b;[0,1]}$.
    
    
\item
    \notationidx{notation-mapping-h-k-t-b-M-LDP}{$h^{(k)|b}_{M\downarrow}$}: ODE mapping under $a_M,\sigma_M$ with $k$ perturbations (truncated at level $b>0$) on $[0,1]$; see \eqref{def: perturb ode mapping h k b, truncated at M, 1}-\eqref{def: perturb ode mapping h k b, truncated at M, 3}
    \chr{how about $h^{(k)|b}_{M\downarrow}$}\xw{Done}

\linkdest{location, notation index I}

\item 
    $\notationidx{notation-exit-domain-I}{I}$:
    ${I} \delequal (s_\text{left},s_\text{right})$.

\item
    $\notationidx{notation-exit-domain-I-epsilon}{I_\epsilon}$:
    ${I_\epsilon} \delequal (s_\text{left} + \epsilon,s_\text{right} - \epsilon)$

\item 
    $\notationidx{notation-attraction-field-I-i}{I_i}: I_i \delequal (s_{i-1},s_i)$

\item 
    $\notationidx{notation-attraction-field-I-i-truncated-delta-M}{I_{i;\delta,M}}: {I_{i;\delta,M}} = (s_{i-1} + \delta, s_i - \delta) \cap (-M,M) = (I_i)_\delta \cap (-M,M)$

\item 
    \chra{$E_i(\eta,\delta)$ missing}\xwa{Fixed}
    \notationidx{notation-A-i-concentration-of-small-jumps-2}{$I_i(\eta,\delta)$}:
    $I_i(\eta,\delta) 
    \delequal 
    \big\{j \in \mathbb{N}:\  \tau^{>\delta}_{i-1}(\eta) + 1 \leq j \leq \big(\tau^{>\delta}_{i}(\eta) - 1 \big) \wedge \floor{1/\eta}\big\}.$

\item 
    $\notationidx{notation-transition-marker-metastability}{\hat{\mathcal I}^{\eta,\epsilon|b}_k(x)}: {\hat{\mathcal I}^{\eta,\epsilon|b}_k(x)} = i \iff X^{\eta|b}_{\hat \tau^{\eta,\epsilon|b}_k(x)}(x) \in I_i.$

\linkdest{location, notation index J}

\item
    \notationidx{J-Z-c-n}{$\mathcal{J}_Z(c,n):$}
    $\mathcal{J}_Z(c,n) \delequal \#\{i \in [n]:\ |Z_i| \geq c \}$

\item 
    \notationidx{notation-J-L-c-T}{$\mathcal{J}_L(c,T)$}:
    $
     \mathcal{J}_L(c,T) \delequal \#\big\{ i \in \mathbb{N}:\ \widetilde{\tau}^L_i \leq T,\ |Z^L_i| > c \big\}
    $

\item 
    \notationidx{notation-J-eta-leq-k-L-i}{$J^{\eta;(k)}_L(i)$}:
    $
    J^{\eta;(k)}_L(i) \delequal \min\big\{ j > J^{\eta;(k)}_L(i-1):\ |Z^{L}_j| \geq \bm{Z}^{(k)}_L(\eta)  \big\}
    $

\item 
    \notationidx{notation-J-B-x-jump-number}{$\mathcal{J}(B;x)$}:
    $
\mathcal{J}(B;x)\delequal \min\{ k \geq 0: B\cap \mathbb{D}^{(k)}_{\{x\}}\neq \emptyset \}.
    $

\item 
    \notationidx{notation-jump-number-J-b-B-x}{$\mathcal{J}_b(B;x)$}:
    $\mathcal{J}_b(B;x)\delequal \min\{ k \geq 0: B\cap \mathbb{D}^{(k)|b}_{\{x\}}\neq \emptyset \}.$

\item  
    \notationidx{notation-J-*-first-exit-analysis}{$\mathcal{J}^*_b$}:
    $
    \mathcal{J}^*_b \delequal \ceil{l/b}.
    $
    \chra{?}%
    \xwa{Fixed}%

\item 
     $\notationidx{notation-J-*-b-i}{\mathcal J^*_b(i)}: {\mathcal J^*_b(i)} \delequal \ceil{l_i/b}$

\item 
    $\notationidx{notation-J-*-b-i-j}{\mathcal J^*_b(i,j)}$:
    $
    {\mathcal J^*_b(i,j)}
        \delequal \ceil{l_{i,j}/b}
    $

\item 
    $\notationidx{notation-J-*-b-V}{\mathcal J^*_b(V)}: {\mathcal J^*_b(V)} = \max_{i:\ m_i \in V}\mathcal{J}^*_b(i)$

\linkdest{location, notation index K}
\linkdest{location, notation index L}

\item 
    $\notationidx{notation-r-radius-of-exit-domain}{l}$:
    ${l} \delequal \inf_{x \in I^\complement}|x| = |s_\text{left}| \wedge s_\text{right}$

\item 
    $\notationidx{notation-r-i-radius-of-I-i}{l_i}: {l_i} \delequal{} \inf_{x \in I_i^\complement}|x - m_i|
    = 
    |m_i - s_{i-1}| \wedge |s_i - m_i|$

\item 
    $\notationidx{notation-l-i-j}{l_{i,j}}: {l_{i,j}} \delequal \inf_{x \in I_j}|x - m_i| =
    \begin{cases}
        s_{j-1} - m_i & \text{if }\ j > i \\
        m_i - s_j & \text{if }\ j < i
    \end{cases}$

\item
    \notationidx{notation-levy-process}{$\bm L$}:
    $\bm{L} = \{L_t: t \geq 0\}$ is the L\'evy process with generating triplet $(c_L,\sigma_L,\nu)$ where $c_L \in \mathbb{R}$ is the drift parameter, $\sigma_L \geq 0$ is the magnitude of the Brownian motion term in $L_t$, and $\nu$ is the L\'evy measure.

\item
    \notationidx{notation-scaled-levy-process}{$\bar{\bm L}^\eta$}:
    $\bar{\bm L}^\eta \delequal \big\{ \bar{L}^\eta_t = \eta L_{t/\eta}:\ t \in [0,1]\big\}$

\item
    \notationidx{notation-L-larger-than-delta-eta}{$L^{>\frac{\delta}{\eta}}_t$}:
     $L^{>\frac{\delta}{\eta}}_t \delequal{} \sum_{0 \leq s \leq t}\Delta L_s\mathbf{I}\{ |\Delta L_s| > \delta/\eta\}$
     
 \item
    \notationidx{notation-L-larger-than-delta-eta-scaled}{$\bar{L}^{\eta,>\frac{\delta}{\eta}}_t$}:
     $\bar{L}^{\eta,>\frac{\delta}{\eta}}_t= \eta \cdot L^{>\frac{\delta}{\eta}}_{t/\eta} = \eta\sum_{0 \leq s \leq \frac{t}{\eta}}\Delta L_s\mathbf{I}\{ |\Delta L_s| > \delta/\eta  \}$

\item
    \notationidx{notation-L-smaller-than-delta-eta-N}{$L^{(N,\frac{\delta}{\eta}]}_t$}:
    $L^{(N,\frac{\delta}{\eta}]}_t \delequal{} -\mu_L(N)t + \sum_{0 \leq s \leq t}\Delta L_s\mathbf{I}\{ |\Delta L_s| \in (N,\delta/\eta]  \}$ where $\mu_L(z) \delequal{} \int_{|x| > z}x\nu(dx)$
    
\item 
    \notationidx{notation-L-smaller-than-delta-eta-N-scaled}{$\bar{L}^{\eta,(N,\frac{\delta}{\eta}] }_t$}:
    $\bar{L}^{\eta,(N,\frac{\delta}{\eta}] }_t = \eta \cdot L^{(N,\frac{\delta}{\eta}]}_{t/\eta}$

\item 
    \notationidx{notation-lebesgue-measure-restricted}{$\mathcal{L}_t$}: 
    Lebesgue measure restricted on $(0,t)$

\item 
    \notationidx{notation-lebesgue-measure-on-ordered-[0,t]}{$\mathcal{L}^{k\uparrow}_t$}:
    Lebesgue measure restricted on $(0,t)^{k \uparrow}$

\item
    \notationidx{notation-measure-L-k-up-infty}{$\mathcal{L}^{k\uparrow}_\infty$}:
    Lebesgue measure restricted on $\{ (t_1,\cdots,t_k) \in (0,\infty)^k:\ 0 < t_1 < t_2 < \cdots < t_k \}$.\chr{can't find the definition}\xw{fixed}

    

\item 
    $\notationidx{notation-lambda-scale-function}{\lambda(\eta)}$:
    $
    {\lambda(\eta)} \delequal \eta^{-1}H(\eta^{-1}) \in \RV_{\alpha -1}(\eta)
    $
    as $\eta \downarrow 0$. 
    
\item 
    \notationidx{notation-scale-function-lambda-L}{$\lambda_L(\eta)$}:
    $
    {\lambda_L(\eta)} \delequal \eta^{-1}H_L(\eta^{-1}) \in \RV_{\alpha-1}(\eta)
    $ as $\eta \downarrow 0$

\item 
    $\notationidx{notation-scale-function-lambda-*-b}{\lambda^*_b(\eta)}: 
    {\lambda^*_b(\eta)} \delequal \eta \cdot \lambda^{ \mathcal J^*_b(V) }(\eta) \in \RV_{ \mathcal J^*_b(V)\cdot (\alpha-1)  + 1 }(\eta).$

\item
    $\notationidx{notation-scale-function-metastability-SDE}{\lambda^*_{b;L}(\eta)}: {\lambda^*_{b;L}(\eta)}\delequal \big(\lambda_L(\eta)\big)^{ \mathcal J^*_b(V) } \in \RV_{ \mathcal J^*_b(V)\cdot (\alpha - 1) }(\eta).$
    
\linkdest{location, notation index M}
    
\item
    \notationidx{notation-M-S-exclude-C}{$\mathbb{M}(\mathbb{S}\setminus \mathbb{C})$}:
    $\mathbb{M}(\mathbb{S}\setminus \mathbb{C})
    \delequal 
    \{
    \nu(\cdot)\text{ is a Borel measure on }\mathbb{S}\setminus \mathbb{C} :\ \nu(\mathbb{S}\setminus \mathbb{C}^r) < \infty\ \forall r > 0
    \}.$
    
\item
    \notationidx{notation-M-convergence}{$\mathbb{M}(\mathbb{S}\setminus \mathbb{C})$-convergence}:
    $\mu_n(f) \rightarrow \mu(f)$ for any $f \in \mathcal{C}({ \mathbb{S}\setminus \mathbb{C} })$

\item
    \notationidx{notation-uniform-M-convergence}{$\mathbb{M}(\mathbb{S}\setminus \mathbb{C})$-convergence uniformly over $\Theta$}:
    $ \lim_{\eta \downarrow 0}\sup_{\theta \in \Theta}|\mu^\eta_\theta(f) - \mu_\theta(f)| = 0$ for any $f \in \mathcal{C}({ \mathbb{S}\setminus \mathbb{C} })$

\linkdest{location, notation index N}

\item 
    \notationidx{notation-measure-nu-alpha}{$\nu_\alpha$}: $\nu_\alpha[x,\infty) = p^{(+)} x^{-\alpha},\ \nu_\alpha(-\infty,-x] = p^{(-)} x^{-\alpha}$

\item
    $\notationidx{notation-nu-k-alpha}{\nu_\alpha^k(\cdot)}$:
    $k$-fold product measure of $\nu_\alpha$.

\linkdest{location, notation index O}
\linkdest{location, notation index P}

\item 
    \notationidx{notation-p-plus-and-minus}{$p^{(+)},p^{(-)} \in [0,1]$}: 
    $\lim_{x \rightarrow \infty}H^{(+)}(x)\big/H(x) = p^{(+)},\ \lim_{x \rightarrow \infty}H^{(-)}(x)\big/H(x) = p^{(-)}$.

\linkdest{location, notation index Q}

\item 
    $\notationidx{notation-q-b-i}{q_b(i)}: {q_b(i)} \delequal \widecheck{\mathbf C}^{ ( \mathcal J^*_b(i) )|b }( I_i^c;m_i)$

\item 
    $\notationidx{notation-q-b-i,j}{q_{b}(i,j)}: {q_{b}(i,j)} \delequal \widecheck{\mathbf C}^{ ( \mathcal J^*_b(i) )|b }( I_j;m_i)$


\item 
    $\notationidx{notation-q-i-j}{q(i,j)}: {q(i,j)} \delequal \widecheck{\mathbf C}( I_j;m_i)$

\linkdest{location, notation index R}

\item 
    \notationidx{notation-rho-LDP}{$\rho$}:
    $\rho \delequal \exp(D)$ with $D$ in Assumption \ref{assumption: lipschitz continuity of drift and diffusion coefficients}

\item 
    \notationidx{notation-rho-t-LDP}{$\rho(t)$}:
    $\rho(t) \delequal \exp(Dt)$ with $D$ in Assumption \ref{assumption: lipschitz continuity of drift and diffusion coefficients}
    

    

\item 
    \chra{$\RV_{-\alpha}$ is missing}
    \notationidx{notation-RV-LDP}{$\RV_\beta$}:
    $\phi \in \RV_\beta$ (as $x \rightarrow \infty$) if $\lim_{x \rightarrow \infty}\phi(tx)/\phi(x) = t^\beta$ for any $t>0$;
    $\phi \in \RV_\beta(\eta)$ (as $\eta \downarrow 0$) if $\lim_{\eta \downarrow 0}\phi(t\eta)/\phi(\eta) = t^\beta$ for any $t>0$

\item  
    \notationidx{notation-R-eta-b-epsilon-return-time}{$R^{\eta|b}_\epsilon(x)$}:
    $
    {R^{\eta|b}_\epsilon(x)}  \delequal \min\big\{ j \geq 0:\ X^{\eta|b}_j(x) \in (-\epsilon,\epsilon) \big\}
    $ 

\item 
    $\notationidx{notation-return-time-to-I-i}{R_{i;\epsilon}^{\eta|b}(x)}: 
    {R_{i;\epsilon}^{\eta|b}(x)} \delequal \min\{ j \geq 0:\ X^{\eta|b}_j(x) \in (m_i - \epsilon,m_i + \epsilon) \}$

\linkdest{location, notation index S}

\item 
    $\notationidx{notation-S-delta-boundary-set}{S(\delta)}: {S(\delta)} \delequal \bigcup_{i \in [n_\text{min} - 1]}[s_i - \delta,s_i + \delta]$

\item 
    \notationidx{sigma}{$\sigma$}: diffusion coefficient $\sigma: \mathbb{R} \to \mathbb{R}$
    
\item
    \notationidx{sigma-M}{$\sigma_M$}: diffusion coefficient $\sigma: \mathbb{R} \to \mathbb{R}$ truncated at level $\pm M$

\item 
     $\notationidx{notation-sigma-eta-b-i-epsilon}{\sigma^{\eta|b}_{i;\epsilon}(x)}:
     {\sigma^{\eta|b}_{i;\epsilon}(x)} \delequal \min\{j \geq 0:\ X^{\eta|b}_j(x) \in \bigcup_{ l \neq i }(m_l - \epsilon,m_l + \epsilon)\}$
    

\item 
    \notationidx{notation-support-of-function-g}{$\text{supp} (g)$}:
$\text{supp} (g) \delequal \Big(\{ \xi \in \mathbb{D}[0,t]:\ g(\xi) \neq 0 \}\Big)^-$;
support of $g: \mathbb{D}[0,t] \to \mathbb{R}$

\item 
    \notationidx{notation-borel-sigma-algebra}{$\mathscr{S}_\mathbb{S}$}:
    Borel $\sigma$-algebra of metric space $(\mathbb{S},\bm{d})$

\item 
    \notationidx{notation-support-of-mu}{$\text{supp}(\mu)$} : $\text{supp}(\mu) \delequal \bigcap_{E:\ E\text{ closed},\ \mu(E) = 1}E$;
    support of Borel measure $\mu$

\linkdest{location, notation index T}

\item 
    $\notationidx{notation-t-epsilon-ode-return-time}{ \bm{t}(\epsilon) }:
        { \bm{t}(\epsilon) } \delequal \min\big\{ t \geq 0:\ \bm{y}_t(s_\text{left} + \epsilon) \in [-\epsilon,\epsilon]\text{ and }\bm{y}_t(s_\text{right} - \epsilon) \in [-\epsilon,\epsilon] \big\}$

\item 
    \notationidx{notation-large-jump-time}{$\tau^{>\delta}_i(\eta)$}:
    $\tau^{>\delta}_i(\eta) \delequal{} \min\{ n > \tau^{>\delta}_{i-1}(\eta):\ \eta|Z_n| > \delta  \},\ \tau^{>\delta}_0(\eta) = 0$; arrival time of $j$\textsuperscript{th} large jump
\\\chra{$\tau^{\uparrow\delta}_i(\eta)\rightarrow\tau_j^{>\delta}(\eta)$?}\xwa{Fixed}
    
\item 
    \notationidx{tau-leq-j-i-n}{$\tau^{(j)}_i(\eta)$}:
    $\tau^{(j)}_i(\eta) \delequal \min\big\{ n > \tau^{(j)}_{i-1}(\eta):\ |Z_n| \geq \bm{Z}^{ (j) }(\eta) \big\}$
    
    \chra{$\tau^{(j)}_i(\eta) \rightarrow \tau^{(j)}_i(\eta)$}\xwa{Fixed}

\item
    $\notationidx{notation-tau-eta-b-i-delta-M-x}{\tau^{\eta|b}_{i;\delta,M}(x)}:\ {\tau^{\eta|b}_{i;\delta,M}(x)} \delequal \min\Big\{ j \geq 0:\ X^{\eta|b}_j(x) \notin (s_{i-1} + \delta,s_i - \delta)\cap (-M,M)\Big\}$

\item 
    \notationidx{notation-tau-eta-b-Y-f-g}{$\tau^{\eta|b;(k+1)}_{Y}(x)$}:
    $\tau^{\eta|b;(k+1)}_{Y}(x)
     \delequal 
    \min\Big\{ t > \tau^{\eta|b;(k)}_{Y}(x): \Big| g\big({Y}^{\eta|b;(k)}_{t-}(x)\big) \cdot \Delta \bar L^\eta_t \Big| > b \Big\}$

\item 
    \notationidx{notation-large-jump-time-levy-tau-delta-j-L}{$\tau^{> \delta}_{j;L}(\eta)$}:
    $
    {\tau^{> \delta}_{j;L}(\eta)} \delequal \inf\{ t > \tau^{>\delta}_{j-1;L}(\delta):\ |\Delta \bar{L}^\eta_t| > \delta  \}
    $
    with convention
    $\tau^{> \delta}_{0;L}(\eta)  \equiv 0.$

\item   
    \notationidx{notation-tau-eta-x-first-exit-time}{$\tau^\eta(x)$}:
    $
    \tau^\eta(x) \delequal \min\big\{j \geq 0:\ X^\eta_j(x) \notin I\big\}
    $

\item  
    \notationidx{notation-tau-eta-b-x-first-exit-time}{$\tau^{\eta|b}(x)$}:
    $
    \tau^{\eta|b}(x) \delequal \min\big\{j \geq 0:\ X^{\eta|b}_j(x) \notin I \big\}
    $

\item  
    \notationidx{notation-tau-eta-b-epsilon-exit-time}{$\tau^{\eta|b}_\epsilon(x)$}:
    $
    {\tau^{\eta|b}_\epsilon(x)} \delequal \min\big\{ j \geq 0:\ X^{\eta|b}_j(x) \notin I_\epsilon \big\}
    $

\item 
    $\notationidx{notation-transition-time-metastability}{\hat \tau^{\eta,\epsilon|b}_k(x)}: {\hat \tau^{\eta,\epsilon|b}_k(x)} \delequal \min\Big\{ j \geq \hat \tau^{\eta,\epsilon|b}_{k-1}(x):\ X^{\eta|b}_j(x) \in \bigcup_{i \neq \hat{\mathcal I}^{\eta,\epsilon|b}_{k-1}(x)}(m_i - \epsilon,m_i+\epsilon) \Big\}$,
    the $k$-th transition time to a different $\epsilon$-neighborhood of local minima $m_i$



\linkdest{location, notation index U}

\item 
    \notationidx{notation-U-j-t}{$U_j$}:
    an iid sequence of Unif$(0,1)$
    \chra{$U_j(t)$ is missing}%
    
\item  
    \chra{$U_{(j);k}$ is missing}%
    \notationidx{notation-U-j-k-LDP}{$U_{(j;k)}$}:
    $0 \leq U_{(1;k)} \leq U_{(2;k)} \leq \cdots \leq U_{(k;k)}$;
    the order statistics of iid $\big(U_j\big)_{j = 1}^k$

\linkdest{location, notation index V}

\item 
    $\notationidx{notation-V-of-G-b}{V}: V = \{m_1,\cdots,m_{ n_\text{min} }\}$

\item
    $\notationidx{notation-V-*-b}{V^*_b}: V^*_b \delequal \{m_i:\ \mathcal J^*_b(i) = \mathcal{J}^*_b(V)\}$
    
\item 
    \notationidx{notation-truncation-operator-level-b}{$\varphi_c$}:
    $\varphi_c(w) \delequal \varphi(w,c) \delequal{} (w\wedge c)\vee(-c)$; truncation operator at level $c > 0$

\linkdest{location, notation index W}

   
\item 
    \notationidx{notation-large-jump-size}{$W^{>\delta}_i(\eta)$}: 
    $W^{>\delta}_i(\eta) \delequal{} Z_{\tau^{>\delta}_i(\eta)}$; size of $j$\textsuperscript{th} jump above threshold $\delta/\eta$
\\\chr{$W_j^{\eta/\delta}$}
   
\item 
    \chra{$W_j^*(\cdot)$ is missing}%
    \notationidx{notation-W-*_j}{$W^*_j(\cdot)$}:
    $\P(W^*_j(c) > x) = p^{(+)}\cdot(c/x)^\alpha,\ \ \P(-W^*_j(c) > x) = p^{(-)}\cdot(c/x)^\alpha\ \ \forall x > 0, c > 0$.
    
\item
    \notationidx{W-leq-j-i-n}{$W^{(j)}_i(\eta)$}:
    $W^{(j)}_i(\eta) \delequal Z_{ \tau^{(j)}_i(\eta) }$

\item 
    \notationidx{notation-W-eta-b-Y-f-g}{$W^{\eta|b;(k+1)}_{Y}(x)$}:
    $
    {W^{\eta|b;(k+1)}_{Y}(x)} \delequal 
    \Delta {Y}^{\eta|b;(k)}_{ \tau^{\eta|b;(k + 1)}_{Y}(x)}(x)
    $
    
\item
    \notationidx{notation-large-jump-size-levy-W-delta-j-L}{$W^{> \delta}_{j;L}(\eta)$}:
    ${W^{> \delta}_{j;L}(\eta)} \delequal \Delta \bar{L}^\eta\big(\tau^{> \delta}_{j;L}(\eta)\big) = \eta\cdot{\Delta L\big(\tau^{> \delta}_{j;L}(\eta)\big/\eta\big)}$


\linkdest{location, notation index X}

\item There are many varieties of $X$. Notation system convention:
\begin{itemize}
    \item
        independent of $\eta$: no superscript\\ 
        dependent of $\eta$: superscript $\eta$
    \item 
        discrete: $x$, $X$\\
        continuous: $y$, $Y$
    \item 
        on $[0,1]$: \\
        on $[0,1/\eta]$: 
    \item 
        clipped: \\
        unclipped: 
    \item 
        random: capital letters $X$, $Y$\\ 
        deterministic: lower case letters $x$, $y$
    \item 
        piecewise deterministic perturbed by $j$ largest shocks ($(j)$)\\ 
        piecewise deterministic perturbed by shocks over the threshold: $\hat {\bm x}$)
        
    \item 
        whole process: \\
        process at continuous time $t \in [0,1]$
        process at discrete time $j$
        
    \item 
        number of big jumps: superscript in parenthesis: e.g., $\tau_i^{(j)}(\eta)$
        jumps over threshold: superscript with $>$ and threshold: e.g., $\tau_i^{>\delta}(\eta)$
\end{itemize}

\item 
    \notationidx{notation-discrete-gradient-descent}{$\bm{x}^\eta_{j}(x)$}: (deterministic) difference equation
    $\bm{x}^\eta_{j}(x) = \bm{x}^\eta_{j-1}(x) + \eta a\big(\bm{x}^\eta_{j-1}(x) \big)$ for any $j \geq 1$ with initial condition $\bm{x}^\eta_{0}(x) = x$.

\item 
    \notationidx{notation-breve-X-eta-delta-t}{$\breve{ X}^{\eta,\delta}_t(x)$}: ODE that coincides with $X^{\eta}_{\floor{t/\eta} }(x)$ at discontinuities $t = \eta \tau^{>\delta}_i(\eta)$, $i=1,2,\ldots$.
    
\item  
    \notationidx{notation-breve-X-eta-b-delta-t}{$\breve{X}^{\eta|b;\delta}_t(x)$}:
    ODE that coincides with $X^{\eta|b}_{\floor{t/\eta} }(x)$ at discontinuities $t = \eta \tau^{>\delta}_i(\eta)$, $i=1,2,\ldots$.

\item
    \notationidx{notation-hat-X-clip-b-top-j-jumps}{$\hat{\bm{X}}^{\eta|b; (j) }(x)$}:
    ODE perturbed by \chrin{$j$ largest} jumps with truncation at $b$.
    
\item 
    $\notationidx{notation-marker-process-metastability-hat-X}{\hat X^{\eta,\epsilon|b}_t(x)}$:
     ${\hat X^{\eta,\epsilon|b}_t(x)}$ is the $\Big( \Big( \big({\hat \tau^{\eta,\epsilon|b}_k(x) - \hat \tau^{\eta,\epsilon|b}_{k-1}(x)}\big) \cdot { \lambda^*_b(\eta) } \Big)_{k \geq 1},\big( m_{ \hat{\mathcal I}^{\eta,\epsilon|b}_k(x) } \big)_{k \geq 1} \Big)$ jump process.
    
\item 
    \notationidx{notation-X-j-eta-x}{$X^\eta_j(x)$}: $X^\eta_0(x) = x;\ \ X^\eta_j(x) = X^\eta_{j - 1}(x) +  \eta\big[ a\big(X^\eta_{j - 1}(x)\big) + \sigma\big(X^\eta_{j - 1}(x)\big)Z_j\big]\ \ \forall j \geq 1.$


\item 
    \notationidx{notation-scaled-X-0T-eta-LDP}{$\bm{X}^\eta_{[0,T]}(x)$}:
    $\bm{X}_{[0,T]}^\eta(x) \delequal \big\{ X^\eta_{ \floor{ t/\eta } }(x):\ t \in [0,T] \big\}$
    
    \notationidx{notation-scaled-X-eta-LDP}{$\bm{X}^\eta(x)$}: $\bm{X}^\eta(x) = \bm{X}_{[0,1]}^\eta(x) \delequal \big\{ X^\eta_{ \floor{ t/\eta } }(x):\ t \in [0,1] \big\}$
    \chra{Separate $\bm{X}^\eta(\mu),\bm{X}^\eta(x)$.}\xwa{Done}%


\item 
    \notationidx{notation-X-eta-j-truncation-b-LDP}{$X^{\eta|b}_j(x)$}:
    $X^{\eta|b}_j(x)= X^{\eta|b}_{j - 1}(x)+  \varphi_b\Big(\eta \big[a\big(X^{\eta|b}_{j - 1}(x)\big) + \sigma\big(X^{\eta|b}_{j - 1}(x)\big)Z_j\big]\Big)\ \ \forall j \geq 1$


    

\item 
    \notationidx{notation-scaled-X-eta-mu-truncation-b-LDP}{${\bm{X}}^{\eta|b}_{[0,T]}(x)$}:
    $ {\bm{X}}^{\eta|b}_{[0,T]}(x) \delequal \big\{ X^{\eta|b}_{ \floor{ t/\eta } }(x):\ t \in [0,T] \big\}$.
    \\
    $ {\bm{X}}^{\eta|b}(x) = {\bm{X}}^{\eta|b}_{[0,1]}(x) \delequal \big\{ X^{\eta|b}_{ \floor{ t/\eta } }(x):\ t \in [0,1] \big\}$
    \chra{Separate.}\xwa{done}%



\linkdest{location, notation index Y}

\item 
    \chra{$\bm{x}^\eta(\ \cdot\ ;x)$ missing}
    \notationidx{notation-continuous-gradient-descent}{$\bm{y}_t(x)$}:  Gradient flow path.
    $\frac{d\bm{y}_t(x)}{dt} = a\big(\bm{y}_t(x)\big)$ for any $t > 0$
    with initial condition $\bm{y}_0(x) = x$.

    \chra{Let's denote this with $\bm y(\cdot;)$.}\xwa{Fixed}

\item
    \notationidx{notation-Y-eta-SDE}{$Y^\eta_t(x)$}:
    $dY^\eta_t(x) = a\big(Y^\eta_{t-}(x)\big)dt + \sigma\big(Y^\eta_{t-}(x)\big)d\bar{L}^\eta_t$
    
\item 
    \notationidx{notation-bm-Y-0T-eta-SDE}{$\bm Y^\eta_{[0,T]}(x)$}:
    $
    {\bm Y^\eta_{[0,T]}(x)} = \{Y^\eta_t(x):\ t\in[0,T]\}
    $
    \\
    $\notationidx{notation-bm-Y-eta-SDE}{\bm Y^\eta(x)} = \{Y^\eta_t(x):\ t\in[0,1]\}$

\item
    \notationidx{notation-Y-eta-b-SDE}{$Y^{\eta|b}_t(x)$}:
    $ {Y^{\eta|b}_t(x)}  = Y^{\eta|b;(k)}_{t}(x;a,\sigma)
    \qquad \Longleftrightarrow \qquad 
    t \in \Big[  
    \tau^{\eta|b;(k)}_{Y}(x;a,\sigma), 
     \tau^{\eta|b;(k+1)}_{Y}(x;a,\sigma) 
    \Big).$
    
\item
    \notationidx{notation-bm-Y-0T-eta-b-SDE}{$\bm{Y}^{\eta|b}_{[0,T]}(x)$}:
    $
    {\bm{Y}^{\eta|b}_{[0,T]}(x)} \delequal \big\{Y^{\eta|b}_t(x):\ t \in [0,T]\big\}
    $
    \\
    $\notationidx{notation-bm-Y-eta-b-SDE}{\bm Y^{\eta|b}(x)}= \{Y^{\eta|b}_t(x):\ t\in[0,1]\}$

\item 
    \notationidx{notation-Y-eta-b-k-f-g}{${Y}^{\eta|b;(k)}_{t}(x)$}:
    $
     {
    {Y}^{\eta|b;(k)}_{ \tau^{\eta|b;(k)}_{Y}(x) }(x)
    } \delequal 
    {Y}^{\eta|b;(k)}_{ \tau^{\eta|b;(k)}_{Y}(x)- }(x)
    +
    \varphi_b\Big(
    W^{\eta|b;(k)}_{Y}(x)
    \Big)$

    \qquad\qquad\qquad\qquad
    $
    d{Y^{\eta|b;(k)}_{ t}(x)} \delequal 
    f\big({Y}^{\eta|b;(k)}_{ t-}(x)\big) dt 
    + 
    g\big({Y}^{\eta|b;(k)}_{ t-}(x)\big)d\bar{L}^\eta_t
    \ \ \ \forall t > \tau^{\eta|b;(k)}_{Y}(x)$


\item 
    \notationidx{notation-hat-Y-eta-b-leq-k}{$\hat{Y}^{\eta|b;(k)}(t;x)$}:
    ODE perturbed by k largest discontinuities in $\bar{L}^\eta_t$, clipped by $b$

\item 
    $\notationidx{notation-CTMC-Y-*-b}{Y^{*|b}}$:
    A continuous-time Markov chain over states $\{m_i:\ i \in [n_\text{min},\ \mathcal J^*_b(i) = \mathcal J^*_b(V)\}$;
    i.e., the set of all widest minima.
    Specifically,
     given any $m_\text{init} \in V$, $Y^{*|b}_t(m_\text{init})$ is a $\Big((U_j)_{j \geq 1},(V_j)_{j \geq 1}\Big)$ jump process with 
     $V_1 = m_\text{init}$, $U_1 = 0$,
     and
     $$
     \P\Big( U_{l+1} < t,\ V_{l+1} = m_j\ \Big|\ V_l = m_i \Big)
     = 
    \begin{cases}
        \frac{q_b(i,j)}{q_b(i)} & \text{ if }m_i \notin V^*_b,
        \\ 
        \frac{q_b(i,j)}{q_b(i)} \cdot \Big( 1 - \exp\big(-q_b(i)t\big)\Big) & \text{ if }m_i \in V^*_b.
    \end{cases}
     $$

\item 
    $\notationidx{notation-CTMC-Y-*}{Y^*_t(\cdot)}$:
    a continuous-time Markov chain over states $\{m_i:\ i \in [n_\text{min}]\}$  with generator characterized by $q(i,j)$ and initial condition $Y^*_0(m) = m$

\linkdest{location, notation index Z}

\item 
    \notationidx{notation-Z-iid-noise-LDP}{$Z_j$}:
    $(Z_j)_{j \geq 1}$ is a sequence of iid heavy-tailed RVs with $\E Z_1 = 0$ and $H(x) = \P(|Z_1| > x) \in \RV_{-\alpha}$ as $x \rightarrow \infty$


    

\item
    \notationidx{Z-leq-j-n}{$\bm{Z}^{ (j) }(\eta)$}:
    $\bm{Z}^{ (j) }(\eta) \delequal \max\Big\{ c \geq 0:\ \mathcal{J}_Z(c, \floor{1/\eta} ) \geq j \Big\}$
    



\end{itemize}
\fi

\ifshowtheoremtree
\newpage
\footnotesize
\newgeometry{left=1cm,right=1cm,top=0.5cm,bottom=1.5cm}

\section*{\linkdest{location of theorem tree}Theorem Tree}
\begin{thmdependence}[leftmargin=*]

\thmtreenode{-}
    {Theorem}{theorem: portmanteau, uniform M convergence}
    {0.8}{Portmanteau Theorem for uniform $\M(\S\setminus\C)$-convergence}
\bigskip

\thmtreenodewopf{}
    {Lemma}{lemma: asymptotic equivalence when bounded away, equivalence of M convergence}
    {0.8}{
        asympt.\ equiv.\ of $X_n$ and $Y_n$ w.r.t.\\ $\epsilon_n^{-1}$ in $\M(\mathbb S\setminus\mathbb{C})$ implies the same $\M$-convergence of $\epsilon_n^{-1}\P(X_n\in \cdot)$ and $\epsilon_n^{-1}\P(Y_n\in \cdot)$.%
    }
\bigskip

        

\thmtreenode{\complete}
    {Theorem}{theorem: LDP 1, unclipped}
    {0.8}{
    [A\ref{assumption gradient noise heavy-tailed},\ref{assumption: lipschitz continuity of drift and diffusion coefficients},\ref{assumption: nondegeneracy of diffusion coefficients},\ref{assumption: boundedness of drift and diffusion coefficients}]
    Sample Path Large Deviations for SGD $\bm X^\eta$.
    $\P ({\bm X}^{\eta}(x)\in \cdot)/\lambda^k(\eta) \to \mathbf{C}^{(k)}(\cdot;x)$ in $\mathbb{M}(\D \setminus \D_A^{(k-1)})$ uniformly in $x$ on any compact set $A$ 
    }
    \begin{thmdependence}
    \thmtreenode{-}
            {Lemma}
            {lemma: continuity of h k b mapping}
            {0.8}{
            [A\ref{assumption: lipschitz continuity of drift and diffusion coefficients},\ref{assumption: boundedness of drift and diffusion coefficients}]
            $h^{(k)}_{[0,T]}$ is continuous on $\R\times \R^k \times (0,T)^{k\uparrow}$.
            }
        \begin{thmdependence}
        \thmtreenode{\issue}
            {Lemma}
            {lemma: continuity of h k b mapping clipped}
            {0.8}{
                [A\ref{assumption: lipschitz continuity of drift and diffusion coefficients},\ref{assumption: nondegeneracy of diffusion coefficients}] 
                $h^{(k)|b}$ is continuous.
            }
            \begin{thmdependence}
                 \thmtreenode{-}
            {Corollary}{corollary: existence of M 0 bar delta bar epsilon, clipped case, LDP}
            {0.8}{
                [A\ref{assumption: lipschitz continuity of drift and diffusion coefficients},\ref{assumption: nondegeneracy of diffusion coefficients}] 
                If $d(B, \D_{A|b}^{(k-1)}) >0$, the shocks of the paths in $B \cap \D_{A}^{(k)|b}$ are bounded away from 0. 
            }
        
            \begin{thmdependence}
            \thmtreenode{-}
                {Lemma}{lemma: boundedness of k jump set under truncation, LDP clipped}
                {0.8}{
                    $\sup_{\xi\in \D_{A}^{(k)|b}} \|\xi\| < \infty$ for any $b\in(0,\infty)$, $k\in \mathbb N$, and compact set $A$.
                }
            \end{thmdependence}
            \end{thmdependence}
        \end{thmdependence}

    \thmtreenode{-}
            {Lemma}{lemma: LDP, bar epsilon and delta}
            {0.8}{
                [A\ref{assumption: lipschitz continuity of drift and diffusion coefficients},\ref{assumption: boundedness of drift and diffusion coefficients}] 
                If $d(B, \D_{A}^{(k-1)}) >0$, the shocks of the paths in $B \cap \D_{A}^{(k)}$ are bounded away from 0. 
            }

    \thmtreenode{-}
        {Lemma}{lemma: sequential compactness for limiting measures, LD of SGD}{0.8}
        {
        Verify $\lim_{n \to \infty}\mathbf C^{(k)}(f;x_{n})
            =
            \mathbf C^{(k)}(f;x^*)$
            and
            $\lim_{n \to \infty}\mathbf C^{(k)|b}(f;x_{n})
            =
            \mathbf C^{(k)|b}(f;x^*)$
        }
        \begin{thmdependence}
            \thmtreeref
                {Lemma}{lemma: continuity of h k b mapping}
            \thmtreeref
                {Lemma}{lemma: continuity of h k b mapping clipped}
            \thmtreeref
                {Lemma}{lemma: LDP, bar epsilon and delta}
            \thmtreenode{-}
                {Lemma}
                {lemma: LDP, bar epsilon and delta, clipped version}
                {0.8}{
                [A\ref{assumption: lipschitz continuity of drift and diffusion coefficients}, A\ref{assumption: boundedness of drift and diffusion coefficients}]
                If $d(B, \D_{A|b}^{(k-1)}) >0$, the shocks of the paths in $B \cap \D_{A}^{(k)|b}$ are bounded away from 0.
                }
        \end{thmdependence}

    \thmtreeref
        {Theorem}{theorem: portmanteau, uniform M convergence}

    \thmtreeref
        {Proposition}{proposition: standard M convergence, LDP unclipped}

    \end{thmdependence}

\bigskip

\thmtreenode{\complete}
    {Theorem}{corollary: LDP 2}
    {0.8}{
    [A\ref{assumption gradient noise heavy-tailed},\ref{assumption: lipschitz continuity of drift and diffusion coefficients},\ref{assumption: nondegeneracy of diffusion coefficients}]
    Sample path large deviations for $\bm X^{\eta|b}$.
    $\P ({\bm X}^{\eta|b}(x)\in \cdot)/\lambda^k(\eta) \to \mathbf{C}_b^{(k)}(\cdot;x)$ in $\mathbb{M}(\D \setminus \D_{A|b}^{(k-1)})$ uniformly in $x$ on any compact set $A$ 
    }
    
    \begin{thmdependence}
        \thmtreeref
            {Proposition}{proposition: standard M convergence, LDP clipped}
        \thmtreeref
            {Lemma}{lemma: sequential compactness for limiting measures, LD of SGD}
        \thmtreeref
            {Lemma}{lemma: LDP, bar epsilon and delta, clipped version}
        \thmtreeref
            {Theorem}{theorem: portmanteau, uniform M convergence}
    \end{thmdependence}

\bigskip

\thmtreenode{-}
    {Proposition}{proposition: standard M convergence, LDP unclipped}
    {0.8}{
    [A\ref{assumption gradient noise heavy-tailed},\ref{assumption: lipschitz continuity of drift and diffusion coefficients},\ref{assumption: boundedness of drift and diffusion coefficients}]
        $\P({\bm X}^{\eta_n}(x_n)\in \cdot)/\lambda^k(\eta_n) \to \mathbf{C}^{(k)}(\cdot; x^*)$ in $\M(\D\setminus \D_{A}^{(k-1)})$ if $\eta_n\to 0$, $x_n \to x^*$, and $x_n, x^* \in A$: cpt
    }
    \begin{thmdependence}
    
    \thmtreenode{-}
    {Proposition}{proposition: standard M convergence, LDP clipped}
    {0.8}{ [A\ref{assumption gradient noise heavy-tailed},\ref{assumption: lipschitz continuity of drift and diffusion coefficients},\ref{assumption: nondegeneracy of diffusion coefficients}]
        $\P({\bm X}^{\eta_n}(x_n)\in \cdot)/\lambda^k(\eta_n) \to \mathbf{C}^{(k)}(\cdot; x^*)$ in $\M(\D\setminus \D_{A}^{(k-1)})$ if $\eta_n\to 0$, $x_n \to x^*$, and $x_n, x^* \in A$: cpt
    }
    
    \begin{thmdependence}
    \thmtreeref{Proposition}{proposition: standard M convergence, LDP clipped, stronger boundedness assumption}
    
    \thmtreeref
            {Corollary}{corollary: existence of M 0 bar delta bar epsilon, clipped case, LDP}
        
    \end{thmdependence}


    \thmtreeref{Lemma}{lemma LDP, small jump perturbation}
    
    \thmtreeref
            {Lemma}{lemma: LDP, bar epsilon and delta}

    \end{thmdependence}

\bigskip

\thmtreenode{-}
    {Proposition}{proposition: standard M convergence, LDP clipped, stronger boundedness assumption}
    {0.8}{ [A\ref{assumption gradient noise heavy-tailed},\ref{assumption: lipschitz continuity of drift and diffusion coefficients},\ref{assumption: boundedness of drift and diffusion coefficients}]
        $\P({\bm X}^{\eta_n|b}(x_n)\in \cdot)/\lambda^k(\eta_n) \to \mathbf{C}^{(j)}_b(\cdot; x^*)$ in $\M(\D\setminus \D_{A|b}^{(k-1)})$ 
        if $\eta_n\to 0$, $x_n \to x^*$, and $x_n, x^* \in A$: cpt
    }
    \begin{thmdependence}
        \thmtreeref
            {Lemma}{lemma: asymptotic equivalence when bounded away, equivalence of M convergence}
        
        \thmtreenode{-}
            {Proposition}{proposition: asymptotic equivalence, clipped}
            {0.8}{
                $\bm{X}^{\eta_n|b}(x_n)$ and $\hat{\bm X}^{\eta_n|b;(k)}(x_n)$ are asympt.\ equiv.\ w.r.t.\\ $\lambda^k(\eta)$ in $\M(\D\setminus \D_{A|b}^{(k-1)})$
            }
    
    \begin{thmdependence}
    
        \thmtreeref
                {Lemma}
                {lemma: LDP, bar epsilon and delta, clipped version}

        \thmtreenode{-}
            {Lemma}{lemma: SGD close to approximation x circ, LDP}
            {0.8}
            {
                (time-scaled) SGD $X^{\eta|b}_{\lfloor t/\eta \rfloor}$ is close to (slow) ODE until first large jump
            }
            \begin{thmdependence}
            \thmtreenode{-} 
                {Lemma}{lemmaBasicGronwall}
                {0.8}{
                SGD and GD are close to each other if the noises are small
                }

            \thmtreenode{-}
                {Lemma}{lemma Ode Gd Gap} 
                {0.8}{
                GD ($\bm y^\eta_{\floor{s}}(y)$) and GF ($\bm x^\eta(\cdot;x)$) are close to each other
                }

            \end{thmdependence}
        
        \thmtreenode{-}
            {Lemma}{lemma LDP, small jump perturbation} 
            {0.8}{
                \begin{minipage}[t]{\linewidth}
                (a) 
                $
                \sup_{ (W_i)_{i \geq 0} \in  \bm{\Gamma}_M}\P\Big( \max_{ j \leq \floor{t/\eta} \wedge \big(\tau^{>\delta}_{1}(\eta) - 1\big) }\ \eta\big|\sum_{i = 1}^j W_{i-1}Z_i \big| > \epsilon \Big)   
                = \bm o({\eta^N}) 
                $\\
                
                (b) 
                $
                \sup_{x \in \R} \P\Big( \big(\bigcap_{i = 1}^k A_i(\eta,\epsilon,\delta,t,x)\big)^c \Big) 
                = \bm o(\eta^N)
                $
                \end{minipage}
            }
        
        \thmtreenode{-}
            {Lemma}{lemma: SGD close to approximation x breve, LDP clipped}
            {0.8}{
                $\hat{X}^{\eta|b;>\delta}_{t}(x)$ and $X_{\lfloor t\rfloor}^{\eta|b}(x)$ are close to each other on $\cap_{i=1}^{k+1} A_i(\eta,\epsilon,\delta, x)$
                \\
                \chra{Lemma \ref{lemma: SGD close to approximation x breve, LDP} not used anywhere?}\xwa{Unused lemma removed. Only using Lemma \ref{lemma: SGD close to approximation x breve, LDP clipped} now.}
            }
    \end{thmdependence}
            
    \thmtreenode{-}
        {Proposition}{proposition: uniform weak convergence, clipped}
        {0.8}{ [A\ref{assumption: boundedness of drift and diffusion coefficients}]
        $\P(\hat{\bm X}^{\eta|b;(k)}\in \cdot)/\lambda^k(\eta) \to \mathbf{C}^{(k)|b}   $ in $\M(\D\setminus \D_{A|b}^{(k-1)})$
        }
    
        \begin{thmdependence}
        \thmtreeref{Lemma}{lemma: LDP, bar epsilon and delta, clipped version}
        \thmtreeref
            {Lemma}
            {lemma: continuity of h k b mapping clipped}
        \thmtreenode{-}
            {Lemma}
            {lemma: weak convergence, expectation wrt approximation, LDP, preparation}
            {0.8}{
            On \hyperlink{index, notation-set-E-delta-LDP}{$E^\delta_{c,k}(\eta)$}, continuous and bounded functionals of scaled jump times and jump sizes converge to that of uniform and Pareto distributions.
            }
            \begin{thmdependence}
            \thmtreenode{-}
                {Lemma}{lemma: weak convergence of cond law of large jump, LDP}
                {0.8}{
                Conditional on \hyperlink{index, notation-set-E-delta-LDP}{$E^\delta_{c,k}(\eta)$}, the scaled jump times and jump sizes converge to uniform and Pareto distributions. 
                }
            \end{thmdependence}
        
        \end{thmdependence}
    \end{thmdependence}

\bigskip

\chra{Lemma~\ref{lemma: upper bound on perturbation at initial value, h k b mapping} missing}\xwa{Fixed}



\bigskip

\thmtreenode{-}
    {Theorem}{theorem: first exit time, unclipped}{0.8}
    {
    (a)
    $
        \lim_{\eta \downarrow 0}\P\Big(C^*_b \eta\cdot \lambda^{ \mathcal{J}^*_b }(\eta)\tau^{\eta|b}(x) > t
        \Big)
        =
        \exp(-t)
        $\\
    
    (b)
    $
    \lim_{\eta \downarrow 0}\P\Big(C^* \eta \lambda(\eta)\tau^\eta(x) > t
    \Big)
    =
    \exp(-t)
    $
    }
    \begin{thmdependence}
        \thmtreenode{-}
            {Lemma}{lemmaGeomFront}{0.8}{} 
        \thmtreenode{-}
            {Theorem}{thm: exit time analysis framework} {0.8}
            {
            First Exit Time Analysis Framework: $\P\big(
        \gamma(\eta)\tau_{I^\complement}^{\eta}(x)>t,\,V_{\tau}^\eta(x)\in B  
    \big) \approx C(B)e^{-t}$
            }
        \begin{thmdependence}
            \thmtreenode{-}
                {Proposition}{prop: exit time analysis main proposition}{0.8}{
                First Exit Time Analysis Framework ($\epsilon$-relaxed version):
                $
                \P\big(\gamma(\eta) \tau_{I(\epsilon)^\complement}^\eta(x) > t;\; V^\eta_{\tau_\epsilon}(x) \in B\big)
                \approx C(B)e^{-t} + \delta_{t,B}(\epsilon)
                $
                }
        \end{thmdependence}
        \thmtreenodewopf{}
            {Lemma}{lemma: exit prob one cycle, with exit location B, first exit analysis}{0.8}{Apply general framework: Verifying conditions \eqref{eq: exit time condition lower bound} and \eqref{eq: exit time condition upper bound}}
        \thmtreenodewopf{}
            {Lemma}{lemma: fixed cycle exit or return taking too long, first exit analysis}{0.8}{Apply general framework: Verifying condition \eqref{eq:E3}}
        \thmtreenodewopf{}
            {Lemma}{lemma: cycle, efficient return}{0.8}{Apply general framework: Verifying condition \eqref{eq:E4}}
    \end{thmdependence}

\bigskip

\thmtreenode{-}
    {Lemma}{lemma: exit prob one cycle, with exit location B, first exit analysis}{0.8}{Apply general framework: Verifying conditions \eqref{eq: exit time condition lower bound} and \eqref{eq: exit time condition upper bound}}
    
    \begin{thmdependence}
        \thmtreeref
            {Theorem}{corollary: LDP 2}
        \thmtreenode{-}
            {Lemma}{lemma: choose key parameters, first exit time analysis}{0.8}{}
        \thmtreenode{-}
            {Lemma}{lemma: measure check C J * b, continuity, first exit analysis}{0.8}{
            Verifying  $C(\partial I) = 0$ for the measure $C$ defined in \eqref{def: measure C and scale gamma when applying the exit time framework}
            }
            \begin{thmdependence}
                \thmtreeref{Lemma}{lemma: choose key parameters, first exit time analysis}
            \end{thmdependence}

        \thmtreenode{-}
            {Lemma}{lemma: exit rate strictly positive, first exit analysis}{0.8}{
            Verifying $
                    C^*_b = \widecheck{\mathbf C}^{(\mathcal{J}^*_b)|b}\big( I^\complement\big)\in(0,\infty)
                        $
            for the measure $C$ defined in \eqref{def: measure C and scale gamma when applying the exit time framework}
            }
            \begin{thmdependence}
                \thmtreeref{Lemma}{lemma: choose key parameters, first exit time analysis}
                \thmtreeref{Lemma}{lemma: continuity of h k b mapping clipped}
            \end{thmdependence}
        
        \thmtreenode{-}
            {Lemma}{lemma: limiting measure, with exit location B, first exit analysis}{0.8}{}
            \begin{thmdependence}
                \thmtreeref{Lemma}{lemma: choose key parameters, first exit time analysis}
            \end{thmdependence}
        
    \end{thmdependence}

\thmtreenode{-}
    {Lemma}{lemma: fixed cycle exit or return taking too long, first exit analysis}{0.8}{Apply general framework: Verifying condition \eqref{eq:E3}}
    \begin{thmdependence}
        \thmtreeref
            {Theorem}{corollary: LDP 2}
    \end{thmdependence}

\thmtreenode{-}
    {Lemma}{lemma: cycle, efficient return}{0.8}{Apply general framework: Verifying condition \eqref{eq:E4}}
    \begin{thmdependence}
        \thmtreeref
            {Theorem}{corollary: LDP 2}
    \end{thmdependence}

\bigskip 

\thmtreenode{-}
    {Theorem}{corollary irreducible case}{0.8}{metastability of $X^{\eta|b}_j(x)$: convergence to CTMC that only visits widest minima}
    \begin{thmdependence}
    
    \thmtreenode{-}
        {Lemma}{lemma: metastability, abstract framework}{0.8}{abstract framework for establishing sample-path convergence}

    \thmtreenode{-}
        {Proposition}{proposition: hat X converges to CTMC, metastability}{0.8}{sample-path convergence from $\hat X^{\eta,\epsilon|b}_t(x)$ to $Y^{*|b}_t$}
        \begin{thmdependence}
        \thmtreenode{-}
            {Lemma}{lemma weak convergence of jump process}{0.8}{weak convergence of jump process}

        \thmtreenode{-}
            {Proposition}{proposition: transition time, metastability}{0.8}{asymptotic law of transition times between attraction fields}
            \begin{thmdependence}
                \thmtreeref{Lemma}{lemma: cycle, efficient return}
                \thmtreenode{-}
                    {Lemma}{lemma: unlikely to exit M or visit s i, metastability}{0.8}{
                    Unlikely to exit $(-M,M)$ or get too close to any boundary points $s_i$; i.e.,
                    \\
                    $\limsup_{\eta \downarrow 0}\max_{ i \in [n_\text{min}] }\sup_{ x \in [m_i - \epsilon,m_i + \epsilon] }
    \P\Big( \exists j < \sigma^{\eta|b}_{i;\epsilon}(x)\ s.t.\ 
    X^{\eta|b}_j(x) \in S(\delta)\text{ or }\big|X^{\eta|b}_j(x)\big| \geq M + 1
    \Big) \leq \Delta$}
                \begin{thmdependence}
                    \thmtreenode{-}
                        {Lemma}{lemma: iff conditions for positive mass under check C}{0.8}{}
                        \begin{thmdependence}
                            \thmtreeref{Lemma}{lemma: continuity of h k b mapping clipped}
                            \thmtreeref{Lemma}{lemma: choose key parameters, first exit time analysis}
                        \end{thmdependence}
                    \thmtreeref{Lemma}{lemma: measure check C J * b, continuity, first exit analysis}
                    \thmtreeref{Lemma}{lemma: cycle, efficient return}
                    \thmtreeref{Theorem}{theorem: first exit time, unclipped}
                \end{thmdependence}
                \thmtreeref{Lemma}{lemma: measure check C J * b, continuity, first exit analysis}
                \thmtreeref{Theorem}{theorem: first exit time, unclipped}
            \end{thmdependence}
            
        \end{thmdependence}

    \thmtreenode{-}
        {Proposition}{proposition: hat X close to X, metastability}{0.8}{$\lim_{\eta\downarrow 0}\P\Big( \Big|X^{\eta|b}_{ \floor{ t/\lambda^*_b(\eta) } }(x) - \hat X^{\eta,\epsilon|b}_t(x)\Big| > \epsilon\Big) = 0.$}
        \begin{thmdependence}
            \thmtreeref{Proposition}{proposition: transition time, metastability}
            \thmtreeref{Proposition}{proposition: hat X converges to CTMC, metastability}
            \thmtreeref{Lemma}{lemma: unlikely to exit M or visit s i, metastability}
            \thmtreeref{Lemma}{lemma: cycle, efficient return}
            \thmtreeref{Theorem}{theorem: first exit time, unclipped} 
        \end{thmdependence}
        
    \end{thmdependence}

\thmtreenode{-}
    {Theorem}{theorem: metastability, unclipped}{0.8}{metastability of $X^{\eta}_j(x)$: convergence to CTMC}
    \begin{thmdependence}
        \thmtreeref{Theorem}{corollary irreducible case}
    \end{thmdependence}

\thmtreenode{-}
    {Corollary}{corollary, elimination of sharp minima, metastability}{0.8}{Elimination of sharp minima}
    \begin{thmdependence}
         \thmtreeref{Theorem}{corollary irreducible case}
        \thmtreeref{Proposition}{proposition: transition time, metastability}
    \end{thmdependence}

\bigskip

\end{thmdependence}
\fi

\ifshowtheoremlist
\newpage
\footnotesize
\newgeometry{left=1cm,right=1cm,top=0.5cm,bottom=1.5cm}
\linkdest{location of theorem list}
\listoftheorems
\fi

\ifshowequationlist
\newpage
\linkdest{location of equation number list}
\section*{Numbered Equations}

\eqref{def: X eta b j x, unclipped SGD}
\eqref{def: H, law of Z_j}
\eqref{def: perturb ode mapping h k, 1}
\eqref{def: perturb ode mapping h k, 2}
\eqref{def: perturb ode mapping h k, 3}
\eqref{def: measure nu alpha}
\eqref{def: measure C k t}
\eqref{def: scaled SGD, LDP}
\eqref{def: set of ODE with k jumps}
\eqref{def: X eta b j x, clipped SGD}
\eqref{def: perturb ode mapping h k b, 1}
\eqref{def: perturb ode mapping h k b, 2}
\eqref{def: perturb ode mapping h k b, 3}
\eqref{def: l * tilde jump number for function g, clipped SGD}
\eqref{def: measure mu k b t}
\eqref{def: gradient descent process y}
\eqref{defArrivalTime large jump}
\eqref{defSize large jump}
\eqref{property: large jump time probability}
\eqref{def: Gamma M, set of bounded adapted process}
\eqref{def: event A i concentration of small jumps, 1}
\eqref{def: event A i concentration of small jumps, 2}
\eqref{proof: LDP, small jump perturbation, asymptotics part 1}
\eqref{proof: LDP, small jump perturbation, asymptotics part 2}
\eqref{proof: LDP, small jump perturbation, asymptotics part 3}
\eqref{term E Z 1, lemma LDP, small jump perturbation}
\eqref{proof: LDP, small jump perturbation, choose p}
\eqref{proof: LDP, small jump perturbation, ineq 1}
\eqref{proof: applying berstein ineq, lemma LDP, small jump perturbation}
\eqref{term second order moment hat Z 1, lemma LDP, small jump perturbation}
\eqref{def: E eta delta set, LDP}
\eqref{def: prob measure Q, LDP}
\eqref{proof: lemma weak convergence of cond law of large jump, 1, LDP}
\eqref{proof: choose bar delta, lemma LDP, bar epsilon and delta}
\eqref{bound of difference between xi and xi prime between t_J and t_J+1}
\eqref{goal, lemma: LDP, bar epsilon and delta, clipped version}
\eqref{choice of delta, proof, lemma: continuity of h k b mapping}
\eqref{choice of delta, 2, proof, lemma: continuity of h k b mapping}
\eqref{ineq 1, proof, lemma: continuity of h k b mapping}
\eqref{condition, x prime and x, lemma: continuity of h k b mapping}
\eqref{def: a sigma truncated at M, LDP}
\eqref{def: perturb ode mapping h k b, truncated at M, 1}
\eqref{def: perturb ode mapping h k b, truncated at M, 2}
\eqref{def: perturb ode mapping h k b, truncated at M, 3}
\eqref{ineq, no jump time, a, lemma: SGD close to approximation x circ, LDP}
\eqref{ineq, no jump time, b, lemma: SGD close to approximation x circ, LDP}
\eqref{ineq, with jump time, b, lemma: SGD close to approximation x circ, LDP}
\eqref{proof, ineq gap between X and y, SGD close to approximation x circ, LDP}
\eqref{proof, ineq gap between y and xi, SGD close to approximation x circ, LDP}
\eqref{proof, up to t strictly less than eta tau_1^eta, SGD close to approximation x circ, LDP}
\eqref{goal 1, proposition: standard M convergence, LDP unclipped}
\eqref{goal 3, proposition: standard M convergence, LDP unclipped}
\eqref{goal 4, proposition: standard M convergence, LDP unclipped}
\eqref{goal 5, proposition: standard M convergence, LDP unclipped}
\eqref{subgoal for goal 2, proposition: standard M convergence, LDP unclipped}
\eqref{subgoal, goal 4, proposition: standard M convergence, LDP unclipped}
\eqref{property: choice of M 0, new, proposition: standard M convergence, LDP clipped}
\eqref{property: Y and tilde Y, 1, proposition: standard M convergence, LDP clipped}
\eqref{property: Y and tilde Y, 2, proposition: standard M convergence, LDP clipped}
\eqref{goal new 1, proposition: standard M convergence, LDP clipped}
\eqref{goal new 2, proposition: standard M convergence, LDP clipped}
\eqref{def: objects for definition of hat X leq k}
\eqref{def: time and size of top j jumps before n steps}
\eqref{def: hat X truncated b, j top jumps, 1}
\eqref{def: hat X truncated b, j top jumps, 2}
\eqref{property: equivalence between breve X and hat X, clipped at b}
\eqref{goal: asymptotic equivalence claim, proposition: asymptotic equivalence, clipped}
\eqref{choice of bar delta, proof, proposition: asymptotic equivalence, clipped}
\eqref{choice of bar epsilon, proof, proposition: asymptotic equivalence, clipped}
\eqref{goal, event B 1, clipped, proposition: asymptotic equivalence, clipped}
\eqref{goal, event B 2, clipped, proposition: asymptotic equivalence, clipped}
\eqref{goal, event B 3, clipped, proposition: asymptotic equivalence, clipped}
\eqref{goal, event B 4, clipped, proposition: asymptotic equivalence, clipped}
\eqref{def: x eta b M circ approximation, LDP, 2}
\eqref{def: x eta b M circ approximation, LDP, 3}
\eqref{proof, ineq for mathring x, proposition: asymptotic equivalence, clipped}
\eqref{proof: bounded jump size for breve X, goal, event B 2, clipped, proposition: asymptotic equivalence, clipped}
\eqref{proof: goal 1, lemma: atypical 3, large jumps being too small, LDP clipped}
\eqref{proof: ineq 1, lemma: atypical 3, large jumps being too small, LDP, clipped}
\eqref{choice of bar delta, proposition: uniform weak convergence, clipped}
\smallskip

\noindent
\eqref{def: mapping check g k b, endpoint of path after the last jump, first exit analysis}
\eqref{def: first exit time, J *}
\eqref{proof, observation on xi, lemma: choose key parameters, first exit time analysis}
\eqref{def: t epsilon function, first exit analysis}
\eqref{property: t epsilon function, first exit analysis}
\eqref{proof, goal, lemma: exit rate strictly positive, first exit analysis}
\eqref{def: epsilon relaxed first exit time}
\eqref{def: return time R in first cycle, first exit analysis}
\eqref{proof, ineq 1, theorem: first exit time, unclipped}
\smallskip

\noindent
\eqref{defSDE, initial condition x}
\eqref{def: Y eta b 0 f g, SDE clipped}
\eqref{def, tau and W, discont in Y, eta b, 1}
\eqref{def, tau and W, discont in Y, eta b, 2}
\eqref{def: objects for defining Y eta b, clipped SDE, 1}
\eqref{def: objects for defining Y eta b, clipped SDE, 2}
\eqref{def: objects for defining Y eta b, clipped SDE, 3}
\eqref{def: objects for defining Y eta b, clipped SDE, 4}
\eqref{defSDE, initial condition x, clipped}
\fi

\ifshownavigationpage
\newpage
\normalsize
\tableofcontents

\section*{Navigation Links}

\ifshowtheoremlist
    \noindent
    \hyperlink{location of theorem list}{List of Theorems}
    \bigskip
\fi

\ifshowtheoremtree
    \noindent
    \hyperlink{location of theorem tree}{Theorem Tree}
    \begin{itemize}
    \thmtreeref
        {Proposition}{proposition: standard M convergence, LDP clipped}
    
    \thmtreeref
        {Proposition}{proposition: standard M convergence, LDP unclipped}
        
    \thmtreeref
        {Theorem}{theorem: LDP 1, unclipped}
        
    \thmtreeref
        {Theorem}{theorem: sample path LDP, unclipped}
    
    \thmtreeref
        {Proposition}{proposition: standard M convergence, LDP clipped}
    
    \thmtreeref
        {Theorem}{theorem: LDP 1}
    
    \thmtreeref
        {Theorem}{theorem: first exit time, unclipped}
    
        
    \end{itemize}
\fi

\ifshowreminders
    \noindent
    \hyperlink{location of reminders}{Assumptions, etc}
    \bigskip
\fi

\ifshowtheoremlist
    \noindent
    \hyperlink{location of equation number list}{Numbered Equations}
    \bigskip
\fi

\ifshownotationindex
    \noindent
    \hyperlink{location of notation index}{Notation Index}
    \begin{itemize}
    \item[] 
        \hyperlink{location, notation index A}{A},
        \hyperlink{location, notation index B}{B},
        \hyperlink{location, notation index C}{C},
        \hyperlink{location, notation index D}{D},
        \hyperlink{location, notation index E}{E},
        \hyperlink{location, notation index F}{F},
        \hyperlink{location, notation index G}{G},
        \hyperlink{location, notation index H}{H},
        \hyperlink{location, notation index I}{I},
        \hyperlink{location, notation index J}{J},
        \hyperlink{location, notation index K}{K},
        \hyperlink{location, notation index L}{L},
        \hyperlink{location, notation index M}{M},
        \hyperlink{location, notation index N}{N},
        \hyperlink{location, notation index O}{O},
        \hyperlink{location, notation index P}{P},
        \hyperlink{location, notation index Q}{Q},
        \hyperlink{location, notation index R}{R},
        \hyperlink{location, notation index S}{S},
        \hyperlink{location, notation index T}{T},
        \hyperlink{location, notation index U}{U},
        \hyperlink{location, notation index V}{V},
        \hyperlink{location, notation index W}{W},
        \hyperlink{location, notation index X}{X},
        \hyperlink{location, notation index Y}{Y},
        \hyperlink{location, notation index Z}{Z}
    \end{itemize}
\fi

\fi

\end{document}